\documentclass{article}
\usepackage{amsmath,latexsym,amssymb,verbatim,enumerate,graphicx,mathtools,quantikz, appendix, quantikz, tikz, algorithm, placeins, algpseudocode, svg, subcaption, mathtools, physics, etoc, url, wrapfig}
\PassOptionsToPackage{numbers,sort&compress}{natbib}

\usepackage[preprint]{Styles/neurips_2026} 
\usepackage{xcolor}

\newcommand{\sidecaptionimage}[3]{%
\begin{tikzpicture}
    \node[inner sep=0] (img) {\includegraphics[width=#1]{#2}};
    \node[overlay, rotate=90, anchor=south] at ([xshift=-1mm]img.west) {#3};
\end{tikzpicture}%
}
\usepackage{theorem}
\newtheorem{definition}{Definition}[section]

\newtheorem{lemma}[definition]{Lemma}

\newtheorem{theorem}[definition]{Theorem}
\newtheorem{corollary}[definition]{Corollary}

\def\squareforqed{\hbox{\rlap{$\sqcap$}$\sqcup$}}
\def\qed{\ifmmode\squareforqed\else{\unskip\nobreak\hfil
\parfillskip=0pt\finalhyphendemerits=0\endgraf}\fi}
\def\endenv{\ifmmode\;\else{\unskip\nobreak\hfil
\penalty50\hskip1em\null\nobreak\hfil\;
\parfillskip=0pt\finalhyphendemerits=0\endgraf}\fi}
\newenvironment{proof}{\noindent \textbf{{Proof~} }}{\QEDB}

\newcommand*{\QEDB}{\null\nobreak\hfill\ensuremath{\square}}%

\newcommand{\nc}{\newcommand}
\nc{\rnc}{\renewcommand} 

 \rnc{\bra}[1]{\langle#1|} 
 \rnc{\ket}[1]{|#1\rangle}
 \rnc{\braket}[2]{\langle#1|#2\rangle} 
\def\>{\rangle}
\def\<{\langle}

\nc{\Exp}[1]{$\mathbb{E}\left\[#1\right\]$}
\nc{\Avg}[1]{\langle#1\rangle}

\rnc{\max}{\operatorname{max}} 
\nc{\poly}{\operatorname{poly}}
\nc{\Supp}{{\operatorname{Supp}}}
\nc{\Diam}{{\operatorname{Diam}}}
\nc{\smfrac}[2]{\mbox{$\frac{#1}{#2}$}} 
\nc{\ox}{\otimes} 
\nc{\dg}{\dagger} 
\nc{\rar}{\rightarrow} 
\nc{\lrar}{\longrightarrow}
\nc{\pair}[2]{{\langle #1, #2 \rangle}}

\nc{\vc}[1]{\bar{#1}}
\nc{\rnd}[1]{{\operatorname{rnd}(#1)}}
\nc{\Vol}{{\operatorname{Vol}}}
\nc{\Area}{{\operatorname{Area}}}
\nc{\diag}{{\operatorname{diag}}}
\nc{\mL}{{\mathcal{L}}}
\nc{\Var}{{\operatorname{Var}}}

\nc{\eqn}[1]{Eq.~(\ref{eqn:#1})}
\nc{\eqns}[2]{Eqs.~(\ref{eqn:#1}) and (\ref{eqn:#2})}

\nc{\1}{\openone} 
\nc{\id}{{\operatorname{id}}}

\nc{\cA}{{\cal A}} \nc{\cB}{{\cal B}} \nc{\cC}{{\cal C}}
\nc{\cD}{{\cal D}} \nc{\cE}{{\cal E}} \nc{\cF}{{\cal F}}
\nc{\cG}{{\cal G}} \nc{\cH}{{\cal H}} \nc{\cI}{{\cal I}}
\nc{\cJ}{{\cal J}} \nc{\cK}{{\cal K}} \nc{\cL}{{\cal L}}
\nc{\cM}{{\cal M}} \nc{\cN}{{\cal N}} \nc{\cO}{{\cal O}}
\nc{\cP}{{\cal P}} \nc{\cR}{{\cal R}} \nc{\cS}{{\cal S}}
\nc{\cT}{{\cal T}} \nc{\cU}{{\cal U}} \nc{\cV}{{\cal V}}
\nc{\cX}{{\cal X}} \nc{\cW}{{\cal W}} \nc{\cZ}{{\cal Z}}

\def\dD{|\mathcal{D}|}
\def\cC{\mathcal{C}_{x,t}}

\nc{\RR}{{{\mathbb R}}}
\nc{\CC}{{{\mathbb C}}}
\nc{\FF}{{{\mathbb F}}}
\nc{\HH}{{{\mathbb H}}}
\nc{\II}{{{\mathbb I}}}
\nc{\NN}{{{\mathbb N}}}
\nc{\ZZ}{{{\mathbb Z}}}
\nc{\PP}{{{\mathbb P}}}
\nc{\QQ}{{{\mathbb Q}}}
\nc{\UU}{{{\mathbb U}}}
\nc{\WW}{{{\mathbb W}}}
\nc{\EE}{{{\mathbb E}}}
\rnc{\SS}{{{\mathbb S}}}

\nc{\bfn}{\beta^{(n)}}
\nc{\Esm}{E_{sm}}
\nc{\Elap}{{E_{Lap}}}
\nc{{\Ediff}}{{E_{diff}}}
\nc{{\Ecross}}{{E_{cross}}}
\nc{{\Eapsf}}{{E_{A2}}}
\nc{{\EOeff}}{{E_{O,\text{eff}}}}
\nc{{\EDeff}}{{E_{D,\text{eff}}}}
\nc{{\Edev}}{{E_{dev}}}
\nc{{\Cbabai}}{{C_{\text{Babai}}}}
\nc{{\TV}}{{\operatorname{TV}}}
\nc{\alg}[1]{\textsf{#1}}
\nc{\vs}{{v^{(s)}}}

\author{%
  Henry Hunt$^*$ \\
  Department of Physics\\
  Stanford University\\
  Stanford, CA 94305 \\
  \texttt{hshunt@stanford.edu} \\
  \And
  Mason Kamb$^*$ \\
  Department of Applied Physics\\
  Stanford University\\
  Stanford, CA 94305 \\
  \texttt{kambm@stanford.edu} \\
  \And
  Surya Ganguli \\
  Department of Applied Physics\\
  Stanford University\\
  Stanford, CA 94305 \\
  \texttt{sganguli@stanford.edu} \\
}

\title{An exact information theory of generalization phase transitions in Bayesian diffusion models}

\begin{document}


\maketitle

\begin{abstract}

How diffusion models circumvent the curse of dimensionality to learn complex distributions over high dimensional spaces from a finite training set, instead of memorizing it, remains a fundamental mystery. To address this, we introduce analytically tractable Bayesian information restricted diffusion (BIRD) models, in which each pixel observes restricted information about noisy data. A BIRD model time-reverses diffusion by inferring which past training sample produced its current restricted observation using the Bayesian posterior. This model class generalizes existing analytical diffusion models that use spatially local information restriction. We show that spatially local BIRD models closely approximate trained diffusion models \textit{early in training}, across different architectures such as UNets and DiTs. Under minimal assumptions on the data distribution, we identify an information-theoretic phase boundary between memorization and generalization in the joint space of amount of training data, time in the reverse generative process, and amount of information restriction: a BIRD model memorizes when the mutual information between its restricted noisy observations and the training data exceeds the log number of training points, and it generalizes otherwise. Experiments across a range of datasets confirm our theoretically predicted location for the transition. We find that generation proceeds near the edge of memorization: both spatially local BIRD models and early-training diffusion models track the memorization-generalization phase boundary by increasingly restricting information over time. Overall, our results reveal a fundamental role for information restriction in generative AI to circumvent the curse of dimensionality.

\end{abstract} 

\section{Introduction and related work}

Remarkably, generative AI can now learn complex distributions over high dimensional spaces with tractable amounts of training data. For example,  diffusion models \citep{sohl2015deep,ho2020denoising,song2020denoising} robustly and consistently generalize: when trained on entirely disjoint subsets of data, two independent models will \textit{converge to the same general model} when the data subset size is moderately large \citep{kadkhodaie2023generalization}. Below this modest data threshold, the two models may instead memorize their own training data. Naively, however the curse of dimensionality would suggest that an intractable (e.g. exponential in ambient dimension) amount of data should be required to transition from memorization to generalization for arbitrary data distributions \citep{shalevshwartz2014understanding}. This raises a fundamental mystery as to when the transition from memorization to generalization occurs for diffusion models, and why it does not require that much training data.
Several theoretical approaches have been taken to explore this mystery. These approaches often involve analyzing the behavior of simplified diffusion models on simplified datasets. Diffusion models reverse a forward diffusion process that converts data into noise.  There exists an exact mathematical reversal of this process that starts from noise and follows the ideal, or empirical score function for the finite training set (App. \ref{app:bayes_opt_mem}, \cite{biroli2024dynamical}). However, such empirical score models require an exponential in dimension amount of data to avoid memorization \citep{biroli2024dynamical, de2022convergence}. While several papers have explored the hypothesis that the low dimensionality of the data distribution might rescue generalization of empirical score models \citep{achilli2025memorization, li2024adapting, chen2023score}, extensive experimental work shows empirical score models fail to generalize on datasets of realistic dimensionality under realistic modeling conditions \cite{li2024good, scarvelis2023closed, kambanalytic, niedobatowards, yi2023generalization, gu2023memorization, pham2025memorization}.
\begin{wrapfigure}{r}{0.6\textwidth} 
    \begin{center}
    \includegraphics[width=\linewidth]{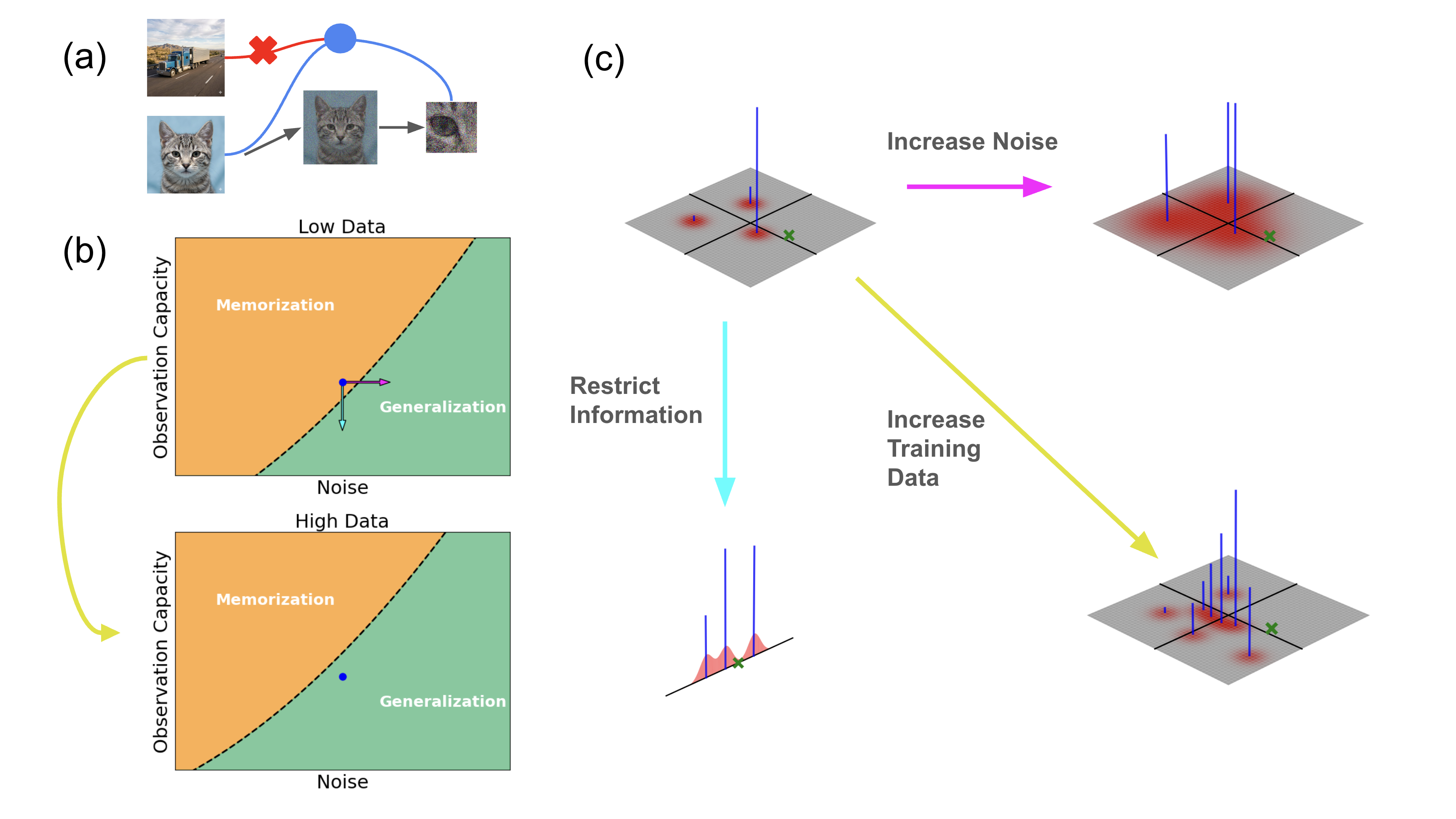}
    \end{center}
    \caption{\label{fig:cartoon_figure} A schematic figure showing different transitions from memorizing to generalizing phases in BIRD models. (a) A depiction of the Bayesian guessing game performed by a pixel as it assigns posterior probabilities to training set images. The forward process starts with a training image of a cat, adds noise to it, then the pixel observes a local patch (e.g. the eye). Based on this restricted observation, the pixel must guess whether this noisy patch was generated from the truck or the cat. The pixel memorizes if the posterior concentrates on the cat. (b) The phase diagram of the generalization/memorization transition for a BIRD model. Consider the blue point in the memorization phase (top phase diagram). Increased noise (pink arrow) and decreased observation capacity (blue arrow) each increase posterior entropy, making Bayesian guessing harder, and drive a transition from memorization to generalization. Also, increased dataset size (yellow arrow) can shift the phase transition boundary, and increase posterior entropy, so the original location (blue) dot is now in the generalization phase (bottom phase diagram). (c) In each of the $4$ plots, red indicates the distribution of a noised training set, the green x indicates a particular observed image, and the blue bars depict the posterior distribution over the training data conditioned on the observed image. The upper left plot corresponds to the blue point in the top phase diagram of panel b. The posterior concentrates on a single training point and so it is in the memorization phase.  The 3 colored arrows again correspond to the $3$ distinct ways one can increase posterior entropy and transition to the generalizing phase as in panel (b): by increasing noise (pink), by decreasing observation capacity (by projecting out a dimension) (blue), or by increasing the amount of training data (yellow).}
\end{wrapfigure}

The inability of empirical score models to generalize with tractable amounts of training data suggests that successful generalization in diffusion models may originate from limiting inductive biases that prevent learning the empirical score function, such as smoothness \citep{scarvelis2023closed,chen2025interpolation}, linearity \cite{wang2024unreasonable, wang2026random}, low-dimensionality \cite{he2026diffusion, wang2025diffusion}, geometry-adaptive bases \cite{kadkhodaie2023generalization}, early training \cite{favero2025bigger,bonnairediffusion,favero2025compositional,bardone2026theory}, and model parameterization/architecture \cite{boffi2024shallow, cui2025solvable,yoon2023diffusion,cui2024analysis,buchanan2026edge}. However, none of these works derived a highly accurate predictive theory for {\it what} trained, nonlinear diffusion models actually converge to in the generalizing phase on a wide variety of benchmark datasets.

One such theory to do so arose in \cite{kambanalytic,niedobatowards} which posited \textit{local score models} (LS), which are optimal diffusion models subject to locality constraints in which each pixel can only use a local spatial neighborhood of the image to compute the score. \cite{kambanalytic} also introduced \textit{equivariant} (ES) and (boundary-broken) \textit{equivariant-local} score models (ELS), which incorporated the additional inductive bias of (translation) group equivariance and its breaking via boundaries. In particular \cite{kambanalytic} showed that such analytic models could predict image outputs of trained convolution only diffusion models on a case by case basis with high quantitative accuracy.
\cite{lukoianov2025locality} then provided a method for obtaining locality scales and patch geometry for LS models via optimal linear denoising.  Interestingly, these works found that image generation proceeds starting from a large coarse spatial scale to a small fine one. Local generalization mechanisms have also since been investigated in later works \citep{bradley2025local, finn2025origins, hu2025local, ambrogioni2026out, gottwald2025localized, nguyen2026analytic}.

Importantly, while \cite{kambanalytic} developed a predictive analytic theory of convolutional diffusion models deep in the generalization phase, this and other works left open the mystery of when and where the memorization to generalization transition occurs on complex benchmark datasets for more powerful models. Answering this question in general is extremely challenging as it could depend {\it a priori} in a joint fashion on explicit biases from architecture, implicit biases from training dynamics, the structure of the data distribution, time or noise in the generative process, and of course the amount of data.  

In order to understand this question theoretically, in this paper we introduce a general class of diffusion models which we call Bayesian information restricted diffusion (BIRD) models. They are inspired by and encompass all the equivariant/local score models of \cite{kambanalytic,niedobatowards}. BIRD models possess a fortuitous combination of intuitive appeal, predictive capacity, {\it and} theoretical tractability, yielding significant conceptual insights into the memorization-generalization transition. Our main contributions are:

(1) We introduce an intuitively appealing general concept of BIRD models in which each image pixel $x$ can be thought of as an agent which observes at time $t$ in the reverse generation process restricted information $\mathcal C_{x,t}$ about a noisy image $\phi_t$ and uses optimal Bayesian inference over the training data to reverse the forward diffusion process (Fig. \ref{fig:cartoon_figure}).  This constitutes a well-defined diffusion model obtained from a finite training set without the need for any explicit architectural assumptions or learning dynamics. 

(2) We show spatially local BIRD models in which the restricted observation $\mathcal C_{x,t}$ corresponds to a local image patch around $x$, yield exceedingly good predictions for what many neural diffusion models learn \textit{early in training}: for early epochs, they can predict the {\it individual} image outputs of neural diffusion models {\it on a case by case basis} for a wide variety of architectures, including UNets with self-attention and diffusion transformers (DiTs), achieving high median $r^2\sim0.85-0.93$. This latter result goes beyond that of \cite{kambanalytic} which showed similarly high predictivity for small CNNs only. 

(3) Theoretically, under very mild assumptions on the data distribution, we develop a precise information theoretic criterion for the location of a memorization to generalization phase transition for general BIRD models in the joint space of amount of training data, time (or noise level in the reverse process), and amount of information restriction: if, under the forward diffusion process, the mutual information between the observation $\cC$ and the true data distribution, is greater (less) than the logarithm of the amount of data, then the BIRD model will memorize (generalize), because Bayesian guessing of the underlying image will be to easy (hard). We experimentally confirm the location of our theoretically predicted phase boundary on a variety of datasets. These results reveal that the amount of data need only be exponential in the mutual information to avoid memorization, and highlight a fundamental role of information restriction in promoting generalization. 

(4) Focusing on spatially local BIRD models in which information restriction is obtained by each pixel observing a location patch of patch scale $L$, we use various information theoretic analyses, to reveal the shape of the memorization to generalization phase transition boundary in the joint space of reverse generation time $t$ (or noise level $\sigma_t$) and patch scale $L$. We find a time dependent critical patch scale $L_c(t)$ that monotonically decreases in the reverse generation process as both time $t$ and noise $\sigma_t$ decrease from $t=1$ to $t=0$.  For any $t$, $L(t) > L_c(t)$ ($L(t) < L_c(t)$) then BIRD models will memorize (generalize).  Optimal generation and denoising proceeds along this phase transition boundary: in essence {\it successful generation proceeds at the edge of memorization} (Fig. \ref{fig:cartoon_figure}b).

(5) We explicitly compute the phase transition boundary between memorization and generalization for BIRD models on image data with power law fall off in the power spectral density $P(k) \sim k^{-2-\epsilon}$ with spatial frequency $k$. Here $\epsilon=0$ corresponds to scale invariant images \cite{ruderman1993statistics}.  We then show that a BIRD model will be in a generalization phase and effectively denoise for all time and noise levels if the amount of data grows with linear image size $L_I$ (e.g. $L_I$ by $L_I$ images) as $\exp L_I^\epsilon$.  This reveals the striking finding that BIRD models do not suffer from a curse of dimensionality for scale invariant images. Moreover, for small departures from scale invariance (e.g. $\epsilon \sim 0.1$ to $0.3$ as observed in practice), the data requirements are only exponential in $L_I^\epsilon$, not the total dimensionality of the space of images, which is $O(L_I^2)$.

Together, these results provide fundamental insights into the phase transition between memorization and generalization in a general class of diffusion models, and highlight the role of information restriction in allowing generalization with relatively modest amounts of data.  We summarize our results below, with much more detailed exposition in the Appendix (see e.g. App. \ref{app:notation_background} for a summary of all mathematical notation and background).

\section{Overall framework}

Diffusion models sample from a complex data distribution $\varphi \sim P_0(\varphi)$ by time reversing a forward diffusion process that transports $P_{0}(\varphi)$ to a simple isotropic Gaussian $\eta \sim \mathcal{N}(0,I)$ through a time dependent family of distributions of noised data $\phi_t \sim P_t(\phi_t)$ as time proceeds from $t=0$ to $t=1$. A single sample at any fixed time $t$ can be obtained via the interpolation $\phi_t(\varphi,\eta) = \sqrt{\bar{\alpha}_t} \varphi + \sqrt{1 - \bar{\alpha}_t} \eta$. The temporal schedule of the signal power is chosen to monotonically decrease from $\bar{\alpha}_0 = 1$ at time $t=0$ to  $\bar{\alpha}_1 = 0$. 
We define an important noise-to-signal ratio $\sigma_t^2 = \frac{1 - \bar{\alpha}_t}{\bar{\alpha}_t}$ that monotonically increases from $t = 0$ to $t = 1$. 

Diffusion models generally aim to reverse this flow through a reverse process SDE or ODE that evolves pure noise $\phi_1 \sim \mathcal{N}(0,I)$ backwards in time from $t=1$ to $t=0$, matching the marginals $P_t(\phi_t)$ of the forward process. These reverse processes use the \textit{score function} $\nabla \log P_t$ of $P_t(\phi_t)$, which must be computed or modeled. The score computation can be reframed as optimal denoising, via Tweedie's theorem:
\begin{align} \label{eq:Tweedie_m}
    \nabla \log P_t(\phi_t) = -\frac{\mathbb{E}[\eta | \phi_t]}{\sqrt{1 - \bar{\alpha}_t}} = \frac{\sqrt{\bar{\alpha}_t}}{1 - \bar{\alpha}_t} \mathbb{E}[\varphi | \phi_t] - \frac{\phi_t}{1 - \bar{\alpha}_t}.
\end{align}
One can approximate the posterior mean $\mathbb{E}[\varphi | \phi_t]$ by training a neural denoiser model $M_{t,\theta}[\phi_t]$ to predict $\varphi$ from $\phi_t$. In training this model, we usually do not have direct access to the true data distribution $P_0(\varphi)$; instead, we only have a {\it finite} dataset $\mathcal D$ consisting of $|\mathcal D|$ training samples $\varphi \in \mathcal D$ each drawn i.i.d from the true data distribution $P_0(\varphi)$. For this finite dataset, the Bayes optimal MMSE denoiser (also known as the empirical score function) is given by
\begin{align}\label{eq:full_bayes_optimal_m}
    M_t[\phi_t] = \mathbb{E}[\varphi | \phi_t ] = \sum_{\varphi \in \mathcal{D}} \varphi \,P_{train}(\varphi | \phi_t).
\end{align}
Here $P_{train}(\varphi | \phi_t)$ is the posterior probability that a clean training data point $\varphi \in \mathcal D$ at time $t=0$ led to $\phi_t$ in the forward process (Fig. \ref{fig:cartoon_figure}b). The structure of this analytic model has an appealing interpretation in terms of a Bayesian guessing game \cite{kambanalytic} where the model guesses or forms beliefs about which past data point $\varphi \in \mathcal D$ led to the noised image $\phi_t$ using the posterior $P_{train}(\varphi | \phi_t)$, then reverse flows towards the posterior mean.

While the Bayes optimal diffusion model is appealing in terms of both its analytic and conceptual simplicity, it unfortunately always memorizes the training data by eventually flowing to one of the training points $\varphi \in \mathcal D$ (see App. \ref{app:bayes_opt_mem} and \cite{biroli2024dynamical,kambanalytic}). In essence, the Bayesian guessing game, given the {\it entire} noised image $\phi_t$ becomes too easy at small enough times $t$, corresponding to small enough noise $\sigma_t$. Thus this model cannot generalize and is inconsistent with the behavior of trained neural diffusion models.

\subsection{Bayes Information Restricted Diffusion (BIRD) Models}\label{sec:BIRD_def}

The unfortunate ease of the Bayesian guessing game for the optimal denoiser raises the question of whether one can obtain alternate Bayesian diffusion models that retain analytic tractability and conceptual simplicity, but {\it also} generalize, thereby behaving more like trained neural network diffusion models. We show that we can {\it simultaneously} achieve {\it both} theoretical tractability {\it and} realistic generalization, by combining two key approaches: 1) making the Bayesian guessing game harder by restricting the image information the model uses to guess the training image; and 2) allowing different parts of the image to use {\it different} pieces of restricted information. We call such models Bayesian Information Restricted Diffusion (BIRD) models, which are inspired by and generalize \cite{kambanalytic,niedobatowards}.

In the general class of BIRD models (see App. \ref{app:BIRD} for more details) we imagine that each pixel $x$ is only allowed to observe the noisy image $\phi_t$ through a restricted, possibly stochastic, observation $\mathcal C_{x,t} \in \mathbb R^{n_{x,t}}$ of $\phi_t \in \mathbb R^d$. The observation dimension $n_{x,t}$ is less than the image dimension $d$. Then the pixel $x$ computes a denoised pixel value $M_{x,t}[\mathcal C_{x,t}]$.  The optimal Bayesian MMSE individual pixel denoiser has a structure similar to that of $M_t$ in \eqref{eq:full_bayes_optimal_m}: 
\begin{align}
    M_{x,t}[\mathcal{C}_{x,t}] = \sum_{\varphi \in \mathcal{D}} \varphi_x \, P_{train}(\varphi | \mathcal{C}_{x,t}).
    \label{eq:restricted_bayes_optimal_m}
\end{align}
Here $P_{train}(\varphi | \mathcal{C}_{x,t})$ is pixel $x$'s posterior belief distribution that any past training sample $\varphi \in \mathcal D$ led to the pixel's observation $\cC$ in the forward training Markov chain $\varphi \rightarrow \phi_t \rightarrow \mathcal C_{x,t}$ (see App. \ref{app:notation_background} for notation and App. \ref{app:BIRD} for explicit expressions for $P_{train}(\varphi | \mathcal{C}_{x,t})$). The pixel's observation restriction $\mathcal C_{x,t}$ at the end of this forward information channel makes the Bayesian guessing game harder, and ideally prevents memorization.  The MMSE denoiser for pixel $x$ in \eqref{eq:restricted_bayes_optimal_m} simply returns the posterior mean training image pixel value at pixel $x$. 

A simple channel restriction $\mathcal{C}_{x,t}$ is a deterministic linear projection operator $\mathcal P_{x,t}$ applied to $\phi_t$. 
The LS model of \cite{kambanalytic} is a special case where $\mathcal P_{x,t}$ projects an image $\phi_t$ onto a coordinate subspace of pixel components within a local spatial patch $\Omega_{x,t}$ surrounding pixel $x$, yielding the restricted observation $\cC[\phi_t] = \phi_{\Omega_{x,t}}$, where $\phi_{\Omega_{x,t}}$ is a vector of all pixel values of $\phi_t$ {\it restricted} to the local patch $\Omega_{x,t}$ (see App.\ref{app:notation_background} for notation).  The BIRD framework also generalizes ES and ELS machine of \cite{kambanalytic} (see App.\ref{app:BIRD_examples} for details).  We call such models spatially local BIRD models, indicating that spatial locality is the specific form of information restriction employed, though BIRD models and our theory of them are more general. 

We next show in Sec.~\ref{sec:describe_early} that spatially local BIRD models are highly relevant in that they describe well the early learning phase of many diffusion model architectures. And after that in Sec.~\ref{sec:theory_gen_mem} we show that the general class of all BIRD models provides a highly useful abstraction, in that a general theory of the memorization-generalization phase transition can be derived for {\it all} such BIRD models for essentially {\it any} data distribution.

\section{BIRD models describe the early learning phase of diffusion models} \label{sec:describe_early}

\begin{figure}[t]
    \centering
    \begin{subfigure}[t]{0.24\textwidth}
        \centering

        \sidecaptionimage{\linewidth}{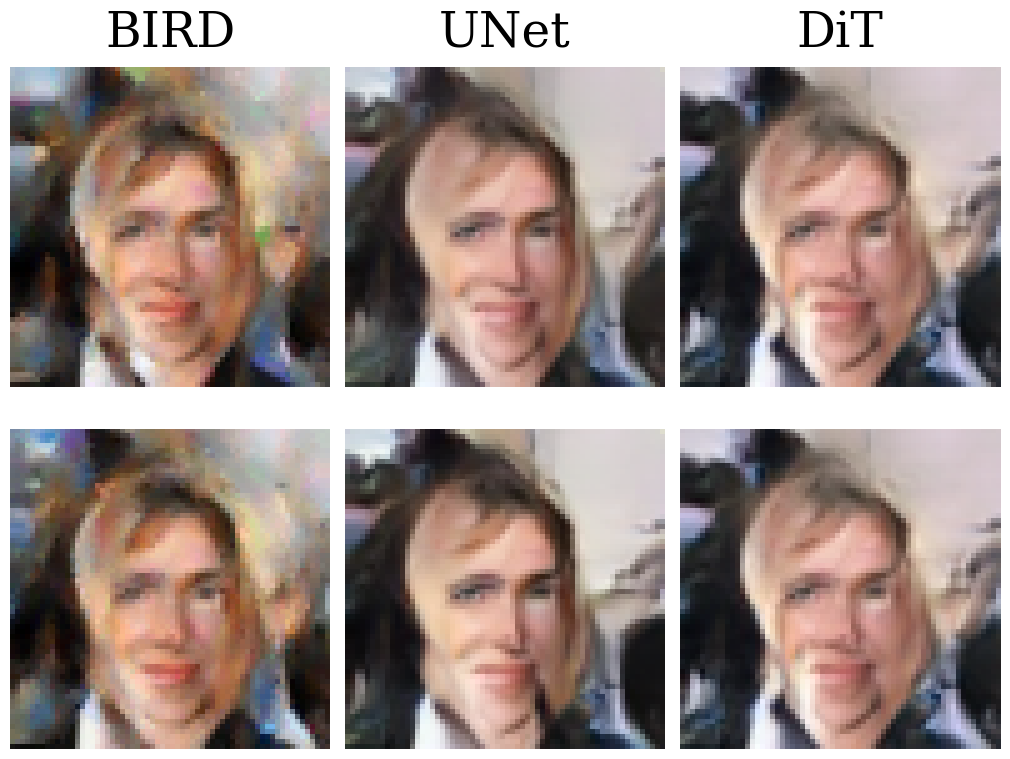}{$\mathcal{D}_2\,\,\,\,\,\,\,\,\,\,\,\,\,\mathcal{D}_1\,\,\,$}

        \vspace{0.25em}

        \sidecaptionimage{\linewidth}{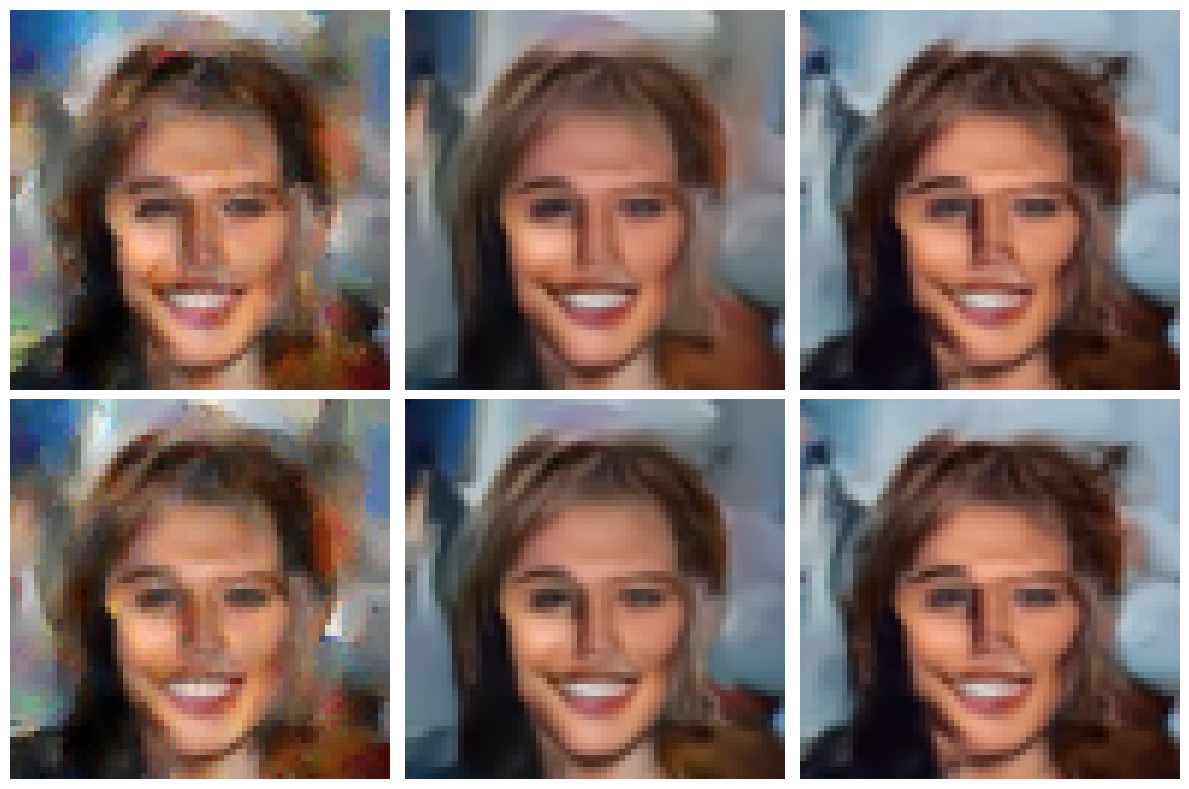}{$\mathcal{D}_2\,\,\,\,\,\,\,\,\,\,\,\,\,\,\mathcal{D}_1$}
        
        \vspace{0.25em}

        \sidecaptionimage{\linewidth}{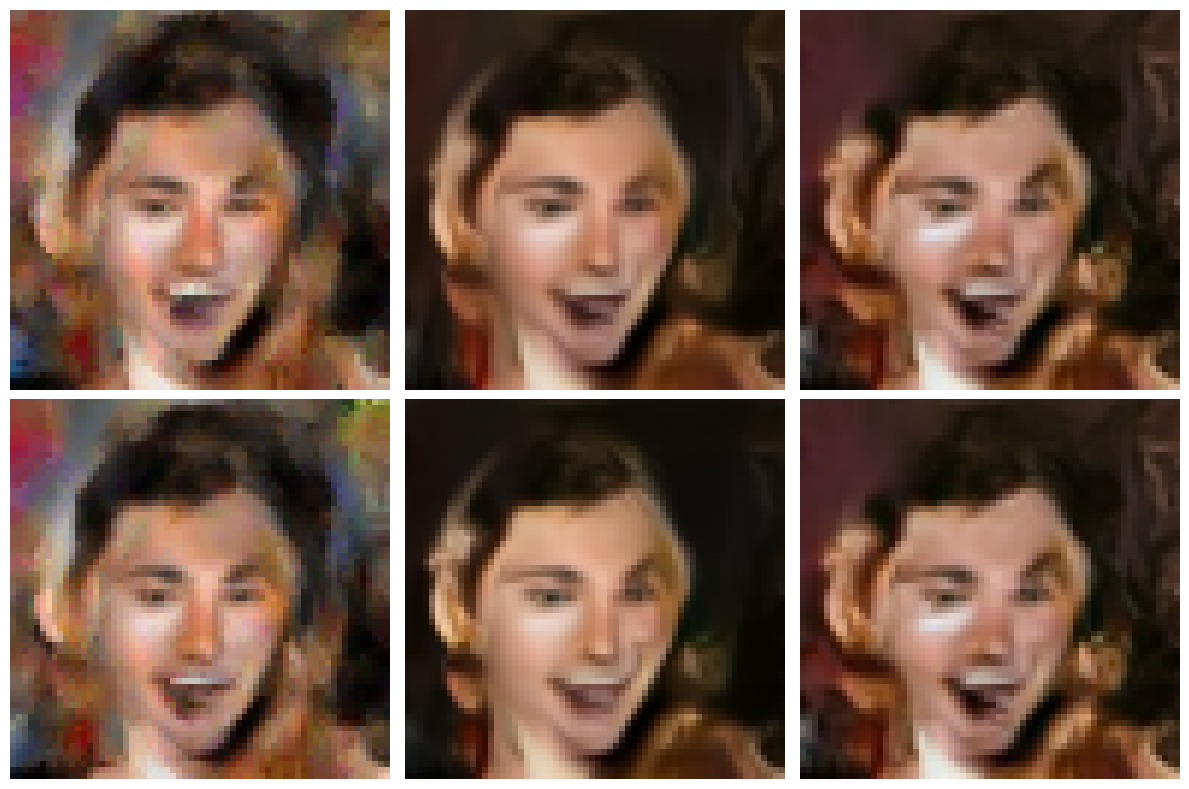}{$\mathcal{D}_2\,\,\,\,\,\,\,\,\,\,\,\,\,\mathcal{D}_1$}

        \caption{Celeba64}
    \end{subfigure}
    \begin{subfigure}[t]{0.24\textwidth}
        \centering

        \includegraphics[width=\linewidth]{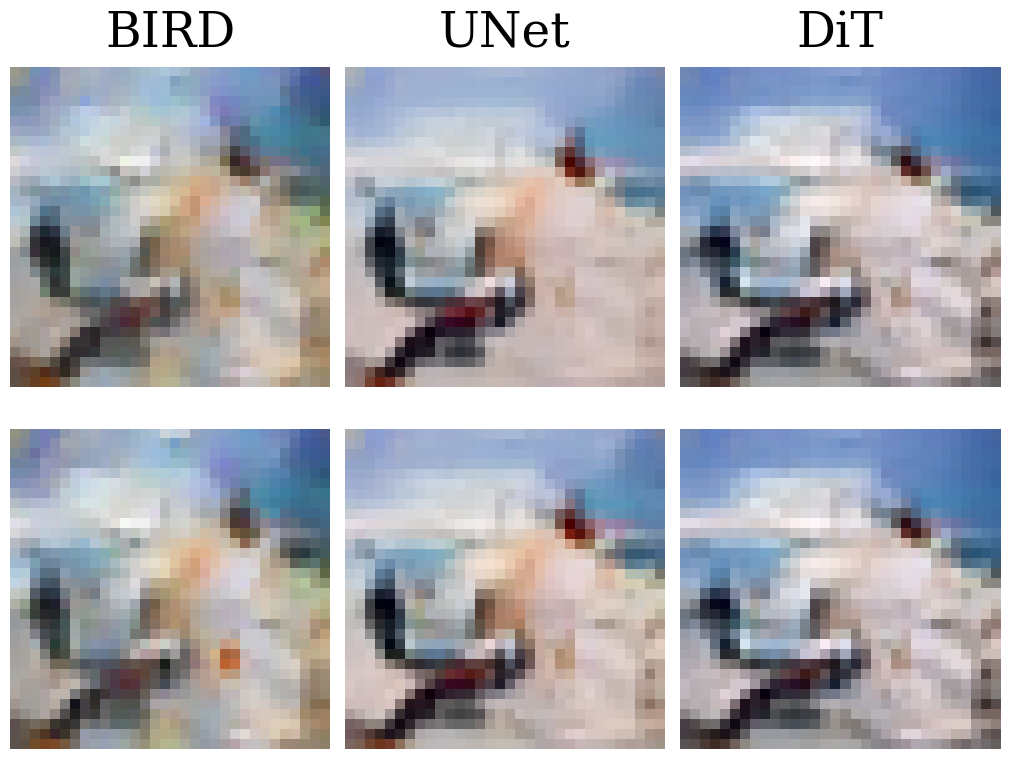}

        \vspace{0.25em}

        \includegraphics[width=\linewidth]{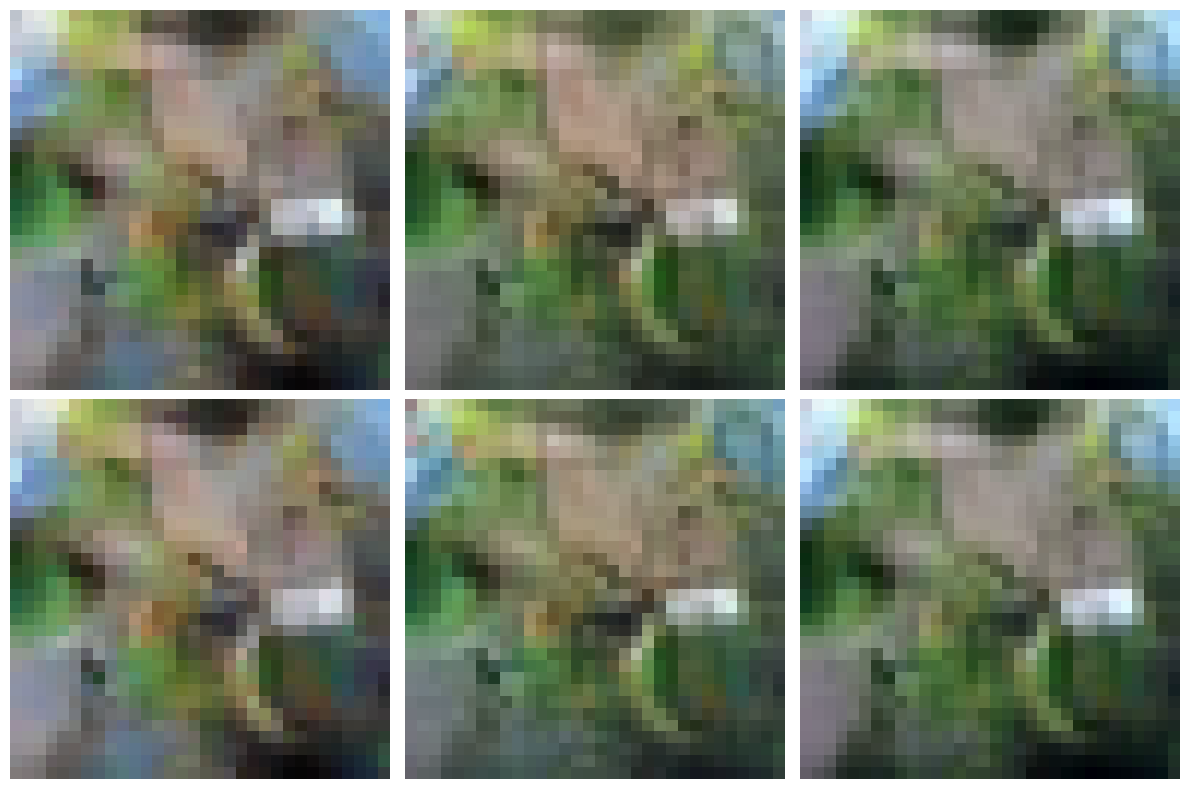}
        
        \vspace{0.25em}

        \includegraphics[width=\linewidth]{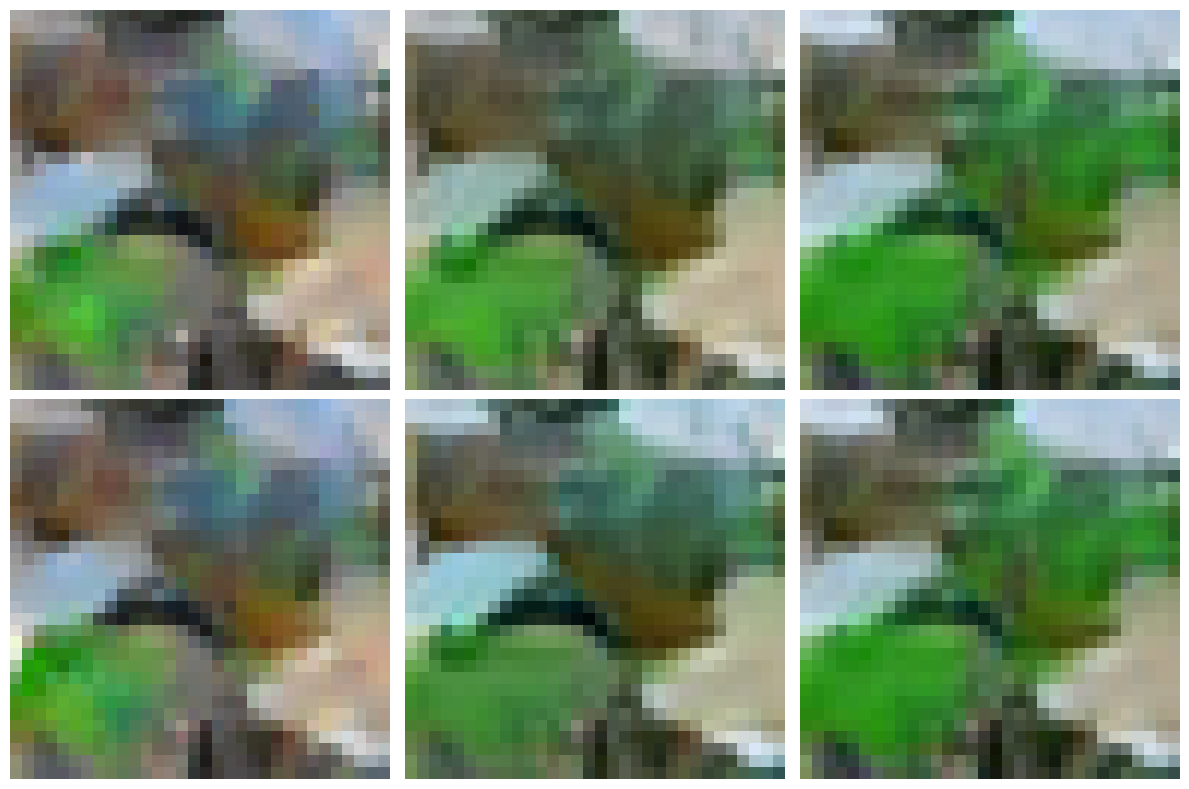}

        \caption{Cifar10}
    \end{subfigure}
    \begin{subfigure}[t]{0.24\textwidth}
        \centering

        \includegraphics[width=\linewidth]{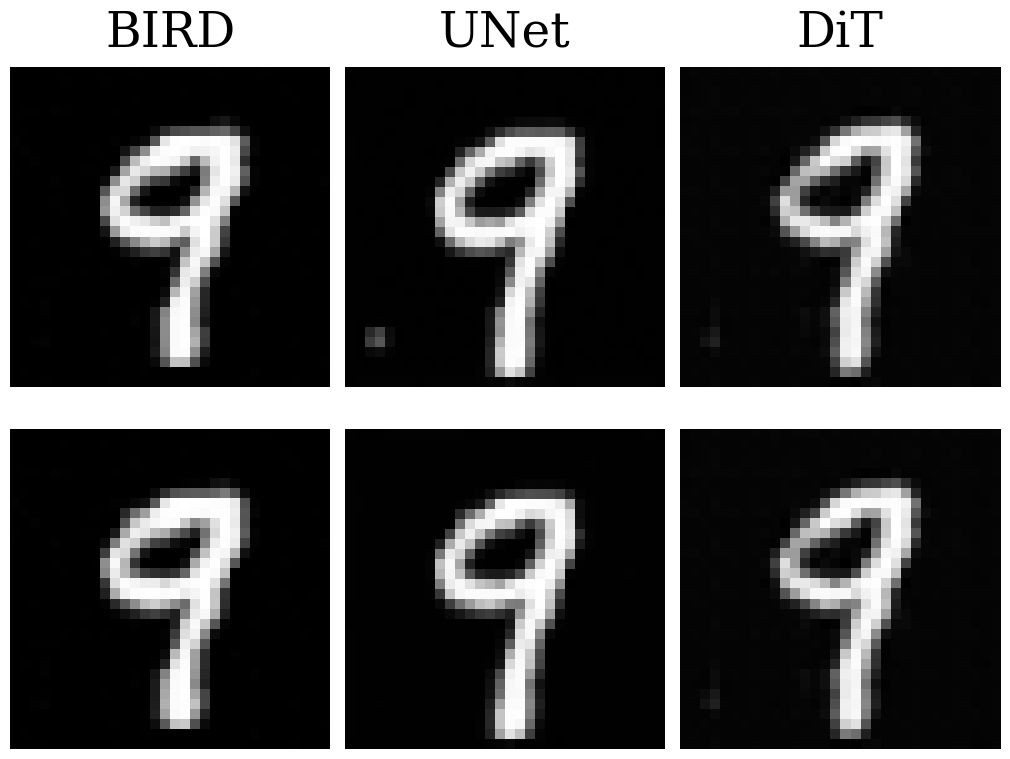}

        \vspace{0.25em}

        \includegraphics[width=\linewidth]{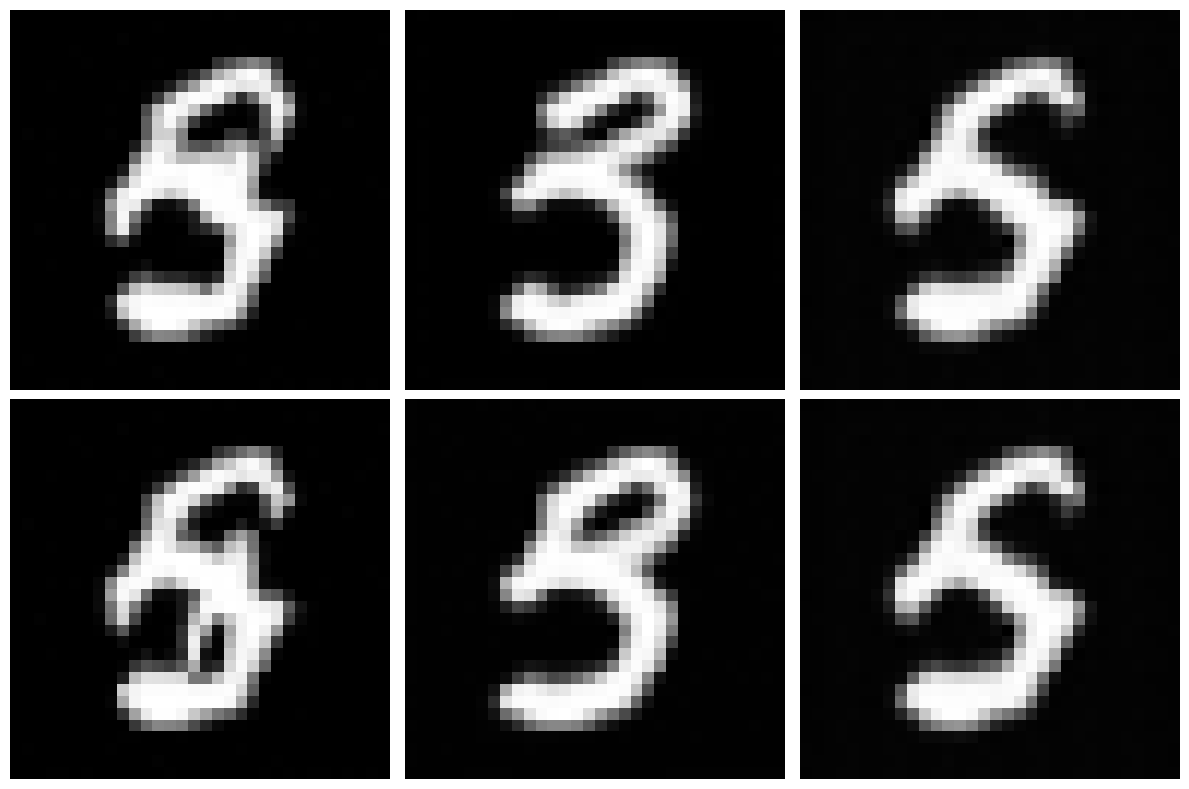}
        
        \vspace{0.25em}

        \includegraphics[width=\linewidth]{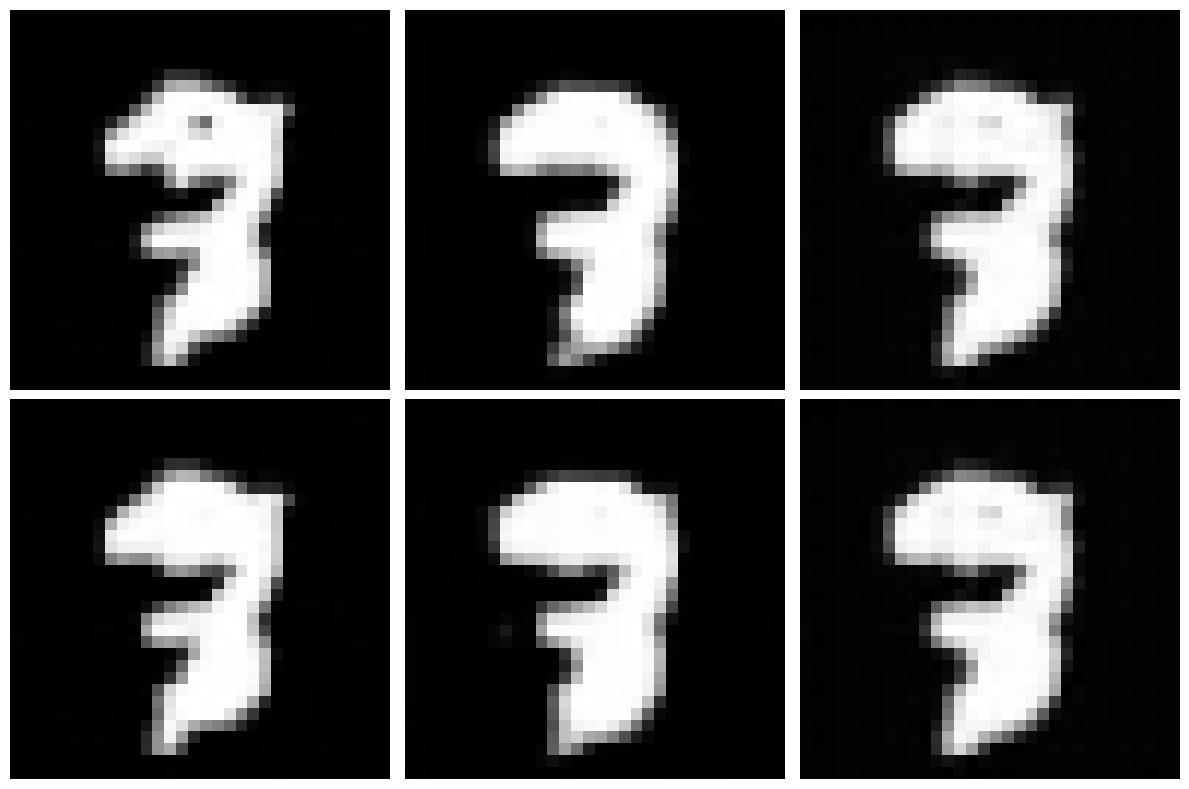}

        \caption{MNIST}
    \end{subfigure}
    \begin{subfigure}[t]{0.24\textwidth}
        \centering

        \includegraphics[width=\linewidth]{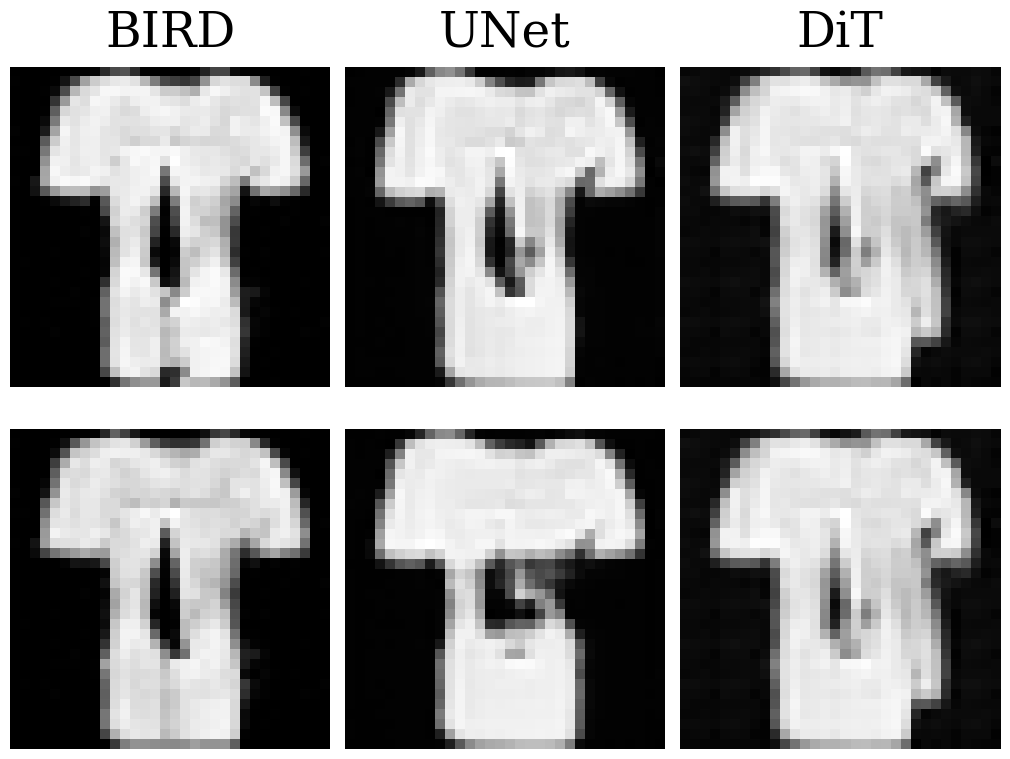}

        \vspace{0.25em}

        \includegraphics[width=\linewidth]{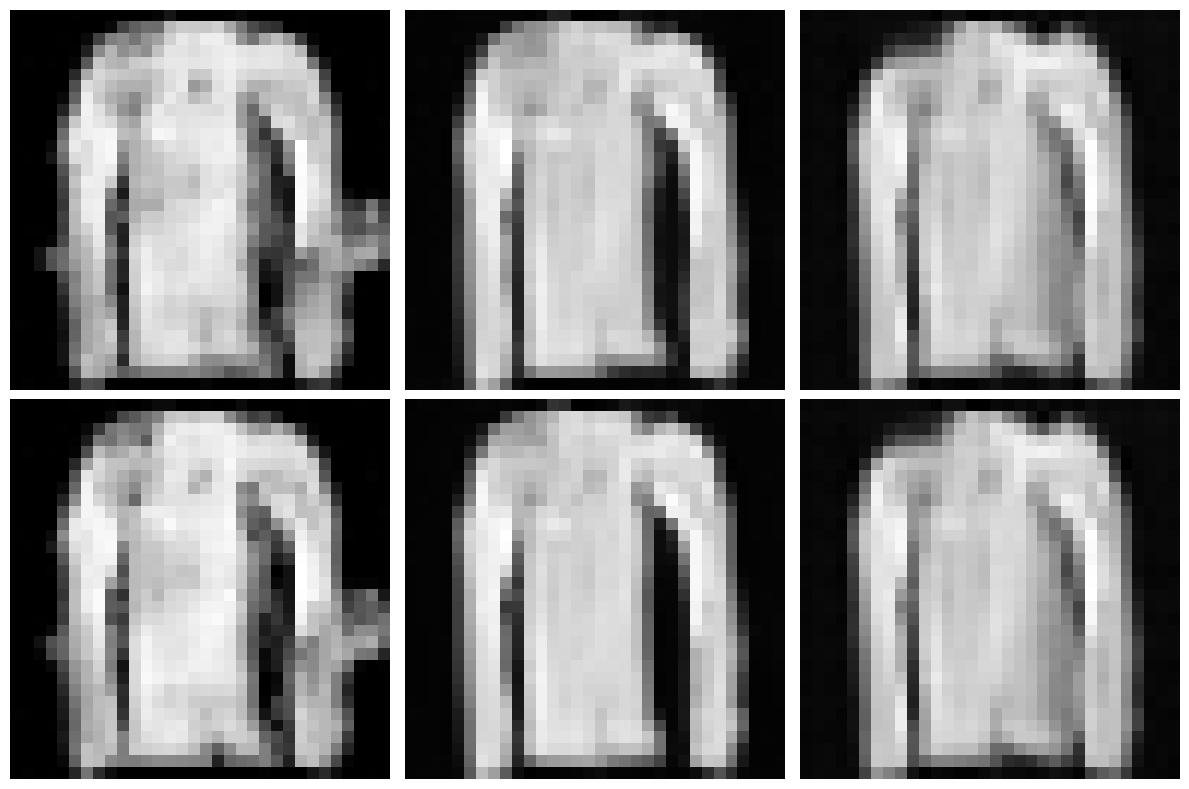}
        
        \vspace{0.25em}

        \includegraphics[width=\linewidth]{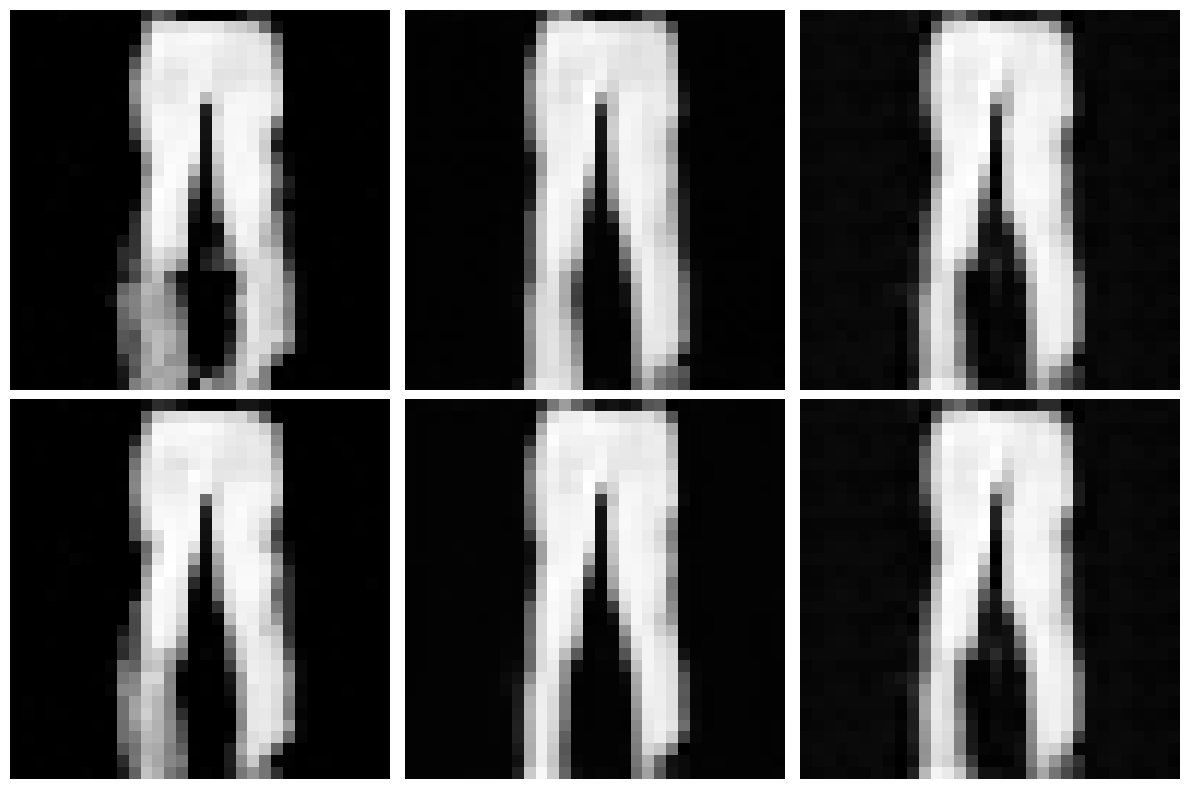}

        \caption{FashionMNIST}
    \end{subfigure}

    \caption{The consistent and robust generalization of diffusion models early in training, across various datasets and architectures, is well-predicted by spatially local BIRD models with a patch size near the critical scale, as described below. Here UNets and DiTs are trained on disjoint dataset subsets $\mathcal{D}_1$ and $\mathcal{D}_2$, but produce nearly identical images when fed the same noise input. These images are both similar to the corresponding BIRD model outputs. Many more such samples are shown in App. \ref{app:samples} in figs. \ref{fig:app_samps_cba}-\ref{fig:carveout}, along with details of the experiment and training procedure in App. \ref{app:early_training}.
    }
    \label{fig:early_training_consistency}
\end{figure}

\cite{kambanalytic} showed that spatially local BIRD models could accurately predict the {\it individual} image outputs of fully trained small CNN based diffusion models on a {\it case by case} basis for a range of datasets with a very high median $r^2 \approx 0.9$ between theory and experiment.  Here we extend this result by showing that spatially local BIRD models can {\it also} predict {\it individual} image outputs of more sophisticated architectures, like UNets with self-attention and diffusion transformers (DiTs), when they are in the \textit{earliest stages of the training process}. 

In our experiments on four standard datasets (Celeba64, CIFAR10, FashionMNIST, and MNIST), and two architectues (UNets and DiTs), we find that the agreement between spatially local BIRD theory and early trained diffusion models peaks at values of $r^2 \sim 0.9$ around $10$ to $30$ epochs of training before slowly declining thereafter. See tables \ref{tab:training_quantitative_1} and \ref{tab:training_quantitative_1_u} for detailed numerical values of the agreement, and see App. \ref{app:early_training} for more experimental parameters, and for our approach to patch selection and scale calibration, and further samples.  

We also show, in this early training regime, both BIRD models and trained architectures exhibit the phenomenon of consistent and robust generalization, where different models trained on disjoint data subsets nevertheless generate the same images when fed the same input noise (Fig. \ref{fig:early_training_consistency}). 
In our experiments, we compare the outputs of the $3$ distinct models (BIRD, UNets, DiTs) each obtained from two disjoint datasets $\mathcal D_1$ and $\mathcal D_2$. 
We find that all $3\times 2=6$ images are highly similar (Fig. \ref{fig:early_training_consistency}). 
This indicates BIRD models provide a good description of the generalization phase of both DiTs and UNets early in training, and that such BIRD models themselves are in a generalizing phase. More samples are included in App.~\ref{app:early_training}, as well as quantitative validations in Fig. \ref{fig:early_training_1} and tables \ref{tab:training_quantitative_1} and \ref{tab:training_quantitative_1_u}. 

This strong agreement between BIRD model theory and neural experiments motivates the use of BIRD models as a theoretical laboratory within which to study the nature of the memorization-generalization phase transition. We elucidate this theory next in Sec.\ref{sec:theory_gen_mem}.

\vspace{-1em}
\section{The memorization-generalization phase transition in BIRD models} \label{sec:theory_gen_mem}
\vspace{-1em}
\begin{figure}
    \centering
    \includesvg[width=\linewidth]{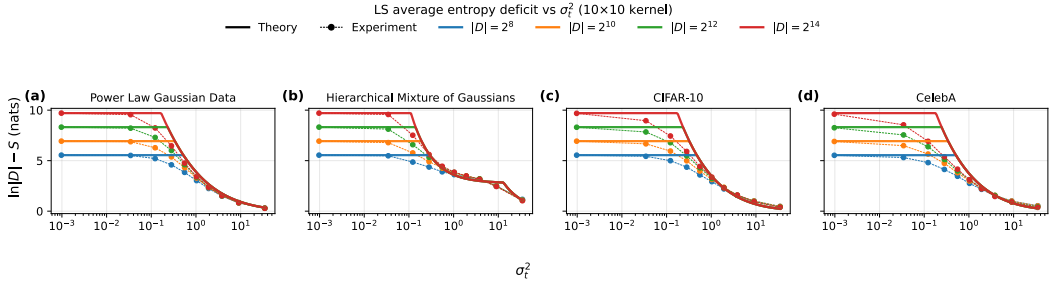}
    \caption{Comparison of theory and experiment for the memorization-generalization phase transition. Experimental curves plot numerical estimates of the average entropy deficit ($\ln |\mathcal{D}| - S[P(\varphi|\phi_\Omega)]$) in a spatially local BIRD model with a patch size $5 \times 5$. $\sigma_t^2$ is the noise strength. Theory curves plot mutual information or its Gaussian upper bounds. As noise $\sigma^2_t$ decreases, the theoretically derived mutual information curves rise in the generalizing phase then saturate to $\ln \mathcal |D|$ in the memorizing phase. Moreover the mutual information theory curves closely track, or upper bound the entropy deficit experimental curves.  Each sub-figure corresponds to a different true data distribution: (a) Translationally invariant Gaussian data with $\alpha=1.7$ power law spectrum (see Sec.\ref{sec:spec_len}) (b) An isotropic hierarchical mixture of Gaussians model (see App. \ref{app:hierarchical_mixture_of_gaussians} for theory and experiment); (c) CIFAR10; (d) CelebA32. The theory curves for CIFAR10 and CelebA32 are given by a Gaussian upper bound on mutual information (see App. \ref{app:gaussian_bound}). See App.\ref{app:empirical_post_entropy} for further details.}

    \label{fig:posterior_entropy}
\end{figure}
We first show, via Fano's inequality, that perfect guessing in the pixel Bayesian guessing game (i.e. memorization) implies
$I(\mathcal{C}_{x,t}; \varphi) \geq \ln |D|$, where $I(\mathcal{C}_{x,t}; \varphi)$ is the mutual information between a {\it test} data point $\varphi$ drawn from the {\it true} data distribution $P_0(\varphi)$ and the pixel observation $\mathcal{C}_{x,t}$ under the forward testing Markov chain $\varphi \rightarrow \phi_t \rightarrow \cC$ (see App.\ref{app:notation_background} for notation and App.~\ref{app:fanos_bound} for a proof).  However, an approach via Fano's inequality only proves that in the memorization phase, we have $I(\mathcal{C}_{t,x}; \varphi) \geq \ln |\mathcal{D}|$. It does {\it not} prove the converse, i.e. that in the generalizing phase, the opposite holds: $I(\mathcal{C}_{t,x}; \varphi) < \ln |\mathcal{D}|$. 

To actually locate the phase transition between memorization and generalization, in App.~\ref{app:main_proof}, we prove for posteriors with sub-exponential tails, in large dataset size and high observation dimension limit, that
\begin{align}\label{eq:posterior_entropy_formula_main}
    S[P_{train}(\varphi | \mathcal{C}_{x,t})] &= \begin{cases}
        \ln|\mathcal{D}| - I(\varphi; \cC) & \ln|\mathcal{D}| > I(\varphi; \cC) \,\,\,\,\,\,\,(\text{generalization})\\ 
        0 & \ln|\mathcal{D}| \leq I(\varphi; \cC)\,\,\,\,\,\,\,(\text{memorization}).\\
    \end{cases}
\end{align}
Thus the phase transition boundary is {\it exactly} given by the equation
\begin{equation} \label{eq:memorization}
\ln|\mathcal{D}| = I(\varphi; \cC).
\end{equation}
Intuitively, more mutual information than the entropy $\ln|\mathcal{D}|$ of the uniform prior on the training data makes Bayesian guessing easy with zero posterior entropy $S[P_{train}(\varphi | \mathcal{C}_{x,t})]$ leading to memorization, while less mutual information makes Bayesian guessing hard with positive posterior entropy leading to generalization.  

We can experimentally test \eqref{eq:posterior_entropy_formula_main} because $S[P_{train}(\varphi | \cC)]$ can be numerically computed for any finite dataset, while the mutual information $I(\varphi; \cC)$ can be exactly calculated for simple data distributions and upper bounded for real data sets. In order to collapse the curves from different dataset sizes, and to compare theory and experiment,  we plot $\ln \dD - S[P_{train}(\varphi | \phi_{\Omega})]$, obtained from numerical experiments, and compare it either to: (1) the theoretically obtained mutual information (which it should be equal to in the generalizing phase), or (2) $\ln \dD$ (which it should be equal to in the memorizing phase). For real datasets, we cannot calculate the exact mutual information, so we use the Gaussian upper bound derived in App.\ref{app:gaussian_bound} which only depends on the second order statistics of the data.
Fig. \ref{fig:posterior_entropy} exhibits numerical validation of \eqref{eq:posterior_entropy_formula_main} for both real and toy datasets. 

Our results above generalize previous work \cite{biroli2024dynamical}, which studied the memorization to generalization phase transition for the empirical score function (in which each pixel observes the entire image).  They defined the transition time $t$ to occur when $S[P(\phi_{t})] = S_{sep}^{|\mathcal{D}|}$ where $S_{sep}^{|\mathcal{D}|}$ is entropy of $\dD$ separated Gaussians and computed this transition time for the training set only, when the data came from an isotropic Gaussian distribution.  We include information restriction, generalize to essentially arbitrary data distributions, and find a general information theoretic criterion for the memorization-generalization transition for both training and testing data, as well as a pointwise condition for any single observation (App. \ref{app:main_proof}).

\vspace{-1em}
\section{How BIRD models circumvent the curse of dimensionality}
\vspace{-1em}

\begin{figure}
    \centering
    \begin{subfigure}[b]{0.45\textwidth}
        \centering
        \includegraphics[width=\textwidth]{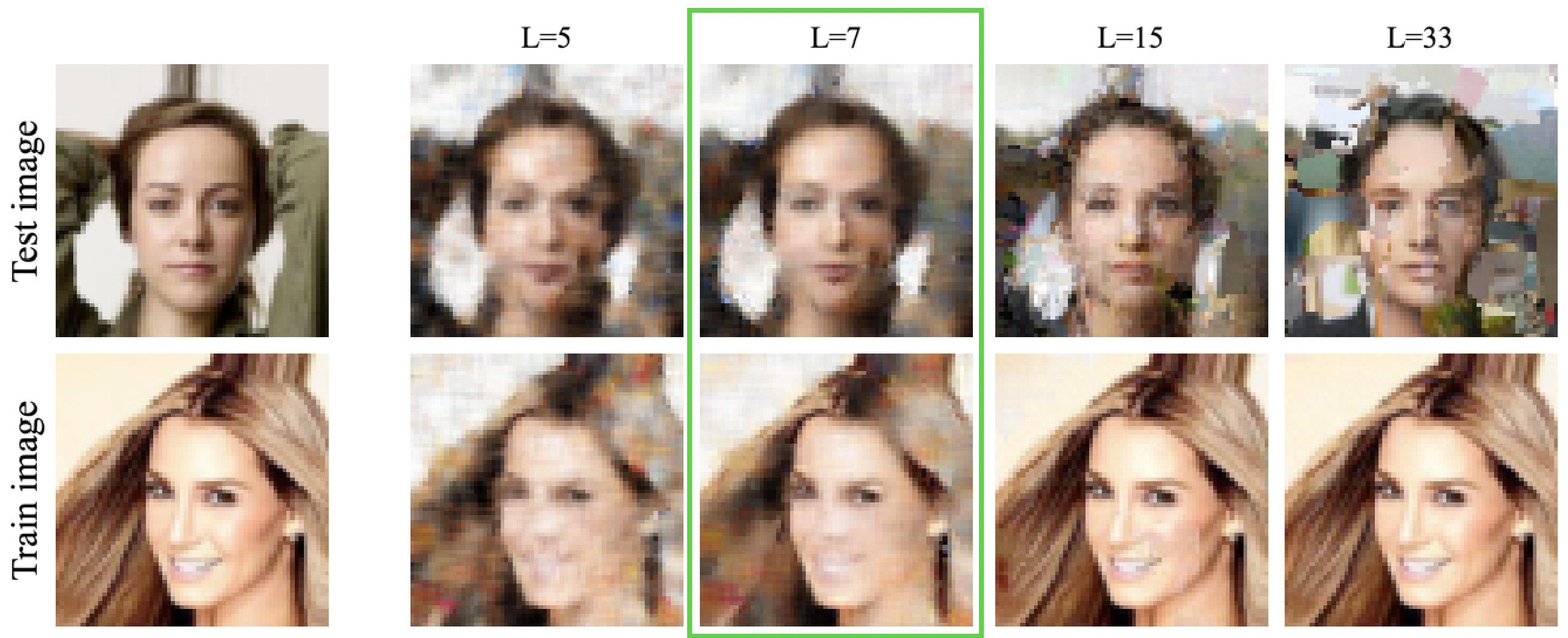}
        \caption{}
        \label{fig:pixelization}
    \end{subfigure}
    \begin{subfigure}[b]{0.5\textwidth}
        \centering
        \includegraphics[width=0.69\textwidth]{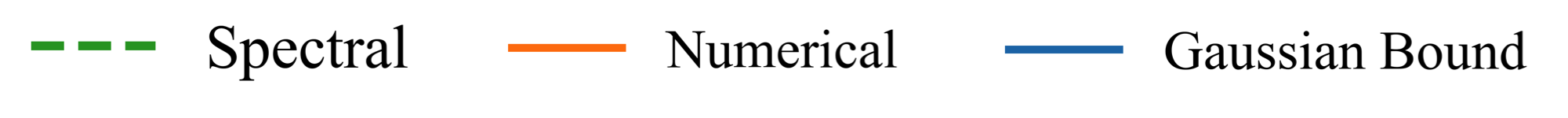}

        \includegraphics[width=0.32\textwidth]{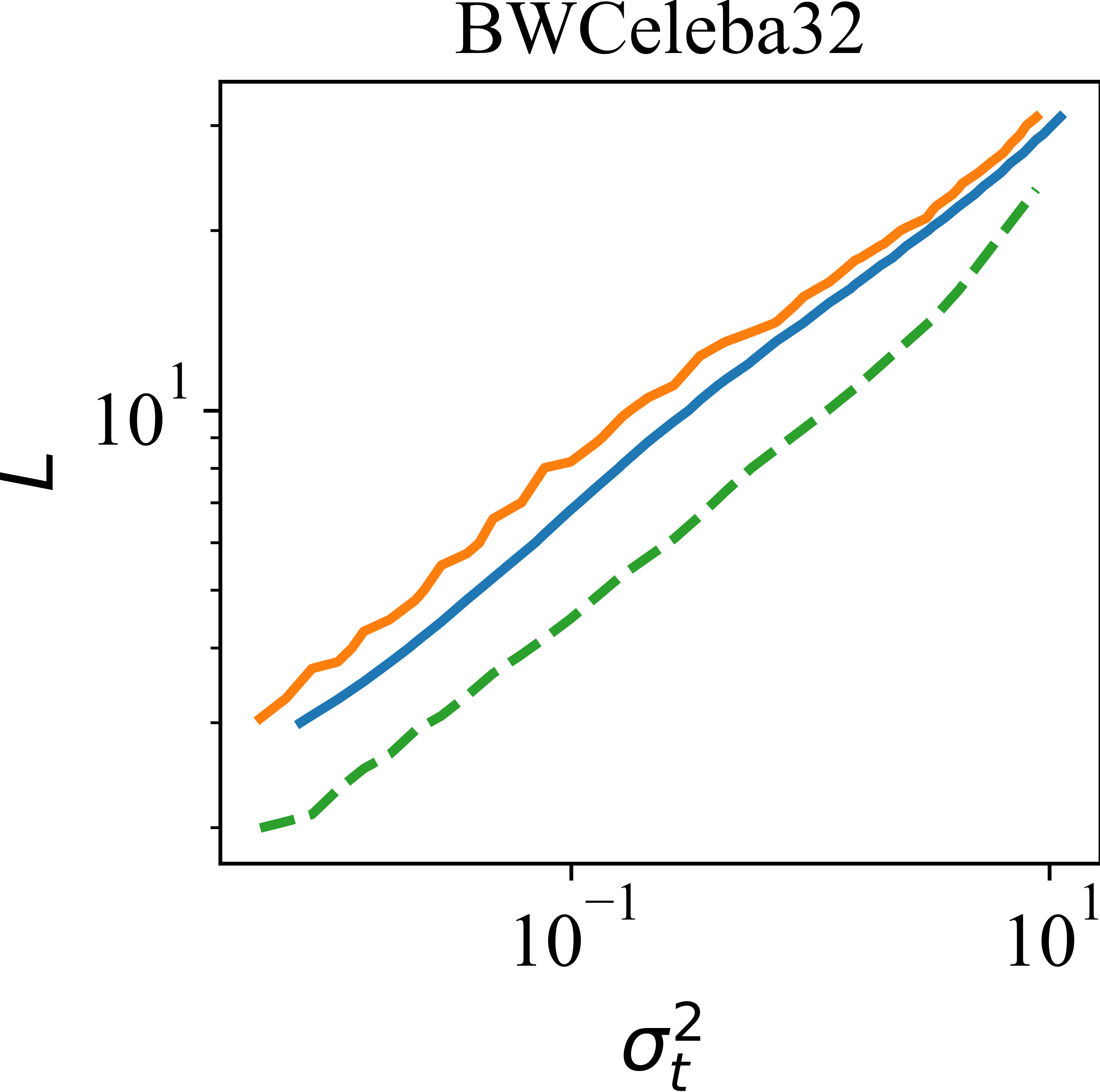}
        \includegraphics[width=0.32\textwidth]{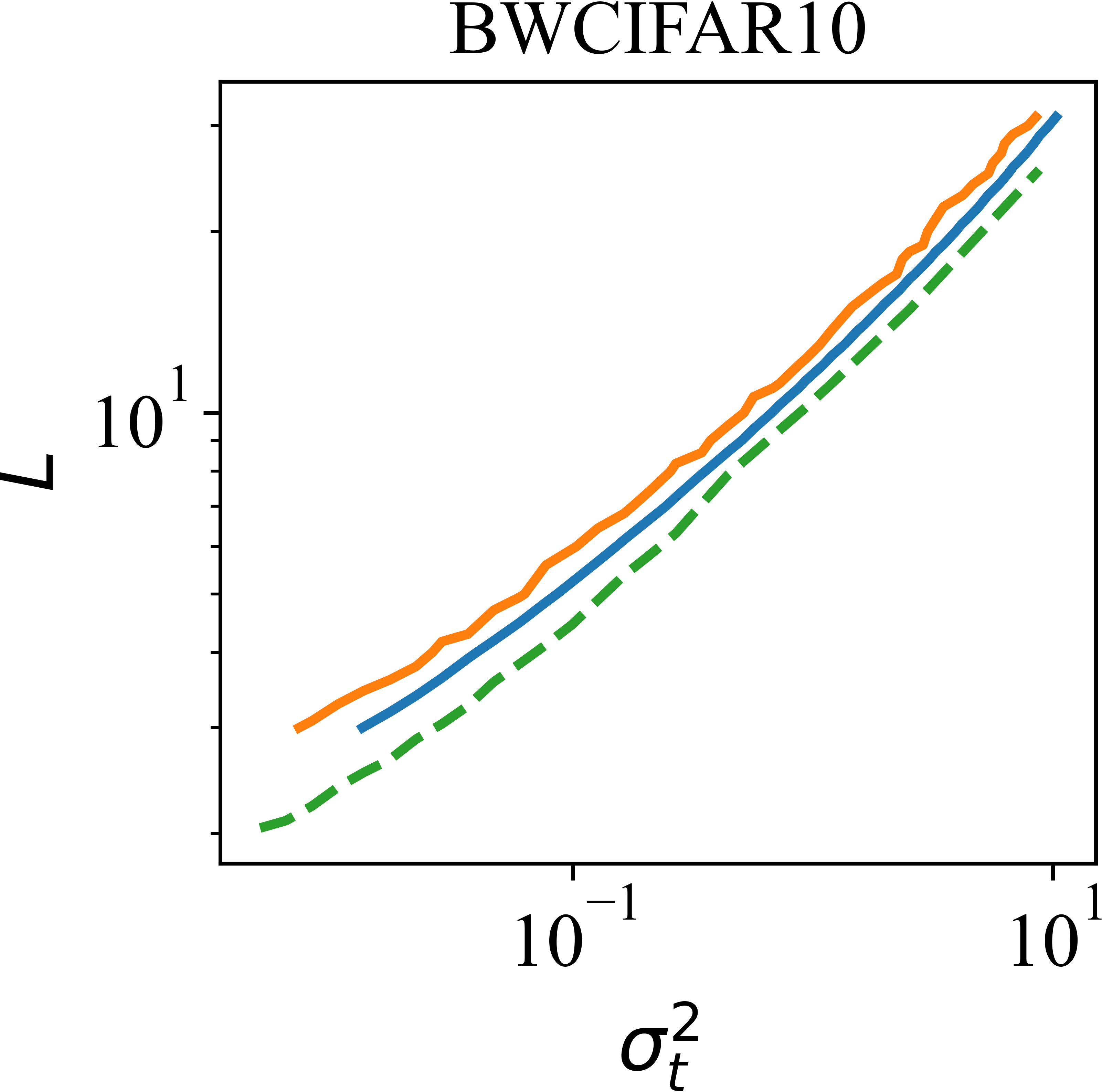}
        \includegraphics[width=0.32\textwidth]{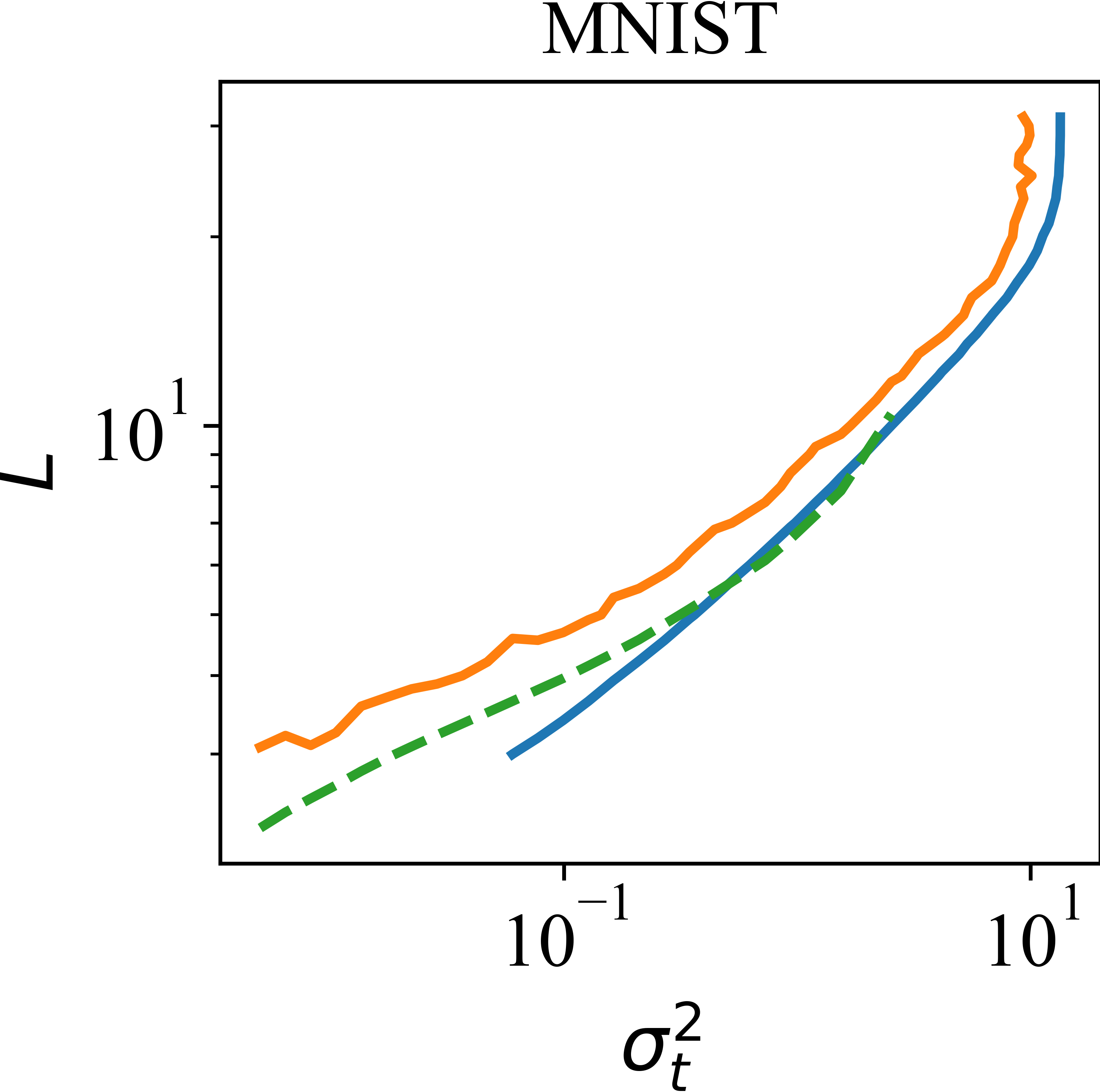}

        \caption{}
        \label{fig:mem_vs_spec}
    \end{subfigure}

    \caption{Effect of varying the patch scale $L$ (for square patches) in a spatially local BIRD model. (a) The top (bottom) shows the outcome of denoising of a single test (train) image using BIRD models of different patch sizes $L$ starting from noise $\sigma_t=1$. Optimal denoising occurs when $L=7$ (in general this optimal denoising scale coincides with $L_c$). For $L=5 < L_c$, from theory we know the denoiser is robust, but denoising performance is poor both on the training and test image. For $L=15 > L_c$, the model is in the memorization phase, and it poorly denoises the test image (compare rightmost to leftmost image in top row), but perfectly denoises (i.e. memorizes) the training image (compare rightmost to leftmost image in bottom row). (b) Comparisons between the anisotropic Gaussian bound (Lemma \ref{lemma:gaussian_length_bound_main}), the numerically computed values for the critical scales, and the spectral scale for BWCeleba32, BWCIFAR10, and MNIST, which are generally very close (see App. \ref{app:empirical_scaling} for details).\vspace{-1em}}
    \label{fig:multiscale}
\end{figure}


We now explain how, for scale invariant natural images, spatially local BIRD models can effectively denoise while avoiding memorization, even for large images.  

\subsection{The critical patch scale at the memorization-generalization transition}\label{sec:crit_scale}

A very natural information restriction approach, studied in \citep{kambanalytic}, is to allow each pixel $x$ to observe a local patch $\Omega_{L,x}$, where we define the patch scale $L = \sqrt{|\Omega_{L,x}|}$. Following $\cite{kambanalytic}$ we denote this projection as $\phi_{t;\Omega_{L,x}}$. For square patches \citep{kambanalytic} $L$ is simply the side length. We consider a one parameter family of patches, indexed by scale $L$, obeying $\Omega_{L',x} \supset \Omega_{L,x}$ if $L' > L$. The fact that larger patches include smaller patches implies that $I(\varphi, \phi_{t,{\Omega_{x,L}}})$ is monotonically increasing with the patch scale $L$. This implies, according to \eqref{eq:posterior_entropy_formula_main} that there is a \textit{critical patch scale} $L_{c}(\sigma_t)$, for any fixed noise level $\sigma_t$, obeying the equality
\begin{align}\label{eq:collapse_condition_m}
    \ln|\mathcal{D}| = I(\varphi ; \phi_{t,{\Omega_{L,x}}}),
\end{align}
such that if $L<L_c$ ($L>L_c)$ the spatially local BIRD model is in the generalizing (memorizing) phase.  $L_c(\sigma_t)$ then defines a phase boundary between memorization and generalization in the joint space of patch scale $L$ and noise $\sigma_t$ (or time $t$ in the reverse process). 

In high dimensions, the critical scale also generally sets the \textit{optimal} scale (with respect to denoising error on the test set) for the spatially local BIRD model. The reason for this is as follows. If the model uses a scale $L > L_c$, denoising will reduce to nearest-neighbor search over the training set, which will fail to correspond to the underlying population distribution over images and thus yield significant statistical error (see \ref{fig:pixelization} for an example of this failure mode as the length scale $L$ is taken large). Meanwhile, if $L < L_c$, the noised training set distribution provides a robust estimate for the true distribution over noised images, but the information the model learns about the sample $I(\varphi ; \phi_{\Omega_{L;x}})$ will be less than the critical value $I(\varphi ; \phi_{\Omega_{L_{c};x}})$, so it will perform worse than a higher sub-critical scale. In low dimensions, there is no longer a `sharp' phase transition as above, but the optimal scale $L$ will nonetheless generally sit around $L \approx L_c$.

Without {\it any} assumptions on the data distribution, we can make a very general statement about the shape of the phase boundary $L_c(\sigma_t)$. Since  $I(\varphi ; \phi_{t,{\Omega_{L,x}}})$ monotonically increases with $L$ and decreases with $\sigma_t$, implicit differentiation of \eqref{eq:collapse_condition_m} with respect to $L$ implies that $L_{c}(\sigma_t)$ monotonically increases with the noise $\sigma_t$. 

With relatively few additional assumptions, we can make more precise statements about the phase boundary $L_{c}(\sigma_t)$ and its dependence on $|\mathcal{D}|$. We first establish the following bounds, which we call the `anisotropic' and `isotropic' Gaussian bounds (see App. \ref{app:scaling_bounds}):
\begin{lemma}\label{lemma:gaussian_length_bound_main}
    Suppose $\varphi$ is drawn from an arbitrary data distribution $P_0$ with mean $\mu$ and covariance $\Sigma$. We define the distribution $Q_1 = \mathcal{N}(\mu, \Sigma)$ to be a Gaussian with identical second-order statistics to $P_0$, and $Q_2 = \mathcal{N}(\mu, \text{Diag}[\Sigma])$ to be the Gaussian with identical diagonal covariance and zero off-diagonal elements. We then have $L_{c;P}(\sigma_t) \geq L_{c; Q_1}(\sigma_t) \geq L_{c; Q_2}(\sigma_t)$.
\end{lemma}
The Gaussian bounds are helpful numerically because they can be computed directly from the data covariance matrix, which enables direct numerical computation of a lower bound on the critical scale for realistic data (fig. \ref{fig:mem_vs_spec}). It also theoretically provides a general lower bound of the form $L_{c}(\sigma_t) \sim \Omega(\sigma_t \sqrt{\ln|\mathcal{D}|})$ (see App. \ref{app:gaussian_bound} for details). We also show (see App. thm. \ref{theorem:extensive_scaling_appendix} for proof) a very general scaling theorem, which covers a wide range of naturalistic images and shows that $L_{spec} \sim O(\sigma_t)$ scaling of the critical scale is \textit{generic}:
\begin{theorem}\label{theorem:extensive_scaling_main}
    Suppose we have a translationally-invariant distribution over images $\varphi$ with bounded pointwise variance and asymptotically image-size-extensive entropy, with an entropy density $\lim_{|\Omega| \to \infty} \frac{1}{|\Omega|} S(\varphi_{\Omega}) = \tilde{S}$. Then, for $\ln|\mathcal{D}|$ which is $O(1)$ with respect to $\sigma_t$,
    \begin{align}\label{eq:bulk_mem_scaling_m}
         L_{c}(\sigma_t) \sim O\big(\sigma_t \sqrt{\ln|\mathcal{D}|}\big)
    \end{align}
\end{theorem}
The intuition behind this theorem is that for sufficiently high noise, each pixel will confer an amount of information proportional to the signal-to-noise ratio $O(1/\sigma_t^2)$ about the underlying image $\varphi$, so the whole patch provides $O(\sigma_{t}^{-2}|\Omega|)$ bits of information; setting this equal to $\ln|\mathcal{D}|$ and rearranging gives (\ref{eq:bulk_mem_scaling_m}). 

Our theory implies if $L>L_c(\sigma_t)$, the spatially local BIRD denoiser will memorize.  We examine the denoising performance of such a denoiser as a function of the patch scale $L$ and confirm this memorization behavior in Fig. \ref{fig:pixelization}.  However, merely avoiding memorization is not sufficient to guarantee that a model can denoise \textit{effectively}: for instance, as shown in Fig. \ref{fig:pixelization}, BIRD model with a small scale such as $L \ll L_c$ may be robust to the training data (i.e. generalize) but be a poor denoiser. The reason is that with such a small patch size $L$, it cannot access correlations that are important for denoising (at a given noise level).  We next examine another length scale for natural images that is relevant for (linear) denoising. 

\subsection{The spectral length scale for effective linear denoising} \label{sec:spec_len}

While natural images contain correlations at many different length scales, existing explorations of spatial locality in diffusion models \citep{lukoianov2025locality, dieleman2024spectral_autoregression} emphasize a quantity which we call the \textit{spectral scale}, which corresponds to the scale of relevant second-order correlations in the data. Naturalistic images have a well-known power-law falloff in their power spectral density of $P(k) \approx A|k|^{-\alpha}$ \citep{ruderman1993statistics}. The exponent $\alpha$ is typically $\alpha \approx 2 - \epsilon$, with $\epsilon$ typically in the range 0.1-0.3, near the perfectly scale-invariant value $\alpha = 2$. When performing denoising, there is a natural spectral length scale, defined as the inverse of the spatial frequency where the noise power $\sigma_t^2$ equals the signal power, yielding 
 $ L_{spec} = ( \frac{\sigma_t^2}{A})^{1/\alpha}$. 
For exactly scale-invariant images with $\alpha = 2-\epsilon$, this reduces to $L_{spec} \sim \sigma_t^{2/(2-\epsilon)}$. This scale can also be interpreted as the approximate radius of the optimal linear denoiser, or the Wiener filter \citep{wiener1949extrapolation}. Because of the relatively compact support of this filter, a spatially local BIRD model with patch scale $L \geq L_{spec}$ should still denoise with relatively small error. Indeed, \cite{lukoianov2025locality} established empirically that BIRD models with $L \sim L_{spec}$ generated reasonable samples. Conversely if $L \ll L_{spec}$ we should expect poor denoising, as confirmed in Fig.\ref{fig:pixelization}.  

This raises a critical issue for BIRD models: their patch scale $L$ may need to be close to the spectral scale $L_{spec}$ to achieve good denoising performance, but if $L_{spec} \gg L_c$, then the BIRD model cannot access the spectral scale without exhibiting deleterious memorization. 


How much data is required so that a BIRD model can at least access the spectral scale $L_{spec}$ without memorizing. Our theoretical results in Sec.\ref{sec:crit_scale} imply that for $L_{c} \sim L_{spec}$ at all noise levels $\sigma_t$, we need the proportionality $L_{c} \sim \sigma_t \sqrt{\ln|\mathcal{D}|} \sim \sigma_t^{\frac{2}{2 - \epsilon}} \sim L_{spec}$ to hold, which implies $\ln|\mathcal{D}| \sim \sigma_t^{\frac{2\epsilon}{2 - \epsilon}} \sim L_{spec}^{\epsilon}$. The largest spectral scale that the model may need to access is the image size $L_I$. This implies that the minimum dataset size $|\mathcal D|$ required to avoid memorization must scale with image size $L_I$ as
\begin{align}
    \ln|\mathcal{D}| \sim L_I^{\epsilon}
\end{align}
When $\epsilon$ is taken to 0 (i.e. for perfectly-scale invariant data), we see that there is \textit{no required scaling of dataset size with image size}, showing that spatially local BIRD models easily circumvent the curse of dimensionality for perfectly scale-invariant data.


Next, we numerically compare the critical scale $L_c(\sigma_t)$ and the spectral scale $L_{spec}(\sigma_t)$ for several standard image datasets. We compute $L_c(\sigma_t)$ by solving \eqref{eq:collapse_condition_m} using an estimator for the mutual information. We can also compute the spectral scale $L_{spec}$ from the dataset's power spectral density. In Fig \ref{fig:mem_vs_spec}, we show $L_{c}(\sigma_t)$ and $L_{spec}(\sigma_t)$ and a theoretical prediction for $L_{c}$ from a Gaussian bound. We see clearly in these curves that the scaling behavior of $L_{c}$ and $L_{spec}$ are highly aligned, consistent with our theoretical expectations for scale-invariant image data. See App. \ref{app:empirical_scaling} for more experimental details.

\section{Discussion}\label{sec:discussion}
Using information theory, we were able to obtain the memorization-generalization phase diagram for a highly general class of BIRD models and for a very general class data distributions. While BIRD models introduce inductive biases through information restriction, a significant limitation of our theory is that it cannot capture potentially other inductive biases due to architectural choices or learning dynamics. For example, BIRD models cannot express a bias towards linearity which may arise in trained models. However, we are not aware of any theory of diffusion models that can precisely delineate memorization-generalization phase transitions boundaries with arbitrary architectures and data distributions. All existing theory either use simplified architectures or simple models of data. BIRD models therefore provide an exciting new theoretical laboratory to study memorization and generalization across many possible data distributions. 

\bibliographystyle{unsrtnat}
\bibliography{bib}

\begin{thebibliography}{52}
\providecommand{\natexlab}[1]{#1}
\providecommand{\url}[1]{\texttt{#1}}
\expandafter\ifx\csname urlstyle\endcsname\relax
  \providecommand{\doi}[1]{doi: #1}\else
  \providecommand{\doi}{doi: \begingroup \urlstyle{rm}\Url}\fi

\bibitem[Sohl-Dickstein et~al.(2015)Sohl-Dickstein, Weiss, Maheswaranathan, and Ganguli]{sohl2015deep}
Jascha Sohl-Dickstein, Eric Weiss, Niru Maheswaranathan, and Surya Ganguli.
\newblock Deep unsupervised learning using nonequilibrium thermodynamics.
\newblock In \emph{International conference on machine learning}, pages 2256--2265. PMLR, 2015.

\bibitem[Ho et~al.(2020)Ho, Jain, and Abbeel]{ho2020denoising}
Jonathan Ho, Ajay Jain, and Pieter Abbeel.
\newblock Denoising diffusion probabilistic models.
\newblock \emph{Advances in neural information processing systems}, 33:\penalty0 6840--6851, 2020.

\bibitem[Song et~al.(2020)Song, Meng, and Ermon]{song2020denoising}
Jiaming Song, Chenlin Meng, and Stefano Ermon.
\newblock Denoising diffusion implicit models.
\newblock \emph{arXiv preprint arXiv:2010.02502}, 2020.

\bibitem[Kadkhodaie et~al.(2023)Kadkhodaie, Guth, Simoncelli, and Mallat]{kadkhodaie2023generalization}
Zahra Kadkhodaie, Florentin Guth, Eero~P Simoncelli, and St{\'e}phane Mallat.
\newblock Generalization in diffusion models arises from geometry-adaptive harmonic representation.
\newblock \emph{arXiv preprint arXiv:2310.02557}, 2023.

\bibitem[Shalev-Shwartz and Ben-David(2014)]{shalevshwartz2014understanding}
Shai Shalev-Shwartz and Shai Ben-David.
\newblock \emph{Understanding Machine Learning: From Theory to Algorithms}.
\newblock Cambridge University Press, 2014.

\bibitem[Biroli et~al.(2024)Biroli, Bonnaire, De~Bortoli, and M{\'e}zard]{biroli2024dynamical}
Giulio Biroli, Tony Bonnaire, Valentin De~Bortoli, and Marc M{\'e}zard.
\newblock Dynamical regimes of diffusion models.
\newblock \emph{arXiv preprint arXiv:2402.18491}, 2024.

\bibitem[De~Bortoli(2022)]{de2022convergence}
Valentin De~Bortoli.
\newblock Convergence of denoising diffusion models under the manifold hypothesis.
\newblock \emph{arXiv preprint arXiv:2208.05314}, 2022.

\bibitem[Achilli et~al.(2025)Achilli, Ambrogioni, Lucibello, M{\'e}zard, and Ventura]{achilli2025memorization}
Beatrice Achilli, Luca Ambrogioni, Carlo Lucibello, Marc M{\'e}zard, and Enrico Ventura.
\newblock Memorization and generalization in generative diffusion under the manifold hypothesis.
\newblock \emph{Journal of Statistical Mechanics: Theory and Experiment}, 2025\penalty0 (7):\penalty0 073401, 2025.

\bibitem[Li and Yan(2024)]{li2024adapting}
Gen Li and Yuling Yan.
\newblock Adapting to unknown low-dimensional structures in score-based diffusion models.
\newblock \emph{Advances in Neural Information Processing Systems}, 37:\penalty0 126297--126331, 2024.

\bibitem[Chen et~al.(2023)Chen, Huang, Zhao, and Wang]{chen2023score}
Minshuo Chen, Kaixuan Huang, Tuo Zhao, and Mengdi Wang.
\newblock Score approximation, estimation and distribution recovery of diffusion models on low-dimensional data.
\newblock In \emph{International Conference on Machine Learning}, pages 4672--4712. PMLR, 2023.

\bibitem[Li et~al.(2024)Li, Chen, and Li]{li2024good}
Sixu Li, Shi Chen, and Qin Li.
\newblock A good score does not lead to a good generative model.
\newblock \emph{arXiv preprint arXiv:2401.04856}, 2024.

\bibitem[Scarvelis et~al.(2023)Scarvelis, Borde, and Solomon]{scarvelis2023closed}
Christopher Scarvelis, Haitz S{\'a}ez de~Oc{\'a}riz Borde, and Justin Solomon.
\newblock Closed-form diffusion models.
\newblock \emph{arXiv preprint arXiv:2310.12395}, 2023.

\bibitem[Kamb and Ganguli(2025)]{kambanalytic}
Mason Kamb and Surya Ganguli.
\newblock An analytic theory of creativity in convolutional diffusion models.
\newblock In \emph{Forty-second International Conference on Machine Learning}, 2025.

\bibitem[Niedoba et~al.(2025)Niedoba, Zwartsenberg, Murphy, and Wood]{niedobatowards}
Matthew Niedoba, Berend Zwartsenberg, Kevin~Patrick Murphy, and Frank Wood.
\newblock Towards a mechanistic explanation of diffusion model generalization.
\newblock In \emph{Forty-second International Conference on Machine Learning}, 2025.

\bibitem[Yi et~al.(2023)Yi, Sun, and Li]{yi2023generalization}
Mingyang Yi, Jiacheng Sun, and Zhenguo Li.
\newblock On the generalization of diffusion model.
\newblock \emph{arXiv preprint arXiv:2305.14712}, 2023.

\bibitem[Gu et~al.(2023)Gu, Du, Pang, Li, Lin, and Wang]{gu2023memorization}
Xiangming Gu, Chao Du, Tianyu Pang, Chongxuan Li, Min Lin, and Ye~Wang.
\newblock On memorization in diffusion models.
\newblock \emph{arXiv preprint arXiv:2310.02664}, 2023.

\bibitem[Pham et~al.(2025)Pham, Raya, Negri, Zaki, Ambrogioni, and Krotov]{pham2025memorization}
Bao Pham, Gabriel Raya, Matteo Negri, Mohammed~J Zaki, Luca Ambrogioni, and Dmitry Krotov.
\newblock Memorization to generalization: Emergence of diffusion models from associative memory.
\newblock \emph{arXiv preprint arXiv:2505.21777}, 2025.

\bibitem[Chen(2026)]{chen2025interpolation}
Zhengdao Chen.
\newblock On the interpolation effect of score smoothing in diffusion models.
\newblock In \emph{The Fourteenth International Conference on Learning Representations}, 2026.
\newblock URL \url{https://openreview.net/forum?id=O33LAUliUF}.

\bibitem[Wang and Vastola(2024)]{wang2024unreasonable}
Binxu Wang and John~J Vastola.
\newblock The unreasonable effectiveness of gaussian score approximation for diffusion models and its applications.
\newblock \emph{arXiv preprint arXiv:2412.09726}, 2024.

\bibitem[Wang et~al.(2026)Wang, Zavatone-Veth, and Pehlevan]{wang2026random}
Binxu Wang, Jacob Zavatone-Veth, and Cengiz Pehlevan.
\newblock A random matrix theory perspective on the consistency of diffusion models.
\newblock \emph{arXiv preprint arXiv:2602.02908}, 2026.

\bibitem[He et~al.(2026)He, Qiu, and Tao]{he2026diffusion}
Ye~He, Yitong Qiu, and Molei Tao.
\newblock Diffusion model's generalization can be characterized by inductive biases toward a data-dependent ridge manifold.
\newblock \emph{arXiv preprint arXiv:2602.06021}, 2026.

\bibitem[Wang et~al.(2025)Wang, Zhang, Zhang, Chen, Ma, and Qu]{wang2025diffusion}
Peng Wang, Huijie Zhang, Zekai Zhang, Siyi Chen, Yi~Ma, and Qing Qu.
\newblock Diffusion models learn low-dimensional distributions via subspace clustering.
\newblock In \emph{2025 IEEE 10th International Workshop on Computational Advances in Multi-Sensor Adaptive Processing (CAMSAP)}, pages 211--215. IEEE, 2025.

\bibitem[Favero et~al.(2025{\natexlab{a}})Favero, Sclocchi, and Wyart]{favero2025bigger}
Alessandro Favero, Antonio Sclocchi, and Matthieu Wyart.
\newblock Bigger isn't always memorizing: Early stopping overparameterized diffusion models.
\newblock \emph{arXiv preprint arXiv:2505.16959}, 2025{\natexlab{a}}.

\bibitem[Bonnaire et~al.(2026)Bonnaire, Urfin, Biroli, and Mezard]{bonnairediffusion}
Tony Bonnaire, Rapha{\"e}l Urfin, Giulio Biroli, and Marc Mezard.
\newblock Why diffusion models don{\textquoteright}t memorize: The role of implicit dynamical regularization in training.
\newblock In \emph{The Thirty-ninth Annual Conference on Neural Information Processing Systems}, 2026.
\newblock URL \url{https://openreview.net/forum?id=BSZqpqgqM0}.

\bibitem[Favero et~al.(2025{\natexlab{b}})Favero, Sclocchi, Cagnetta, Frossard, and Wyart]{favero2025compositional}
Alessandro Favero, Antonio Sclocchi, Francesco Cagnetta, Pascal Frossard, and Matthieu Wyart.
\newblock How compositional generalization and creativity improve as diffusion models are trained.
\newblock In \emph{International Conference on Machine Learning}, pages 16286--16306. PMLR, 2025{\natexlab{b}}.

\bibitem[Bardone et~al.(2026)Bardone, Merger, and Goldt]{bardone2026theory}
Lorenzo Bardone, Claudia Merger, and Sebastian Goldt.
\newblock A theory of learning data statistics in diffusion models, from easy to hard.
\newblock \emph{arXiv preprint arXiv:2603.12901}, 2026.

\bibitem[Boffi et~al.(2024)Boffi, Jacot, Tu, and Ziemann]{boffi2024shallow}
Nicholas~M Boffi, Arthur Jacot, Stephen Tu, and Ingvar Ziemann.
\newblock Shallow diffusion networks provably learn hidden low-dimensional structure.
\newblock \emph{arXiv preprint arXiv:2410.11275}, 2024.

\bibitem[Cui et~al.(2025)Cui, Pehlevan, and Lu]{cui2025solvable}
Hugo Cui, Cengiz Pehlevan, and Yue~M Lu.
\newblock A solvable model of learning generative diffusion: theory and insights.
\newblock \emph{arXiv preprint arXiv:2501.03937}, 2025.

\bibitem[Yoon et~al.(2023)Yoon, Choi, Kwon, and Ryu]{yoon2023diffusion}
TaeHo Yoon, Joo~Young Choi, Sehyun Kwon, and Ernest~K Ryu.
\newblock Diffusion probabilistic models generalize when they fail to memorize.
\newblock In \emph{ICML 2023 workshop on structured probabilistic inference $\{$$\backslash$\&$\}$ generative modeling}, 2023.

\bibitem[Cui et~al.(2024)Cui, Krzakala, Vanden-Eijnden, and Zdeborov{\'a}]{cui2024analysis}
Hugo Cui, Florent Krzakala, Eric Vanden-Eijnden, and Lenka Zdeborov{\'a}.
\newblock Analysis of learning a flow-based generative model from limited sample complexity.
\newblock In \emph{International Conference on Learning Representations}, volume 2024, pages 51929--51955, 2024.

\bibitem[Buchanan et~al.(2026)Buchanan, Pai, Ma, and De~Bortoli]{buchanan2026edge}
Sam Buchanan, Druv Pai, Yi~Ma, and Valentin De~Bortoli.
\newblock On the edge of memorization in diffusion models.
\newblock \emph{Advances in Neural Information Processing Systems}, 38:\penalty0 96113--96157, 2026.

\bibitem[Lukoianov et~al.(2025)Lukoianov, Yuan, Solomon, and Sitzmann]{lukoianov2025locality}
Artem Lukoianov, Chenyang Yuan, Justin Solomon, and Vincent Sitzmann.
\newblock Locality in image diffusion models emerges from data statistics.
\newblock \emph{arXiv preprint arXiv:2509.09672}, 2025.

\bibitem[Bradley(2025)]{bradley2025local}
Arwen Bradley.
\newblock Local mechanisms of compositional generalization in conditional diffusion.
\newblock \emph{arXiv preprint arXiv:2509.16447}, 2025.

\bibitem[Finn et~al.(2025)Finn, Keller, Theodosis, and Ba]{finn2025origins}
Emma Finn, T~Anderson Keller, Manos Theodosis, and Demba~E Ba.
\newblock Origins of creativity in attention-based diffusion models.
\newblock \emph{arXiv preprint arXiv:2506.17324}, 2025.

\bibitem[Hu et~al.(2025)Hu, Liu, Zhang, and Gao]{hu2025local}
Fangjun Hu, Guangkuo Liu, Yifan Zhang, and Xun Gao.
\newblock Local diffusion models and phases of data distributions.
\newblock \emph{arXiv preprint arXiv:2508.06614}, 2025.

\bibitem[Ambrogioni(2026)]{ambrogioni2026out}
Luca Ambrogioni.
\newblock How out-of-equilibrium phase transitions can seed pattern formation in trained diffusion models.
\newblock \emph{arXiv preprint arXiv:2603.20092}, 2026.

\bibitem[Gottwald et~al.(2025)Gottwald, Liu, Marzouk, Reich, and Tong]{gottwald2025localized}
Georg~A Gottwald, Shuigen Liu, Youssef Marzouk, Sebastian Reich, and Xin~T Tong.
\newblock Localized diffusion models.
\newblock \emph{arXiv preprint arXiv:2505.04417}, 2025.

\bibitem[Nguyen et~al.(2026)Nguyen, Do, Pauwels, and Weiss]{nguyen2026analytic}
Minh~Hai Nguyen, Quoc~Bao Do, Edouard Pauwels, and Pierre Weiss.
\newblock An analytic theory of convolutional neural network inverse problems solvers.
\newblock \emph{arXiv preprint arXiv:2601.10334}, 2026.

\bibitem[Ruderman and Bialek(1993)]{ruderman1993statistics}
Daniel Ruderman and William Bialek.
\newblock Statistics of natural images: Scaling in the woods.
\newblock \emph{Advances in neural information processing systems}, 6, 1993.

\bibitem[Dieleman(2024)]{dieleman2024spectral_autoregression}
Sander Dieleman.
\newblock Diffusion is spectral autoregression, 2024.
\newblock URL \url{https://sander.ai/2024/09/02/spectral-autoregression.html}.

\bibitem[Wiener(1949)]{wiener1949extrapolation}
Norbert Wiener.
\newblock \emph{Extrapolation, interpolation, and smoothing of stationary time series: with engineering applications}.
\newblock The MIT press, 1949.

\bibitem[Karras et~al.(2022)Karras, Aittala, Aila, and Laine]{karras2022elucidating}
Tero Karras, Miika Aittala, Timo Aila, and Samuli Laine.
\newblock Elucidating the design space of diffusion-based generative models.
\newblock \emph{Advances in neural information processing systems}, 35:\penalty0 26565--26577, 2022.

\bibitem[Achilli et~al.(2026)Achilli, Benedetti, Biroli, and Mézard]{biroli2026speciation}
Beatrice Achilli, Marco Benedetti, Giulio Biroli, and Marc Mézard.
\newblock Theory of speciation transitions in diffusion models with general class structure, 2026.
\newblock URL \url{https://arxiv.org/abs/2602.04404}.

\bibitem[Derrida(1980)]{derrida1980randomenergy}
B.~Derrida.
\newblock Random-energy model: Limit of a family of disordered models.
\newblock \emph{Phys. Rev. Lett.}, 45:\penalty0 79--82, Jul 1980.
\newblock \doi{10.1103/PhysRevLett.45.79}.
\newblock URL \url{https://link.aps.org/doi/10.1103/PhysRevLett.45.79}.

\bibitem[Mézard and Montanari(2009)]{mezard2009info}
Marc Mézard and Andrea Montanari.
\newblock \emph{Information, Physics, and Computation}.
\newblock Oxford University PressOxford, January 2009.
\newblock ISBN 9780191718755.
\newblock \doi{10.1093/acprof:oso/9780198570837.001.0001}.
\newblock URL \url{http://dx.doi.org/10.1093/acprof:oso/9780198570837.001.0001}.

\bibitem[Cover and Thomas(2006)]{cover2006elementsofinfo}
Thomas~M. Cover and Joy~A. Thomas.
\newblock \emph{Elements of Information Theory (Wiley Series in Telecommunications and Signal Processing)}.
\newblock Wiley-Interscience, USA, 2006.
\newblock ISBN 0471241954.

\bibitem[Ventura et~al.(2024)Ventura, Achilli, Silvestri, Lucibello, and Ambrogioni]{ventura2024manifolds}
Enrico Ventura, Beatrice Achilli, Gianluigi Silvestri, Carlo Lucibello, and Luca Ambrogioni.
\newblock Manifolds, random matrices and spectral gaps: The geometric phases of generative diffusion.
\newblock \emph{arXiv preprint arXiv:2410.05898}, 2024.

\bibitem[Niedoba et~al.(2024)Niedoba, Zwartsenberg, Murphy, and Wood]{niedoba2024towards}
Matthew Niedoba, Berend Zwartsenberg, Kevin Murphy, and Frank Wood.
\newblock Towards a mechanistic explanation of diffusion model generalization.
\newblock \emph{arXiv preprint arXiv:2411.19339}, 2024.

\bibitem[Yu et~al.(2024)Yu, Shen, Huang, Li, and Zhao]{yu2024unmasking}
Hu~Yu, Li~Shen, Jie Huang, Hongsheng Li, and Feng Zhao.
\newblock Unmasking bias in diffusion model training.
\newblock In \emph{European Conference on Computer Vision}, pages 374--390. Springer, 2024.

\bibitem[Deck and Bischoff(2023)]{deck2023easing}
Katherine Deck and Tobias Bischoff.
\newblock Easing color shifts in score-based diffusion models.
\newblock \emph{arXiv preprint arXiv:2306.15832}, 2023.

\bibitem[Choi et~al.(2022)Choi, Lee, Shin, Kim, Kim, and Yoon]{choi2022perception}
Jooyoung Choi, Jungbeom Lee, Chaehun Shin, Sungwon Kim, Hyunwoo Kim, and Sungroh Yoon.
\newblock Perception prioritized training of diffusion models.
\newblock In \emph{Proceedings of the IEEE/CVF conference on computer vision and pattern recognition}, pages 11472--11481, 2022.

\bibitem[Salimans and Ho(2022)]{salimans2022progressive}
Tim Salimans and Jonathan Ho.
\newblock Progressive distillation for fast sampling of diffusion models.
\newblock \emph{arXiv preprint arXiv:2202.00512}, 2022.

\end{thebibliography}

\clearpage
\begin{appendices}
\appendixpage
\addappheadtotoc

\section{Notation and background on diffusion models.} \label{app:notation_background}

\subsection{Notation conventions}

We use the following notation:
\begin{itemize}
    \item $\varphi \in \mathbb{R}^d$ denotes an example from the training set. For images or patches of size $L$ pixels by $L$ pixels by $C$ color channels, we have $d=L\times L \times C$.  
    \item $\phi$ represents any arbitrary image (or other data) that can serve as input to either the score function or diffusion model. 
    \item $P_0(\varphi)$ denotes the true data (i.e. target) distribution, from which we sample both our training and test data.
    \item $\mathcal{D}$ represents a finite set of training data consisting of $|\mathcal{D}|$ training samples. Each training data point $\varphi \in \mathcal D$ is drawn i.i.d. from the true data distribution $\varphi \sim P_0$. 
    \item $P_{train}(\varphi) = \frac{1}{|\mathcal{D}|} \sum_{\varphi' \in \mathcal{D}} \delta(\varphi - \varphi')$ denotes the empirical training distribution for any fixed realization $\mathcal{D}$ of the training set.  It is simply a uniform distribution over $|\mathcal{D}|$ specific training examples $\varphi \in \mathcal D$.
    \item $x$ represents a pixel location in an image.
    \item For image data, $\phi_x$ and $\varphi_x$ will represent the pixel values of the images $\phi$ and $\varphi$ at pixel location $x$; both are elements of $\mathbb{R}^C$. 
    \item $M[\phi]: \mathbb{R}^d \to \mathbb{R}^d$ represents a model that takes as input an image $\phi$ and produces a new image (e.g. an estimate of the score function, or a denoised version of $\phi$). We will denote by $M_x[\phi] \in \mathbb{R}^C$ the value of the outputs of this model, given an input $\phi$, at the pixel location $x$.
    \item $\Omega_x$ denotes a patch neighborhood indexed by the pixel $x$. The pixel location $x$ will be contained within $\Omega_x$. We will denote by $|\Omega_x|$ the number of pixels in the patch. We will also denote by $\Omega_{x,L}$ a patch neighborhood of $x$ with a characteristic length $L$, such that the size of the patch is $|\Omega_{x;L}| = L^2$.
    \item $\phi_{\Omega_x}$ and $\varphi_{\Omega_x}$ denotes the restriction of images $\phi$ and $\varphi$ to a neighborhood $\Omega_x$ around a pixel $x$. For images with $C$ channels, we have $\phi_{\Omega_x} \in \mathbb{R}^{|\Omega_x| \times C}$.
    \item $\mathcal{C}_{x,t} : \mathbb{R}^d \to \mathbb{R}^n$ denotes a (possibly stochastic) observation map that takes as input a noised image $\phi_t \in \mathbb R^d$, and outputs the $n$ dimensional observation $\mathcal{C}_{x,t}(\phi_t)$ that pixel $x$ makes about the image $\phi_t$. If the map is stochastic we denote its conditional distribution by $P(\mathcal C_{x,t} | \phi_t)$. If the map is deterministic we denote it by the function $\mathcal C_{x,t}(\phi_t)$.  For brevity we also sometimes drop the argument $\phi_t$, and then $\mathcal{C}_{x,t} \in \mathbb R^n$ denotes the pixel $x$'s observation value, reflecting the state of knowledge of pixel $x$ at time $t$ in the diffusion process. 
    \item We denote the {\it training} Markov process $\varphi \rightarrow \phi_t \rightarrow \mathcal C_{x,t}$ for pixel $x$ at time $t$ by the joint distribution $P_{train}(\varphi, \phi_t, \mathcal C_{x,t}) = P_{train}(\varphi)P(\phi_t | \varphi)P(\mathcal C_{x,t} | \phi_t)$.  Here $\varphi \sim P_{train}(\varphi)$ is a randomly chosen training image $\varphi \in \mathcal D$, $P(\phi_t | \varphi)$ is the conditional distribution of a noisy image $\phi_t$ under the forward diffusion process, starting from the training image $\varphi$ , and $P(\mathcal C_{x,t} | \phi_t)$ is the (possibly stochastic) observation that pixel $x$ makes about image $\phi_t$.  We can think about the Markov process $\varphi \rightarrow \phi_t \rightarrow \mathcal C_{x,t}$ as an information channel from the past training data $\mathcal D$ at time $t=0$ to the current pixel observation $\mathcal C_{x,t}$ at time $t$. We denote by $P_{train}(\varphi | \mathcal C_{x,t})$ the Bayesian posterior distribution over the training data given a pixel $x$'s channel observation $\mathcal C_{x,t}$.  This posterior assigns a probability to each of the $|\mathcal D|$ training images $\varphi \in \mathcal D$.  If a pixel $x$ had to guess, based on its (restricted) observation $\mathcal C_{x,t}$ of the  noisy image $\phi_t$, which clean training image $\varphi \in \mathcal D$ at time $0$ in the past led to $\phi_t$ under the forward process, then any optimal guess would be based on the posterior $P_{train}(\varphi | \mathcal C_{x,t})$, which encapsulates the pixel's belief state about the past training data origin of its current observation. 
    \item We denote the {\it testing} Markov process $\varphi \rightarrow \phi_t \rightarrow \mathcal C_{x,t}$ for pixel $x$ at time $t$ by the joint distribution $P_{test}(\varphi, \phi_t, \mathcal C_{x,t}) = P_{0}(\varphi)P(\phi_t | \varphi)P(\mathcal C_{x,t} | \phi_t).$  This is identical to the training Markov process except the prior is now the {\it true} data distribution $P_0(\varphi)$ from which a test sample is drawn, rather than the empirical training distribution $P_{train}(\varphi)$ for a specific dataset $\mathcal D$. We denote the posterior distribution of the input given the output under this Markov process by $P_{test}(\varphi | \mathcal C_{x,t})$.  Unlike $P_{train}(\varphi | \mathcal C_{x,t})$, the support of $P_{test}(\varphi | \mathcal C_{x,t})$ need not be restricted  to the training data $\varphi \in \mathcal D$. 
    \item $\bar{\alpha}_t$ denotes the signal variance in the forward diffusion process.
    \item $\sigma_t^2 = \frac{1 - \bar{\alpha}_t}{\bar{\alpha}_t}$ denotes the noise-to-signal ratio at a time $t$.
    \item $I(A; B)$ will indicate the mutual information between two random variables $A$ and $B$. $I(A ; B \sim P)$ will indicate the mutual information between these variables with a particular emphasis that the random variable $B$ is distributed according to $P$.
    \item $S[P]$ will denote the entropy of a distribution $P$. $S_P[A]$ will indicate the entropy of the random variable $A$ under the distribution $P$.
    \item $\mathcal{N}(x|\mu, \Sigma)$ represents the PDF of the normal distribution with mean $\mu$ and covariance $\Sigma$.  We also use the short-hand $\mathcal{N}(\mu, \Sigma)$ when we do not need to refer to the name of a specific random variable.
\end{itemize}

\subsection{Background on diffusion models}
Diffusion models are designed to solve the problem of generating samples from a complex data distribution $\varphi \sim P_0(\varphi)$ in high dimensions (e.g. for $\varphi \in \mathbb R^d$ for large $d$). To do so, they first develop, through a forward diffusion process, an intermediate bridge between the complex data distribution $P_{0}(\varphi)$ and simple isotropic Gaussian noise distribution $\eta \sim \mathcal{N}(0,I)$.  This bridge corresponds to a time dependent family of distributions of  noised data $\phi_t \sim P_t(\phi_t)$ that interpolate between data points $\varphi$ at time $t=0$ and pure isotropic noise $\eta$ at time $t=1$. A single sample from this bridge at any fixed time $t$ can be obtained via the interpolation 
\begin{align}\label{eq:noise_interpolation}
    \phi_t(\varphi,\eta) = \sqrt{\bar{\alpha}_t} \varphi + \sqrt{1 - \bar{\alpha}_t} \eta.
\end{align}
The parameter $\bar{\alpha}_t$ can be thought of as the signal power, while $1-\bar{\alpha}_t$ can be thought of as the noise power $\bar{\alpha}_t$.  The temporal schedule of the signal power is chosen to monotonically decrease from $\bar{\alpha}_0 = 1$ at time $t=0$ to  $\bar{\alpha}_1 = 0$ at time $t=1$. Thus over the time interval $t \in [0,1]$, the forward diffusion process gradually turns data $\varphi$ into noise $\eta$ through the family of noised samples $\phi_t$. We define a useful quantity $\sigma^2_t$, which characterizes the noise to signal ratio at a time $t$:
\begin{align} \label{eq:noise_to_signal}
    \sigma_t^2 = \frac{1 - \bar{\alpha}_t}{\bar{\alpha}_t}.
\end{align}
(This parameter is the direct analog of the $\sigma_t^2$ parameter in variance-exploding forward processes \citep{karras2022elucidating}). The interpolation in \eqref{eq:noise_interpolation} induces a flow on distributions $P_t(\phi_t)$ defined by the forward diffusion partial differential equation
\begin{align} \label{eq:forward_diffusion}
    \frac{d P_t(\phi_t)}{dt} = \gamma_t\nabla \cdot (\phi_t P_t(\phi_t) ) + \gamma_t \nabla^2 P_t(\phi_t),
\end{align}
with $\gamma_t = -\frac{1}{2}\partial_t \ln \bar{\alpha}_t$. Diffusion models in deterministic mode (DDIM parameterization, \cite{song2020denoising}) aim to reverse this flow by evolving pure noise samples $\phi_1 \sim \mathcal{N}(0,I)$ backwards in time under the ODE
\begin{align} \label{eq:reverse_process}
    \frac{d\phi_t}{dt} = -\gamma_t(\phi_t + \nabla \log P_t(\phi_t)),
\end{align}
where $\nabla \log P_t$ is the \textit{score function} for the distribution $P_t(\phi_t)$ over noised samples induced by the  interpolation in \eqref{eq:noise_interpolation} or equivalently the forward diffusion equation in \eqref{eq:forward_diffusion}. This ODE provably transports the distribution $\mathcal{N}(0,I)$ back to the data distribution $P_0(\varphi)$ through the intermediate family of distributions $P_t(\phi_t)$ as time flows backwards from $t=1$ to $t=0$. Thus the reverse process in \eqref{eq:reverse_process} exactly reverses the forward diffusion in \eqref{eq:forward_diffusion} in terms of the time-dependent family of marginal distributions $P_t(\phi_t)$.

However, implementing the reverse process in \eqref{eq:reverse_process} requires access to the score function $\nabla \log P_t$. When direct calculation of the score is not available or not desirable, there is a convenient identity that reframes the score computation as a problem of optimal denoising, known as Tweedie's theorem:
\begin{align} \label{eq:Tweedie}
    \nabla \log P_t(\phi_t) = -\frac{\mathbb{E}[\eta | \phi_t]}{\sqrt{1 - \bar{\alpha}_t}} = \frac{\sqrt{\bar{\alpha}_t}}{1 - \bar{\alpha}_t} \mathbb{E}[\varphi | \phi_t] - \frac{\phi_t}{1 - \bar{\alpha}_t}
\end{align}
Thus the calculation of the posterior mean $\mathbb{E}[\varphi | \phi_t]$ over the data $\varphi$ at time $t=0$, conditioned on the noised sample $\phi_t$ at any point $t>0$ in the reverse flow, becomes a central quantity required to successfully implement the reverse flow to transport $P_t(\phi_t)$ back to $P_0(\phi_0 = \varphi)$. 
In turn, one can approximate this posterior mean by minimizing the following squared loss over the parameters $\theta$ of a neural network $M_{t,\theta}[\phi_t]$:
\begin{align}\label{eq:neural_loss}
    \mathcal{L}_t = \frac{1}{2} \mathbb{E}_{\varphi \sim P_0, \eta \sim \mathcal{N}(0,I)}[ \norm{\varphi - M_{t,\theta}[\phi_t(\varphi,\eta)]}^2].
\end{align}
This loss is a denoising loss that trains the neural network function $M_{t,\theta}$ to take any noisy sample $\phi_t$ at time $t$ as input, and convert it back to the clean sample $\varphi$ that led to $\phi_t$ under the forward diffusion process.  It is well known that the minimum mean squared error (MMSE) solution to any optimization of the form in \eqref{eq:neural_loss} will yield the posterior mean $\mathbb{E}[\varphi | \phi_t]$ as the Bayes optimal estimator or denoiser.  More precisely, if the family of neural network functions $M_{t,\theta}$ is expressive enough to compute this posterior mean, then for the optimal $\theta^*$ that achieves zero loss in \eqref{eq:neural_loss} we have the MMSE result
\begin{equation} \label{eq:MMSE_denoise}
M_{t,\theta^*}[\phi_t] = \mathbb{E}[\varphi | \phi_t].
\end{equation}
Learning the reverse process from a finite training set and then implementing it through a parametric neural network $M_{t,\theta}$ then corresponds at a high level to four main steps:
\begin{enumerate}
\item Replace the average over the true data distribution $P_0(\varphi)$ with an empirical average over a finite training set.
\item Minimize the loss in \eqref{eq:neural_loss} to obtain the learned neural network $M[t,\theta](\phi_t)$.
\item Replace the posterior mean $\mathbb{E}[\varphi | \phi_t]$ in Tweedie's formula in \eqref{eq:Tweedie} with its approximation $M_{t,\theta}[\phi_t]$, thereby obtaining an approximate score function.
\item Insert this approximate score into the deterministic ODE representing the reverse process in \eqref{eq:reverse_process}.
\end{enumerate}

In practice, for purposes of variance reduction, rather than predicting the clean data $\varphi$ from the noisy sample $\phi_t$, one can also predict the noise $\eta$ or the velocity $v_t$ through the following two losses respectively: 
\begin{align}
    \mathcal{L}^\eta_t &= \frac{1}{2} \mathbb{E}_{\varphi \sim P_0, \eta \sim \mathcal{N}(0,I)}[ \norm{\eta - M_{t,\theta}[\phi_t(\varphi,\eta)]}^2]\\
    \mathcal{L}^v_t &= \frac{1}{2} \mathbb{E}_{\varphi \sim P_0, \eta \sim \mathcal{N}(0,I)}[ \norm{v_t(\varphi,\eta) - M_{t,\theta}[\phi_t(\varphi,\eta)]}^2].
\end{align}
Here the velocity $v_t(\varphi,\eta)$ roughly points from the clean data $\varphi$ to the noisy sample $\phi_t$ with weighting coefficients:
\begin{align*}
    v_t(\varphi,\eta) = \sqrt{\frac{\bar{\alpha}_t}{1 - \bar{\alpha}_t}} \,\phi_t - \frac{\varphi}{\sqrt{1 - \bar{\alpha}_t}}.
\end{align*}
In analogy to \eqref{eq:MMSE_denoise}, the MMSE neural network (assuming no expressivity constraints) computes the following posterior means for the losses $\mathcal{L}^\eta_t$ and $\mathcal{L}^v_t$ respectively:
\begin{align*}
    \mathbb{E}[\eta | \phi_t] &= \frac{\phi_t}{\sqrt{1 - \bar{\bar{\alpha}_t}}} - \sqrt{\frac{\bar{\alpha}_t}{1 - \bar{\alpha}_t}}\,\, \mathbb{E}[\varphi | \phi_t]\\
    \mathbb{E}[v_t(\phi_t,\eta) | \phi_t] &= \sqrt{\frac{\bar{\alpha}_t}{1 - \bar{\alpha}_t}} \,\phi_t - \frac{1}{\sqrt{1 - \bar{\alpha}_t}} \mathbb{E}[\varphi | \phi_t]. 
\end{align*}
The MMSE solutions to all three losses above produce a linear combination of $\phi_t$ and the optimal MMSE denoiser $\mathbb{E}[\varphi | \phi_t]$.  Thus we can easily interchange theoretical results derived for one approach to any other approach.  For the remainder of the theory section, we will focus primarily on the loss (\ref{eq:neural_loss}) for which the optimal model simply computes $\mathbb{E}[\varphi | \phi_t]$.  However, for numerical experiments in our paper, we use velocity prediction, which shows significant benefits from the perspective of training stability and the stability of generated outputs, especially early in training.

\section{Bayes-optimal diffusion models always memorize finite training sets} \label{app:bayes_opt_mem}


In the denoising loss in \eqref{eq:neural_loss} we never have direct access to the true data distribution $P_0(\varphi)$. Instead, we only have a {\it finite} dataset $\mathcal D$ consisting of $|\mathcal D|$ training samples $\varphi \in \mathcal D$ each drawn i.i.d from the true data distribution $P_0(\varphi)$.  This leads to an empirical training distribution 
\begin{equation} \label{eq:P_train}
P_{train}(\varphi) = \frac{1}{|\mathcal{D}|} \sum_{\varphi' \in \mathcal{D}} \delta(\varphi - \varphi'), 
\end{equation}
which is simply a uniform sum of delta-functions over $|\mathcal{D}|$ specific training examples $\varphi \in \mathcal D$. It follows that the Bayes optimal MMSE denoiser that achieves minimal loss in \eqref{eq:neural_loss}, when $P_0(\varphi)$ replaced with $P_{train}(\varphi)$, computes the posterior mean over the {\it finite} training set:
\begin{align}\label{eq:full_bayes_optimal}
    M_t[\phi_t] = \mathbb{E}[\varphi | \phi_t ] = \sum_{\varphi \in \mathcal{D}} \varphi \,P_{train}(\varphi | \phi_t).
\end{align}
Here $P_{train}(\varphi | \phi_t)$ is the posterior probability that a clean training data image $\varphi \in \mathcal D$ at time $t=0$ leads to the noisy observed image $\phi_t$ at a given time $t$ under the forward diffusion process.  Because the forward diffusion process shrinks the training data and adds isotropic Gaussian noise to it (see \eqref{eq:noise_interpolation}), the forward conditional probability of obtaining $\phi_t$ starting from $\varphi \in \mathcal D$ is simply an isotropic Gaussian distribution whose mean shrinks with time and whose variance grows with time: 
\begin{align}
    P(\phi_t | \varphi) = \mathcal{N}(\phi_t | \sqrt{\bar{\alpha}_t} \varphi, (1 - \bar{\alpha}_t) I).
\end{align}
And since the prior over the training data in \eqref{eq:P_train} is uniform, Bayes rule simply yields 
\begin{align}\label{eq:posterior_train}
    P_{train}(\varphi | \phi_t) = \frac{P(\phi_t | \varphi)}{\sum_{\varphi'} P(\phi_t | \varphi')} =  \frac{\mathcal{N}(\phi_t | \sqrt{\bar{\alpha}_t} \varphi, (1 - \bar{\alpha}_t) I)}{\sum_{\varphi' \in \mathcal{D}} \mathcal{N}(\phi_t | \sqrt{\bar{\alpha}_t} \varphi', (1 - \bar{\alpha}_t) I)}.
\end{align}

In summary, the Bayesian posterior distribution $P_{train}(\varphi | \phi_t)$ in \eqref{eq:posterior_train} over the past training data $\varphi \in \mathcal D$ given the current image $\phi_t$ provides a full analytic solution to both the Bayes optimal denoiser in \eqref{eq:full_bayes_optimal} and the reverse process in \eqref{eq:reverse_process} through Tweedie's theorem in \eqref{eq:Tweedie}. 

As explained in \citep{kambanalytic}, the structure of this analytic solution has an appealing interpretation in terms of a Bayesian guessing game about the past of the forward diffusion process given a noised image $\phi_t$. 
This Bayesian guessing game is played by both the denoiser and the reverse process. 
In essence, both the denoiser and the reverse process are trying to guess which training image $\varphi \in \mathcal D$ at time $t=0$ led to $\phi_t$ under the forward diffusion. 
The optimal Bayesian observer would summarize this guess in the Bayesian posterior distribution $P_{train}(\varphi | \phi_t)$ in \eqref{eq:posterior_train}.  
The optimal MMSE Bayesian denoiser then uses this posterior to compute the posterior mean training set image through \eqref{eq:full_bayes_optimal}. Similarly, the task of the reverse process is to flow $\phi_t$ backwards in time from $t$ to $0$ and reverse the forward diffusion. 
The Bayes-optimal reverse process (obtained by inserting \eqref{eq:posterior_train} into \eqref{eq:full_bayes_optimal}, 
then inserting \eqref{eq:full_bayes_optimal} into \eqref{eq:Tweedie}, 
and then finally inserting \eqref{eq:Tweedie} into \eqref{eq:reverse_process}) 
uses the posterior to flow $\phi_t$ back to each data point $\varphi \in \mathcal D$ with a Bayesian posterior weight controlled by $P_{train}(\varphi | \phi_t)$.  
In essence, the optimal Bayesian reverse process guesses the past training data origin of $\phi_t$ under the forward diffusion and flows back towards this origin.  

While the Bayes optimal diffusion model is appealing in terms of its both its analytic and conceptual simplicity, its sampling behavior is unfortunately quite unrealistic and does not explain what trained neural network based diffusion models actually do.  The key issue is that for any finite training set, the Bayes optimal reverse process memorizes the training data, and can only flow to one of the training set points $\varphi \in \mathcal D$ (see also \citep{kambanalytic}).  Intuitively, as time decreases from $t=1$ to $t=0$ in the reverse process, the posterior distribution in \eqref{eq:posterior_train} increases its probability on training points $\varphi$ close to $\phi_t$, and in turn,  $\phi_t$ then increasingly flows towards such training points. The positive feedback between posterior belief $P_{train}(\varphi | \phi_t)$ and reverse flow of $\phi_t$ to images $\varphi$ of high belief, eventually causes $\phi_t$ to converge to a {\it single} training set point $\varphi \in \mathcal D$. 

In terms of the Bayesian guessing game interpretation, with small amounts of data and low enough noise levels, the guessing game is {\it too easy}. With only a small number of training data points $\varphi \in \mathcal D$, which are likely to be well separated (relative to the noise scale $\sigma_t$) in an exponentially large high dimensional space, $\phi_t$ will generically be much closer to the data point $\varphi^*$ that generated it than any other data point $\varphi \neq \varphi^*$.  In this situation, the posterior belief $P_{train}(\varphi | \phi_t$) will have a very high probability when $\varphi=\varphi*$ and a low probability otherwise. Thus the posterior belief $P_{train}(\varphi | \phi_t$) {\it correctly} guesses the origin of $\phi_t$ in the forward process.  

Even more intuitively, in terms of images, if one observes {\it the entirety} of a noised image $\phi_t$ at moderately low levels of noise, and one knows that the noised image came from a small set of widely separated clean training images, it is easy to guess which training image it came from.  Thus the ease of the Bayesian guessing game in high dimensions is intimately tied to the deleterious memorization behavior of the Bayes optimal diffusion model. This memorization behavior hurts denoising on a test image; this denoiser will be overly influenced by a particular training set image.  Also, from the diffusion model perspective,  this memorization behavior also impedes creative generation of new images different from any image in the training set.

\section{Bayesian Information Restricted Diffusion (BIRD) Models}
\label{app:BIRD}

We have seen that optimal Bayesian diffusion models memorize finite training sets precisely because the Bayesian guessing game of which training set image $\varphi \in \mathcal D$ leads to a given noise image $\phi_t$ in the forward diffusion is too easy, especially for small numbers of training images in a high dimensional space at low levels of noise (or equivalently last in the reverse process). This raises the question of whether one can obtain alternate Bayesian diffusion models that retain analytic tractability and conceptual simplicity, but {\it also} generalize appropriately, thereby behaving more like trained neural network diffusion models.  

We show that we can {\it simultaneously} achieve {\it both} theoretical tractability {\it and} more realistic generalization. The key idea is to prevent memorization by making the Bayesian guessing game harder {\it without} increasing the amount of training data $|\mathcal D|$ or increasing the noise level. In particular we combine two key approaches: 1) we instead make the Bayesian guessing game harder by restricting the image information the Bayesian guesser uses to guess the past of the forward diffusion process; and 2) we embrace  diversity by allowing different parts of the image, down to the level of {\it individual pixels}, to use {\it different} pieces of restricted information to guess the past.  The former prevents memorization while the latter enhances creative generalization.  We call such models Bayesian Information Restricted Diffusion (BIRD) models. 

To introduce the general class of BIRD models, we first note that the denoising loss \eqref{eq:neural_loss}, as a simple squared loss, decouples across pixels $x$.  In particular, we can write it as a sum over pixels $x$ (or more generally, as a sum over any orthonormal basis over pixels):
\begin{align} \label{eq:BIRD_loss}
    \mathcal{L}_t = \sum_x \mathcal{L}_{x,t} \quad \text{where} \quad \mathcal{L}_{x,t} = \frac{1}{2} \mathbb{E}_{\varphi \sim P_{train}, \phi_t} \norm{\varphi_x - M_{x,t}[\phi_t]}^2.
\end{align}
Here $M_{x,t}[\phi_t] \in \mathbb R^C$ is the evaluation of the model denoiser at time $t$ at {\it single} pixel location $x$, and it is trained to recover the single clean training image pixel value $\varphi_x \in \mathbb R^C$.  We can thus posit a separate denoiser model for each pixel and analyze them individually. Without this decoupling, analyzing the behavior of diffusion models with inductive biases becomes significantly less tractable; see e.g. \cite{finn2025origins} for an example of the complexities that arise when couplings between pixels are introduced.  

To achieve the restricted information, we ensure that the denoiser function $M_{x,t}$ is not allowed to observe the entire noised image $\phi_t$. Instead it is only allowed to make a restricted, possibly stochastic, observation $\mathcal C_{x,t} \in \mathbb R^{n_{x,t}}$ of $\phi_t \in \mathbb R^d$.  We usually assume the observation dimensionality $n_{x,t}$ is less than the full image dimensionality $d$.  If the observation is stochastic we denote its conditional distribution by $P(\mathcal C_{x,t} | \phi_t)$. If it is deterministic we denote it by the function $\mathcal C_{x,t}(\phi_t)$.  For brevity we also sometimes drop the argument $\phi_t$, and then $\mathcal{C}_{x,t} \in \mathbb R^{n_{x,t}}$ denotes the pixel $x$'s observation value.  Furthermore we assume the that different pixels $x$ can make {\it different} observations $\mathcal C_{x,t}$ about the noised image $\phi_t$.  This diversifies the state of knowledge of each pixel. The BIRD model is then simply the collection of individual pixel denoisers $M_{x,t}[\mathcal{C}_{x,t}]$ that each achieve the MMSE loss for $\mathcal L_{x,t}$ in \eqref{eq:BIRD_loss}.

Just like the full optimal Bayesian model, the MMSE BIRD model also has an analytic solution for pixel x given by
\begin{align}
    M_{x,t}[\mathcal{C}_{x,t}] = \sum_{\varphi \in \mathcal{D}} \varphi_x \, P_{train}(\varphi | \mathcal{C}_{x,t}), \label{eq:restricted_bayes_optimal}
\end{align}
where the Bayesian posterior is given by
\begin{align}\label{eq:bayes_weights_channel}
    P_{train}(\varphi | \mathcal{C}_{x,t}) = \frac{P(\mathcal{C}_{x,t}| \varphi)}{\sum_{\varphi'} P(\mathcal{C}_{x,t} | \varphi')}.
\end{align}
We can then combine these pixel-wise estimates into the BIRD model for the full image
\begin{equation}
    M_{t}[\varphi] = \sum_{x} M_{x,t}[\mathcal{C}_{x,t}]\hat{x}
\end{equation}
To understand this solution, consider the training Markov process $\varphi \rightarrow \phi_t \rightarrow \mathcal C_{x,t}$ for pixel $x$ at time $t$, denoted by the joint distribution $P_{train}(\varphi, \phi_t, \mathcal C_{x,t}) = P_{train}(\varphi)P(\phi_t | \varphi)P(\mathcal C_{x,t} | \phi_t)$.  Here $\varphi \sim P_{train}(\varphi)$ is a randomly chosen training image $\varphi \in \mathcal D$, $P(\phi_t | \varphi)$ is the conditional distribution of a noisy image $\phi_t$ under the forward diffusion process, starting from the training image $\varphi$, and $P(\mathcal C_{x,t} | \phi_t)$ is the (possibly stochastic) observation that pixel $x$ makes about image $\phi_t$.  We can think about the training Markov process $\varphi \rightarrow \phi_t \rightarrow \mathcal C_{x,t}$ as an information channel from the past training data $\mathcal D$ at time $t=0$ to the current pixel observation $\mathcal C_{x,t}$ at time $t$. 

$P_{train}(\varphi | \mathcal C_{x,t})$ in \eqref{eq:bayes_weights_channel} is then the Bayesian posterior distribution over the training data $\varphi \in \mathcal D$ given a pixel $x$'s channel observation $\mathcal C_{x,t}$. This posterior assigns a probability to each of the $|\mathcal D|$ training images $\varphi \in \mathcal D$.  If a pixel $x$ had to guess, based on its (restricted) observation $\mathcal C_{x,t}$ of the  noisy image $\phi_t$, which clean training image $\varphi \in \mathcal D$ at time $0$ in the past led to $\phi_t$ under the forward diffusion, then any optimal guess would be based on the posterior $P_{train}(\varphi | \mathcal C_{x,t})$, which encapsulates the pixel's belief state about the past training data origin of its current observation. The individual Bayes optimal pixel denoiser at pixel $x$, based on its restricted observation $\mathcal C_{x,t}$, then simply returns in \eqref{eq:restricted_bayes_optimal} the weighted average of pixel values $\varphi_x$ of pixel $x$ for all images $\varphi$ in the training dataset $\mathcal D$, where the weights are determined by the posterior $P_{train}(\varphi | \mathcal{C}_{x,t})$. 

The channel restriction $\mathcal C_{x,t}$ at the end of the forward information channel $\varphi \rightarrow \phi_t \rightarrow \mathcal C_{x,t}$ makes the backwards Bayesian guessing game harder, and ideally prevents the posterior $P(\varphi | \mathcal{C}_{x,t})$ from memorizing by concentrating on a single correct training example $\varphi \in \mathcal{D}$.  This behavior would be in contrast to the more informative posterior $P_{train}(\varphi | \phi_t)$ that always forces the unrestricted Bayes optimal diffusion model to memorize at low enough levels of noise (or equivalently late in the reverse process), as discussed in App. \ref{app:bayes_opt_mem}.



\subsection{Examples of BIRD models: LS, ES, and ELS}\label{app:BIRD_examples}

A natural and simple channel restriction $\mathcal{C}_{x,t}$ is a deterministic low dimensional projection applied to $\phi_t$. 
In the case of a linear projection, $\mathcal C_{x,t}(\phi_t) = \mathcal P_{x,t} \phi_t$ where $\mathcal P_{x,t}$ is an $n_{x,t}$ by $d$ matrix with orthonormal rows. Under this restriction, the Bayesian posterior is given by
\begin{align}
    P(\varphi | \mathcal P_{x,t} \phi_t) = \frac{\exp(- \frac{1}{2(1 - \bar{\alpha}_t)} \norm{\mathcal P_{x,t} (\phi_t - \sqrt{\bar{\alpha}_t} \varphi }_S^2)}{\sum_{\varphi'} \exp(- \frac{1}{2(1 - \bar{\alpha}_t)} \norm{\mathcal P_{x,t} (\phi_t - \sqrt{\bar{\alpha}_t} \varphi' }_S^2)}.
\end{align}
The Local Score (LS) machine of \cite{kambanalytic} is an instance of this framework when the projection operator $\mathcal P_{x,t}$ projects onto a coordinate subspace of pixel components within a local patch $\Omega_x$ surrounding pixel $x$:
\begin{align}
    P(\varphi | \phi_{t;\Omega_x}) = \frac{\exp(- \frac{1}{2(1 - \bar{\alpha}_t)} \norm{\phi_{t,\Omega_x} - \sqrt{\bar{\alpha}_t}\varphi_{\Omega_x} }^2)}{\sum_{\varphi'} \exp(- \frac{1}{2(1 - \bar{\alpha}_t)} \norm{\phi_{t,\Omega_x} - \sqrt{\bar{\alpha}_t} \varphi_{\Omega_x} }^2 )}.
\end{align}
Here we follow the notation of \cite{kambanalytic}, where $\phi_{\Omega_x}$ refers to a vector of concatenated pixel values of the image $\phi$, but restricted only to all pixels inside the local patch of pixels $\Omega_x$. Projections onto subspaces in more general bases may be admitted, as hypothesized by \cite{lukoianov2025locality}, although the latter paper did not propose any specific basis other than the pixel basis used by \cite{kambanalytic}. 

While this framework is already fairly general, there is a further generalization of it that can be considered, which is to regress onto a correlated target $T_x$ rather than onto $\varphi$ itself:
\begin{align}
    \mathcal{L} = \sum_{x} \frac{1}{2} \big\langle \norm{T_x - M_x[\mathcal{C}_{x,t}] }^2 \big\rangle_{\varphi, T_x, \mathcal{C}_{x,t}},
\end{align}
for which the minimizer is given by
\begin{align}\label{eq:correlated_target}
    M_x[\mathcal{C}_{x,t}] = \sum_{\varphi} \int \,dT_x\, T_x P(T_x | \varphi, \mathcal{C}_{x,t}) P(\varphi | \mathcal{C}_{x,t}) 
\end{align}
This generalization may be useful for various stochastic flow-matching setups. It also allows us to reproduce the Equivariant Score (ES) machine and Equivariant Local Score (ELS) machine of \cite{kambanalytic} under a particular correlated choice of target $T_x$ and channel $\mathcal{C}_{x,t}$. \cite{kambanalytic} consider the setting of a model that is \textit{equivariant} under a group $G$, which means that for any $g \in G$ the model $M[\phi]$ satisfies $M[U_g \phi] = U_g M[\phi]$, where $U_g$ is a unitary representation of the group $G$. Under the following choice of correlated channel and target: 
\begin{align}
    g &\sim \text{Haar}(G)\\
    \mathcal{C}_{x,t} &= U_g^\dagger \phi_t\\
    T_x &= U_g \varphi
\end{align}
we reproduce the ES machine with the equation (\ref{eq:correlated_target}):
\begin{align}
    M_x[\phi_t] &= \sum_{\varphi \in \mathcal{D}}\sum_{g \in G} U_g \varphi \,P(\varphi, g | U_g^\dagger \phi_t)\\
    P(\varphi,g | U_g \phi_t) &= \frac{\exp(-\frac{1}{2(1 - \bar{\alpha}_t)} \norm{\varphi - U_g^\dagger \phi_t}^2)}{\sum_{\varphi' \in \mathcal{D}} \sum_{g' \in G} \exp(-\frac{1}{2(1 - \bar{\alpha}_t)} \norm{\varphi' - U_{g'}^\dagger \phi_t}^2)}
\end{align}
The ELS machine can be reproduced by specializing $g$ to the group of translations $\mathcal{T}$ on images with $U_g$ the fundamental representations, and then subsequently projecting the output $U_g^\dagger \phi_t$ into a patch $\Omega_x$:
\begin{align}
    M_x[\phi_t] &= \sum_{\varphi \in \mathcal{D}}\sum_{g \in \mathcal{T}} (U_g \varphi)_x \,P(\varphi, g | (U_g^\dagger \phi_t)_{\Omega_x})\\
    P(\varphi,g | U_g \phi_t) &= \frac{\exp(-\frac{1}{2(1 - \bar{\alpha}_t)} \norm{\varphi_{\Omega_x} - (U_g^\dagger \phi_t)_{\Omega_x}}^2)}{\sum_{\varphi' \in \mathcal{D}} \sum_{g' \in G} \exp(-\frac{1}{2(1 - \bar{\alpha}_t)} \norm{\varphi'_{\Omega_x} - (U_{g'}^\dagger \phi_t)_{\Omega_x}}^2)}
\end{align}

\section{A theory of BIRD memorization-generalization phase transitions}\label{app:mem_gen_theory}

In App. \ref{app:bayes_opt_mem} we noted that unrestricted Bayes optimal diffusion models always memorize at low enough levels of noise (corresponding to late enough times in the reverse process). Motivated by this issue we introduced a very general class of BIRD models in App. \ref{app:BIRD}. In this section we now wish to develop a general theory of when such BIRD models memorize versus generalize. 

\subsection{A high level overview of memorization versus generalization}

Memorization versus generalization can be assessed by the behavior of the posterior distribution $P_{train}(\varphi | \mathcal{C}_{x,t})$ in \eqref{eq:bayes_weights_channel} which encapsulates each pixel $x$'s belief state about which past training image $\varphi \in \mathcal D$ at time $t=0$ led to the pixel's current observation $C_{x,t}$ at the current time $t$ under the training Markov process for forward diffusion and observation $\varphi \rightarrow \phi_t \rightarrow C_{x,t}$, specified by the joint distribution $P_{train}(\varphi, \phi_t, \mathcal C_{x,t}) = P_{train}(\varphi)P(\phi_t | \varphi)P(\mathcal C_{x,t} | \phi_t)$.  In the memorization phase, the posterior $P_{train}(\varphi | \mathcal{C}_{x,t})$ would concentrate on a single, or very small number of training images $\varphi \in \mathcal D$. 

An important consequence of the concentration of the posterior $P_{train}(\varphi | \mathcal{C}_{x,t})$ is that the accompanying BIRD denoiser and diffusion model then both become extremely sensitive to the particular realization of the training data $\mathcal D$.  For example, consider two {\it distinct} data sets $\mathcal D$ and $\mathcal D'$ each drawn i.i.d. from the same true data distribution $P_0(\varphi)$.  Because the posterior $P_{train}(\varphi | \mathcal{C}_{x,t})$ concentrates on individual (or a small number of) data points in the memorizing phase, the BIRD model denoiser and reverse process would yield {\it different} outputs even for the {\it same} observation inputs $\mathcal C_{x,t}$ when using the two different training datasets $\mathcal D$ and $\mathcal D'$.  Thus the BIRD model would not be {\it robust} to the choice of the training data $\mathcal D$. 

In contrast, the hallmark of the generalization phase is that a BIRD model {\it should be robust} to the choice of training data $\mathcal D$.  In particular, averages of quantities with respect to the posterior $P_{train}(\varphi | \mathcal{C}_{x,t})$, should not depend on the detailed realization of the training data $\mathcal D$. An important such average is the posterior mean computed by the denoiser output $M_{x,t}[\mathcal C_{x,t}]$ in \eqref{eq:restricted_bayes_optimal}.  In this robust regime, two different denoisers, and two different BIRD model reverse samplers, obtained from two different datasets $\mathcal D$ and $\mathcal D'$, {\it consistently generalize} to the same denoiser and sampler.

How do we describe this common denoiser in the robust, consistent generalization phase?  Well, a simple way to achieve this phase is to take the limit of a large number of data points $|\mathcal D|$. Then averages of quantities with respect to $P_{train}(\varphi)$ in \eqref{eq:P_train} should converge to averages with respect to the true distribution $P_0(\varphi)$ from which a new test example $\varphi$ would be drawn.  Thus we can consider the {\it testing} Markov process $\varphi \rightarrow \phi_t \rightarrow \mathcal C_{x,t}$ for pixel $x$ at time $t$, defined by the joint distribution $P_{test}(\varphi, \phi_t, \mathcal C_{x,t}) = P_{0}(\varphi)P(\phi_t | \varphi)P(\mathcal C_{x,t} | \phi_t).$  This is identical to the training Markov process above, except the prior is now the {\it true} data distribution $P_0(\varphi)$ from which a test sample is drawn, rather than the empirical training distribution $P_{train}(\varphi)$ for a specific dataset $\mathcal D$. We denote the posterior distribution of the input $\varphi$ given the output $\mathcal C_{x,t}$ under this Markov process by $P_{test}(\varphi | \mathcal C_{x,t})$. This posterior distribution is the outcome of the Bayesian guessing game in the limit of a large amount of data.   

With the testing Markov process in place, we can see that in the robust, consistent generalization phase, the denoiser output in \eqref{eq:restricted_bayes_optimal} converges to a common denoiser for any dataset $\mathcal D$, where the dataset dependent posterior $P_{train}(\varphi | \mathcal C_{x,t})$ is replaced with the dataset indpendent posterior $P_{test}(\varphi | \mathcal C_{x,t})$.  This common denoiser then determines a common reverse sampler as described in App. \ref{app:bayes_opt_mem}.  In this fashion, in the generalization phase, when two different BIRD models are trained on two different datasets, they both consistently converge to this same sampler.

\subsection{The entropy of Bayesian guessing determines memorization and generalization}

We would like to derive sharper conditions on the minimum amount of training data $|\mathcal D|$ required to achieve generalization instead of memorization, and how such a data amount threshold depends on time $t$ in the reverse process (or equivalently the amount of noise), as well as the amount of information the restricted observation $\mathcal C_{x,t}$ provides about the training data $\varphi \in \mathcal D$. 

To do so, it is useful to focus on the entropy of the outcome of the Bayesian guessing game, encapsulated by posterior distribution $P_{train}(\varphi | \mathcal{C}_{x,t})$.  This entropy is given by 
\begin{align}
    S[P_{train}(\varphi | \mathcal{C}_{x,t})] = -\sum_{\varphi \in \mathcal{D}} P_{train}(\varphi | \mathcal{C}_{x,t}) \ln P_{train}(\varphi | \mathcal{C}_{x,t}).
\end{align}
This entropy can be written as the prior entropy $S[P_{train}(\varphi)] = \ln|\mathcal{D}|$, minus the amount of information learned about the past training data $\varphi \in D$ relative to the prior by the current measurement $\mathcal{C}_{x,t}$, also known as the \textit{Bayesian surprise}:
\begin{equation}
    S[P_{train}(\varphi | \mathcal{C}_{x,t})] = \ln|\mathcal{D}| - D_{KL}(P_{train}(\varphi | \mathcal{C}_{x,t})\,\, || \,\, P_{train}(\varphi) ).
\end{equation}
In the memorization phase, $P_{train}(\varphi | \mathcal{C}_{x,t})$ concentrates on a single point and so the entropy vanishes. In a generalizing phase, the posterior will have support on many samples, and so we may expect the Bayesian surprise $D_{KL}(P_{train}(\varphi | \mathcal{C}_{x,t})\,\, || \,\, P_{train}(\varphi) )$ to converge to its value under the true data distribution, i.e. $D_{KL}(P_{test}(\varphi | \mathcal{C}_{x,t})\,\, || \,\, P_{0}(\varphi) )$.

As we will see below, when both the logarithm of the amount of data $\ln |\mathcal D|$ and the data dimension $d$ are large, there is a sharp phase transition between the two phases. We find that whenever $\ln|\mathcal D| - D_{KL}(P_{test}(\varphi | \mathcal{C}_{x,t})\,\, || \,\, P_{0}(\varphi) ) \geq 0$, the training surprise matches the test surprise.  This corresponds to the generalization phase.  Conversely, when this inequality does not hold, the posterior entropy is zero.  Thus we obtain the following formula in the large $\ln \dD$ limit (see Theorem \ref{thm:pointwise_memorization} in App. \ref{app:main_proof} for a formal theorem statement and proof):
\begin{equation}
    S[P_{train}(\varphi|\mathcal{C}_{x,t})] = \begin{cases}
        \ln |\mathcal{D}| - D_{KL}(P_{test}(\varphi | \mathcal{C}_{x,t}) || P_0(\varphi)) & \ln |\mathcal{D}| > D_{KL}(P_{test}(\varphi |\mathcal{C}_{x,t}) || P_0(\varphi)) \\
        0 & \ln |\mathcal{D}| \leq D_{KL}(P_{test}(\varphi | \mathcal{C}_{x,t}) || P_0(\varphi))
    \end{cases}
\end{equation}
This result yields a pointwise lower threshold for the minimum amount of data $|\mathcal{D}|$ required to avoid memorization and achieve generalization for any fixed observation $\mathcal C_{x,t}$.  If $\ln |\mathcal D|$ is less than the Bayesian test surprise $D_{KL}(P_{test}(\varphi | \mathcal{C}_{x,t}) || P_0(\varphi))$, then the BIRD model is in the memorization regime given that observation.  

To understand this result intuitively, note that the Bayesian surprise quantifies how much a pixel $x$ needs to update its prior belief distribution about the training data $P_0(\phi)$ to obtain the posterior belief distribution $P_{test}(\varphi | \mathcal{C}_{x,t})$ after making the observation $C_{x,t}$. If this surprise is high, then intuitively $\mathcal{C}_{x,t}$ is highly informative about $\varphi$. Thus the higher the surprise, the more data $|\mathcal D|$ is required to avoid memorization for any fixed observation $\mathcal C_{x,t}$. Indeed we obtain the extremely simple condition that the amount of data $|\mathcal D|$ needs to be at least exponential in the surprise to achieve generalization.  Conversely, when $\ln |\mathcal D|$ is greater than the surprise, the posterior entropy is nonnegative, indicating the generalization phase. And quantitatively it is equal to the prior entropy minus the Bayesian surprise, yielding the remaining entropy after the observation. 

While the above result holds pointwise for a {\it single} observation $\mathcal C_{x,t}$, we can also find a simple averaged condition when $\cC$ is drawn \textit{randomly} from the test distribution via the testing Markov process $\varphi \rightarrow \phi_t \rightarrow \cC$ where $\varphi \sim P_0(\phi)$.  As is well known, the average of the Bayesian surprise  $D_{KL}(P_{test}(\varphi | \mathcal{C}_{x,t}) || P_0(\varphi))$ over the marginal distribution $P_{test}(\cC)$ yields the mutual information $I(\varphi; \mathcal{C}_{x,t})$.  Thus if the average surprise is high, so is the mutual information (as they are equal).  This yields an extremely simple and highly general information theoretic criterion for the memorization-generalization phase transition boundary in any BIRD model and a wide variety of data distributions: 
\begin{equation} \label{eq:memorization_app}
    \begin{cases}
        \ln |\mathcal{D}| > I(\varphi; \mathcal{C}_{x,t}) \quad \text{Generalization phase.}  \\
        \ln |\mathcal{D}| < I(\varphi; \mathcal{C}_{x,t}) \quad \text{Memorization phase.}
    \end{cases}
\end{equation}
This result then implies that restricting information contained in observations can reduce the minimum data threshold required for generalization in BIRD models. 

In summary, \eqref{eq:memorization_app} reveals that the amount of data $|\mathcal D|$ needs to be at least exponential in the mutual information between current observations $C_{x,t}$ and the true test data $\varphi \sim P_0(\varphi)$ in the forward testing Markov process $\varphi \rightarrow \phi_t \rightarrow \cC$, to avoid memorization and achieve generalization.  

More formal statements of these results and their proofs are given below.

\subsection{Three types of phase transitions between memorization and generalization.}

One can think of the phase transition boundary between the memorization phase and generalization phase, determined by the equality condition $I(\mathcal{C}_{t,x}; \varphi) = \ln |\mathcal{D}|$ in several ways:  

\paragraph{A data phase transition.} First, in terms of the amount of data $|\mathcal D|$ for a fixed observation method $\mathcal C_{x,t}$ and fixed time $t$, as $\ln |\mathcal D|$ increases from below $I(\mathcal{C}_{t,x}; \varphi)$ to above $I(\mathcal{C}_{t,x}; \varphi)$, we transition from the memorization to generalization phase. The critical amount of data $D_c$ at which this transition occurs obeys the equation $I(\mathcal{C}_{t,x}; \varphi) = \ln D_c$. For $|\mathcal D| < D_c$ any BIRD model memorizes while for $| \mathcal D| > D_c$ it generalizes.

\paragraph{A temporal phase transition.} Second, we can consider time $t$ in the reverse process at a fixed amount of data $|\mathcal D|$ (and an observation method that does not depend explicitly on time).  At times $t$ near $1$ near the beginning of the reverse process, the noise to signal ratio $\sigma_t^2$ is large and therefore the mutual information $I(\mathcal{C}_{t,x}; \varphi)$ is small.  However at times $t$ near $0$ near the end of the reverse process, the noise to signal ratio $\sigma_t^2$ is small and therefore the mutual information $I(\mathcal{C}_{t,x}; \varphi)$ is large.  Therefore, as the reverse process proceeds from $t=1$ down to $t=0$, the mutual information $I(\mathcal{C}_{t,x}; \varphi)$ could in principle transition from below $\ln |\mathcal D|$ to above $\ln |\mathcal D|$, implying the reverse process could transition from a generalizing phase to a memorizing phase. The temporal boundary $t_b$ at which this transition occurs obeys the equation $I(\mathcal{C}_{t_b,x}; \varphi) = \ln |\mathcal D|$.  For $t > t_b$ any BIRD model (with a fixed, time-independent observation method) generalizes early in the reverse process,  while for $t < t_b$ it memorizes late in the reverse process. 

\paragraph{An observation phase transition.} Third, we can consider a one parameter family of observations $\mathcal C_{x,t}$ of varying information content,  at a fixed time $t$ and a fixed amount of data $|\mathcal D|$.  For example, let $n_{x,t}$ be the dimensionality of the observation $\mathcal C_{x,t}$ such that reducing $n_{x,t}$ monotonically drops measurements of $\phi_t$. As a concrete example, let $\mathcal C_{x,t}(\phi) = \phi_{\Omega_x}$ be the restriction of the image $\phi$ to a local $L$ by $L$ image patch $\Omega_x$ centered at $x$, yielding the restricted image $\phi_{\Omega_x}$ of dimension $n_{x,t} = L\times L \times C$. In this example, reducing $n_{x,t}$ by reducing the patch size $L$ strictly reduces the information content of the observation $\mathcal C_{x,t}$. Now when $n_{x,t}$ varies from large to small, the mutual information $I(\mathcal{C}_{t,x}; \varphi)$ can also vary from large to small.  Thus as $n_{x,t}$ varies from large to small, the mutual information $I(\mathcal{C}_{t,x}; \varphi)$ can in principle transition from above $\ln |\mathcal D|$ to below $\ln |\mathcal D|$, implying a transition from memorizing to generalizing as the dimensionality or information content of the observation $\mathcal C_{x,t}$ is reduced. The critical dimension $n_c$ at which this transition occurs obeys the equation $I(\mathcal{C}^{n_c}_{t,x}; \varphi) = \ln | \mathcal D |$. For $n_{x,t} > n_c$ any BIRD model memorizes for highly informative observations, while for $n_{x,t} < n_c$ it generalizes for information restricted observations.     

\paragraph{A general co-dimension one phase boundary.} Of course in full generality, in the {\it joint} space of the amount of data $|\mathcal D|$, the time $t$ in the reverse process, and any family of observations $\mathcal C_{x,t}$ of varying information content, the equation $I(\mathcal{C}_{t,x}; \varphi) = \ln |\mathcal D|$ determines a co-dimension one phase boundary between the memorization and generalization phase.  The three factors of amount of data, time in reverse process, and information content of observations will compete to determine which phase the BIRD model is in.  Small amounts of data, late or small times $t$, and large observational information content, each favor the memorization phase, while opposite directions favor generalization.  As we will see, BIRD models can combat memorization at late times in the reverse process, by reducing the information content of observations $\mathcal C_{x,t}$ as $t$ decreases. The observational restriction impedes the memorization at small times $t$, and if strong enough, favors generalization instead.

\subsection{\label{app:simple_entropy_calc}A simple calculation of the posterior entropy in the generalizing phase}
Consider a channel $\mathcal{C}_{x,t}$ and a set of codewords or images $\dD$ drawn i.i.d. from the distribution $P_0$. Given a input $\varphi$ the channel produces the output $\mathcal{C}_{x,t}$ with probability density $P(\mathcal{C}_{x,t} | \varphi)$. We can define the training set probability distribution conditioned on the channel outputs $P_{train}: \mathbb{R}^d  \rightarrow \mathbb{R}$:
\begin{equation}
    P_{train}(\mathcal{C}_{x,t}) = \frac{1}{|\mathcal{D}|}\sum_{\varphi \in \mathcal{D}} P(\mathcal{C}_{x,t} | \varphi) 
\end{equation}

We can then use Bayes rule to find the Bayesian probability over the training set given $\mathcal{C}_{x,t}$ $P_{train}: \mathcal{D} \times \mathbb{R}^{d} \rightarrow \mathbb{R}$.
\begin{equation}
    P_{train}(\varphi | \mathcal{C}_{x,t})  = \frac{P(\mathcal{C}_{x,t} | \varphi)}{P_{train}(\mathcal{C}_{x,t})} = \frac{P(\mathcal{C}_{x,t} | \varphi)}{\sum_{\varphi' \in \mathcal{D}} P(\mathcal{C}_{x,t} | \varphi') }
\end{equation}
Note that both $P_{train}$ distributions are random variables dependent on the codewords $\mathcal{D}$. We also define the test set probabilities, which are not random.
\begin{align}
    P_{test}(\mathcal{C}_{x,t}) &= \int d\varphi \: P_0(\varphi) P(\mathcal{C}_{x,t} | \varphi) \\
    P_{test}(\varphi | \mathcal{C}_{x,t}) &= \frac{P(\mathcal{C}_{x,t} | \varphi)}{P_{test}(\mathcal{C}_{x,t})}
\end{align}
We would like to understand the entropy of the Bayesian posterior on the training set:
\begin{align}
    S[P_{train}(\varphi | \mathcal{C}_{x,t})] &= \sum_{\varphi \in \mathcal{D}} -P_{train} (\varphi | \mathcal{C}_{x,t}) \ln P_{train} (\varphi | \mathcal{C}_{x,t}) \\
    &= \sum_{\varphi \in \mathcal{D}} -\frac{P(\cC | \varphi)}{\sum_{\varphi' \in \mathcal{D}}P(\cC | \varphi')} \ln \frac{P(\cC | \varphi)}{\sum_{\varphi' \in \mathcal{D}}P(\cC | \varphi')} \\
    &= \ln \Big(\sum_{\varphi \in \mathcal{D}}P(\cC | \varphi) \Big) - \frac{1}{\sum_{\varphi' \in \mathcal{D}}P(\cC | \varphi)}\sum_{\varphi \in \mathcal{D}} P(\cC | \varphi) \ln P(\cC | \varphi) \\
    &= \ln\dD + \ln \Big( \frac{\sum_{\varphi \in \mathcal{D}} P(\mathcal{C}_{x,t} | \varphi)}{\dD}\Big) \label{eq:S_train_means_start} \\
    &\hspace{20pt} - \frac{1}{\sum_{\varphi \in \mathcal{D}} P(\mathcal{C}_{x,t} | \varphi)/\dD}\frac{1}{\dD}\sum_{\varphi \in \mathcal{D}} P(\mathcal{C}_{x,t} | \varphi) \ln P (\mathcal{C}_{x,t} | \varphi) \label{eq:S_train_means_end}
\end{align}
All the terms in this sum are independent random variables. As the dataset becomes large, the central limit theorem lets us replace the empirical means with their actual means up to a $O(\frac{1}{\sqrt{\dD}})$ corrections, so long as their means are non-zero. This shows that
\begin{align}
    \frac{1}{\dD}\sum_{\varphi \in \mathcal{D}} P(\mathcal{C}_{x,t} | \varphi) &\rightarrow \EE_{\varphi \sim P_0}[P(\mathcal{C}_{x,t} | \varphi)]\\
    &= \int d\varphi P_0(\varphi) P(\mathcal{C}_{x,t} | \varphi) = P_{test}(\mathcal{C}_{x,t}) \label{eq:p_mean_to_p_test}\\
    -\frac{1}{\dD}\sum_{\varphi \in \mathcal{D}} P(\mathcal{C}_{x,t} | \varphi) \ln P(\mathcal{C}_{x,t} | \varphi) &\rightarrow \EE_{\varphi \sim P_0}[ - P(\mathcal{C}_{x,t} | \varphi) \ln P(\mathcal{C}_{x,t} | \varphi)] \\
    &= -\int d\varphi P_0(\varphi) P(\mathcal{C}_{x,t} | \varphi) \ln P(\mathcal{C}_{x,t} | \varphi) \\
    &\hspace{-80pt}= -\int d\varphi P_{test}(\mathcal{C}_{x,t}, \varphi) (\ln P_{test}(\mathcal{C}_{x,t}) + \ln \frac{P(\mathcal{C}_{x,t}, \varphi)}{P_0(\varphi) P_{test}(\mathcal{C}_{x,t})}) \\
    &\hspace{-80pt}= P_{test}(\mathcal{C}_{x,t}) \Big(-\ln P_{test}(\mathcal{C}_{x,t}) - \int d\varphi P_{test}(\varphi|\mathcal{C}_{x,t}) (\ln \frac{P_{test}(\varphi|\mathcal{C}_{x,t})}{P_0(\varphi)}) \Big) \\
    &\hspace{-80pt}= P_{test}(\mathcal{C}_{x,t}) \Big(-\ln P_{test}(\mathcal{C}_{x,t}) - D_{KL}(P_{test}(\varphi | \mathcal{C}_{x,t}) || P_0(\varphi))\Big) \label{eq:p_mean_to_DK}
\end{align}
Plugging equations \ref{eq:p_mean_to_p_test} and \ref{eq:p_mean_to_DK} into equation \ref{eq:S_train_means_start}-\ref{eq:S_train_means_end} we see,
\begin{align}
    S[P_{train}(\varphi | \mathcal{C}_{x,t})] &= \ln \dD + \ln P_{test}(\mathcal{C}_{x,t}) \\
    &+ \frac{1}{P_{test}(\mathcal{C}_{x,t})} P_{test}(\mathcal{C}_{x,t}) \Big(-\ln P_{test}(\mathcal{C}_{x,t}) - D_{KL}(P_{test}(\varphi | \mathcal{C}_{x,t}) || P_0(\varphi))\Big) \\
    &+ O(\frac{1}{\sqrt{\dD}})\\
    &= \ln \dD - D_{KL}(P_{test}(\varphi | \mathcal{C}_{x,t}) || P_0(\varphi)) + O(\frac{1}{\sqrt{\dD}})
\end{align}
When either average above becomes smaller than $O(\frac{1}{\sqrt{D}})$ the fluctuations might dominate over the mean. At the same time, we know that entropy must be strictly non-negative and should decrease as the information increases, so we can guess:
\begin{equation}
    S[P_{train}(\varphi | \mathcal{C}_{x,t})] = \max(\ln \dD - D_{KL}(P_{test}(\varphi | \mathcal{C}_{x,t}) || P_0(\varphi)), 0)
\end{equation}
Because the central limit theorem no longer applies, it is difficult to estimate the error in this formula. As it turns out, $O(1/\sqrt{D})$ is too optimistic in the memorizing phase. In the section below, we use the random energy model to prove that this formula becomes exact in the large $\dD$ limit with $O(\frac{\ln \ln \dD}{\ln \dD})$ finite size corrections. If we average this result over $\mathcal{C}_{x,t}$ drawn from $P_{test}$.
\begin{equation}
    \EE_{\mathcal{C}_{x,t} \sim P_{test}}[S[P_{train}(\varphi | \mathcal{C}_{x,t})]] = \max(\ln \dD - I(\varphi; \mathcal{C}_{x,t}), 0)
\end{equation}
\subsection{\label{app:main_proof} A proof from the Random Energy Model}
In this section, we give several proofs based on methods from statistical mechanics, inspired by recent approaches in the literature of diffusion models as in \cite{biroli2024dynamical, biroli2026speciation}. We start by defining an energy function $E$ and inverse temperature $\beta = 1/T$ so that $P_{train}$ corresponds to a Boltzmann distribution:
\begin{align}
     - \beta E(\varphi | \mathcal{C}_{x,t}) = \ln P(\mathcal{C}_{x,t} | \varphi) \label{eq:energy_def}\\
     P_{train}(\varphi | \mathcal{C}_{x,t}) = \frac{e^{-\beta E(\varphi | \mathcal{C}_{x,t})}}{\sum_{\varphi' \in \mathcal{D}} e^{-\beta E(\varphi' | \mathcal{C}_{x,t})}}
     \label{eq:Boltzmann_def}
\end{align}
Therefore, the energy $E(\varphi | \mathcal{C}_{x,t})$ is simply the temperature $T$ times the minus log likelihood of making the channel observation $\mathcal{C}_{x,t}$ at time $t$ under the forward diffusion process starting from a data point $\varphi$ at time $t=0$.  If we fix the channel observation $\mathcal{C}_{x,t}$, we can think of $E(\varphi | \mathcal{C}_{x,t})$ in \eqref{eq:energy_def} as an energy function on the training data $\varphi \in \mathcal D$, with a high (low) energy assigned to $\varphi$ indicating a low (high) likelihood that $\varphi$ was the origin of the channel observation $\mathcal{C}_{x,t}$ in the forward diffusion process. After normalization, the Boltzmann distribution in \eqref{eq:Boltzmann_def} is simply the posterior probability that a particular training data point $\varphi \in \mathcal D$ at time $t=0$ was the origin of the conditioned channel observation $\mathcal{C}_{x,t}$ in the forward diffusion process. 
\begin{theorem}[Independent Pointwise Posterior Entropy]\label{thm:pointwise_memorization}
    If $\mathcal{D}$ is a set of training images drawn i.i.d. from $P_0$ and $\mathcal{C}_{x,t}$ is a channel such that $P(E | \mathcal{C}_{x,t})$ has a rate function for large $E$, with a continuous first derivative. We also assume that both the rate function and its derivative have a finite number of maxima and minima at $E = O(1)$.
    \begin{equation}
        S[P_{train}(\varphi|\mathcal{C}_{x,t})] = \begin{cases}
            \ln |\mathcal{D}| - D_{KL}(P_{test}(\varphi | \mathcal{C}_{x,t}) || P_0(\varphi)) & \ln |\mathcal{D}| > D_{KL}(P_{test}(\varphi | \mathcal{C}_{x,t}) || P_0(\varphi)) \\
            0 & \ln |\mathcal{D}| \leq D_{KL}(P_{test}(\varphi | \mathcal{C}_{x,t}) || P_0(\varphi))
        \end{cases}
    \end{equation}
    up to $O(\frac{\ln \ln |\mathcal{D}|}{\ln |\mathcal{D}|})$ corrections.
\end{theorem}
\begin{proof}
We will begin by considering the case where $\mathcal{C}_{x,t}$ is drawn from the test set. This means $\mathcal{C}_{x,t}$, as a random channel observation, is statistically {\it independent} of the random training data $\mathcal{D}$. If $\mathcal{C}_{x,t}$ were drawn from the training set, then the dataset sample that generates $\cC$ would not be identically distributed and would break the i.i.d assumption in hypothesis of this theorem. We will first fix $\mathcal{C}_{x,t}$ to be an arbitrary value, then we will average over the distribution of $\mathcal{C}_{x,t}$. 
Now since the data points $\varphi \in \mathcal D$ are drawn i.i.d. from $P_0(\varphi)$, the energies $E(\varphi | \mathcal{C}_{x,t})$ of each data point $\varphi$ can be thought of as independent random variables. 
Therefore, the Boltzmann distribution in \eqref{eq:Boltzmann_def} is equivalent to an independent random energy model (REM) as first described by \cite{derrida1980randomenergy} (see also \cite{mezard2009info} for a nice pedagogical treatment).
The REM is solvable by a saddle point approximation which becomes exact in a ``thermodynamic" limit, which in our case, corresponds to a large amount of training data, or large $|\mathcal D|$. We define 
\begin{equation}
    P(E' | \mathcal{C}_{x,t}) = \int d\varphi \: \delta(E' - E(\varphi | \mathcal{C}_{x,t})) P_0(\varphi)
\end{equation}
This is simply the distribution of energies $E'$ we expect to see conditioned on a fixed channel observation $\mathcal{C}_{x,t}$ where the randomness comes from the distribution of a single training point $P_0(\varphi)$.
In the REM, the log number of states, $\ln \dD$, is the extensive parameter, originally corresponding to the number of spins.  For ML readers not familiar with the statistical mechanics language, the extensive parameter, $\ln \dD$, is used to group together different values based on their relative size to it.
\begin{align}
    \text{intensive} &\rightleftharpoons O(1) \\
    \text{extensive} &\rightleftharpoons O(\ln \dD)  \\
    \text{exponentially large} &\rightleftharpoons O(e^{\alpha \ln \dD}) = O(\dD^{\alpha}) \\
    \text{exponentially small} &\rightleftharpoons O(e^{-\alpha \ln \dD}) = O(\frac{1}{\dD^{\alpha}})
\end{align}
In some strict sense, this means that we are imagining that all of the variables in the problem have some dependence on $\ln \dD$ which allows us to take the limit $\ln \dD \rightarrow \infty$, called the thermodynamic limit. To make this analysis quite general, we will avoid explicit expressions for this dependence on $\ln \dD$ and only express asymptotic scaling. We would like to define this limit so there are sensible answers for the entropy of the posterior and relevant quantities at polynomial orders in $\ln \dD$. By that we mean that we are interested in finding asymptotic series as $\ln \dD \rightarrow \infty$ like $\EE[O] = \sum_{n=-\infty}^{N} O_n (\ln \dD)^n$. In this context, quantities which are exponentially small do not show up in any order of these asymptotic series since they decay faster than any rational function. Therefore, exponentially small quantities can simply be set to 0.   For the REM to have a sensible thermodynamic limit, we require that the energy is extensive. This is equivalent to the existence of a function $I$ called the rate function defined in the following way:
\begin{align}
    \epsilon(\varphi | \cC) &= \frac{E(\varphi | \cC)}{\ln \dD} \\
    I(\epsilon | \mathcal{C}_{x,t}) &:=\frac{-\ln P(\epsilon \ln |\mathcal{D}| \Big| \mathcal{C}_{x,t})}{\ln|\mathcal{D}|} = \Omega(1) \label{eq:rate_fnc_definition} \\
    P(\epsilon | \cC) &= e^{- \ln \dD I(\epsilon | \cC)}
\end{align}
$\epsilon$ is called the intensive energy because we removed the $\ln \dD$ scaling of the energy by dividing. We also assume that the inverse temperature is intensive. The existence of the rate function implies that there is sensible probability distribution for the intensive energy.


If we unwrap the asymptotic expression in \ref{eq:rate_fnc_definition}, we find
\begin{align}
    \frac{-\ln P(\epsilon \ln |\mathcal{D}| \Big| \mathcal{C}_{x,t})}{\ln|\mathcal{D}|}  &\geq C_1 \\
    \ln P(\epsilon \ln |\mathcal{D}| \Big| \mathcal{C}_{x,t}) &\leq - C_1 \ln \dD \\
    P(\epsilon \ln \dD \Big| \mathcal{C}_{x,t}) &\leq C_2 e^{- C_1 \ln \dD} \\
    \Leftarrow P(E| \mathcal{C}_{x,t}) &\leq C_2 e^{- C_1 E - o(1)}.
\end{align}
We'll now look at the "density of states" which measures how many energies sit inside of a window around $\epsilon$.
\begin{equation}
    \rho(\epsilon | \mathcal{C}_{x,t}) = \frac{1}{W}\sum_{\varphi \in \mathcal{D}} \mathbf{1}_{\epsilon(\varphi | \mathcal{C}_{x,t}) \in [\epsilon, \epsilon+W]}
\end{equation}
where $\mathbf{1}_{E(\varphi | \mathcal{C}_{x,t}) \in [\epsilon, \epsilon+W]}$ is the indicator function
\begin{equation}
    \mathbf{1}_{E(\varphi | \mathcal{C}_{x,t}) \in [\epsilon, \epsilon+W]} = \begin{cases}
        1 & \epsilon(\varphi | \mathcal{C}_{x,t}) \in [\epsilon, \epsilon+W] \\
        0 & \epsilon(\varphi | \mathcal{C}_{x,t}) \notin [\epsilon, \epsilon+W]
    \end{cases}
\end{equation}
Saddle-point approximation leads to an expected density of states 



\begin{align}
    \mathbb{E}[\rho(\epsilon | \mathcal{C}_{x,t})] &= \mathbb{E}[\sum_{\varphi \in \mathcal{D}} \mathbf{1}_{E(\varphi | \mathcal{C}_{x,t}) \in [\epsilon, \epsilon+W]}]/W= |\mathcal{D}| \mathbb{E}_{\varphi \sim P_0}[\mathbf{1}_{E(\varphi | \mathcal{C}_{x,t}) \in [\epsilon, \epsilon+W]}]/W \\
    &= \int_{\epsilon}^{\epsilon + W} \frac{d\epsilon}{W} \: e^{\ln |\mathcal{D}| \Big(1 - I(\epsilon | \mathcal{C}_{x,t})\Big)} \\
    &= \begin{cases}
        e^{\ln|\mathcal{D}|(1 - I(\epsilon| \mathcal{C}_{x,t})) + o(1)} & 1 > I(\epsilon| \mathcal{C}_{x,t}) \\
        0 + O(e^{-\ln |\mathcal{D}|}) & 1 \leq I(\epsilon | \mathcal{C}_{x,t})
    \end{cases}
\end{align}
Since $\mathbf{1}_{E(\varphi | \mathcal{C}_{x,t}) \in [\epsilon, \epsilon+W]}$ is its own square, the variance of $\rho$ is equal to its expectation which implies the standard deviation is $O(\sqrt{\dD})$ where the leading order term is $O(\dD)$. Therefore, the leading order term of the density of states is deterministic
\begin{equation}
    \rho(\epsilon | \mathcal{C}_{x,t}) = \begin{cases}
        e^{\ln|\mathcal{D}|(1 - I(\epsilon| \mathcal{C}_{x,t}))} + O(e^{\ln|D|/2}) & 1 > I(\epsilon| \mathcal{C}_{x,t}) \\
        0 + O(e^{- \ln |\mathcal{D}|}) & 1 \leq I(\epsilon | \mathcal{C}_{x,t})
    \end{cases}
\end{equation}
The micro-canonical entropy $S(\epsilon)$ is just the log density of states. It is deterministic to all orders in $\ln \dD$.
\begin{align}
    S(\epsilon | \mathcal{C}_{x,t}) &= \ln \rho(\epsilon | \cC) \\
    &= \begin{cases}
        \ln\Big(e^{\ln|\mathcal{D}|(1 - I(\epsilon| \mathcal{C}_{x,t}))} + O(e^{\ln|D|/2})\Big) & 1 > I(\epsilon| \mathcal{C}_{x,t}) \\
        \ln(0 + O(e^{-\ln |\mathcal{D}|})) & 1 \leq I(\epsilon| \mathcal{C}_{x,t})
    \end{cases} \label{eq:before_log_trick}\\
    &= \begin{cases}
        \ln|\mathcal{D}|(1 - I(\epsilon| \mathcal{C}_{x,t})) + \ln(1 + O(e^{-\ln|D|/2})) & 1 > I(\epsilon| \mathcal{C}_{x,t}) \\
        -\infty + O(e^{-\ln |\mathcal{D}|}) & 1 \leq I(\epsilon| \mathcal{C}_{x,t})
    \end{cases} \label{eq:after_log_trick}\\
    &= \begin{cases}\label{eq:deterministic_entropy}
        \ln|\mathcal{D}|(1 - I(\epsilon| \mathcal{C}_{x,t})) + O(e^{-\ln |\mathcal{D}|/2}) & 1 > I(\epsilon| \mathcal{C}_{x,t}) \\
        -\infty + O(e^{-\ln |\mathcal{D}|}) & 1 \leq I(\epsilon| \mathcal{C}_{x,t})
    \end{cases}
\end{align}
Between equation \ref{eq:before_log_trick} and equation \ref{eq:after_log_trick}, we multiplied and divided by the deterministic contribution to the density of states and then used the logarithm to separate the deterministic contribution into its own term. From equation \eqref{eq:deterministic_entropy}, we see that the entropy is deterministic up to exponentially small corrections.
We now define the annealed micro-canonical entropy density $s_{ann, MC}(\epsilon | \mathcal{C}_{x,t})$ and the micro-canonical entropy density $s(\epsilon | \cC)$ as follows:

\begin{align}
    s_{ann,MC}(\epsilon | \mathcal{C}_{x,t}) &:= \frac{\ln \mathbb{E}[\rho(\epsilon | \mathcal{C}_{x,t})]}{\ln |\mathcal{D}|}  = 1 - I(\epsilon | \mathcal{C}_{x,t})\\
    s(\epsilon | \mathcal{C}_{x,t}) := \frac{S(\epsilon | \cC)}{\ln \dD} &= \begin{cases}
        s_{ann, MC}(\epsilon | \mathcal{C}_{x,t}) & s_{ann, MC}(\epsilon | \mathcal{C}_{x,t}) > 0 \\
        -\infty & s_{ann, MC}(\epsilon | \mathcal{C}_{x,t}) \leq 0
    \end{cases}\label{eq:quench_to_ann_entropy}
\end{align}
Our goal is to calculate the entropy of the Gibbs distribution which is the posterior over training images. This requires we switch from the micro-canonical ensemble perspective to the canonical ensemble perspective. We can define the standard thermodynamic functions for the canonical ensemble:
\begin{align*}
    \text{The Partition Function} \rightleftharpoons Z(\beta | \cC) &:= \sum_{\varphi \in \mathcal{D}} e^{-\beta E(\varphi | \cC)} \\
        \text{Gibbs Distribution} \rightleftharpoons P_{\beta} (\varphi | \cC) &:= \frac{e^{-\beta E(\varphi | \cC)}}{Z} \\
    \text{The Free Energy} \rightleftharpoons F(\beta | \cC) &:= \frac{\ln Z(\beta | \cC)}{-\beta} \\
    \text{The Expected Energy} \rightleftharpoons E(\beta | \cC) &:= \sum_{\varphi \in \mathcal{D}} E(\varphi | \cC )P_{\beta}(\varphi | \cC)\\
    \text{The Canonical Entropy} \rightleftharpoons S(\beta | \cC) &= \sum_{\varphi} -P_{\beta}(\varphi | \cC) \ln P_{\beta}(\varphi | \cC) \\
    &= \sum_{\varphi} -\frac{e^{-\beta E(\varphi | \cC)}}{Z} (-\ln Z - \beta E(\varphi | \cC)) \\
    &= \ln Z + \beta \sum_{\varphi} E(\varphi | \cC) P_{\beta}(\varphi | \cC) \\
    &= -\beta F(\beta | \cC) + \beta E(\beta | \cC)
\end{align*}
In models with disorder, like the random energy model, you often look at the disorder averaged or "annealed" partition function. We can define the annealed versions of the functions in the canonical ensemble. Since I've proven that some of the above functions are deterministic, I can replace them with their annealed averages where appropriate.
\begin{align*}
    \text{Annealed Partition Function} \rightleftharpoons Z_{ann}(\beta | \cC) &:= \EE_{\mathcal{D}}[\sum_{\varphi \in \mathcal{D}} e^{-\beta E(\varphi | \cC)}] \\
    \text{(Canonical) Annealed Free Energy} \rightleftharpoons F_{ann}(\beta | \cC) &:= \frac{\ln Z_{ann}(\beta | \cC)}{-\beta} \\
    \text{(Canonical) Annealed Expected Energy} \rightleftharpoons E_{ann}(\beta | \cC) &:= \frac{1}{Z_{ann}}\EE_{\mathcal{D}}[\sum_{\varphi \in \mathcal{D}} E(\varphi | \cC )e^{- \beta E(\varphi | \cC)}]\\
    \text{(Canonical) Annealed Canonical Entropy} \rightleftharpoons s_{ann, MC}(\beta | \cC) &= -\beta F_{ann}(\beta | \cC) + \beta E_{ann}(\beta | \cC)
\end{align*}
The free energy, expected energy, and canonical entropy are extensive so its natural to define the intensive densities for each.
\begin{align}
    \text{The Free Energy Density} \rightleftharpoons  f(\beta | \cC) &:= \frac{F(\beta | \cC)}{\ln \dD} \\
    \text{Expected Energy Density} \rightleftharpoons \epsilon(\beta | \cC) &:= \frac{E(\beta | \cC)}{\ln \dD}\\
    \text{The Canonical Entropy Density} \rightleftharpoons s(\beta | \cC) &:= \frac{S(\beta | \cC)}{\ln \dD}
\end{align}
Because this model has an exact Saddle point in the thermodynamic limit, the canonical ensemble functions will just be Legendre transforms of their micro-canonical counter parts, but the regions where density of states drop to zero constrain the energy. Similar to the micro-canonical entropy, they will be deterministic to all orders in $\ln \dD$. We can see this by using saddle point approximation to find the partition function.


\begin{align*}
   Z(\beta | \mathcal{C}_{x,t}) &= \sum_{\varphi \in \mathcal{D}} e^{-\beta E(\varphi | \cC)}  = \int_{\RR} d\epsilon \: e^{-\ln|\mathcal{D}| \beta \epsilon} \rho(\epsilon | \cC) \\
   &= \int_{s_{ann, MC}(\epsilon | \mathcal{C}_{x,t}) > 0} d\epsilon \: e^{\ln|\mathcal{D}| (s_{ann, MC}(\epsilon | \mathcal{C}_{x,t})  - \beta \epsilon)} + O(e^{\ln \dD /2}) \\
   &= \exp(\ln \dD \Big(\max_{\epsilon \: s.t.\: s_{ann, MC}(\epsilon) > 0} s_{ann, MC}(\epsilon | \mathcal{C}_{x,t})  - \beta \epsilon \Big) + O(\ln \ln \dD))
\end{align*}
The $O(\ln \ln |D|)$ error term above accounts for the normalization of the Gaussian fluctuations around the saddle point. Expanding up to second order in $\epsilon' - \epsilon(\beta)$ leads to a Gaussian with variance $\sigma^2 = O((\ln|D|)^{-1})$. The normalization of a Gaussian integral is $(2 \pi \sigma^2)^{-1/2} = \exp(-\ln \sigma + O(1)) = \exp(O(\ln \ln \dD))$, which produces this term. If $\epsilon(\beta)$ lies at a boundary of integration, the lowest order fluctuation is exponential and the normalization still scales as $O((\ln\dD)^{-1}) = \exp(O(\ln \ln \dD))$, so the same logic applies.  Now, using this result to find the free energy density:
\begin{equation}
    -\beta f(\beta | \cC) = \max_{\epsilon \: s.t.\: s_{ann, MC}(\epsilon) > 0} s_{ann, MC}(\epsilon | \mathcal{C}_{x,t})  - \beta \epsilon + O(\frac{\ln \ln \dD}{\ln \dD}).
\end{equation}
We can similarly solve for the energy density
\begin{align*}
    \epsilon(\beta | \cC) &= \frac{1}{Z} \int_{\RR} d\epsilon \: \rho(\epsilon | \cC) \epsilon e^{-\beta \ln \dD \epsilon} \\
    &= \frac{1}{Z} \int_{s_{ann, MC}(\epsilon) > 0} \epsilon e^{\ln \dD(s_{ann, MC}(\epsilon| \cC) -\beta \epsilon)} \\
    &= \text{argmax}_{s_{ann, MC}(\epsilon) > 0} s_{ann, MC}(\epsilon| \cC) -\beta \epsilon + O(\frac{\ln \ln \dD}{\ln \dD})
\end{align*}
For the rest of the proof, I will omit the $O(\ln \ln |D|/ \ln |D|)$ term. It should be assumed that all canonical ensemble functions have corrections of that order. We make one last transformation from micro-canonical to canonical transformation for annealed thermodynamic functions. As always in thermodynamics, saddle point approximation shows that the annealed canonical functions concentrate to the annealed micro-canonical functions in the large $\ln \dD$ limit. I will elide some details since this is the third similar calculation. Note that $\epsilon$ is now allowed to range over the entire real line instead of only where $s_{ann, MC}(\epsilon) > 0$. 
\begin{align*}
    Z_{ann}(\beta | \cC) &:= \EE_{\mathcal{D}}[\sum_{\varphi \in \mathcal{D}} e^{-\beta E(\varphi | \cC)}] = |\mathcal{D}| \EE_{\varphi}[  e^{-\beta E(\varphi | \cC)}] \\
    &= |\mathcal{D}| \int d\epsilon   e^{\ln \dD (s_{ann, MC}(\epsilon | \cC) -\beta \epsilon)} = \exp (\ln \dD \max_{\epsilon} s_{ann, MC}(\epsilon| \cC) -\beta \epsilon) \\
    \implies f_{ann}(\beta | \cC) &= \frac{\ln Z_{ann}(\beta | \cC)}{-\beta \ln \dD} \rightarrow \frac{-1}{\beta}\max_{\epsilon} s_{ann, MC}(\epsilon| \cC) -\beta \epsilon \\
    \implies \epsilon_{ann}(\beta | \cC) &= \text{argmax}_\epsilon s_{ann, MC}(\epsilon| \cC) -\beta \epsilon \\
    \implies s_{ann}(\beta | \cC) &= - \beta f_{ann}(\beta | \cC) + \beta \epsilon_{ann}(\beta | \cC) = s_{ann, MC}(\epsilon_{ann}(\beta)| \cC)
\end{align*}

Note the difference between the annealed and actual (quenched) energy in the thermodynamic limit of the canonical ensemble
\begin{align*}
    \epsilon_{ann}(\beta) &= \text{argmax}_{\epsilon} &s_{ann, MC}(\epsilon| \cC) -\beta \epsilon \\
    \epsilon(\beta) &= \text{argmax}_{s_{ann,MC}(\epsilon) > 0} &s_{ann, MC}(\epsilon| \cC) -\beta \epsilon
\end{align*}
If $\epsilon_{ann}(\beta) = \epsilon(\beta)$, then $s_{ann}(\epsilon_{ann}(\beta)) = s_{ann}(\epsilon(\beta)) \geq 0$. On the other hand, if $\epsilon_{ann}(\beta) \neq \epsilon(\beta)$ then we can conclude $s_{ann}(\epsilon_{ann}(\beta)) < 0$ since no $\epsilon$ with $s_{ann}(\epsilon) \geq 0$ was able to achieve this minimum. We now must apply more of our hypothesizes to determine what happens to $s_{ann, MC}(\epsilon(\beta))$ when actual (quenched) and annealed energies disagree.

By assumption, $I(\epsilon)$ and $I'(\epsilon)$ have a finite number of local maxima and minima that must lie between some $E_a$ and $E_b$. This means all local minima and maxima of $s_{ann,MC}(\epsilon)$ and $s_{ann,MC}'(\epsilon)$ must lie between $\frac{E_a}{\ln \dD}$ and $\frac{E_b}{\ln \dD}$. Outside this sub-extensive ($O(\frac{1}{\ln \dD})$) region, the derivative of entropy has a constant sign. Since $P(\ln \dD \epsilon) = e^{- \ln\dD I(\epsilon)}$ and normalization implies $\lim_{E \rightarrow - \infty} P(E) = 0$, $\lim_{\epsilon \rightarrow -\infty} I(\epsilon) = \infty$. Therefore, $I'(\epsilon)$ must be negative and $s_{ann, MC}'(\epsilon)$ must be positive below $\frac{E_a}{\ln \dD}$. By similar logic, we conclude that $s_{ann, MC}'(\epsilon)$ is negative above $\frac{E_b}{\ln \dD}$. Since $\beta > 0$ and $\epsilon = O(1)$, we conclude that $\epsilon_{ann}(\beta) < \frac{E_a}{\ln \dD}$ and $\epsilon(\beta) < \frac{E_a}{\ln \dD}$ for all beta. Since the derivative is monotonic in this region, there can only be one solution to the equation $\frac{\partial s_{ann, MC}}{\partial \epsilon} = \beta$ which $\epsilon_{ann}$ must satify. Therefore, if $\epsilon(\beta) \neq \epsilon_{ann}(\beta)$, then $\frac{\partial s_{ann, MC}}{\partial \epsilon}(\epsilon(\beta)) \neq \beta$. This implies $\epsilon(\beta)$ cannot lie at a local maximum of the optimization problem and therefore must lie on the boundary where $s_{ann,MC}(\epsilon(\beta)) = 0$. So, we see
\begin{equation*}
    \epsilon(\beta) \neq \epsilon_{ann}(\beta) \implies s_{ann,MC}(\epsilon(\beta)) = 0
\end{equation*}
We can combine these results with equation \ref{eq:quench_to_ann_entropy} to get the entropy.
\begin{equation}
    s(\beta | \mathcal{C}_{x,t})  = \begin{cases}
        s_{ann, MC}(\epsilon_{ann}(\beta | \cC) | \mathcal{C}_{x,t}) & s_{ann, MC}(\epsilon_{ann}(\beta| \cC) | \mathcal{C}_{x,t}) > 0 \\
        0 & s_{ann, MC}(\epsilon_{ann}(\beta| \cC)| \mathcal{C}_{x,t}) \leq 0
    \end{cases}
\end{equation}
So, we can see this thermodynamic model has two different phases. One where the Gibbs distribution is spread over many states and $s_{ann, MC}(\epsilon_{ann}(\beta) | \mathcal{C}_{x,t}) > 0$ and one where all the probability concentrates on a single state and $s_{ann, MC}(\epsilon_{ann}(\beta) | \mathcal{C}_{x,t}) \leq 0$. These correspond exactly to the generalizing and the memorizing phase described in the main text. Note that simply calculating $s_{ann, MC}(\epsilon(\beta) | \mathcal{C}_{x,t})$ for a given temperature and $\cC$ is enough to determine which phase the model is in. Work previously done by \cite{biroli2024dynamical} explicitly solved for $s(\epsilon)$ and $\epsilon(\beta)$ in the case of isotropic Gaussian data.

However, stepping back from these saddle-point equations and examining the context of the problem allows us to work in more generality. One can generically write the canonical annealed canonical entropy in terms of the temperature using the annealed free energy and annealed energy


\begin{align}
    S_{ann}(\beta | \mathcal{C}_{x,t}) &= -\beta F_{ann}(\beta) + \beta E_{ann}(\beta) = \ln \mathbb{E}[Z| \mathcal{C}_{x,t}] + \beta \frac{\mathbb{E}[E Z| \mathcal{C}_{x,t}]}{\mathbb{E}[Z| \mathcal{C}_{x,t}]} \\
    &= \ln |\mathcal{D}| + \ln \mathbb{E}_{E}[e^{-\beta E}| \mathcal{C}_{x,t}] + \beta \frac{\mathbb{E}_{E}[E e^{-\beta E}| \mathcal{C}_{x,t}]}{\mathbb{E}_{E}[e^{-\beta E}| \mathcal{C}_{x,t}]}
\end{align}
Usually, evaluating this expression would require knowing the full moment generating function for the $E$ distribution, which would be as difficult as solving for $s(\epsilon)$. However, we are not interested in the behavior of our model at arbitrary temperatures. Recall that equation \ref{eq:energy_def} defines a specific temperature for the model.
\begin{align}
    \mathbb{E}_{E}[e^{-\beta E} | \mathcal{C}_{x,t}] &= \int d\varphi \: P_0(\varphi) e^{-\beta E(\varphi | \mathcal{C}_{x,t})} \\
    &= \int d\varphi \: P_0(\varphi) P(\mathcal{C}_{x,t} | \varphi) = P_{test}(\mathcal{C}_{x,t})
\end{align}
So, we see that at the temperature given by \ref{eq:energy_def}, the moment generating function of the energy distribution is just the test set probability of $\mathcal{C}_{x,t}$. Similarly, for the energy expectation, we find
\begin{align}
    \frac{\beta \mathbb{E}_{E}[E e^{-\beta E}| \mathcal{C}_{x,t}]}{\mathbb{E}_{E}[e^{-\beta E}| \mathcal{C}_{x,t}]} &= \int d\varphi \: \frac{P_0(\varphi)}{P_{test}(\mathcal{C}_{x,t})} \beta E(\varphi | \mathcal{C}_{x,t}) e^{-\beta E(\varphi | \mathcal{C}_{x,t})} \\
    &= \int d\varphi \: \frac{P_0(\varphi)}{P_{test}(\mathcal{C}_{x,t})} P(\mathcal{C}_{x,t} | \varphi) \: (- \ln P(\mathcal{C}_{x,t} | \varphi)) \\
    &= \int d\varphi \:P_{test} (\varphi | \mathcal{C}_{x,t})\: (- \ln P(\mathcal{C}_{x,t} | \varphi))
\end{align}
Combining everything, we find
\begin{align*}
    \ln\dD \: \: s_{ann,MC}(\epsilon_{ann}(\beta) | \cC) &= S_{ann}(\beta | \mathcal{C}_{x,t}) \\
    &= \ln |\mathcal{D}| + \ln P_{test}(\mathcal{C}_{x,t}) + \int d\varphi \:P_{test} (\varphi | \mathcal{C}_{x,t})\: (- \ln P(\mathcal{C}_{x,t} | \varphi)) \\
    &= \ln |\mathcal{D}| + \int d\varphi \:P_{test} (\varphi | \mathcal{C}_{x,t})\: (- \ln \frac{P_{test}(\varphi | \mathcal{C}_{x,t})}{P_0(\varphi)}) \\
    &= \ln |\mathcal{D}| - D_{KL}(P_{test}(\varphi | \mathcal{C}_{x,t}) || P_0(\varphi))
\end{align*}
Plugging this into the formula for the actual (quenched) entropy,
\begin{equation}
    S(\beta | \mathcal{C}_{x,t}) = \max\Big(0, \ln |\mathcal{D}| - D_{KL}(P_{test}(\varphi | \mathcal{C}_{x,t}) || P_0(\varphi))\Big)
\end{equation}
From this, we can conclude the pointwise collapse condition.
\end{proof}

This pointwise condition is useful for situations where $\mathcal{C}_{x,t}$ is not drawn from a known distribution. So, long as it is drawn independently from $\mathcal{D}$, we can calculate the posterior entropy. If $\mathcal{C}_{x,t}$ is drawn from the test set, the expression further simplifies:
\begin{theorem}[Test Set Averaged Posterior Entropy]\label{thm:average_memorization_condition}
    If $\mathcal{D}$ is a set of training images drawn i.i.d. from $P_0$ and $\mathcal{C}$ is a channel  such that $P(E | \mathcal{C}_{x,t})$ has sub-exponential tails and $\sqrt{Var(D_{KL}(P_{test}(\varphi | \mathcal{C}_{x,t}) || P_0(\varphi)))} = o(1)$, then when a noisy image is drawn from the test set
    \begin{equation}
        \mathbb{E}_{\mathcal{C}_{x,t} \sim P_{test}}[S[P_{train}(\varphi|\mathcal{C}_{x,t})]] = \begin{cases}
            \ln |\mathcal{D}| - I(\varphi; \mathcal{C}_{x,t}) & \ln |\mathcal{D}| > I(\varphi; \mathcal{C}_{x,t}) \\
            0 & \ln |\mathcal{D}| \leq I(\varphi; \mathcal{C}_{x,t})
        \end{cases}
    \end{equation}
    up to $o(1)$ corrections.
\end{theorem}
\begin{proof}
    Since the annealed entropy is assumed to have sub extensive standard deviation, we can replace it with its mean:
    \begin{align}
        \mathbb{E}_{\mathcal{C}_{x,t} \sim P_{test}} [S_{ann}(\varphi | \mathcal{C}_{x,t})] &= \ln \dD - \mathbb{E}_{\mathcal{C}_{x,t} \sim P_{test}}[D_{KL}(P_{test}(\varphi | \mathcal{C}_{x,t}) || P_0(\varphi))]\\
        &\hspace{-60pt}= \ln \dD + \int d\varphi d\mathcal{C}_{x,t} \:P_{test}(C_{x,t}) P (\varphi | \mathcal{C}_{x,t})\: (- \ln \frac{P_{test}(\varphi | \mathcal{C}_{x,t})}{P_0(\varphi)})\\
        &\hspace{-60pt}=  \ln |\mathcal{D}| + \int d\varphi d\mathcal{C}_{x,t} \:P_{test} (\varphi, \mathcal{C}_{x,t})\: (- \ln \frac{P_{test}(\varphi, \mathcal{C}_{x,t})}{P_{test}(\mathcal{C}_{x,t})P_0(\varphi)}) \\
        &\hspace{-60pt}= \ln |\mathcal{D}| - I(\varphi; \mathcal{C}_{x,t})
    \end{align}
\end{proof}
\begin{theorem}[Training Set Posterior Entropy]
    If $\mathcal{D}/\varphi_0$ is a set of training images drawn i.i.d. from $P_0$ and $\mathcal{C}$ is a channel such that $P(E | \mathcal{C}_{x,t})$ satisfies the conditions in theorem \ref{thm:pointwise_memorization}, $\sqrt{Var(D_{KL}(P_{test}(\varphi | \mathcal{C}_{x,t}) || P_0(\varphi)))} = o(1)$, and $Var(\ln \frac{P_{test}(\cC, \varphi)}{P_{test}(\cC) P_0(\varphi)}) = o(1)$, then when a noisy image is drawn from the train set
    \begin{equation}
        \mathbb{E}_{\mathcal{C}_{x,t} \sim P_{train}}[S[P_{train}(\varphi|\mathcal{C}_{x,t})]] = \begin{cases}
            \ln |\mathcal{D}| - I(\varphi; \mathcal{C}_{x,t}) & \ln |\mathcal{D}| > I(\varphi; \mathcal{C}_{x,t}) \\
            0 & \ln |\mathcal{D}| \leq I(\varphi; \mathcal{C}_{x,t})
        \end{cases}
    \end{equation}
    up to $o(1)$ corrections.
\end{theorem}

\begin{proof}
    If $\mathcal{C}_{x,t}$ is not drawn independently from $\mathcal{D}$, for example if $\mathcal{C}_{x,t}$ is generated from the training set, there is less that we can say. This violates the i.i.d. assumption imposed on the $\mathcal{D}$ training set in theorem \ref{thm:pointwise_memorization}. A training image $\varphi_0 \in \mathcal{D}$ is first chosen uniformly and then $\mathcal{C}_{x,t}$ is drawn from $P(\mathcal{C}_{x,t} | \varphi)$. This leads to additional contribution to the annealed partition function
\begin{equation}
    Z_{train} = P(\mathcal{C}_{x,t}| \varphi_0) + \sum_{\varphi \in \mathcal{D}/\{\varphi_0\}} P(\mathcal{C}_{x,t} | \varphi)
\end{equation}
Each term in the second sum is identically distributed and their contribution from the second term is identical to the test partition function with $|\mathcal{D}| - 1$ samples. However, this $P(\mathcal{C}_{x,t}| \varphi_0)$ term must be specifically accounted for.
In the error correction context, this quantity is usually taken to be much smaller than $\ln |\mathcal{D}|$ and therefore ignored. If we don't ignore it, it can lead to the collapse transition happening faster.
\begin{align}
    Z_{train} &= P(\mathcal{C}_{x,t}| \varphi_0) + \max((|\mathcal{D}| - 1) P_{test}(\mathcal{C}_{x,t}), e^{-\ln \dD \beta \epsilon^*}) \\
\end{align}
where $\epsilon^*$ is the intensive energy such that $I(\epsilon^* | \mathcal{C}_{x,t}) = 1$. We must assume that $\ln P(\mathcal{C}_{x,t} | \varphi_0)/ \ln \dD$  concentrates to some deterministic $O(1)$ value when $\varphi_0$ is drawn from $P_0$ and $\cC$ from $P(\cdot | \phi_0)$. Because the two terms are exponentially small in $\ln\dD$, $Z_{train}$ is always only dominated by one of the two terms. If $ \frac{\ln P(\mathcal{C}_{x,t} | \varphi_0)}{\beta \ln \dD} > \epsilon^*$, then there are exponentially many other states around the energy $\frac{\ln P(\mathcal{C}_{x,t} | \varphi_0)}{\beta \ln \dD}$ and $Z_{train} = Z_{test} + O(1)$ we conclude the theorem by applying theorem \ref{thm:average_memorization_condition}. If $ \frac{\ln P(\mathcal{C}_{x,t} | \varphi_0)}{\beta \ln \dD} < \epsilon^*$ then,
\begin{equation*}
    -\beta F(\beta) = \ln Z = \begin{cases}
        \ln (\dD -1 ) + \ln P_{test}(\mathcal{C}_{x,t}) & \ln (\dD -1 ) + \ln P_{test}(\mathcal{C}_{x,t}) >  \ln P(\mathcal{C}_{x,t} | \varphi_0)\\
        \ln P(\mathcal{C}_{x,t} | \varphi_0) & \ln (\dD -1 ) + \ln P_{test}(\mathcal{C}_{x,t}) \leq \ln P(\mathcal{C}_{x,t} | \varphi_0)
    \end{cases}
\end{equation*}
In the collapsed phase, we know the energy associated with the training sample that generated $\cC$ is $\frac{\ln P(\mathcal{C}_{x,t} | \varphi_0)}{\beta \ln \dD}$. In the generalizing phase, we can use the conclusion of theorem \ref{thm:pointwise_memorization}  since $Z_{train} = Z_{test} + O(e^{-\ln |D|})$.
\begin{equation*}
    S(\beta) = \begin{cases}
        \ln (\dD -1 ) - D_{KL}(P_{test}(\varphi | \mathcal{C}_{x,t}) || P_0(\varphi)) & \ln (\dD -1 ) + \ln P_{test}(\mathcal{C}_{x,t}) >  \ln P(\mathcal{C}_{x,t} | \varphi_0)\\
        0 & \ln (\dD -1 ) + \ln P_{test}(\mathcal{C}_{x,t}) \leq \ln P(\mathcal{C}_{x,t} | \varphi_0)
    \end{cases}
\end{equation*}
Now using Bayes rule $P(\mathcal{C}_{x,t} | \varphi_0) = \frac{P_{test}(\mathcal{C}_{x,t}, \varphi_0)}{P_0(\varphi_0)}$
\begin{equation*}
    S[P_{train}(\varphi|\mathcal{C}_{x,t})] = \begin{cases}
        \ln (\dD -1 ) - D_{KL}(P_{test}(\varphi | \mathcal{C}_{x,t}) || P_0(\varphi)) & \ln (\dD -1 ) > \ln \frac{P_{test}(\mathcal{C}_{x,t} , \varphi_0)}{P_{test}(\mathcal{C}_{x,t})P_0(\varphi_0)} \\
        0 & \ln (\dD -1 ) \leq \ln \frac{P_{test}(\mathcal{C}_{x,t} , \varphi_0)}{P_{test}(\mathcal{C}_{x,t})P_0(\varphi_0)}
    \end{cases}
\end{equation*}
Now we note that $\ln \frac{P_{test}(\mathcal{C}_{x,t} , \varphi_0)}{P_{test}(\mathcal{C}_{x,t})P_0(\varphi)}$ is the observable averaged in the expression for the mutual information. 
\begin{equation}
    I(\mathcal{C}_{x,t} , \varphi_0) = \int d\varphi d\cC \: P_{test}(\mathcal{C}_{x,t} , \varphi_0) \ln \frac{P_{test}(\mathcal{C}_{x,t} , \varphi_0)}{P_{test}(\mathcal{C}_{x,t})P_0(\varphi)}
\end{equation}
Therefore, if we assume that $\ln \frac{P_{test}(\mathcal{C}_{x,t} , \varphi_0)}{P_{test}(\mathcal{C}_{x,t})P_0(\varphi_0)}$ also concentrates, then
\begin{equation}
    S[P_{train}(\varphi|\mathcal{C}_{x,t})] = \begin{cases}
        \ln (\dD -1 ) - D_{KL}(P_{test}(\varphi | \mathcal{C}_{x,t}) || P_0(\varphi)) & \ln (\dD -1 ) > I(\mathcal{C}_{x,t}; \varphi) \\
        0 & \ln (\dD -1 ) \leq I(\mathcal{C}_{x,t}; \varphi)
    \end{cases}
\end{equation}
So, we see that the posterior entropy for a noisy training image is essentially the same when appropriate strong assumptions are made about $P(\mathcal{C}_{x,t} | \varphi)$.
\end{proof}

\subsection{Relationship to Biroli et al.'s Collapse Condition}\label{app:relation_to_biroli}
In \cite{biroli2024dynamical} the memorization to generalization phase transition for the empirical score function was studied using the random energy model. Our work with the random energy model in section \ref{app:main_proof} was inspired by theirs. The empirical score function is, in some sense, the simplest BIRD model in which all pixels see the entire noisy image. The Markov chain for the observation is just a white noise channel
\begin{equation}
    \varphi \rightarrow \phi := \sqrt{\bar{\alpha}_t} \varphi + \sqrt{1 - \bar{\alpha}_t} \eta
\end{equation}
They were able to solve for the entropy of the generated distribution and memorization time in the large dimension limit exactly in the case of isotropic Gaussian data by directly calculating the micro-canonical entropy. They conjectured that for the memorization time $t^*$, the identity
\begin{equation}\label{eq:biroli_condition}
    S[P(\phi_{t^*})] = S_{sep}^{|\mathcal{D}|}
\end{equation}
would hold for data $\varphi$ from arbitrary distributions, where $S_{sep}^{|\mathcal{D}|}$ is the entropy of $|\mathcal{D}|$ well-separated isotropic Gaussians with variance $(1 - \bar{\alpha}_t) I$:
\begin{equation}
    S_{sep}^{|\mathcal{D}|} = \ln \dD + S[\mathcal{N}(\eta | 0, (1 - \bar{\alpha}_t) I )] = \ln \dD + \frac{d}{2}(1 + \ln(2 \pi (1 - \bar{\alpha}_t)))
\end{equation}
Using the mutual information formula for the white noise channel with this parametrization (theorem \ref{thm:white_noise_MI})
\begin{equation}
    \ln |D| = S[P(\phi_{t^*})] - S[\mathcal{N}(\eta | 0, (1 - \bar{\alpha}_t) I )] = I(\phi_{t^*}; \varphi)
\end{equation}
which corresponds to the collapse condition (\ref{eq:memorization}) for the trivial channel $\mathcal{C}_{x,t}(\phi_t) = \phi_t$. Thus, our analysis proves Biroli's conjecture using random energy model methods, but also significantly generalizes their notions to cases where the channel $\mathcal{C}_{x,t}$ is potentially nonlinear or stochastic, for which their conjectured identity (\ref{eq:biroli_condition}) no longer serves as the correct collapse condition.


\section{Useful Mutual Information Bounds and Calculations}\label{app:mutual_information_calculations}

\subsection{Memorization implies a lower bound on mutual information}\label{app:fanos_bound}

Here we will use Fano's inequality to prove a simple lower bound on the mutual information $I(\mathcal{C}_{x,t}; \varphi)$ whenever a BIRD model is in the memorization phase. 

Consider again the mutual information $I(\mathcal{C}_{x,t}; \varphi)$ between a pixel $x$'s current observation $\mathcal C_{x,t}$ at time $t$ and the past training image $\varphi \in \mathcal D$ at time $t=0$ that led to the observation under the forward diffusion.   In BIRD models, the pixel $x$ attempts to play a reverse time Bayesian guessing game to guess the correct past training image $\varphi_c \in \mathcal D$ that led to $\mathcal C_{x,t}$. Its knowledge about the past training data is encapsulated in the posterior distribution $P(\varphi | \cC)$ for $\varphi \in \mathcal D$. The probability the pixel guesses the correct training point $\varphi_c$ is $P(\varphi_c | \cC)$. Thus the probability a pixel $x$ loses the Bayesian guessing game by making an incorrect guess or error is $p_e = 1 - P(\varphi_c | \cC)$.  The pixel $x$ has then memorized the training data, by definition, if $p_e$ is close to $0$ and the posterior $P(\varphi | \cC)$ concentrates on the correct data point $\varphi_c$.  

We can think of this process as an error correction problem where the $|\mathcal D|$ training samples $\varphi \in \mathcal D$ can be thought of as $|\mathcal D|$ codewords.  
One of these codewords $\varphi_c$ is sent through the forward diffusion plus observation communication channel $\varphi \rightarrow \phi_t \rightarrow \mathcal C_{x,t}$, yielding the channel output $\cC$. 
Then the pixel $x$ must guess the correct input codeword $\varphi_c$ based on this output.  
Memorization then corresponds to successful error correction in which $p_e \rightarrow 0$ (which is deleterious though for BIRD models as it leads to the failure to generalize).   

With this error correction analogy in place, we can now apply Fano's inequality (see \cite{cover2006elementsofinfo} for a standard proof) which provides a lower bound on the error probability $p_e$ in terms of the conditional entropy $S(\varphi|\cC)$.  Applying Fano's inequality yields 
\begin{align*}
    H_{b}(p_e) + p_e \ln(|\mathcal{D}| - 1) &\geq S(\varphi | \mathcal{C}_{x,t}) \\
    &= S(\varphi) - I(C_{t,x}; \varphi) \\
    &= \ln \dD - I(C_{t,x}; \varphi).
\end{align*}
Here $H_{b}(p_e)$ is the entropy of a binary random variable with probability parameter $p_e$. We next take the memorization limit $p_e \rightarrow 0$. This makes the left hand side $0$, yielding a lower bound on mutual information in the memorization phase:
\begin{equation}
    I(\mathcal{C}_{t,x}; \varphi) \geq \ln |\mathcal{D}|. \label{eq:fano_lower_bound}
\end{equation}

Importantly, this analysis only proves that in the memorization phase, corresponding to $p_e \rightarrow 0$, we have $I(\mathcal{C}_{t,x}; \varphi) \geq \ln |\mathcal{D}|$. It does {\it not} prove that in the generalizing phase, when $p_e$ is bounded away from $0$, that the opposite holds: $I(\mathcal{C}_{t,x}; \varphi) < \ln |\mathcal{D}|$. One would ideally like to have this latter inequality also in the generalizing phase to prove that the boundary between the generalizing phase and memorizing phase occurs exactly when $I(\mathcal{C}_{t,x}; \varphi) = \ln |\mathcal{D}|$. Indeed, we are able to prove this result in App. \ref{app:main_proof}, when $\cC$ is drawn from the testing Markov process.

\subsection{The Gaussian Bound on Mutual Information for the Additive White Noise Channel}\label{app:gaussian_bound}
In the prior sections, we keep the description of Bayes optimal denoisers as generic as possible, but calculating bounds on the channel mutual information requires specificity on the type of channel that we are considering. In the case of the ideal score function, the local score function, the equivariant local score function, the channels for each agent or pixel can be viewed as the composition of two channels. First, a non-invertible, possibly random linear map applied to a noiseless image, and second, rescaling and adding white noise. For example, in the case of the local score machine,
\begin{equation}
    \mathcal{C}_{x,t}: \varphi \rightarrow \varphi_{\Omega_x} \rightarrow \sqrt{\bar{\alpha}_t} \varphi_{\Omega_x} + \sqrt{1 - \bar{\alpha}_t} \eta_{\Omega_x}
\end{equation}
where $\eta$ is an isotropic Gaussian. Note that this is the opposite of the order which these two channels are applied in the actual computation. We actually apply a projection to the noisy image. However, because the white noise at each pixel is independent and the projection map is linear we can view the two as happening in either order. $P_{\Omega_x}: \RR^{H\times W \times C} \rightarrow \RR^{k \times k \times C}$
\begin{equation}
    \mathcal{C}_{x,t} = P_{\Omega_x}(\sqrt{\bar{\alpha}}_t \varphi + \sqrt{1- \bar{\alpha}_t} \eta_1) = \sqrt{\bar{\alpha}}_t P_{\Omega_x} \varphi + \sqrt{1 - \bar{\alpha}_t} \eta_2
\end{equation}
where $\eta_1 \sim \NN(0, I_{H\times W \times C})$ and $\eta_1 \sim \NN(0, I_{k\times k \times C})$. By the data processing inequality,
\begin{equation}
    I(\mathcal{C}_{x,t}; \varphi) \leq I(\mathcal{C}_{x,t}; P_{\Omega_x} \varphi)
\end{equation}
In other words, because $\cC$ only depends on $\varphi$ through its dependence on $P_{\Omega_x} \varphi$, $\cC$ cannot have more information about $\varphi$ than was contained in $P_{\Omega_x} \varphi$. For ease of notation in the following sections, I'll define $\Phi_x$ as the noise free version of the observation $\cC$.
\begin{equation}
    \Phi := P_{\Omega_x} \varphi
\end{equation}


Therefore, we can understand most Bayes optimal diffusion models just by understanding the mutual information of the additive white noise channel. Let $\Phi$ be some random variable inserted into an additive white noise channel.
\begin{align}
    I(\mathcal{C}_{x,t} := \sqrt{\bar{\alpha}_t} \Phi + \sqrt{1 - \bar{\alpha}_t} \eta; \Phi) &= S(\mathcal{C}_{x,t}) + S(\Phi) - S(\mathcal{C}_{x,t}, \Phi)\\
    &= S(\mathcal{C}_{x,t}) - S(\mathcal{C}_{x,t} | \Phi) \\
    &= S(\mathcal{C}_{x,t}) - S(\sqrt{\bar{\alpha}_t} \Phi + \sqrt{1 - \bar{\alpha}_t} \eta | \Phi)
\end{align}
After conditioning on $\Phi$, $\sqrt{\bar{\alpha}_t} \Phi$ is just a deterministic shift and $\eta$ is independent. So, this simplifies to:
\begin{align}
    I(\mathcal{C}_{x,t}; \Phi) &= S(\mathcal{C}_{x,t}) - S(\sqrt{1 - \bar{\alpha}_t}\eta) \\
    &= S(\mathcal{C}_{x,t}) + \int d\eta \mathcal{N}(\eta | 0, (1-\bar{\alpha}_t) I) \ln \mathcal{N}(\eta | 0, (1-\bar{\alpha}_t) I)\\
    &= S(\mathcal{C}_{x,t}) - \int d\eta \Big( \frac{1- \bar{\alpha}_t}{2} ||\eta||^2 + \frac{d}{2} \ln (2\pi(1-\bar{\alpha}_t)) \Big) \frac{\exp(-\frac{1- \bar{\alpha}_t}{2} ||\eta||^2)}{\sqrt{(2\pi (1-\bar{\alpha}_t))^d}}\\
    &=S(\mathcal{C}_{x,t}) - \frac{d}{2}(1 +  \ln(2 \pi (1-\bar{\alpha}_t))) 
\end{align}

This recovers the well known theorem describing the Mutual information between the input and output of an additive Gaussian white noise channel: 
\begin{theorem}\label{thm:white_noise_MI}
    For the additive white noise channel with input $\Phi$ and output $\mathcal{C}_{x,t} := \sqrt{\bar{\alpha}_t} \Phi + \sqrt{1 - \bar{\alpha}_t} \eta$ with $\eta \sim \mathcal{N}(0, I)$, 
    \begin{equation}
        I(\mathcal{C}_{x,t}; \Phi) = S(\mathcal{C}_{x,t}) - \frac{d}{2}(1 +  \ln(2 \pi (1-\bar{\alpha}_t)))
    \end{equation}
    where $d$ is the data dimension or the number of pixels.
\end{theorem}


As we can see, the mutual information of the additive white noise channel only depends on the input, or data distribution through the entropy of the noisy distribution $\phi$. Therefore, we can bound the mutual information of the channel by choosing a maximal entropy distribution for $\phi$. Under the constraint that the data covariance matrix is $\Sigma_0$ and its distribution is centered, the maximal entropy distribution of $\phi$ is Gaussian with covariance $\Sigma_t = \bar{\alpha}_t \Sigma + (1 - \bar{\alpha}_t)I$. Let $\tilde{\phi} \sim \mathcal{N}(0,\Sigma_t)$ be the Gaussianized version of the data.
\begin{align*}
    S(\tilde{\phi}) &= -\int d \tilde{\phi} \mathcal{N}(\tilde{\phi} | 0,\Sigma_t) \ln \mathcal{N}(\tilde{\phi} | 0,\Sigma_t) \\
   &= \frac{1}{2}\ln \det(2 \pi \Sigma_t) + \int d \tilde{\phi} \frac{\tilde{\phi}^{T} \Sigma_t^{-1} \phi}{2} \frac{\exp(-\tilde{\phi}^{T} \Sigma_t^{-1} \phi / 2)}{\sqrt{\det(2 \pi \Sigma_t)}} \\
   &= \frac{1}{2}\mathbf{Tr}{\ln(2 \pi \Sigma_t)} + \frac{1}{2} \mathbf{Tr}[\Sigma_t^{-1} \Sigma_t] \\
   &= \frac{d}{2} + \frac{1}{2}\mathbf{Tr}{\ln(2 \pi \Sigma_t)}
\end{align*}
Now using this to find the mutual information
\begin{align}
    I(\tilde{\phi}; \tilde{\Phi}) &= S(\tilde{\phi}) - \frac{d}{2}(1 +  \ln(2 \pi (1-\bar{\alpha}_t))) \\
    &= \frac{1}{2}\mathbf{Tr}{\ln(\frac{\Sigma_t}{1 - \bar{\alpha}_t})} \\
    &= \frac{1}{2}\mathbf{Tr}{\ln(\frac{\bar{\alpha}_t\Sigma + (1-\bar{\alpha}_t)I}{1 - \bar{\alpha}_t})} \\
    &= \frac{1}{2}\mathbf{Tr}{\ln(I + \frac{\bar{\alpha}_t}{1 - \bar{\alpha}_t} \Sigma)}\label{eq:gaussian_mi}
\end{align}
Therefore, we can conclude theorem
\begin{theorem}\label{thm:gaussian_bound}
    Given a patch $\Phi$ with covariance $\Sigma$ the mutual information of the additive white noise channel, $\phi := \sqrt{\bar{\alpha}_t} \Phi + \sqrt{1 - \bar{\alpha}_t} \eta$, is bounded above by
    \begin{equation}
        I(\phi; \Phi) \leq \frac{1}{2} \mathbf{Tr} \ln[I + \frac{\bar{\alpha}_t}{1 - \bar{\alpha}_t} \Sigma].
    \end{equation}
\end{theorem}


\subsection{Asymptotic Expansion for the Mutual Information of the AWNC at High Noise}\label{app:asymptotic_expansion}
In this section we will again consider a simple white noise channel. In section \ref{app:gaussian_bound}, we argued that the Gaussianized distribution provides an upper bound on the mutual information. Here, we provide a universal expansion for the mutual information of the white noise channel regardless of the underlying distribution. This helps give intuition for the surprisingly good agreement between the Gaussian prediction and the posterior entropies measured on real data sets. Recall for the additive white noise channel:
\begin{equation}
    I(\varphi; \phi) = S[P_{test}(\phi)] - S[\mathcal{N}(0, (1-\bar{\alpha}_t)I)].
\end{equation}
We can define the cumulant generating function for the test distribution as follows:
\begin{equation}
    K_{\phi}(\hat{\phi}) = \ln \EE_{\phi \sim P_{test}}[\exp(\hat{\phi} \cdot \phi)]
\end{equation}
The formal Taylor expansion for this function provides the cumulants of the distribution of $\phi$. In addition, the cumulant generating function with an imaginary argument can be used to entirely reconstruct the distribution. Because the cumulant generating function of the sum of two independent random variables is just the sum of each variable's cumulant generating function, we have
\begin{align}
    K_{\phi}(\hat{\phi}) &= \ln \EE_{\varphi \sim P_{0}, \eta \sim \mathcal{N}(0,I)}[\exp(\hat{\phi} \cdot (\sqrt{\bar{\alpha}_t} \varphi + \sqrt{1-\bar{\alpha}_t} \eta))] \\
    &= \ln  \EE_{\varphi \sim P_{0}}[\exp(\hat{\phi} \cdot (\sqrt{\bar{\alpha}_t} \varphi))] + \ln \EE_{\eta \sim \mathcal{N}(0,I)}[\exp(\sqrt{1-\alpha}_t \hat{\phi} \cdot \eta))] \\
    &= K_{\varphi}(\sqrt{\bar{\alpha}_t} \hat{\phi}) + K_{\eta}(\sqrt{1 - \bar{\alpha}_t} \hat{\phi}) \\
    &= K_{\varphi}(\sqrt{\bar{\alpha}_t} \hat{\phi}) + \frac{1 - \bar{\alpha}_t}{2} ||\hat{\phi}||^2.
\end{align}
The cumulant generating function of the isotropic Gaussian is just a quadratic. In addition, $\sqrt{\bar{\alpha}_t}$ ranges from 0 to 1, so one would expect that the low order terms in $\sqrt{\bar{\alpha}_t}$ describe the behavior at most time steps. This means that only the lowest order cumulants of the data distribution have a significant effect on the cumulants of the generated distribution. By simply representing the entropy in terms of the cumulant generating function, we should be able to expand the expression to lowest order terms and find a cumulant expansion for the mutual information.  We start with the expression for the distributio in terms of its CGF:
\begin{align}
    P_{test}(\phi) &= \int \frac{d\hat{\phi}}{(2 \pi)^d} \exp(-i \hat{\phi} \cdot \phi + K_{\phi}(i\hat{\phi})). \label{eq:distribution_in_term_of_CGF}
\end{align}
Taking derivatives with respect to $\phi$ has the effect of adding factors of $-i \hat{\phi}$ to the integrand in equation \eqref{eq:distribution_in_term_of_CGF}.
\begin{align}
    P_{test}(\phi) &= \int \frac{d\hat{\phi}}{(2 \pi)^d} \exp(-i \hat{\phi} \cdot \phi + K_{\phi}(i\hat{\phi})) \\
    &= \int \frac{d\hat{\phi}}{(2 \pi)^d} \exp(K_{\varphi}(i\sqrt{\bar{\alpha}_t}\hat{\phi})) \exp(-i \hat{\phi} \cdot \phi - \frac{1 - \bar{\alpha}_t}{2} ||\hat{\phi}||^2) \\
    &= \exp(K_{\varphi}(- \sqrt{\bar{\alpha}_t}\partial_\phi)) \int \frac{d\hat{\phi}}{(2 \pi)^d}  \exp(-i \hat{\phi} \cdot \phi - \frac{1 - \bar{\alpha}_t}{2} ||\hat{\phi}||^2) \\
    &= \exp(K_{\varphi}(- \sqrt{\bar{\alpha}_t} \partial_\phi)) \frac{\exp(-\frac{||\phi||^2}{2(1-\bar{\alpha}_t)})}{\sqrt{\det(2 \pi \Sigma_t)}} \label{eq:asymtotic_gen_for_distribution}
\end{align}
In the last line, 
\begin{align}
    S[P_{test}(\phi)] &= -\int d\phi \: P_{test}(\phi) \ln P_{test}(\phi) \\
    &\hspace{-20pt}= -\int d\phi \:\exp(\tilde{K}_{\varphi}(- \sqrt{\bar{\alpha}_t} \partial_\phi)) \mathcal{N}(\phi | 0, (1 - \bar{\alpha}_t)I) \\
    &\hspace{20pt} \times \ln \exp(\tilde{K}_{\varphi}(- \sqrt{\bar{\alpha}_t} \partial_\phi)) \mathcal{N}(\phi | 0, (1 - \bar{\alpha}_t)I) \label{eq:alpha_asymmtotic_generating_expression}
\end{align}
Expanding equation \eqref{eq:alpha_asymmtotic_generating_expression} in powers of $\bar{\alpha}_t$ will provide an expansion of the entropy in terms of higher cumulants of the distribution and expectations of Gaussian random variables. For example, the first order term in $\bar{\alpha}_t$ is
\begin{align}
    S[P_{test}(\phi)] &= \frac{d}{2}(1 + \ln(2\pi (1-\bar{\alpha}_t))) + \frac{1}{2}\Big(\frac{\bar{\alpha}_t}{1 - \bar{\alpha}_t}\Big) \mathbf{Tr}[\Sigma] \\
    &- \frac{1}{4}\Big(\frac{\bar{\alpha}_t}{1 - \bar{\alpha}_t}\Big)^2 (\mathbf{Tr}[\Sigma^2] + \frac{\kappa^{abcd}}{12}(\delta_{a,b}\delta_{c,d} + \delta_{a,c}\delta_{b,d} + \delta_{a,d}\delta_{b,c})) \\
    &+ O(\bar{\alpha}_t^{5/2})
\end{align}
where $\Sigma$ is the covariance matrix of the dataset and $\kappa_{abcd}$ is the four-cumulant of the data. We can convert this into a prediction for the mutual information of the white noise channel
\begin{align}
    I(\phi; \varphi) &=  S[P_{test}(\phi)] - S[\mathcal{N}(0, (1 - \bar{\alpha}_t)I)] \\
    &= \frac{1}{2}\Big(\frac{\bar{\alpha}_t}{1 - \bar{\alpha}_t}\Big) \mathbf{Tr}[\Sigma] \\
    &- \frac{1}{4}\Big(\frac{\bar{\alpha}_t}{1 - \bar{\alpha}_t}\Big)^2 (\mathbf{Tr}[\Sigma^2] + \frac{\kappa^{abcd}}{12}(\delta_{a,b}\delta_{c,d} + \delta_{a,c}\delta_{b,d} + \delta_{a,d}\delta_{b,c})) \\
    &+ O(\bar{\alpha}_t^{5/2})
\end{align}
Higher order terms of this expansion are quite tedious to calculate.
\subsection{A Data Model for Speciation with Exactly Solvable Mutual Information}\label{app:hierarchical_mixture_of_gaussians}
Imagine that the true image distribution is a mixture of Gaussians where the parameters are themselves random.
\begin{equation}
    P_0(\varphi) = \frac{1}{|M|}\sum_{\mu \in M} \mathcal{N}(\varphi | \mu, \Sigma_0)\label{eq:mix_gaussian_model}
\end{equation}
where the means $M$ are drawn from some other distribution $\mathcal{N}(0, \Sigma_M)$. If we assume that $|M|$ is large, we can exactly solve for the mutual information of this data passed through the white noise channel. This can be achieved by noting that $P_0(\varphi)$ itself is exactly the same as the training distribution created by $M$ samples from $\mathcal{N}(0, \Sigma_M)$ then passed through an additive noise channel $\mu + \eta$ with $\eta \sim \mathcal{N}(\varphi | \mu, \Sigma_0)$. However, the theory that we've built to this point would only allow you to predict the entropy of the posterior distribution, that is given $\varphi$ what is the probability that it came from the Gaussian centered at $\mu$. However, we can instead make use of our prediction for the partition function
\begin{equation}
    Z(\varphi) = \sum_{m \in M} \mathcal{N}(\varphi | \mu, \Sigma_0) = \begin{cases}
        |M| \mathcal{N}(\varphi | 0, \Sigma_0 + \Sigma_M) & \ln |M| > D_{KL}(P(\mu | \varphi) | P(\mu)) \\
        e^{-\beta E^*} & \ln |M| \leq D_{KL}(P(\mu | \varphi) | P(\mu))
    \end{cases}
\end{equation}
where $e^{-\beta E^*}$ does not concentrate, but behaves like M well separated Gaussian distributions. This predicts the probability density at $\varphi$ in the large $|M|$ limit. Leading to the counterintuitive result that $P_0(\varphi)$ concentrates to a value independent of $M$. This leads to
\begin{equation}
    S(\varphi) = \begin{cases}
        S[\mathcal{N}( \mu, \Sigma_0 + \Sigma_M)] & \ln |M| > I(\mu; \varphi) \\
        \ln|M| + S[\mathcal{N}(0, \Sigma_0)] & \ln |M| \leq I(\mu; \varphi)
    \end{cases}
\end{equation}
Adding in the white noise channel, this becomes
\begin{equation}
    S(\phi_t) = \begin{cases}
        S[\mathcal{N}( \mu, \Sigma_0 + \Sigma_M + \sigma_t^2 I)] & \ln |M| > I(\mu; \varphi) \\
        \ln|M| + S[\mathcal{N}(0, \Sigma_0 + \sigma_t^2 I)] & \ln |M| \leq I(\mu; \varphi)
    \end{cases}
\end{equation}
So, the mutual information is
\begin{equation}
    I(\phi_t; \varphi) = \begin{cases}
        S[\mathcal{N}( \mu, \Sigma_0 + \Sigma_M + \sigma_t^2 I)] - S[\mathcal{N}( 0,\sigma_t^2 I)] & \ln |M| > I(\mu; \varphi) \\
        \ln|M| + S[\mathcal{N}(0, \Sigma_0 + \sigma_t^2 I)] - S[\mathcal{N}( 0,\sigma_t^2 I)] & \ln |M| \leq I(\mu; \varphi)
    \end{cases}
\end{equation}
Plugging in the Gaussian entropy formula
\begin{equation}
    I(\phi_t; \varphi) = \begin{cases}
        Tr[\ln( I + \frac{1}{\sigma_t^2}(\Sigma_0 + \Sigma_M))] & \ln |M| > I(\mu; \varphi) \\
        \ln|M| + Tr[\ln( I + \frac{1}{\sigma_t^2}\Sigma_0)] & \ln |M| \leq I(\mu; \varphi)
    \end{cases}
\end{equation}
This can also be written
\begin{equation}
    I(\phi_t; \varphi) = \min(\ln|M| + Tr[\ln( I + \frac{1}{\sigma_t^2}\Sigma_0)], Tr[\ln( I + \frac{1}{\sigma_t^2}(\Sigma_0 + \Sigma_M))])
\end{equation}
This leads to the formula for the posterior entropy
\begin{equation}
    S[P(\varphi | \phi)] = \max(0, \ln\frac{\dD}{|M|} - Tr\Big[\ln( I + \frac{1}{\sigma_t^2}\Sigma_0)\Big], \ln \dD - Tr\Big[\ln( I + \frac{1}{\sigma_t^2}(\Sigma_0 + \Sigma_M))\Big])
\end{equation}
So, we see that the speciation transition has the exact same behavior as a memorization/generalization transition in terms of a discontinuity in the derivative of the posterior entropy with the noise level. You can see an example of this in sub-figure \ref{fig:posterior_entropy}b where the posterior entropy seems to plateau twice. More generally, we could imagine a Hierarchical Gaussian model where a set of means $M_1$ is drawn from $\mathcal{N}(0, \Sigma_{M_1})$ to generate a mixture of Gaussian distribution with covariance of each Gaussian $\Sigma_{M_2}$ from which I draw as set of means $M_2$, and so on until the final set of means $M_N$. In this case, we will have a posterior entropy like
\begin{equation}
    S[P(\varphi | \phi)] = \max(0, \max_{n \leq N} \ln \frac{\dD}{\Pi_{i=1}^n| M_i|} - Tr[\ln ( I + \frac{1}{\sigma_t^2}\Sigma_0 + \frac{1}{\sigma_t^2} \sum_{i={n+1}}^N \Sigma_{M_i})])
\end{equation}
where each change of the maximum $n$ corresponds to a speciation transition.
\section{Generalization in Local Score Models and the critical scale}\label{app:scaling}

As we have discussed above, Bayes-optimal diffusion models \textit{generalize} only when the channel through which they are observing their image lets pass fewer than $\log |\mathcal{D}|$ bits of information, preventing them from being able to resolve down to a single training set image. Increasing the noise level clearly decreases this mutual information, and extensive prior work has gone into studying the location of the memorization/generalization transition as a function of this parameter \citep{biroli2024dynamical, ventura2024manifolds}. Another parameter directly relevant to the capacity of the channel, however, is the \textit{locality scale} \citep{kambanalytic, niedoba2024towards, lukoianov2025locality}. When diffusion models have {locality biases}, they may only effectively observe a small proportion of the overall image. Under such constraints, the effective amount of information that they may in principle use to deduce the underlying image will be substantially \textit{lower} than the global information contained in the image. We thus ask the following question: where is the threshold of criticality, \textit{as a function of both the noise level and the locality scale}?

In this section, for conciseness, we will use the invariant noise-to-signal ratio to characterize the noise level at a time $t$:
\begin{align}
    \sigma_t^2 = \frac{1 - \bar{\alpha}_t}{\bar{\alpha}_t}
\end{align}
We will consider a family of patches $\Omega_{L,x}$ indexed uniquely by the parameter $L$, which satisfies $L = \sqrt{|\Omega_{L,x}|}$. For square patches \citep{kambanalytic} this quantity is simply the length of one side of the patch. We will demand that $\Omega_{L',x} \supset \Omega_{L,x}$ if $L' > L$. The \textit{critical scale} $L_{c}(\sigma_t)$ is then defined as the scale where a Local Score model using the patch $\Omega_{L,x}$ achieves the average collapse condition (\ref{eq:memorization_app}) at a noise level $\sigma_t$:
\begin{align}\label{eq:collapse_condition}
    \ln|\mathcal{D}| = I(\varphi ; \phi_{t,{\Omega_{x,L}}})
\end{align}
Without making any assumptions about the data-generating distribution, we are in a position to make a very general statement. Firstly, increasing the scale of a patch given a fixed noise scale {monotonically increases} the mutual information $I(\varphi ; \phi_{t;\Omega})$, while increasing the noise scale at a fixed patch size monotonically decreases this mutual information. It follows from implicit differentiation of the collapse condition (\ref{eq:collapse_condition}) with respect to $L$ that the values of $L_{c}(\sigma_t)$ must be {monotonically increasing with noise level}. This is phenomenologically consistent with the observed coarse-to-fine scaling of diffusion models \citep{kambanalytic, lukoianov2025locality, niedobatowards, dieleman2024spectral_autoregression}. We can formalize this into this theorem.
\begin{theorem}[Monotonicity of Length Scale]
    $L_{c}(\sigma_t)$ is a monotonically increasing function.
\end{theorem}
\begin{proof}
If $L < L'$, then $\Omega_{L',x} \supset \Omega_{L,x}$. This implies that the observation $C_{L,x,t}$ can be viewed as first restricting to the patch $\Omega_{L',x}$ and then restricting to the patch $\Omega_{L,x}$. This leads to the Markov chain
\begin{equation}
    \varphi \rightarrow \phi_t \rightarrow \phi_{t,{\Omega_{x,L'}}} \rightarrow \phi_{t,{\Omega_{x,L}}}
\end{equation}
The information processing inequality implies
\begin{equation}
    I(\varphi; \phi_{t,{\Omega_{x,L}}}) \leq I(\varphi; \phi_{t,{\Omega_{x,L'}}})
\end{equation}
Now we need to use the fact that for a projection we can switch the order of restricting the domain and adding the noise. For a positive increment $\Delta t$
\begin{equation}
    \varphi \rightarrow \varphi_{{\Omega_{x,L}}} \rightarrow \phi_{t, {\Omega_{x,L}}} \rightarrow \phi_{t+\Delta t, {\Omega_{x,L}}}
\end{equation}
This implies
\begin{equation}
    I(\varphi; \phi_{t + \Delta t,{\Omega_{x}}}) \leq I(\varphi; \phi_{t,{\Omega_{x, L}}})
\end{equation}
The collapse condition \eqref{eq:memorization_app} then implies the mutual information is held constant at the transition size $\phi_{t, {\Omega_{x, L(\sigma_t)}}}$. Since $\sigma_t$ is monotonically increasing in $t$, this implies
\begin{equation}
    |\Omega_{x,L_{c}(\sigma_t)}| \leq |\Omega_{x,L_{c}(\sigma_t + \Delta \sigma)}|
\end{equation}
\begin{equation}
    L_{c}(\sigma_t) \leq L_{c}(\sigma_t + \Delta \sigma_t)
\end{equation}
\end{proof}

\subsection{What scales do we \textit{need} vs. what scales do we \textit{have}: the spectral scale vs. the critical scale}\label{sec:spectral_scale}

Existing theories of locality \citep{lukoianov2025locality, dieleman2024spectral_autoregression} in diffusion models emphasize a quantity which we term the \textit{spectral scale}. Naturalistic images have a well-known power-law falloff in their power spectral density \citep{ruderman1993statistics}:
\begin{align}\label{eq:powerlaw}
    P(k) \approx A|k|^{-\alpha}
\end{align}
with the power-law exponent  $\alpha \approx 2 - \epsilon$, with $\epsilon$ typically in the range 0.1-0.3. The $\alpha = 2$ case is the exactly scale invariant case, where each logarithmic frequency interval has exactly the same total variance. The observed $\alpha < 2$ implies that natural images are nearly scale-invariant but slightly `UV tilted,' showing a bit more power at higher frequencies than at lower frequencies. When performing denoising diffusion, this power-law distribution is mixed with white noise, with a flat power spectral density $S(k) = \sigma_t^2$. There is a natural length scale in this problem, defined as the inverse of the spatial frequency where the noise power equals the signal power:
\begin{align}\label{eq:spectral_scale}
    L_{spec} = \bigg( \frac{\sigma_t^2}{A} \bigg)^{1/\alpha}
\end{align}
For exactly scale-invariant data with $\alpha = 2$, this reduces to 
\begin{align}\label{eq:spectral_alpha2}
    L_{spec} \sim \sigma_t
\end{align}
i.e., the spectral scale is linearly proportional to the standard deviation of the added noise. 

{Another useful way of interpreting the spectral scale is as the approximate radius of the optimal linear denoiser, which for stationary translationally-invariant data is known as the Wiener filter \cite{wiener1949extrapolation}. The optimal linear denoiser is the minimizer of the denoising objective (\ref{eq:neural_loss}) under the constraint that the model should be an affine function $M_{t}[\phi_t] = \hat{W}_t \phi_t + b_t$. This is a standard linear regression problem, and for a distribution $P_0(\varphi)$ with mean $\mu$ and covariance $\Sigma$, the optimal $\hat{W}_t$ and $b_t$ are given by, respectively,
\begin{align}
    \hat{W}_t &= (\bar{\alpha}_t \Sigma + (1 - \bar{\alpha}_t) I)^{-1} (\sqrt{\bar{\alpha}_t} \Sigma )\\
    b_t &= (I - \hat{W}_t) \sqrt{\bar{\alpha}_t} \,\mu.
\end{align}
The rows of $\hat{W}_t$ can be interpreted as denoising filters applied at a particular location. When the data statistics are translationally-invariant, the matrix $\hat{W}_t$ implements a convolution with a single filter, the Wiener filter. Convolution operators correspond to multiplication in Fourier space; the action of $\hat{W}_t$ for such data is given by
\begin{align}
    \hat{W}_t(k) \hat{\phi}_t(k) &= \frac{\sqrt{\bar{\alpha}_t} \Sigma(k)}{\bar{\alpha}_t \Sigma(k) + (1 - \bar{\alpha}_t)} \,\hat{\phi}_t(k)
\end{align}
where $\Sigma(k) = \text{Var}[\hat{\varphi}(k)] = P(k)$ is the variance of the $k$th Fourier mode of the unnoised images under $P_0(\varphi)$. This filter generally has infinite support, so to make sense of its `size' one has to make a definitional choice; the most natural characteristic width is the point where the signal and noise power meet in the denominator of the filter, i.e.
\begin{align}
    \bar{\alpha}_t \Sigma(k) = (1 - \bar{\alpha}_t)
\end{align}
Plugging in the power-law expression (\ref{eq:powerlaw}) for $P(k) = \Sigma(k)$ and rearranging yields the expression for the `fourier space radius' $|k_{spec}|$. Setting $L_{spec} = |k_{spec}|^{-1}$ yields the spectral scale (\ref{eq:spectral_scale}).}

{Because of the relatively compact support of this filter, it can be inferred that a {strictly local} denoiser with a locality scale $\geq L_{spec}$ should perform well with a relatively small error. Indeed, \cite{lukoianov2025locality} established that calibrating an LS model according to the radius of the Wiener filter yields models that perform capably even to moderately large image sizes. Conversely, we should expect that any reasonably capable denoiser should be able to attend \textit{at least} to roughly the radius of this filter. If the critical scale is much \textit{smaller} than the spectral scale ($L_{spec} \gg L_{c}$), a Bayes-optimal local denoiser \textit{cannot} perform well: it will either incur catastrophic memorization by using a scale $L \gg L_{spec}$ comparable to $L_{spec}$, or its denoising estimate will be robust but highly inaccurate due to using a scale $L < L_{c} \ll L_{spec}$. Figure \ref{fig:pixelization} illustrates both of these failure modes: when the scale the model uses is smaller than the spectral scale, the outputs of the denoiser are robust but suboptimal. Conversely, when the scale the model uses is larger than the critical scale, the model's output incurs catastrophic statistical error.}


Realistic data are highly non-Gaussian, with higher-order dependencies that exist and longer scales than implied by the two-point statistics. Indeed, \cite{lukoianov2025locality} finds that localizing based on the spectral scale breaks down on the highly non-Gaussian dataset MNIST. Nonetheless, achieving a radius of \textit{at least} the spectral scale is clearly \textit{necessary} for any performant local denoiser, and appears to be sufficient for a reasonable baseline level of denoising accuracy. A natural question then arises: \textit{is the critical scale for a Bayes-optimal local denoiser large enough to meet or exceed the spectral scale}? While by no means obvious, we show in this section that, both theoretically and empirically, the critical scale typically \textit{does} scale similarly to the spectral scale. This {remarkable concordance} between the scales \textit{needed} for denoising and the scales at which robust denoising is \textit{possible} (using a non-exponential dataset size) is at the heart of why local Bayes optimal models succeed at `reasonable' generalization on high-dimensional image data.

\subsection{A concrete calculation: the isotropic Gaussian}
The simplest case to analyze is the isotropic Gaussian first analyzed by \cite{biroli2024dynamical}. Given a patch $\Omega$, the mutual information between the $\Omega$-restricted image $\phi_{t;\Omega}$, and the underlying unnoised image $\varphi$ is given by
\begin{align}
    I(\phi_{t;\Omega} ; \varphi) = \frac{|\Omega|}{2} \ln\bigg(1 + \frac{s^2}{\sigma_t^2}\bigg)
\end{align}
where $s^2$ is the variance of $\varphi$, $|\Omega|$ is the number of pixels in the patch. The critical scale $L_{c} = \sqrt{|\Omega_{mem}|}$ is then given by equating this quantity to the number of nats $\ln |\mathcal{D}|$ needed to resolve down to a single training set image:
\begin{align}
    \ln |\mathcal{D}| = \frac{L_{c}^2}{2} \ln\bigg(1 + \frac{s^2}{\sigma_t^2}\bigg)
\end{align}
which upon rearranging yields
\begin{align}\label{eq:isotropic_exact}
    L_{c} = \sqrt{\frac{2 \ln|\mathcal{D}|}{\ln(1 + \frac{s^2}{\sigma_t^2})}}
\end{align}
While presented in slightly different form, under the identification $|\Omega| = L_{c}^2 = d$ this result reproduces exactly the `collapse time' result of \cite{biroli2024dynamical}. At large $\sigma_t^2$, this quantity becomes
\begin{align}\label{eq:isotropic_asymptotic}
    L_{c} \to \sigma_t \,\sqrt{\frac{2  \ln|\mathcal{D}|}{s^2}} 
\end{align}
i.e. we find that at high noise levels, {the critical scale is directly proportional to the noise level}. This asymptotic scaling is promising: it suggests that the requisite spectral scaling (\ref{eq:spectral_alpha2}) may be achieved by a generalizing Bayes-optimal denoiser even for asymptotically large image sizes, without scaling the dataset. As we will see, the form derived above is very general across a wide range of images. 

\subsection{Scaling bounds}\label{app:scaling_bounds}

If we have a more specific model for our data distribution where $I(\phi_{t;\Omega};\varphi)$ is tractable, we can in principle compute the location of the critical threshold exactly. This computation can also be performed numerically; we illustrate some such results in figure \ref{fig:multiscale}. However, doing so directly is often challenging, and extracting an analytic solution for $L_{c}(\sigma_t)$ from the collapse condition (\ref{eq:collapse_condition}) may not be feasible even when the mutual information $I(\phi_{t;\Omega};\varphi)$ is available. In such cases, however, we can produce a number of useful bounds. The first is an immediate corollary of the additive noise bound:
\begin{corollary}\label{cor:diag_gaussian_bound}
    Assume an additive white noise channel with a locality constraint $\phi_{\Omega,t}$, and suppose $\varphi$ is drawn from an arbitrary distribution $P$ with mean $\mu$ and covariance $\Sigma$. We define the distribution $Q_1 = \mathcal{N}(\mu, \Sigma)$ to be the Gaussian distribution with identical second-order statistics to $P$, and $Q_2 = \mathcal{N}(\mu, \text{Diag}[\Sigma])$ to be the Gaussian distribution with identical diagonal covariance and zero off-diagonal elements. We then have
    \begin{align}
        L_{c;P}(\sigma_t) \geq L_{c; Q_1}(\sigma_t) \geq L_{c; Q_2}(\sigma_t)
    \end{align}
\end{corollary}

\begin{proof}
    The Gaussian bound (thm. \ref{thm:gaussian_bound}) entails
    \begin{align}
        I(\phi_{t,\Omega}; \varphi \sim P) \leq I(\phi_{t,\Omega}; \varphi \sim Q_1)
    \end{align}
    We further have a closed formula for the mutual information of a Gaussian variable
    \begin{equation}
        I(\phi_{t,\Omega}; \varphi \sim Q_1) = \Tr[\ln(I + \frac{\Sigma}{\sigma_t^2})]
    \end{equation}
    Since $Tr \ln(\cdot)$ is a strictly concave function and $\text{diag}(\Sigma) \prec \Sigma$ in the ordering of semi-definite matrices, $Tr[\ln(I + \frac{\Sigma}{\sigma_t^2})] \leq Tr[\ln(I + \frac{\text{Diag}(\Sigma)}{\sigma_t^2})]$
    \begin{align}
        I(\phi_{t,\Omega}; \varphi \sim Q_1) \leq I(\phi_{t,\Omega}; \varphi \sim Q_2)
    \end{align}
    From this it follows that, for $L > L'$, $I(\phi_{t,\Omega_L} ; \varphi) \geq I(\phi_{t,\Omega_{L'}} ; \varphi)$. Consequently, we find that at $L_{c;Q_1}$,
    \begin{align}
        \ln|\mathcal{D}| = I(\phi_{t,\Omega_{L_{c};Q_1}} ; \varphi \sim Q_{1}) \geq I(\phi_{t,\Omega_{L_{c};P}} ; \varphi \sim P)
    \end{align}
    so we deduce that $L_{c;Q_1} \leq L_{c;P}$. Applying the same argument for $Q_2$ and $Q_1$ yields $L_{c;Q_2} \leq L_{c;Q_1}$.
\end{proof}

These bounds alone are sufficient to answer our most basic question in the affirmative: \textit{for data with exactly scale-invariant power laws}, the relationship (\ref{eq:isotropic_asymptotic}) provides a {lower bound} on the critical scale. This implies that for any data distribution with translation invariance, bounded pixel-wise variance, and a scale-invariant power law $|k|^{-2}$, we have $L_{c}(\sigma_t) \geq C L_{spec}(\sigma_t)$ for arbitrarily large image sizes even with a fixed dataset size. We also find that these Gaussian bounds are very strong predictors of the critical scale: as shown in figure \ref{fig:mem_vs_spec}, the patch mutual information and concomitant critical scales for even highly non-Gaussian datasets like MNIST tend to be well-predicted by their Gaussian bounds. 

Beyond these Gaussian bounds, a much more general scaling theorem is available, which covers a wide range of naturalistic images:
\begin{theorem}\label{theorem:extensive_scaling_appendix}
    Suppose we have a translationally-invariant field $\varphi$ with bounded pointwise variance and asymptotically image-size-extensive entropy density $\lim_{|\Omega| \to \infty} \frac{1}{|\Omega|} S(\varphi_{\Omega}) = \tilde{S}$. Then, for $\ln|\mathcal{D}|$ which is $O(1)$ with respect to $\sigma_t$,
    \begin{align}
         L_{c}(\sigma_t) \sim O\big(\sigma_t \sqrt{\ln|\mathcal{D}|}\big)
    \end{align}
\end{theorem}
\begin{proof}
    The inequality $L(\sigma_t) \geq C\, \sigma_t \sqrt{\ln|\mathcal{D}|}$ follows immediately from corollary \ref{cor:diag_gaussian_bound}. To show the reverse direction, we use the result
    \begin{align}
        I(\phi_{\Omega}; \varphi) = S(\phi_{\Omega}) - S(\phi_{\Omega} | \varphi) = S(\phi_{\Omega}) - S(\sigma_t \eta_{\Omega})
    \end{align}
    where $\sigma_t \eta_{\Omega}$ is the added noise in the patch $\Omega$. We can produce a lower bound on the information $S(\phi_{\Omega})$ via the entropy power inequality, which states that for any variables $X, Y \in \mathbb{R}^d$,
    \begin{align}
        S(Y + X) \geq \frac{d}{2} \ln( e^{\frac{2}{d} S(Y)} + e^{\frac{2}{d} S(X)}  ) = \frac{d}{2} \ln(1 + e^{\frac{2}{d} [S(X) - S(Y)]}) + S(Y)
    \end{align}
    Applying this with $X = \varphi_{\Omega}$ and $Y = \sigma \eta$ yields
    \begin{align}
        S(\phi_{\Omega}) - S(\sigma_t\eta_{\Omega}) \geq  \frac{|\Omega|}{2} \ln \bigg(1 + \frac{e^{\frac{2}{|\Omega|} S(\varphi_{\Omega})}}{2\pi e \sigma_t^2}\bigg)
    \end{align}
    Since $\sigma_t^2$ is going off to infinity and $L(\sigma_t) \geq C\, \sigma_t \sqrt{\ln|\mathcal{D}|}$, $|\Omega|= L^2$ is going off to infinity as well. Therefore, we can replace $S(\varphi_{\Omega})/|\Omega|$ with its limit.
    \begin{align}
        I(\phi_{\Omega} ; \varphi) \geq \frac{|\Omega|}{2} \ln\bigg(1 + \frac{e^{2 \tilde{S} + o(1)}}{2\pi e \sigma_t^2}\bigg)
    \end{align}
    Applying this to the collapse condition (\ref{eq:memorization_app}) we find that 
    \begin{align}
        \ln |\mathcal{D}| \geq \frac{L_{c}^2(\sigma_t)}{2} \frac{e^{2 \tilde{S}}}{2\pi e \sigma_t^2} + o(\frac{1}{\sigma_t^2})
    \end{align}
    This is in the form of the information condition scaling for an isotropic Gaussian with pointwise variance $\frac{e^{2\tilde{S}}}{2\pi e}$, which asymptotically as $\sigma_t \to \infty$ yields the result
    \begin{align}
        L_{c}(\sigma_t) + o(\frac{1}{\sigma_t^2}) \leq  2\sigma_t \, e^{-\tilde{S}}\sqrt{\pi e\ln|\mathcal{D}|}
    \end{align}
    This completes the proof.
\end{proof}

\subsection{Deviations from scale-invariance}

For power laws $\alpha > 2$, the spectral scale (\ref{eq:spectral_scale}) increases sublinearly with $\sigma_t$. The Gaussian bounds show that $L_{c} \geq C\sigma_t$, which as $\sigma_t \to \infty$ shows that $L_{c}$ asymptotically dominates over the spectral scale and thus is asymptotically learnable with a Bayes-optimal local denoiser. This case, while instructive, is less phenomenologically interesting for natural images. For power laws $\alpha < 2$, the bulk entropy scales extensively, so the conditions for theorem \ref{theorem:extensive_scaling_main} are met and we find $L_{c} \sim \sigma_t \sqrt{\ln|\mathcal{D}|}$. The spectral scale $L_{spec} \sim \sigma_t^{2/\alpha}$ thus grows faster with $\sigma_t$ than the critical scale at fixed $\ln|\mathcal{D}|$. This implies that to match the spectral scale up to arbitrarily large image sizes, we need to scale the dataset size according to
\begin{align}
    \ln |\mathcal{D}| \sim \sigma_t^{\frac{2\epsilon}{2 - \epsilon}} \sim L_{spec}^{\epsilon}
\end{align}
While this scaling is superpolynomial with regards to image size (and therefore experiences a mild form of the curse of dimensionality), it is significantly subexponential, i.e. it significantly improves over the naive $|\mathcal{D}| \sim const^{d}$ scaling characteristic of the empirical score function. The right hand side in practice exhibits a very weak dependence-- e.g., for $\epsilon = 0.1$, taking $L_{spec} = 32$ to $L = 1024$ results in $L_{spec}^\epsilon \sim 1.4$ to $L_{spec}^{\epsilon} = 2$. The exact impact of this scaling on the required $|\mathcal{D}|$ will depend on the constants of proportionality, which are in general problem-dependent. In our experiments, we see little evidence of this $\epsilon$-dependent growth of the spectral scale relative to the critical scale, although such growth may become more pronounced for very large-scale images.

\subsection{Empirical analysis of the Posterior Entropy on Real and Toy Datasets}\label{app:empirical_post_entropy}
In figure \ref{fig:posterior_entropy}, we show comparisons between the real entropy of the posterior distribution and the theoretical prediction for two toy models of data and two real datasets. We take $5 \times 5$ image patches from $32 \times 32$ images and calculate the posterior distribution according to \eqref{eq:bayes_weights_channel} and calculate $-\sum_{\varphi \in \mathcal{D}} P(\varphi | \phi_{\Omega_x,L=5}) \ln P(\varphi | \phi_{\Omega_x,L=5})$ by summing over all training samples as the noise strength $\sigma_t^2$ is varied between $10^{-3}$ and $10^{1.5}$. The entropy is then calculated with dataset sizes $2^8=256$, $2^{10}=1024$, $2^{12} = 4096$, and $2^{14} = $ 16,384. Each dataset size is plotted in its own color. In sub-figure (a), training data samples were drawn from a Gaussian distribution with pixel covariance $\langle \varphi(x) \varphi(y)\rangle = \frac{1}{|x-y|^{2-\alpha}}$ for $x \neq y$ with $\langle \varphi(x)^2 \rangle = 1$. The distance between neighboring pixels is set to 1. The theory curve uses the expression for the mutual information of Gaussian variables \eqref{eq:gaussian_mi}. In sub-figure (b), a mixture of 16 Gaussians was sampled in the fashion described in section \ref{app:hierarchical_mixture_of_gaussians}. Section \ref{app:hierarchical_mixture_of_gaussians} also describes the formula for the theory curve. All covariances were taken to be isotropic with the variance of the means equal to 1 and a variance of the Gaussians equal to 0.15. In sub-figures (c) and (d), data samples were drawn uniformly from the datasets CIFAR10 and CelebA and then downscaled to $32 \times 32$ color images. The theory curves were calculated by plugging the empirical patch covariance matrix into the equation for the Gaussian mutual information \eqref{eq:gaussian_mi}.

\subsection{Empirical analysis of the critical scale and spectral scale}\label{app:empirical_scaling}
\begin{figure}
    \centering
    
    \includegraphics[width=0.8\textwidth]{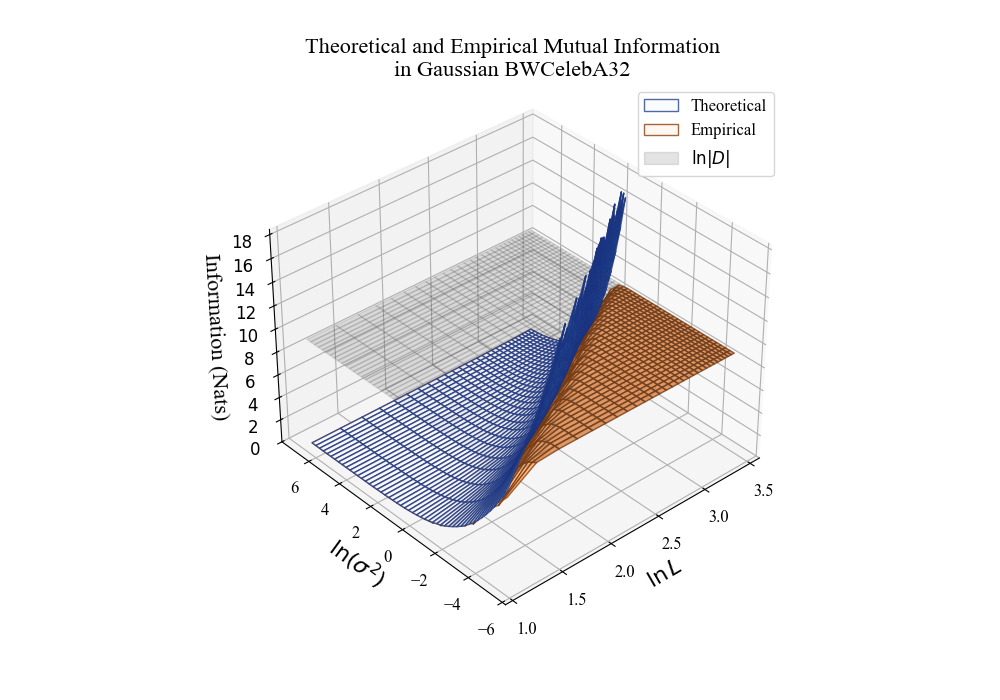}
    \caption{A plot of the mutual information $I(\phi_{\Omega}; \varphi)$ on a Gaussianized version of the BWCeleba32 dataset, as a function of both the noise level $\sigma^2$ and the patch length $L$, relative to the prior for a) the theoretical distribution b) a finite sample of $10^4$ examples, as a function of noise level and patch scale. Plotted as well is the critical threshold $I = \ln|\mathcal{D}|$.}
    \label{fig:3dplotandscaling}
\end{figure}

To test the qualitative validity of our theoretical analyses, we ran a number of experiments on realistic data distributions. Firstly, we wanted to directly test the validity of the collapse condition (\ref{eq:memorization_app}). To do so, we first took the \textit{Gaussianized} version of an exemplar dataset (in our case BWCeleba32), defined as the Gaussian distribution whose mean and covariance match the first and second order statistics of the dataset. The reason for this choice is that it furnishes a dataset with naturalistic statistics, for which the underlying patch mutual information $I(\phi_{\Omega_L, t}; \varphi)$ can be computed analytically using the formula (\ref{eq:gaussian_mi}). We first computed the theoretical mutual information for a large range of length scales $L$ and noise parameters $\sigma_t$; the results are plotted in figure \ref{fig:3dplotandscaling} as the blue `Theoretical' surface. As expected, these results grow without bound as the noise level decreases/as the length scale increases. We then sampled a dataset $\mathcal{D}$ of $10^4$ samples from this Gaussianized prior, and evaluated the mutual information of the noised patch relative to the \textit{discrete} prior induced by this dataset; the result of this evaluation is plotted in figure \ref{fig:3dplotandscaling} as the orange `Empirical' surface. We found, as expected, that the empirical finite-dataset result closely matches the theoretical infinite-dataset result when the latter is below the critical threshold (illustrated by the grey surface in fig. \ref{fig:3dplotandscaling}). When the latter exceeds the critical threshold, the empirical surface `peels off,' never reaching its theoretical upper bound of $\ln|\mathcal{D}|$. 

To test the concordance between the critical scale and the spectral scale, we compute numerically the location of a \textit{sub}critical threshold, defined by $\ln|\mathcal{D}| - \ln d^* = I(\phi_{\Omega_L}; \varphi)$. Intuitively, this defines the necessary amount of information to resolve down to an `effective sample size' of $d^*$ from a dataset of size $|\mathcal{D}|$. Because this threshold is below the critical threshold, we expect the empirical estimator for $I(\phi_{\Omega_L}; \varphi)$ to be robust and match the theoretical infinite-data mutual information with a high degree of accuracy, as demonstrated in the plot (\ref{fig:3dplotandscaling}) in the case where the exact mutual information is known analytically. We take $d^* = 10$ for our experiments. We compute the subcritical threshold curve by sweeping the empirical mutual information over $\sigma_t$ and $L$ and computing numerically the intersection of this surface with the threshold surface. We also compute the Gaussian bound analytically over the same $\sigma_t, L$ domain, using the patch covariance statistics. We find that the Gaussian bound is usually a strong predictor of the (sub)critical scale. Finally, for each dataset, we extract the curve of the spectral scale $L_{spec}(\sigma_t)$ as a function of the noise level. To do this, we identify at each $\sigma_t$ the Fourier band $|k^*(\sigma_t)|$ wherein the average power spectral density for a mode in that band equals the added noise power $\sigma_t^2$, and take $L_{spec}(\sigma_t) = |k^*(\sigma_t)|^{-1}$. We generally find that this curve generally tracks the critical scale remarkably well, validating our thesis that there is a concordance between the spectral and critical scales on realistic datasets.

Finally, to illustrate the impact of the scale on the performance of the Bayes-optimal local denoiser, we took two images, one from the denoiser's training set and an unrelated image from the test set, and applied noise to each at a noise level of $\sigma_t = 1$. We then applied an optimal local denoiser to each image at length scales $L$. At small scales ($L=5$), the Bayes-optimal local denoiser is robust, performing comparably on the training and test set, but produces poor quality images. At the test set-optimal scale, near but below the critical scale, ($L = 7$, boxed in green), the denoising estimate is both accurate and robust on both the training set and test set. Past this scale, statistical error begins to accumulate on the denoising estimate of test image, while the model's performance on the training image increases. At very large scales ($L = 17,33$), the LS model denoises the training image perfectly, while failing catastrophically on the test set image. This result is illustrated in panel \ref{fig:pixelization} in figure \ref{fig:multiscale}.

\section{Why scaling the patch size with the spectral scale is necessary for a local denoiser}\label{app:renormalization}
Our work has emphasized the critical importance of using a local denoiser with a length scale that corresponds to the spectral scale. We can make the rationale for this somewhat more precise in the context of Gaussian power-law data. Roughly speaking, any filter family that achieves this scale will be able to reconstruct the power spectral density $P(k) \sim |k|^{-\alpha}$ of the target distribution; conversely, a filter family that scales according to a different law will produce data distributed according to a different power law spectral density. 

For a power-law Gaussian distribution with power spectral density $P(k) = A |k|^{-\alpha}$, the Wiener filter in Fourier space is given by
\begin{align}
    \hat{w}(\sigma^2, k) = \frac{1}{1 + A^{-1}\sigma^2 |k|^{\alpha}}
\end{align}
This satisfies an interesting invariance condition, which we call the \textit{renormalization condition}:
\begin{align}
    \hat{w}(\beta \sigma^2, k) = \hat{w}(\sigma^2, \beta^{1/\alpha} k)
\end{align}
In real space, this corresponds to the condition
\begin{align}
    w(\beta\sigma^2, x) = \beta^{-2/\alpha} w(\sigma^2, \beta^{-1/\alpha} x)
\end{align}
The content of this statement is that the family of Wiener filters is scale-equivariant: the filter used for a given noise level $\sigma^2$ is simply a rescaled and dilated version of the filter at any other given noise level. The Wiener filter also satisfies two other important properties: it is \textit{normalized} (i.e. it integrates to 1), and it goes to 0 in Fourier space as $k \to \infty$ (spectral \textit{falloff}). The normalization property in particular is important for a consistent denoiser: it entails in particular that as the noise level goes to 0, the estimate of the image produced by the filter converges to the true image, as opposed to a rescaled version of it.

In our paper, we have focused on the question of \textit{strictly local denoising}, where the filter mass is constrained to lie entirely within a particular noise-dependent radius, which we expect to be proportional to the spectral scale. We might wonder what sort of distributions the resulting optimal localized filters generate. It turns out that the three properties of \textit{renormalization}, \textit{normalization}, and \textit{falloff}, on their own are enough to guarantee that a diffusion model, parameterized by a linear denoiser with \textit{any} filter shape $\hat{m}(\sigma^2,k)$, will obtain the correct spectral power law $\sim |k|^{-\alpha}$ (with a uniform error, relative to the optimal filter, on the global amplitude $A$). Conversely, if the filter family satisfies the renormalization property with a renormalization condition that is for a \textit{different} power value $\alpha ' \neq \alpha$ (e.g. if it is scale-equivariant but localized to a radius that is \textit{not} proportional to the spectral scale), this filter family will generate a power law with an \textit{incorrect} exponent. This result is captured below:

\begin{theorem}\label{thm:renormalization}
    Suppose we parameterize a diffusion model with a linear denoising model
    \begin{align}\label{eq:linear_denoiser}
        M_{t}[\phi_t](x) = m(\sigma^2_{t}, x) \,\,\star\,\, \phi_{t}(x)
    \end{align}
    where $m(\sigma^2, x)$ is a filter whose Fourier transform (with respect to $x$) satisfies the renormalization condition
    \begin{align}\label{eq:renormalization}
        \hat{m}(\beta \sigma^2, k) = \hat{m}(\sigma^2, \beta^{1/\alpha} k)
    \end{align}
    the normalization condition
    \begin{align}\label{eq:normalization}
        \int m(\sigma^2, x)\,dx = 1
    \end{align}
    and two falloff conditions for some $p,q>0$.
    \begin{align*}
        \lim_{|k|\to\infty} |k|^p \hat{m}(\sigma^2, k) = 0 & & \lim_{|k|\to 0} |k|^{-q} (\hat{m}(\sigma^2, k) - 1) = 0
    \end{align*}
    These falloff conditions can also be written in terms of $\sigma^2$ for some $p',q'>0$.
    \begin{align*}
        \lim_{\sigma^2\to\infty} (\sigma)^{2p'} \hat{m}(\sigma^2, k) = 0 & & \lim_{\sigma^2 \to 0} (\sigma)^{-2q'} (\hat{m}(\sigma^2, k) - 1) = 0
    \end{align*}
    Then the diffusion model will generate a Gaussian distribution with a power spectral density
    \begin{align}\label{eq:psd}
        P(k) \sim |k|^{-\alpha}
    \end{align}
\end{theorem}
\begin{proof}
    For the purposes of this proof, we will use the `variance-exploding' parameterization:
    \begin{align}
        \phi_t &= \phi_0 + \sigma_t \eta
    \end{align}
    where $\sigma_t \to \infty$ as $t \to \infty$. An analogous proof should be available using other parameterizations, such as the variance-preserving parameterization. The probability flow equation for this parameterization is
    \begin{align}\label{eq:flow}
        \frac{d}{dt} \phi_t &= \frac{1}{2}(\partial_t \log \sigma_t^2) (\phi_t - M_t[\phi_t])
    \end{align}
    which, for the linear denoiser (\ref{eq:linear_denoiser}), reduces to
    \begin{align}\label{eq:flow_linear}
        \frac{d}{dt} \phi_t &= \frac{1}{2}(\partial_t \log \sigma_t^2) \Big(1 - (m(\sigma^2_t, x) \, \star)\Big) \phi_t
    \end{align}
    The linear operator $(1 - m(\sigma^2_t, x) \, \star))$ is diagonalized in the Fourier basis, so we can solve \eqref{eq:flow_linear} by solving the evolution of each Fourier component $\hat{\phi}_t(k)$ independently. This yields
    \begin{align}
        \partial_t \hat{\phi}_t(k) &= \frac{1}{2}(\partial_t \log \sigma_t^2) (1 - \hat{m}(\sigma^2, k)) \hat{\phi}_t(k)
    \end{align}
    Integrating this equation gives
    \begin{align}
        \hat{\phi}_T(k) &= \hat{\phi}_0(k) \exp(\frac{1}{2}\int_0^T (\partial_t \log \sigma_t^2)  (1 - \hat{m}(\sigma_t^2,k)) \,dt)
    \end{align}
    Or, reparameterizing with the substitution $u = \log \sigma_t^2$, 
    \begin{align}\label{eq:phi_T_u}
        \hat{\phi}_T(k) = \hat{\phi}_0(k) \exp(\frac{1}{2} \int_0^{\log \sigma_T^2} \Big( 1 - \hat{m}(u,k) \Big)\,du)
    \end{align}
    In variance-exploding flows, the samples $\phi_0$ are generated by taking samples $\phi_T \sim \sigma_T \eta$ and flowing these samples backwards. We will consider the limit where $T \to \infty$. Substituting this into the equation \eqref{eq:phi_T_u} and rearranging, we find that the variance of $\hat{\phi}_0(k)$ (i.e. the power spectral density $P(k)$) satisfies 
    \begin{align}
        P(k) = \lim_{T \to \infty} \sigma_T^2 \exp(-\int_{-\infty}^{\log \sigma_T^2} \Big(1 - \hat{m}(u,k)\Big)\,du)
    \end{align}
    or in logarithmic terms
    \begin{align}
        \log P(k) &= \lim_{T \to \infty} \bigg[\log \sigma_T^2 - \int_{-\infty}^{\log \sigma_T^2} (1 - \hat{m}(e^u,k)) \,du\bigg] \\
        &= \int_{0}^{\infty} \hat{m}(e^u,k) du + \int_{-\infty}^{0} (\hat{m}(e^u,k) - 1) du
    \end{align}
    One can easily verify that the falloff conditions imply both of these integrals converge by the power test. To show the expected PSD \eqref{eq:psd} holds for this distribution, we can differentiate this quantity with respect to $|k|$. The dominated convergence theorem allows us to swap integration and differentiation to obtain
    \begin{align}\label{eq:p1}
        \partial_{\log|k|} \log P(k) = \int_{-\infty}^{\infty} \partial_{\log|k|} \hat{m}(e^u,k) \,du
    \end{align}
    The renormalization condition \eqref{eq:renormalization} implies that we can interchange the derivative of the filter with respect to $|k|$ with a derivative with respect to $u$:
    \begin{align}
        \partial_{\log |k|} \hat{m}(\sigma^2, k) = \alpha\, \partial_{\log\sigma^2} \,\hat{m}(\sigma^2,k)
    \end{align}
    Substituting this into \eqref{eq:p1} and evaluating the integral with respect to u, we obtain
    \begin{align}\label{eq:limits}
        \partial_{\log |k|} \log P(k) = \alpha [\hat{m}(\sigma^2=\infty, k) - \hat{m}(\sigma^2 = 0,k)]
    \end{align}
    The renormalization condition implies that $\hat{m}(\sigma^2 = 0,k) = \hat{m}(\sigma^2, k = 0)$; applying the normalization condition \eqref{eq:normalization} yields $\hat{m}(\sigma^2 = 0, k) = 1$. Meanwhile, the renormalization condition implies the asymptotic falloff of $\hat{m}(\sigma_T^2, k)$ as $\sigma_T^2 \to \infty$ coincides with the limit $\lim_{|k|\to\infty} \hat{m}(\sigma^2,k) = 0$. Thus \eqref{eq:limits} evaluates to
    \begin{align}
        \partial_{\log |k|} \log P(k) = \alpha(0 - 1) = -\alpha
    \end{align}
    which proves the claim.
\end{proof}

To complete our analysis, we need to prove that, for a power-law Gaussian, the conditional expectation $\mathbb{E}[\varphi | \phi_{\Omega_{L(\sigma^2), x}}]$, with $L(\sigma^2) \propto L_{spec}(\sigma^2)$, satisfies the three properties. This is indeed the case, as we show in the following theorem. 

\begin{theorem}
    For $\varphi$ an image drawn from a power-law Gaussian with PSD $|k|^{-\alpha}$, the model $M_t[\phi_t] = \mathbb{E}[\varphi | \phi_{\Omega_{L(\sigma^2), x}}]= m(\sigma^2, x) \star \phi$, with $L(\sigma^2) \propto \sigma^{2/\alpha}$, defines a filter family $m(\sigma^2,x)$ that satisfies the normalization and falloff conditions, as well as the renormalization condition with parameter $\alpha$.
\end{theorem}
\begin{proof}
    The optimal linear denoising model given a patch $\Omega$ for its center pixel $x$ is given by
    \begin{align}\label{eq:restricted_wiener}
        (\Sigma_{\Omega,x}) \cdot (\Sigma_{\Omega,\Omega} + \sigma_t^2 I)^{-1} \phi_{t}
    \end{align}
    where $\Sigma_{\Omega\Omega}$ is the block covariance (under the underlying distribution $P_0(\varphi)$ for the pixels within $\Omega$, and $\Sigma_{\Omega,x}$ is the row of this block covariance associated with the covariance between each pixel and the center pixel. Letting $C(x) \propto |x|^{\alpha - 2}$ be the two-point correlation function of the field (whose Fourier transform is the power spectral density $P(k)$), the operator $\Sigma_{\Omega\Omega} + \sigma_t^2 I$ can be understood as the convolution operator with the kernel
    \begin{align}
        \textbf{1}(y \in \Omega)C(y) + \sigma^2_t \delta(y)
    \end{align}
    while $\Sigma_{\Omega,x}$ corresponds to the function $C(y - x)$. In Fourier space, the resulting filter $\hat{m}(\sigma^2, k)$ is given by
    \begin{align}
        \hat{m}(\sigma^2, k) = \frac{P(k) \star \hat{\textbf{1}}_{\Omega}}{P(k) \star \hat{\textbf{1}}_\Omega + \sigma^2_t} = \frac{1}{1 + \sigma_t^2 [P(k) \star \hat{\textbf{1}}_\Omega](k)^{-1}}
    \end{align}
    where $\hat{\textbf{1}}_{\Omega}$ is the Fourier transform of the indicator function $\textbf{1}(y \in \Omega)$, so that $P(k) \star \hat{\textbf{1}}_{\Omega}$ is the Fourier transform of $C \textbf{1}_{\Omega}$. We will now assume that $\Omega_{L(\sigma^2)}$ is a patch whose geometry scales according to the noise level with the length scale $L(\sigma^2) = \sigma_t^{2/\alpha}$. This entails that the indicator function satisfies
    \begin{align}
        \textbf{1}_{\Omega_{L(\beta\sigma^2)}}(y) = \textbf{1}_{\Omega_{L(\sigma^2)}}(\beta^{-1/\alpha} y)
    \end{align}
    or, in Fourier space,
    \begin{align}\label{eq:indicator_renorm}
        \hat{\textbf{1}}_{\Omega_{L(\beta \sigma^2)}}(k) = \beta^{2/\alpha} \hat{\textbf{1}}_{\Omega_{L(\sigma^2)}}(\beta^{1/\alpha} k)
    \end{align}
    From this, we observe
    \begin{align*}
        [\hat{P} \star \hat{\textbf{1}}_{\Omega_{L(\sigma^2)}}](\beta^{1/\alpha} k) &= \int \hat{P}(\beta^{1/\alpha} k - v) \hat{\textbf{1}}_{\Omega_{L(\sigma^2)}}(v) \,dv\\
        &= \beta^{2/\alpha} \int P(\beta^{1/\alpha}(k - u)) \hat{\textbf{1}}_{\Omega_{L(\sigma^2)}} (\beta^{1/\alpha} u) \,du
    \end{align*}
    Using \eqref{eq:indicator_renorm} and the property $P(\beta^{1/\alpha} k) = P(k) \beta^{-1}$, this becomes
    \begin{align}
         [\hat{P} \star \hat{\textbf{1}}_{\Omega_{L(\sigma^2)}}](\beta^{1/\alpha} k) = \beta^{-1} [\hat{P} \star \hat{\textbf{1}}_{\Omega(\beta L(\sigma^2))}](k)
    \end{align}
    From this we find that
    \begin{align}
        \hat{m}( \beta\sigma^2, k) &= \bigg(1 + \frac{(\beta \sigma^2)}{[\hat{P} \star \hat{\textbf{1}}_{\Omega(\beta L(\sigma^2))}](k)}\bigg)^{-1} = \bigg(1 + \frac{(\beta \sigma^2)}{\beta [\hat{P} \star \hat{\textbf{1}}_{\Omega( L(\sigma^2))}](\beta^{1/\alpha} k)}\bigg)^{-1}\\
        &= \bigg(1 + \frac{ \sigma^2}{ [\hat{P} \star \hat{\textbf{1}}_{\Omega( L(\sigma^2))}](\beta^{1/\alpha} k)}\bigg)^{-1} = \hat{m}(\sigma^2, \beta^{1/\alpha} k)
    \end{align}
    So $\hat{m}$ satisfies the renormalization condition with parameter $\alpha$. Since $\lim_{\sigma^2 \to 0} m(\sigma^2, k) = 1$ and $\lim_{\sigma^2 \to \infty} m(\sigma^2, k) = 0$, it follows from the renormalization conditions that the falloff and normalization conditions are satisfied.    
\end{proof}

Our results do not establish the result that $\mathbb{E}[\varphi | \phi_{\Omega_{L(\sigma^2), x}}]$, with $L(\sigma^2)$ parameterized by an \textit{incorrect} power law $\sim \sigma^{\frac{2}{\alpha'}}$ with $\alpha' \neq \alpha$, will generate an incorrect power law in the power spectral density, because the restricted statistics are not appropriately self-similar with noise, and therefore this estimator does not define a scale-equivariant family of filters satisfying the renormalization condition. We generally expect, however that if the length scale is significantly less than the spectral scale, that this family will nonetheless generate a power law that is \textit{asymptotically} incorrect at large $|k|$.

\section{Empirical Results in Early Training}\label{app:early_training}

\subsection{Methods}

\begin{figure}
    \centering

    \includegraphics[width=0.49\linewidth]{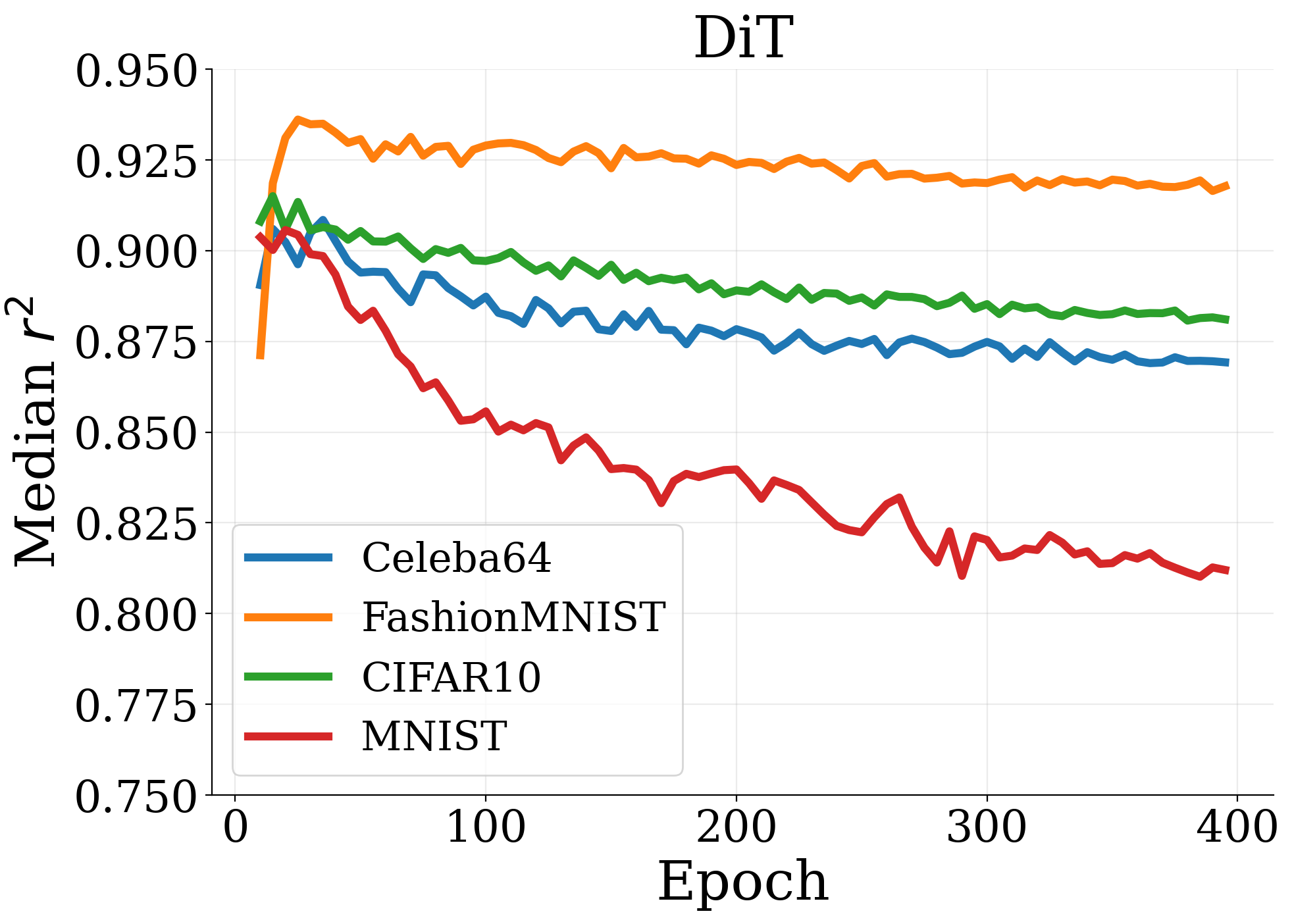} 
    \hspace{1mm}
    \includegraphics[width=0.49\linewidth]{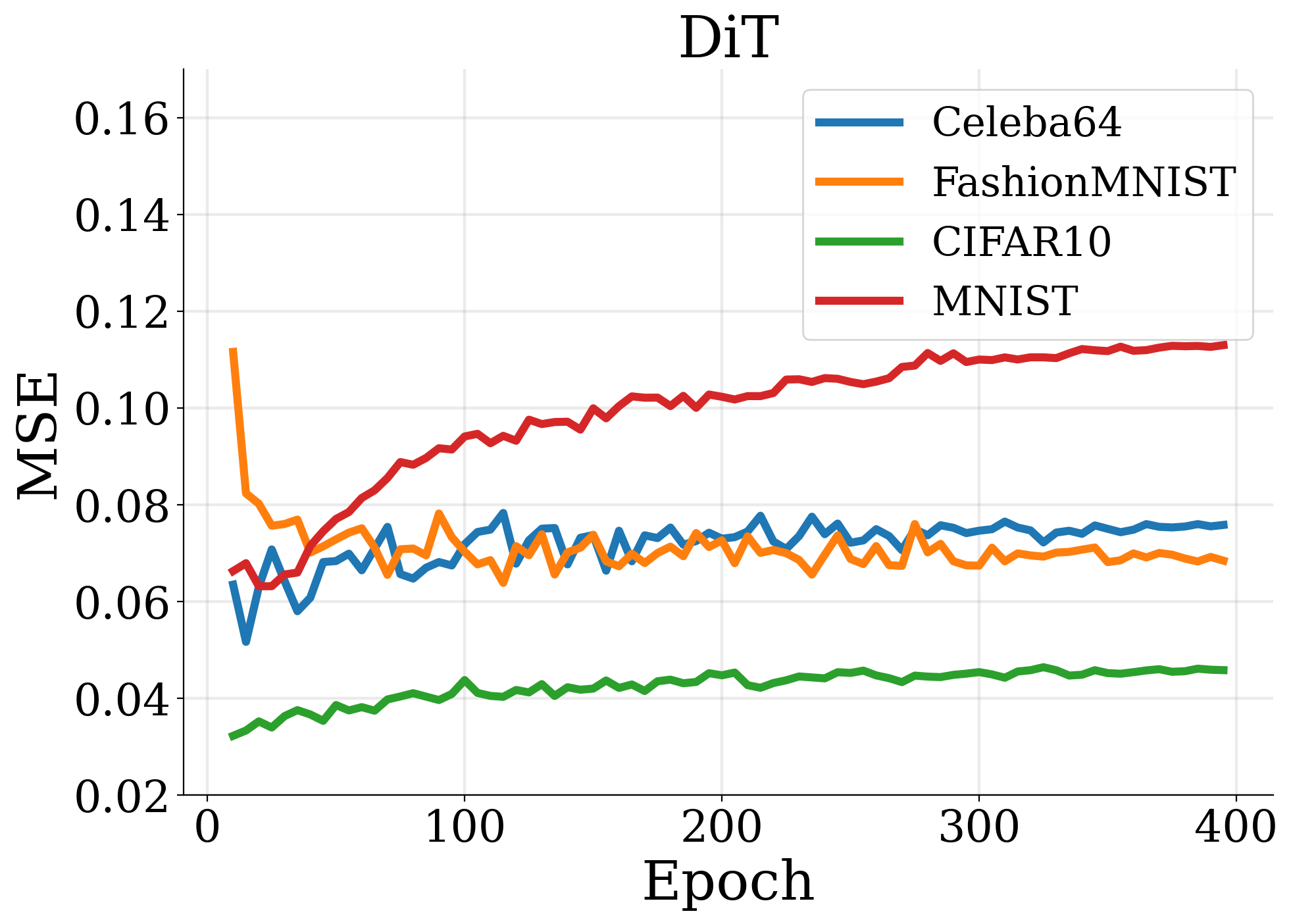}

    \includegraphics[width=0.49\linewidth]{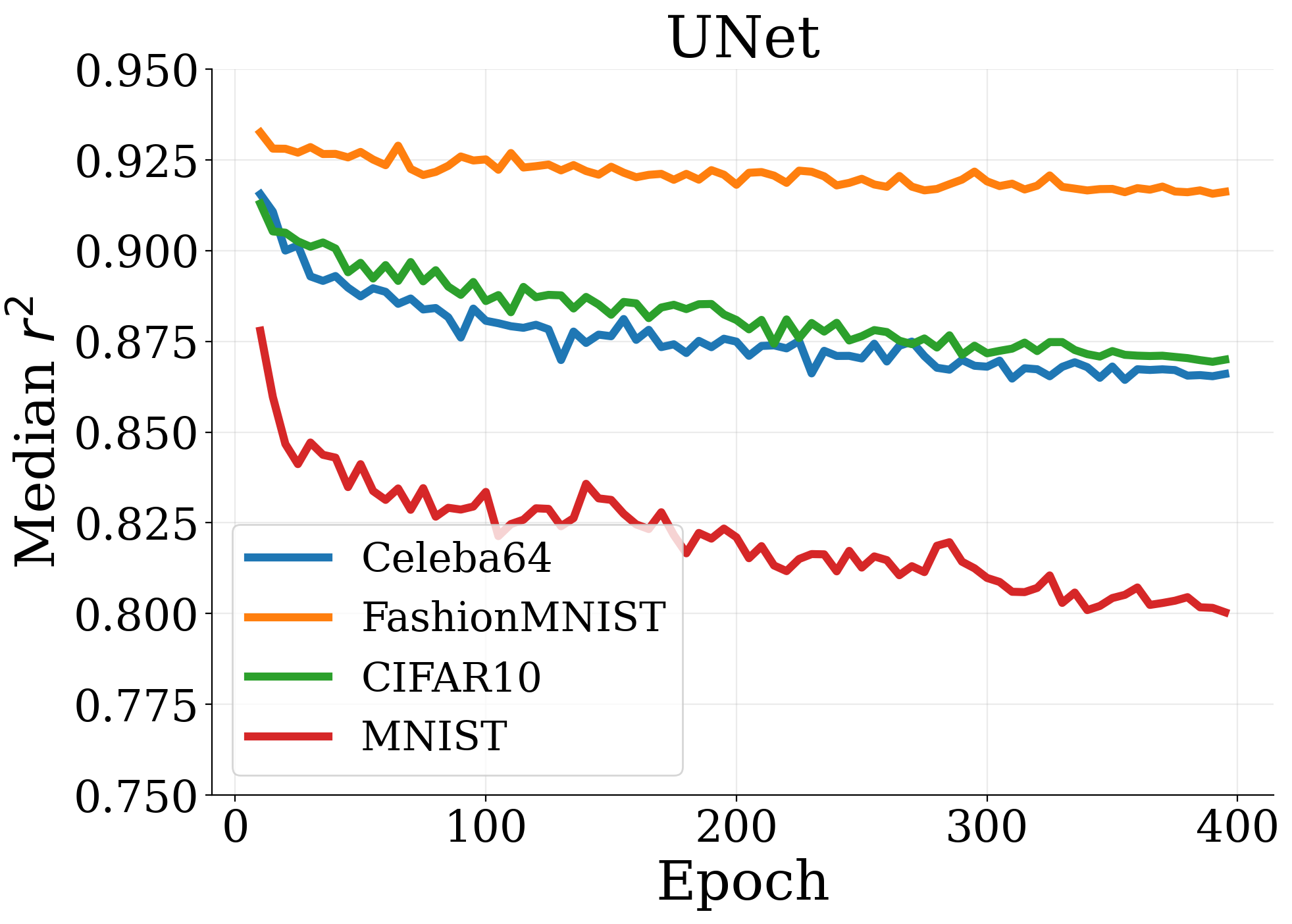} 
    \hspace{1mm}
    \includegraphics[width=0.49\linewidth]{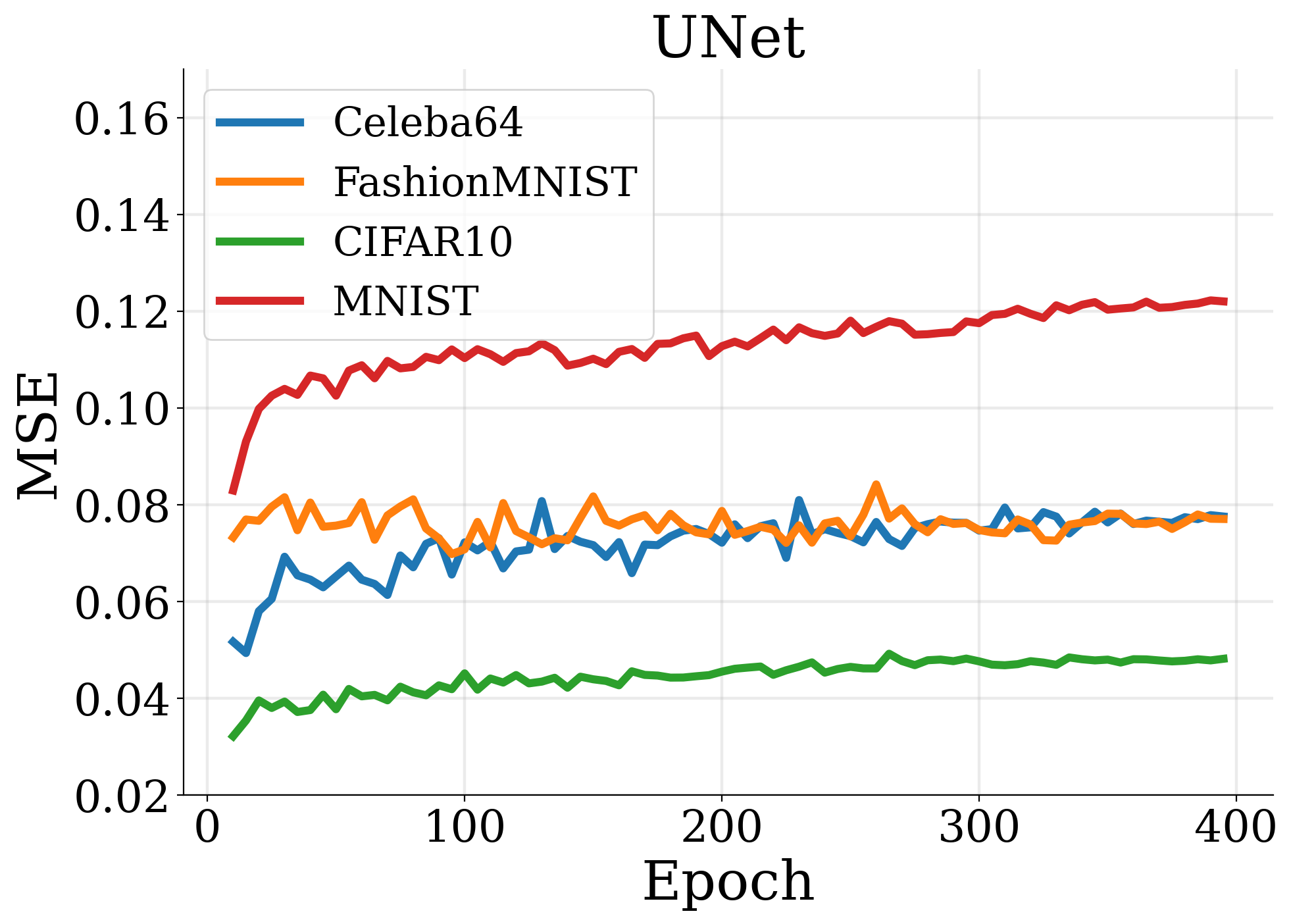}

    \caption{Plots of the Median $r^2$ metric (left) and the MSE metric (right) averaged over comparisons between several independently trained DiTs (top two panels) and UNets (bottom two panels) and their corresponding analytical models, across four datasets Celeba64, FashionMNIST, CIFAR10, and MNIST. Each model is trained on a reduced dataset of $10^4$ samples. The Median $r^2$ metric starts high across all datasets and gradually decreases over the course of training, while the MSE generally starts low and then gradually increases over the course of training. On some datasets these metrics are more level than others, or even decreasing as in MSE on FashionMNIST; however, the generally decreasing trend is mostly consistent.}
    \label{fig:early_training_1}
\end{figure}


\begin{table}
    \centering
        \caption{Quantitative agreement between DiTs trained on a subset of size $10^4$ and corresponding (LS/ELS) analytical models given the same dataset subset. Median $r^2$ values and MSE values are computed using a fixed subset of 100 samples and averaged over 4 model trials. Error bars are given by $2\sigma / \sqrt{n}$ for $n = 400$ the sample size and $\sigma$ the sample standard deviation for the given statistic. Selected metrics are taken from 30 epochs and 400 epochs respectively; the full trajectory of the median $r^2$ and MSE losses are shown in plots \ref{fig:early_training_1}.}
        \vspace{2mm}
    \begin{tabular}{c|cccc}
    \hline
          Dataset & Med. $r^2$ (30 eps.) & MSE (30 eps.) & Med. $r^2$ (400 eps.) & MSE (400 eps.) \\
          \hline
          Celeba64 & $0.905 \;(\pm 0.003)$ & $0.064 \;(\pm 0.003)$ & $0.869 \;(\pm 0.005)$ & $0.076 \;(\pm 0.004)$ \\
          CIFAR10 & $0.906 \;(\pm 0.005)$ & $0.036 \;(\pm 0.002)$ & $0.881 \;(\pm 0.007)$ & $0.046 \;(\pm 0.002)$ \\
          MNIST & $0.899 \;(\pm 0.007)$ & $0.066 \;(\pm 0.003)$ & $0.812 \;(\pm 0.009)$ & $0.113 \;(\pm 0.005)$ \\
          FashionMNIST & $0.935 \;(\pm 0.011)$ & $0.076 \;(\pm 0.008)$ & $0.918 \;(\pm 0.009)$ & $0.068 \;(\pm 0.005)$ \\
    \hline
    \end{tabular}
    \label{tab:training_quantitative_1}
\end{table}


\begin{table}
    \centering
        \caption{Quantitative agreement between UNets trained on a subset of size $10^4$ and corresponding (LS/ELS) analytical models given the same dataset subset. Median $r^2$ values and MSE values are computed using a fixed subset of 100 samples and averaged over 4 model trials. Error bars are given by $2\sigma / \sqrt{n}$ for $n = 400$ the sample size and $\sigma$ the sample standard deviation for the given statistic. Selected metrics are taken from 30 epochs and 400 epochs respectively; the full trajectory of the median $r^2$ and MSE losses are shown in plots \ref{fig:early_training_1}.}
        \vspace{2mm}
    \begin{tabular}{c|cccc}
    \hline
          Dataset & Med. $r^2$ (30 eps.) & MSE (30 eps.) & Med. $r^2$ (400 eps.) & MSE (400 eps.) \\
          \hline
          Celeba64 & $0.893 \;(\pm 0.005)$ & $0.069 \;(\pm 0.004)$ & $0.866 \;(\pm 0.005)$ & $0.078 \;(\pm 0.004)$ \\
          CIFAR10 & $0.901 \;(\pm 0.005)$ & $0.039 \;(\pm 0.002)$ & $0.870 \;(\pm 0.007)$ & $0.048 \;(\pm 0.002)$ \\
          MNIST & $0.847 \;(\pm 0.011)$ & $0.104 \;(\pm 0.006)$ & $0.800 \;(\pm 0.011)$ & $0.122 \;(\pm 0.005)$ \\
          FashionMNIST & $0.929 \;(\pm 0.012)$ & $0.082 \;(\pm 0.009)$ & $0.916 \;(\pm 0.011)$ & $0.077 \;(\pm 0.006)$ \\
    \hline
    \end{tabular}
    \label{tab:training_quantitative_1_u}
\end{table}

In all experiments, we use a cosine noise schedule and a uniformly discretized reverse process. Our neural networks diffusion models employ a 100-step reverse process. Due to the high computational cost of the analytical denoisers, we employ a 20-step reverse process for each of our analytical diffusion models. We used an H100 GPU node with 80GB RAM for all experiments.

To calibrate the locality scale in our experiments, rather than calibrating based on agreement with the neural network \cite{kambanalytic,niedobatowards} or using an absolute threshold on the Wiener filter \cite{lukoianov2025locality}, we simply select a validation set ($\sim 2000$ images) from the training set, and select a patch scale at each time step of denoising based on the \textit{best validation loss} for the analytical denoising model employing the remainder of the dataset. This naturally selects patch sizes around the criticality threshold, which is the largest scale (and therefore most informative) before the onset of catastrophic statistical error (fig. \ref{fig:pixelization}). For MNIST, FashionMNIST, and Celeba64, we use an LS model with patch geometries chosen using the Wiener filter binarization of \cite{lukoianov2025locality}. However, our procedure for picking the patch length differs from that used in this paper because we are only selecting the \textit{family} of patches given by $\Omega_k$ representing the mask of top-$k$ coefficients of the Wiener filter; however, rather than setting $k$ using the coefficients of the Wiener filter, we calibrate it (independently for each pixel) using the cross validation procedure described above. We start our calibration at the lowest noise level, and incrementally increase the patch size $k$ in increments of $\max(4, \lfloor 0.05 k \rfloor)$ pixels until the validation loss at that pixel location fails to decrease. We make the assumption based on our theoretical analysis that the optimal scale should be nondecreasing with noise level, and thus initialize the scale at the next noise level at the optimal scale from the previous iteration. For CIFAR10, we find instead that the better analytical model at early training is the {ELS} rather than {LS}. We use the boundary-broken ELS prescription given in \cite{kambanalytic} with a square patch geometry, and calibrate the scale of this patch using the same procedure described above. For the ELS, however, due to computational constraints, we only use one global scale, rather than a pixel-specific scale.

For our neural network experiments, we use a DiT with a patch size of $4 \times 4$, hidden dimension of 512, 8 attention heads, and a 4-to-1 MLP ratio per block, with 8 transformer blocks, as well as an input and output convolutional stem. We use a UNet with four scales, with channel dimensions (128, 256, 256, 512) at each scale, and use an attention layer with 8 heads at the 16x16 image resolution layers. We train the models using AdamW with betas (0.9,0.95) and weight decay of 5e-2. We use an initial learning rate of 2e-4, with cosine annealing. We train for 400 epochs on dataset subsets of size $10^4$ with a batch size of 64 for the DiTs and 32 for the UNets. 

One of the difficult features of early-training diffusion models is that, at standard (large) learning rates, their outputs have a propensity to oscillate. The global hues of generated images are particularly strongly affected by this oscillation, which can significantly damage the appearance of early-training consistency if not managed. This instability phenomenon has been noted elsewhere \citep{yu2024unmasking, deck2023easing, choi2022perception}, and a number of techniques have been proposed to mitigate this effect. To get consistent results in early training, we find best results are obtained combining the following two techniques:
\begin{itemize}
    \item \textbf{Weight averaging:} exponential moving average weighting of past checkpoints is a commonly employed technique in diffusion modeling to improve checkpoint performance. We use a slightly simpler approach of linearly weight averaging the checkpoints taken from the end of each of the 5 previous epochs. We find that this helps average out fluctuations and stabilize the model's output significantly.
    \item $v$\textbf{-prediction:} this technique was designed \cite{salimans2022progressive} to keep the target variance consistent at all noise levels. We find in our work that this choice significantly stabilizes early training outputs, relative to other standard parameterizations ($\hat{\epsilon}$- and $\hat{x}_0$-prediction).
\end{itemize}

Following the literature on analytical diffusion models \citep{kambanalytic, niedobatowards, lukoianov2025locality}, we use pixel-space metrics (median pixelwise $r^2$ and MSE) to characterize the distance between our neural and analytical model outputs. While the quantitative signals we use show the effects of increased training in the form of slightly worse scores, the effect is somewhat less evident in the statistics than the qualitative perceptual effect would suggest. We do not propose a new metric here, but feel that there is room to develop metrics that capture the more perceptually and phenomenologically relevant dimensions of model training. In our table \ref{tab:training_quantitative_1} and our plots \ref{fig:early_training_1} we report the average of these metrics taken across four iterations and both subsets $\mathcal{D}_1$ and $\mathcal{D}_2$ that we are training with.


\FloatBarrier
\subsection{Samples}\label{app:samples}

We find empirically that the rate of training is somewhat different depending on the architecture and dataset, with maximal similarity to the outputs of the analytical models not always occurring at the same point (see fig \ref{fig:early_training_1} for the evolution of the quantitative metrics). In the samples below, we highlight samples taken from the epoch showing approximately the point of maximum agreement between theory and experiment for each diffusion model/dataset; examples of the evolution of the samples further in training are shown in fig \ref{fig:carveout}. For each sample figure including \ref{fig:early_training_consistency}, the samples for the epochs displayed are:
\begin{itemize}
    \item MNIST: UNet, 10; DiT, 15.
    \item FashionMNIST: UNet, 10; DiT, 15.
    \item CIFAR10: UNet, 10; DiT, 30.
    \item Celeba64: UNet, 10; DiT, 30.
\end{itemize}
These epochs are on the reduced datasets of size $10^4$ used in training, so the equivalent number of epochs on a larger dataset for peak agreement would likely be smaller.

\begin{figure}
    \centering
    \centering
    \sidecaptionimage{0.3\linewidth}{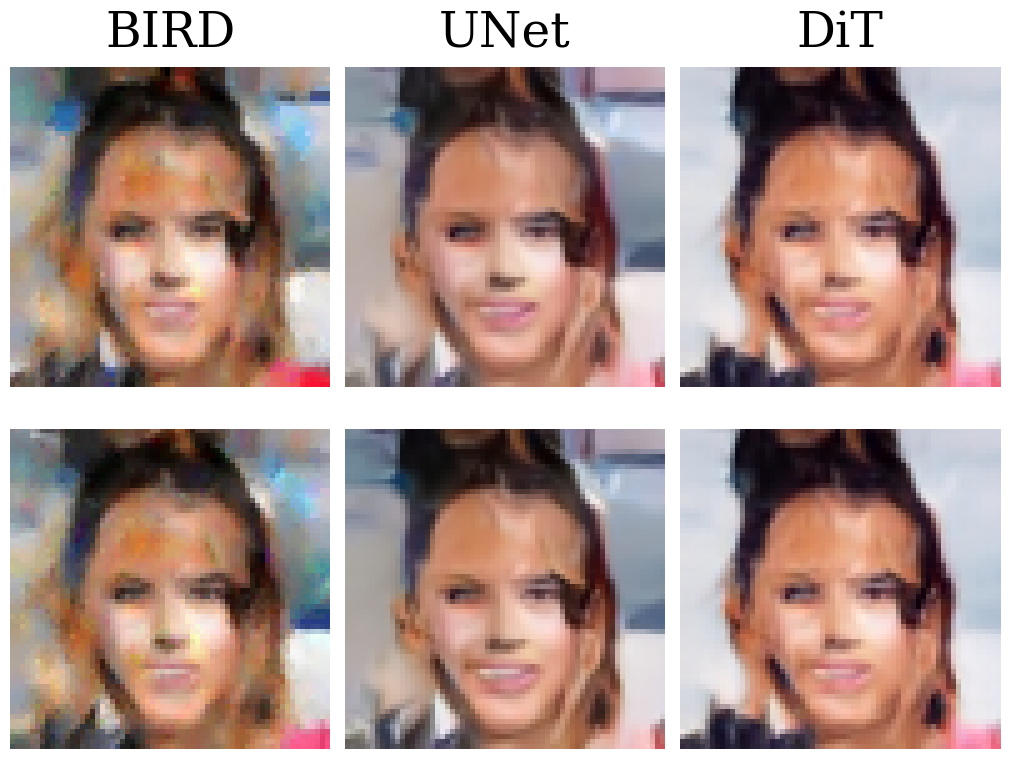}{\,\,\,$\mathcal{D}_2\,\,\,\,\,\,\,\,\,\,\,\,\,\,\,\,\,\,\,\,\mathcal{D}_1$\,\,\,\,}
        \includegraphics[width=0.3\linewidth]{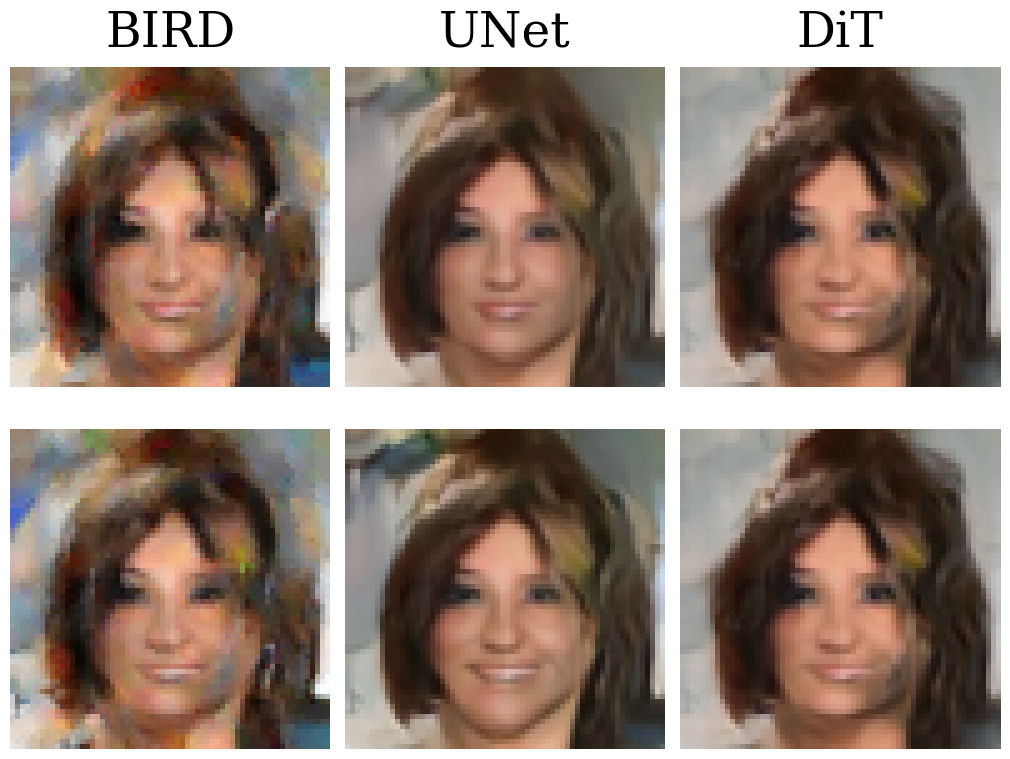}
    \includegraphics[width=0.3\linewidth]{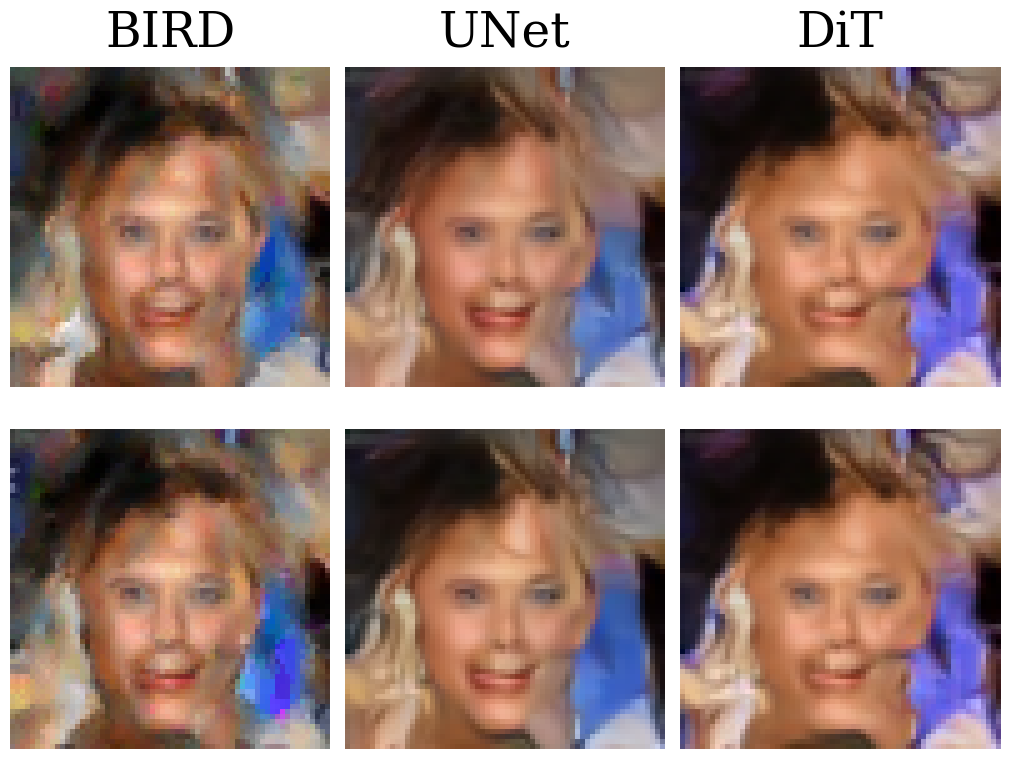}
    
    \sidecaptionimage{0.3\linewidth}{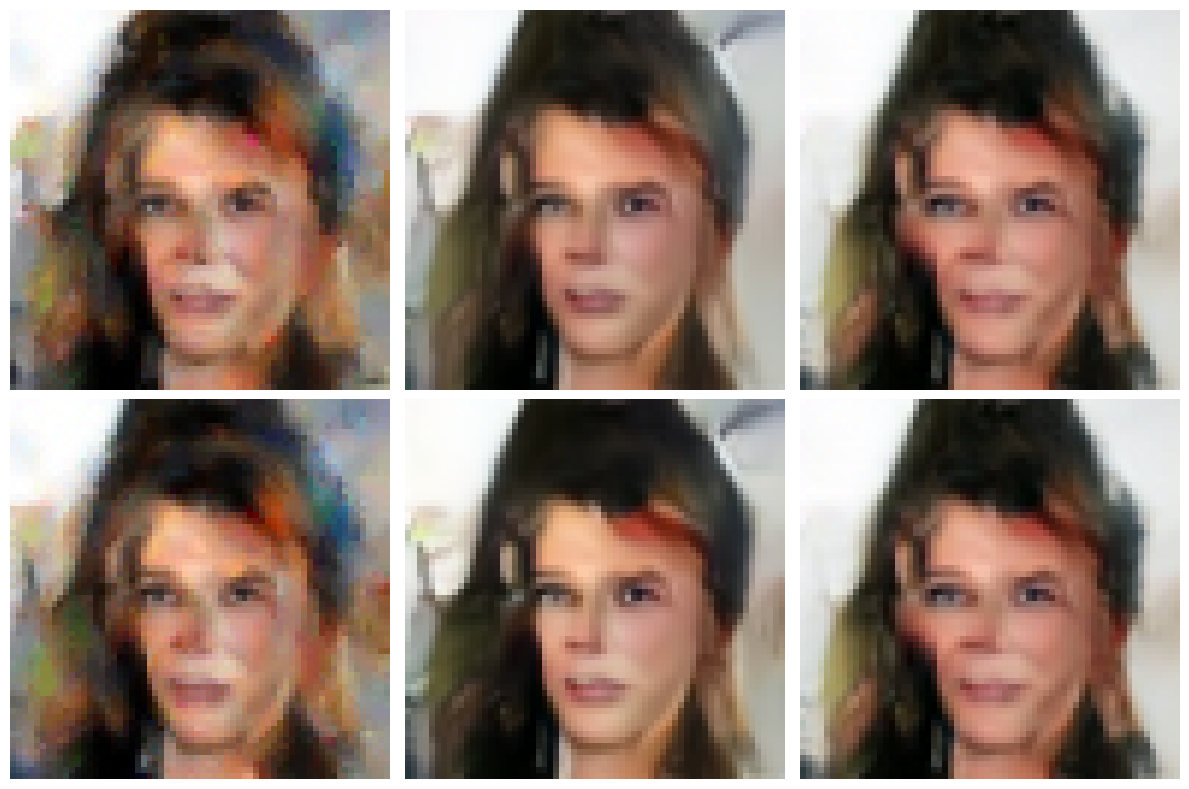}{\,\,\,$\mathcal{D}_2\,\,\,\,\,\,\,\,\,\,\,\,\,\,\,\,\mathcal{D}_1$}
    \includegraphics[width=0.3\linewidth]{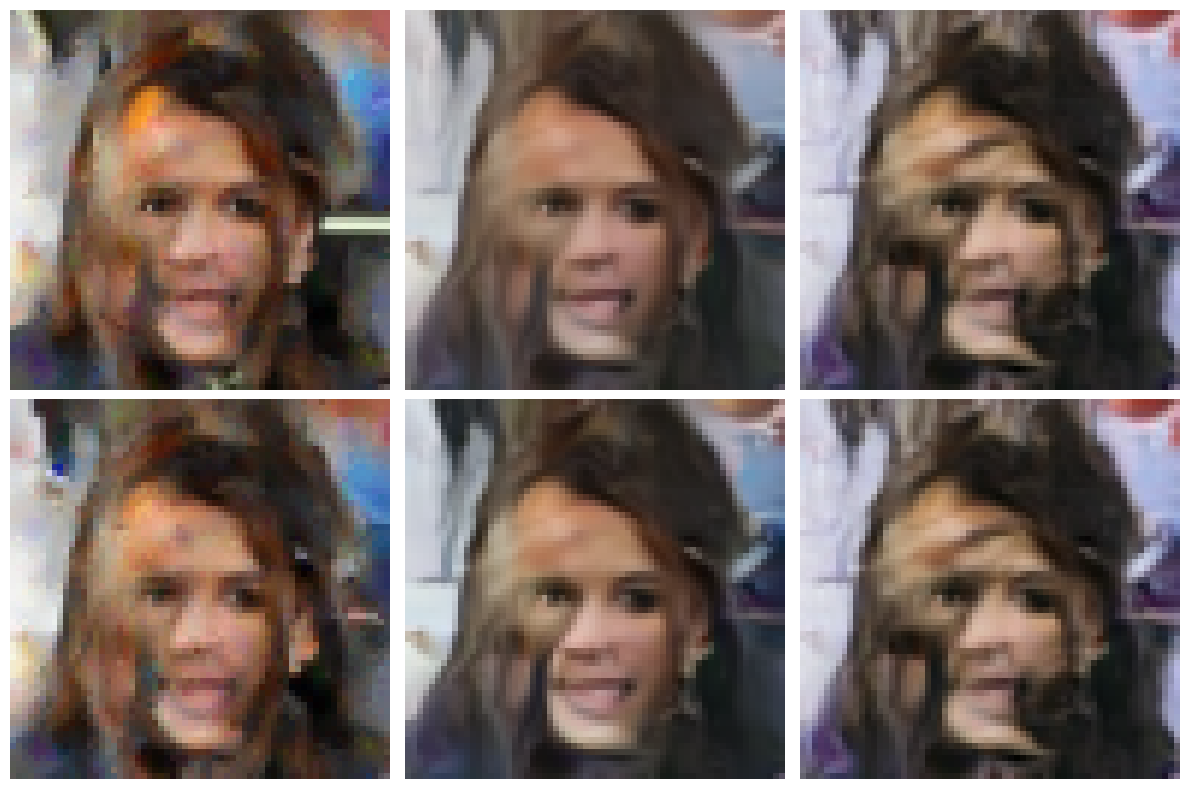}
    \includegraphics[width=0.3\linewidth]{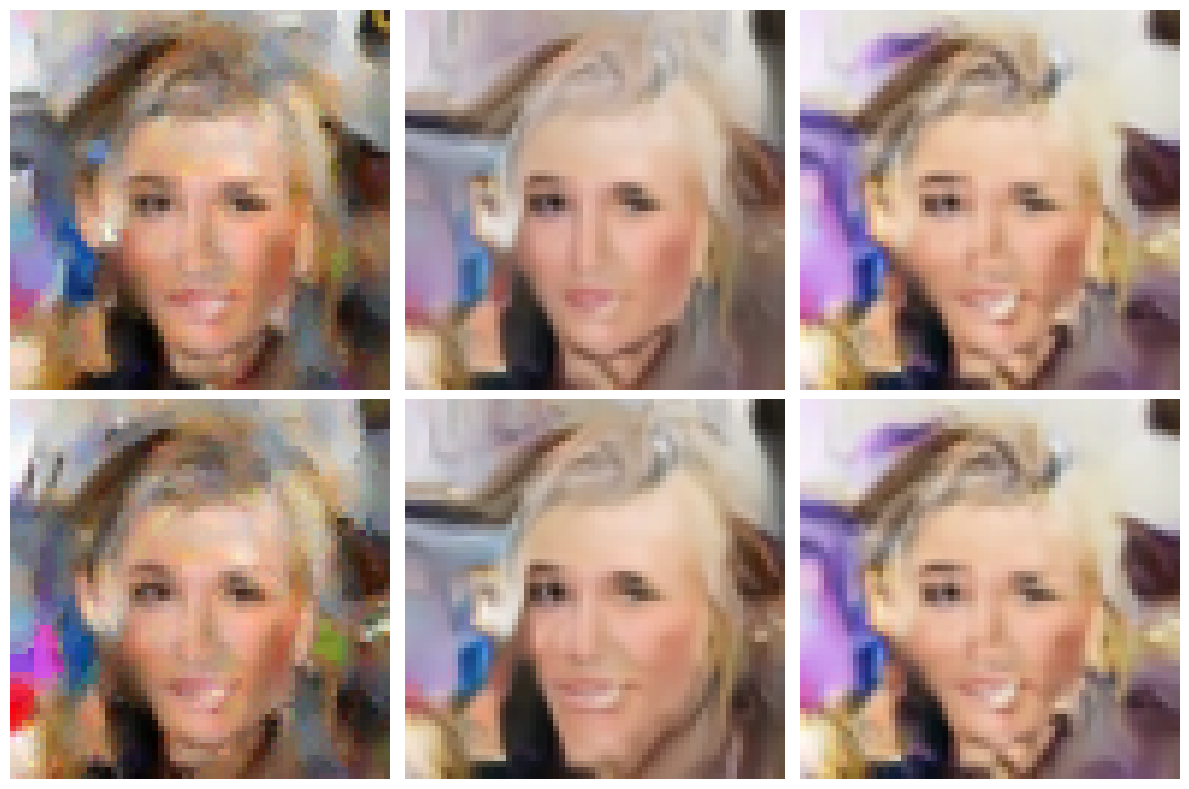}

    \sidecaptionimage{0.3\linewidth}{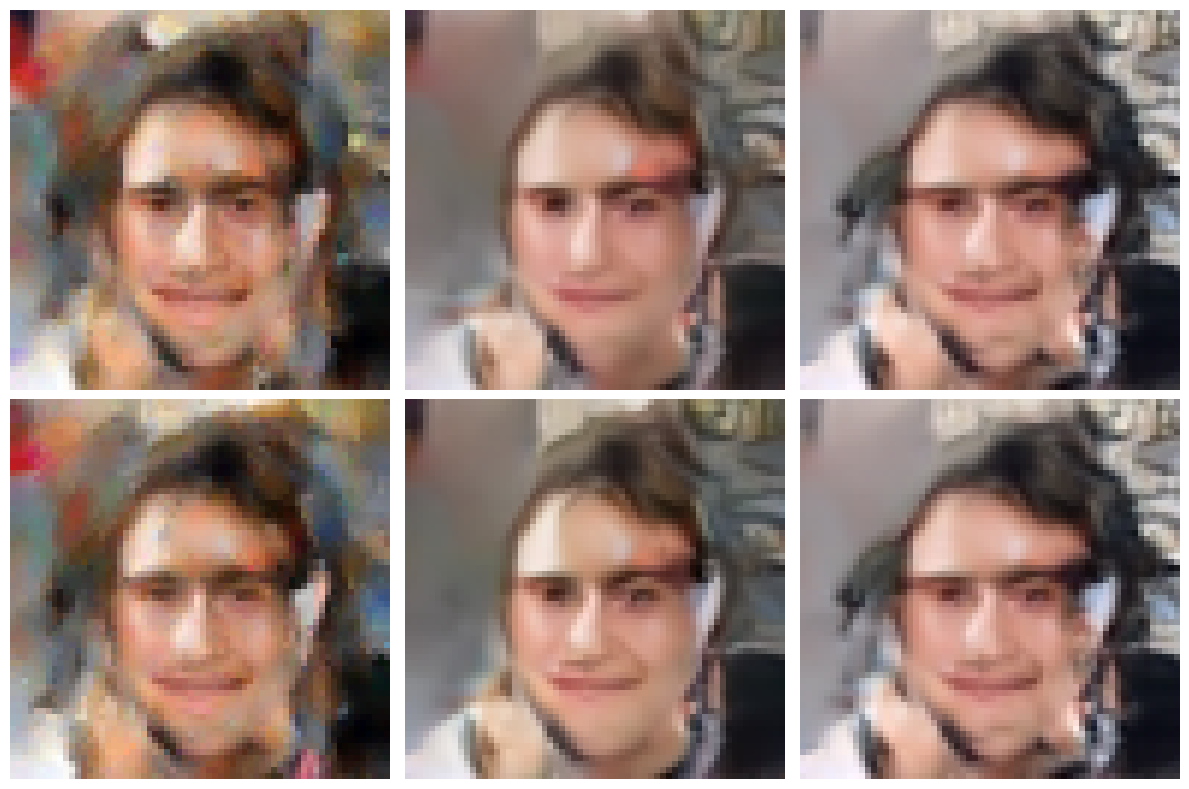}{\,\,\,$\mathcal{D}_2\,\,\,\,\,\,\,\,\,\,\,\,\,\,\,\,\mathcal{D}_1$}
    \includegraphics[width=0.3\linewidth]{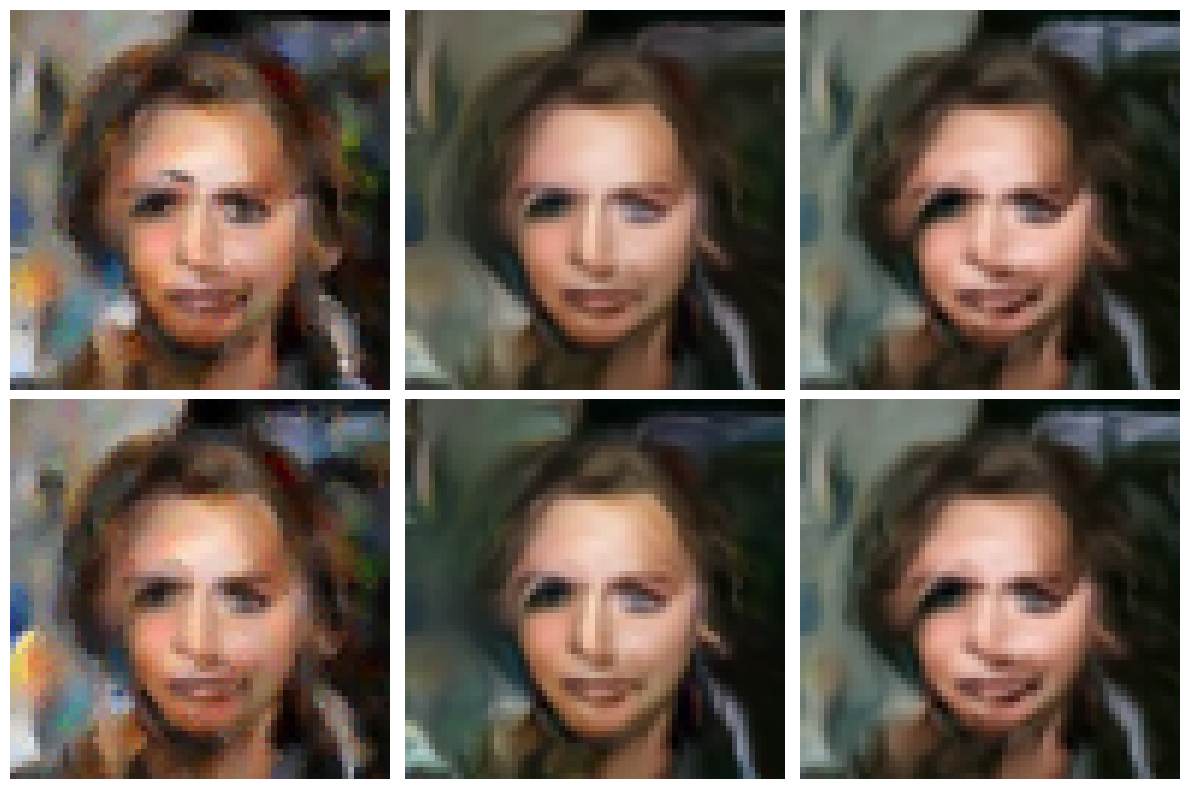}
    \includegraphics[width=0.3\linewidth]{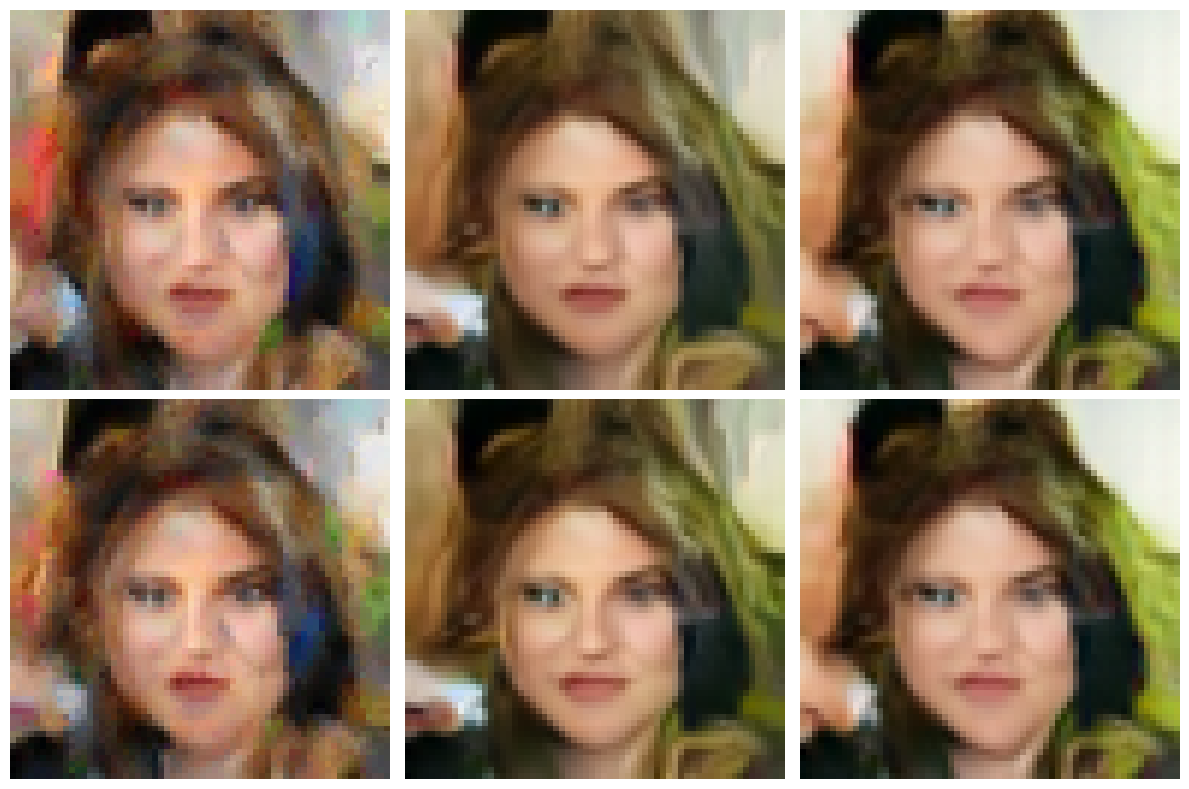}

    \caption{Further samples comparing DiTs and UNets early in training, trained on two disjoint subsets of $10^4$ samples from Celeba64, and the outputs of calibrated local score models. UNet samples are from 10 epochs of training, DiT samples are from 30 epochs of training. Samples not curated for quality.}
    \label{fig:app_samps_cba}
\end{figure}

\begin{figure}
    \centering
    \centering
    \sidecaptionimage{0.3\linewidth}{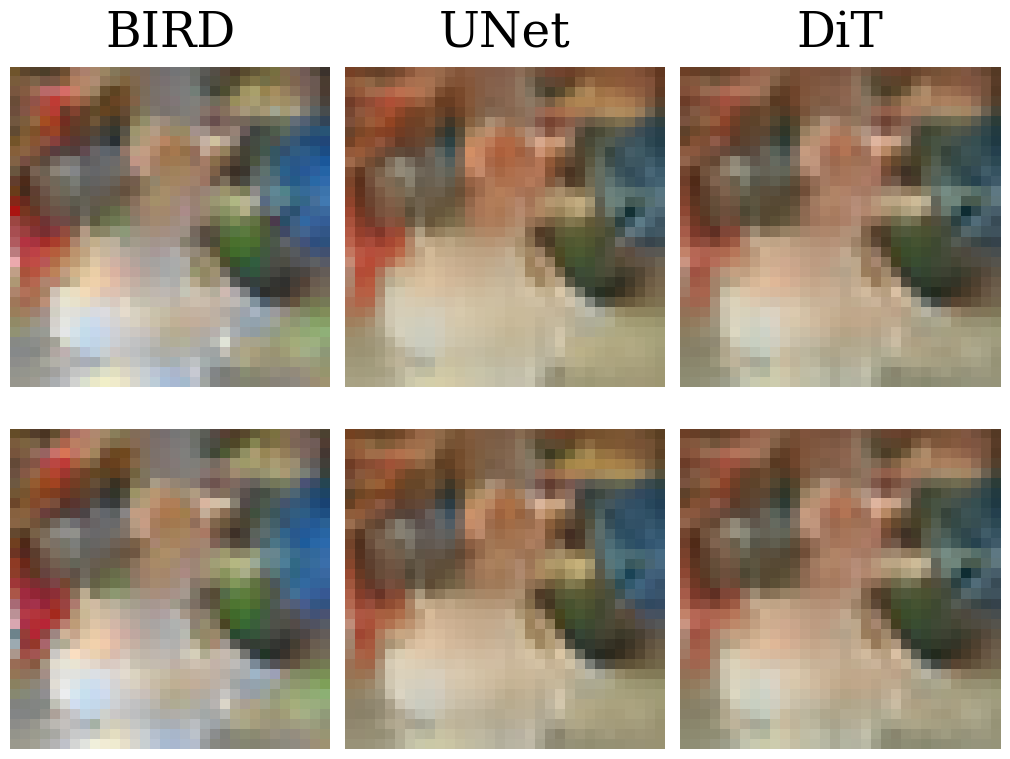}{\,\,\,$\mathcal{D}_2\,\,\,\,\,\,\,\,\,\,\,\,\,\,\,\,\,\,\,\,\mathcal{D}_1$\,\,\,\,}
        \includegraphics[width=0.3\linewidth]{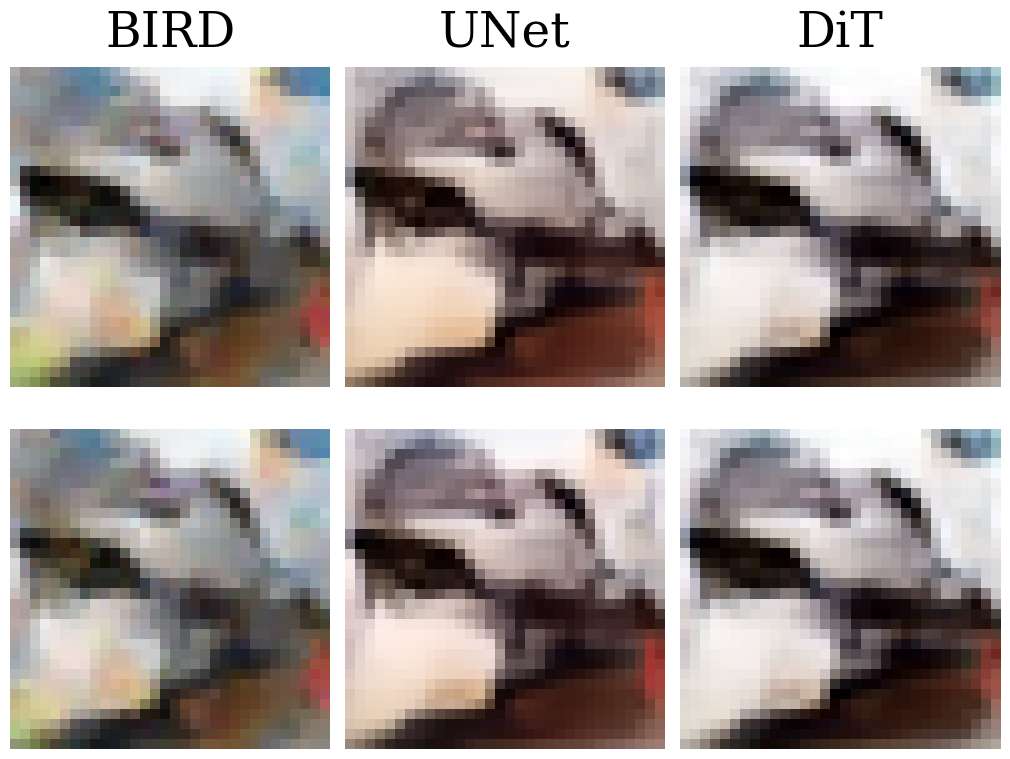}
    \includegraphics[width=0.3\linewidth]{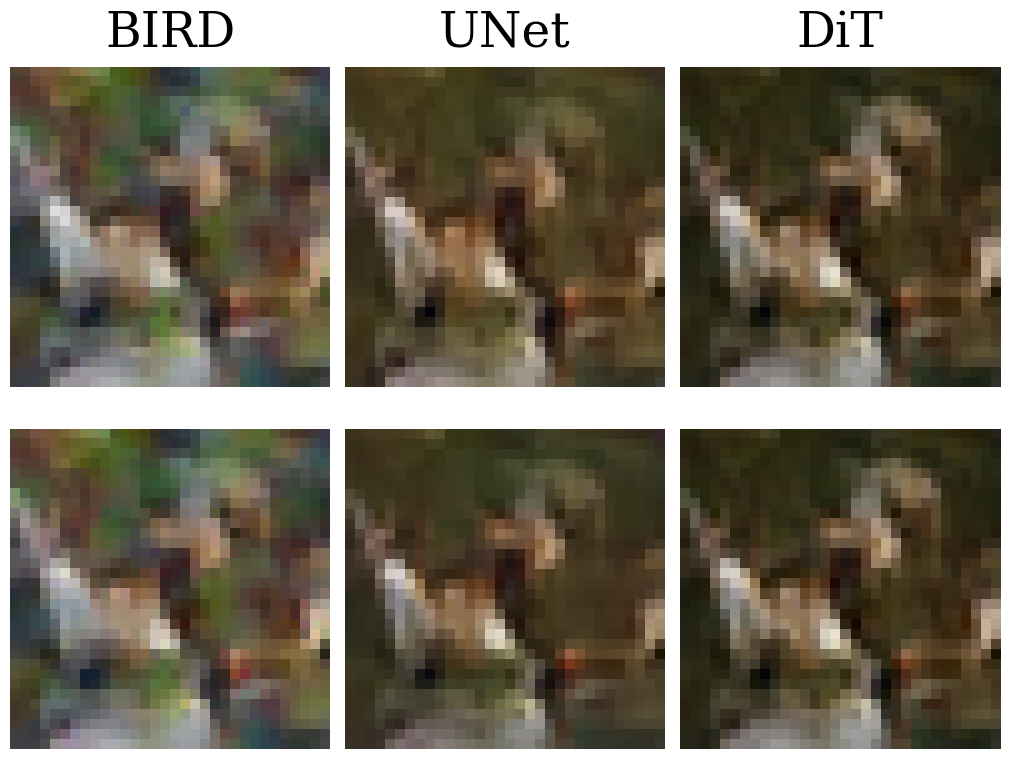}
    
    \sidecaptionimage{0.3\linewidth}{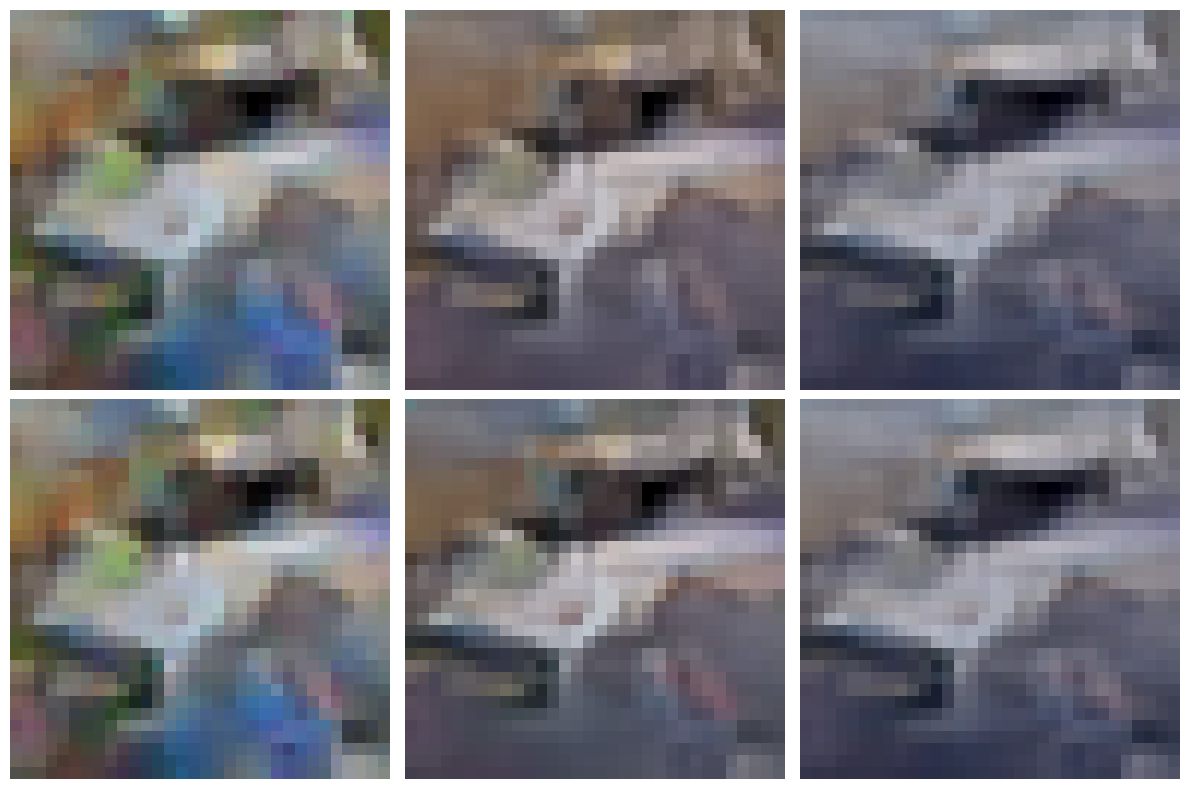}{\,\,\,$\mathcal{D}_2\,\,\,\,\,\,\,\,\,\,\,\,\,\,\,\,\mathcal{D}_1$}
    \includegraphics[width=0.3\linewidth]{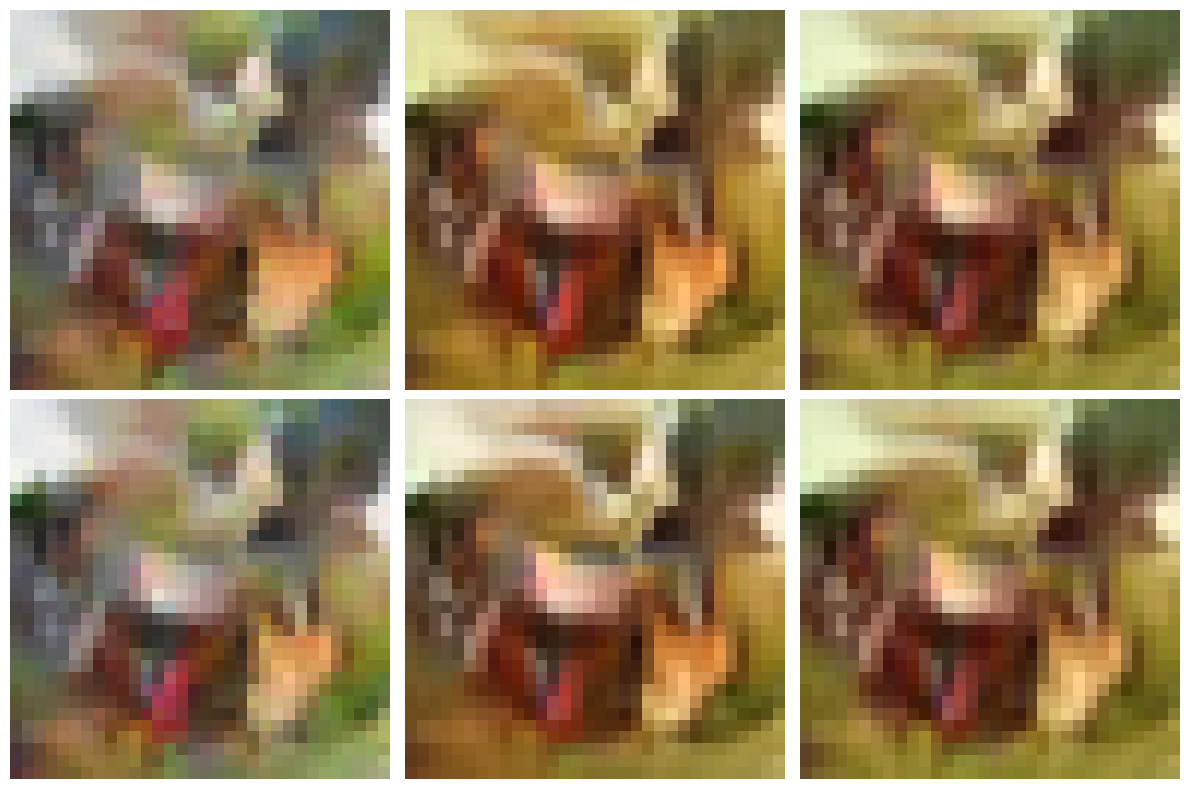}
    \includegraphics[width=0.3\linewidth]{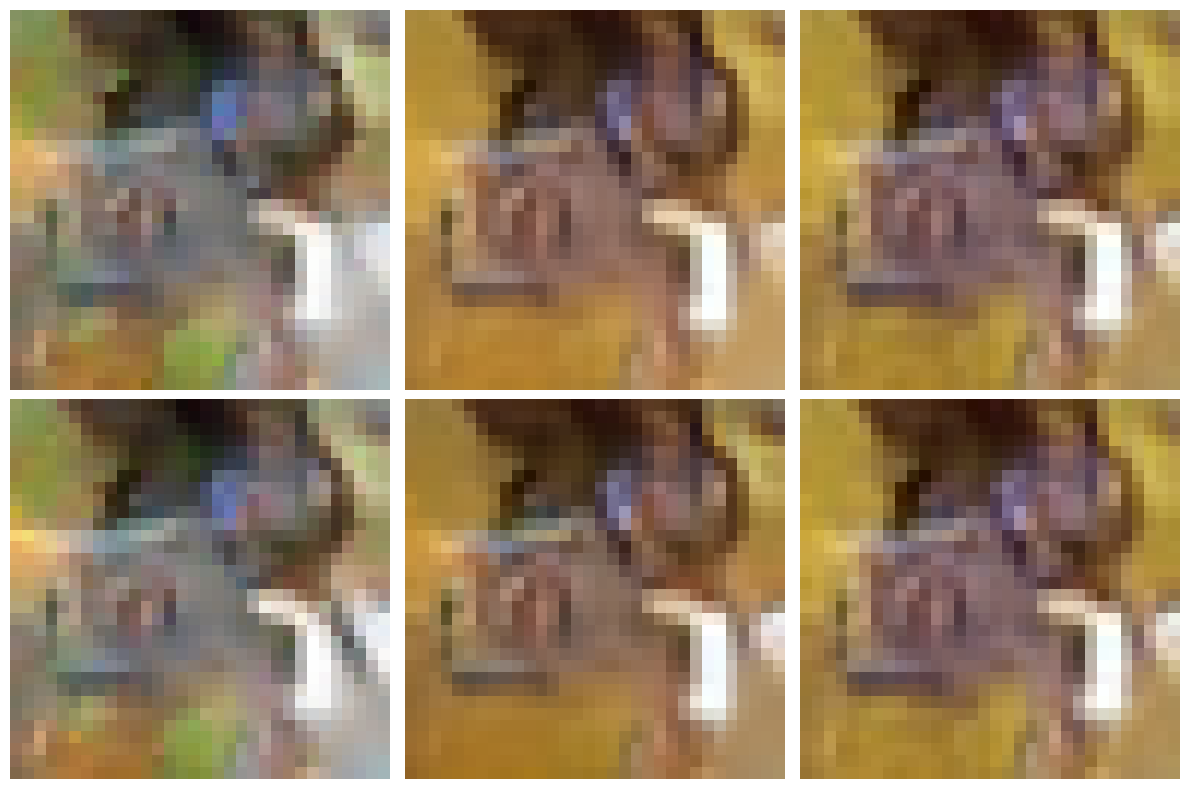}

    \sidecaptionimage{0.3\linewidth}{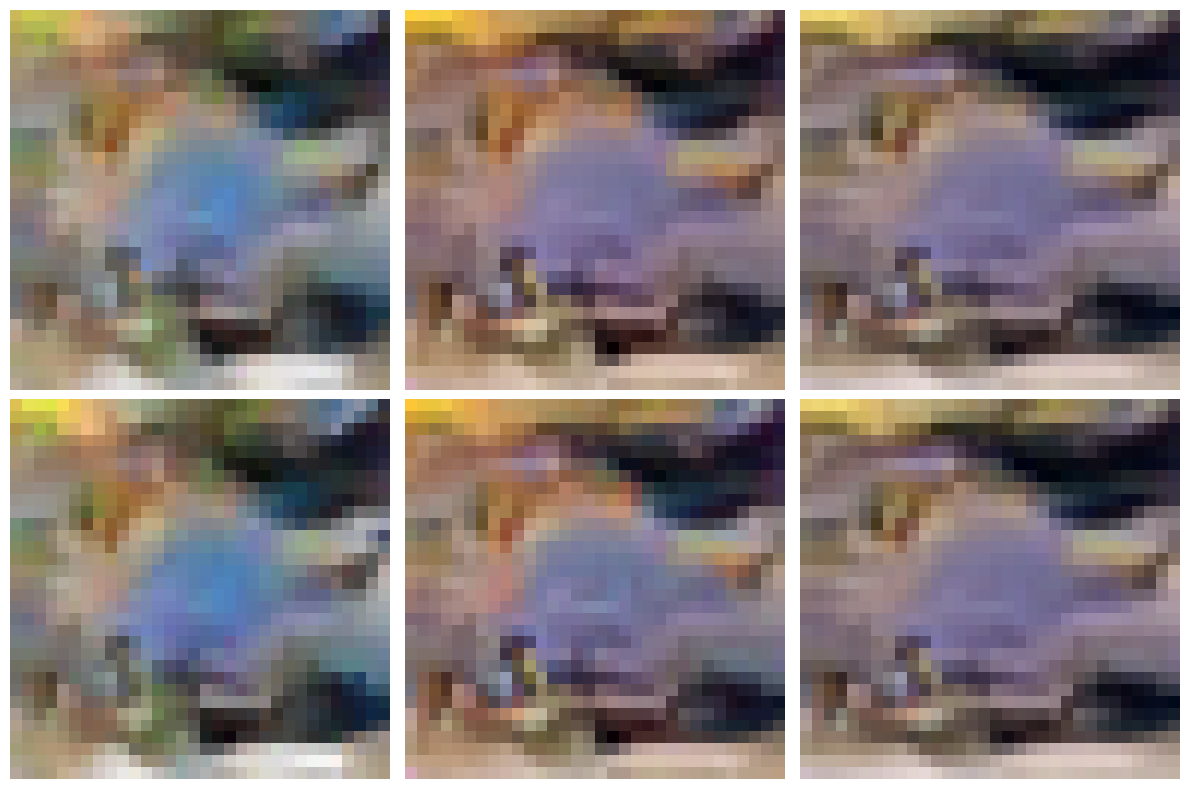}{\,\,\,$\mathcal{D}_2\,\,\,\,\,\,\,\,\,\,\,\,\,\,\,\,\mathcal{D}_1$}
    \includegraphics[width=0.3\linewidth]{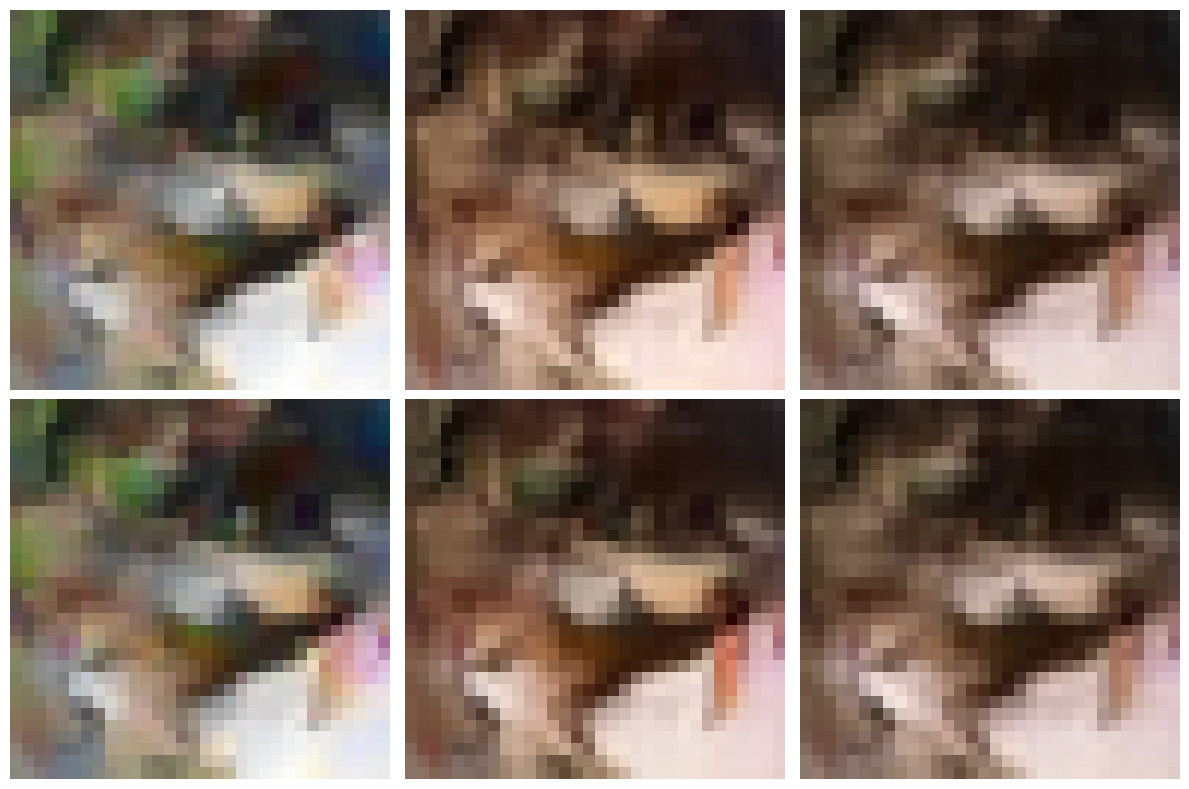}
    \includegraphics[width=0.3\linewidth]{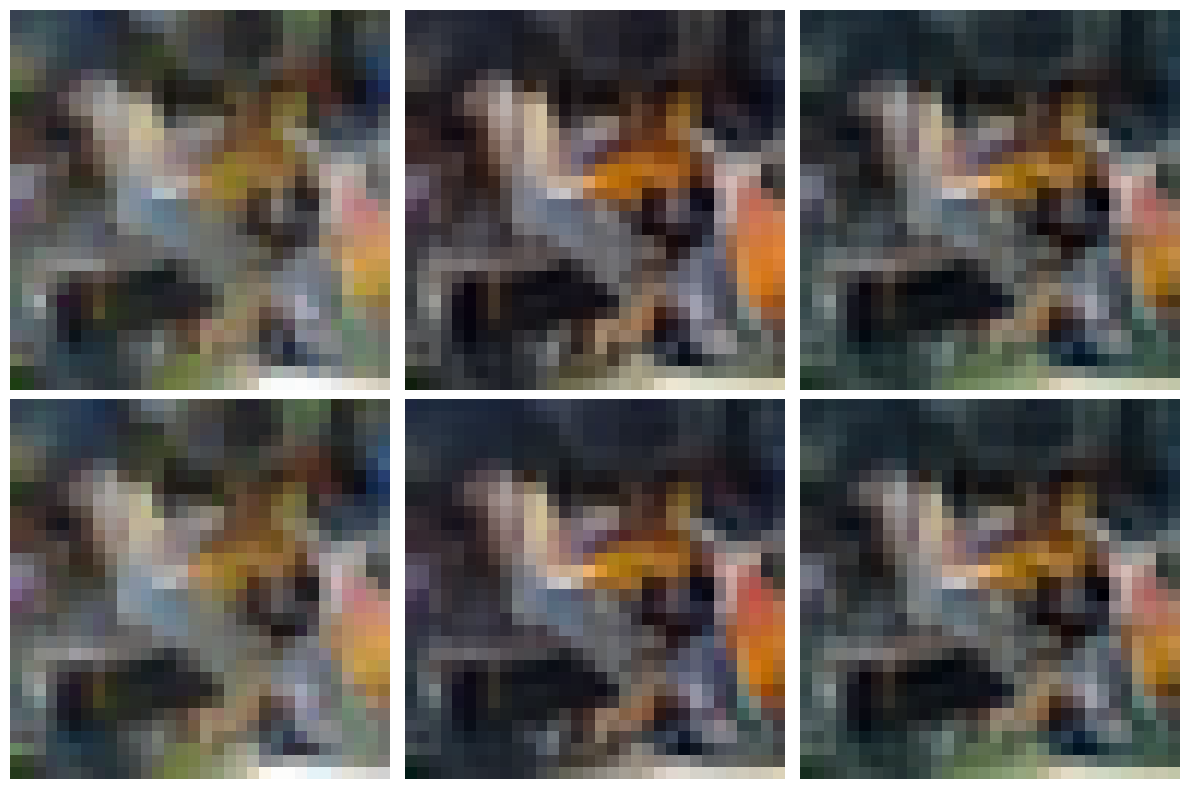}

    \caption{Further samples comparing DiTs and UNets early in training, trained on two disjoint subsets of $10^4$ samples from CIFAR10, and the outputs of calibrated local score models. UNet samples are from 10 epochs of training, DiT samples are from 30 epochs of training. Samples not curated for quality.}
    \label{fig:app_samps_c10}
\end{figure}

\begin{figure}
    \centering
    \sidecaptionimage{0.3\linewidth}{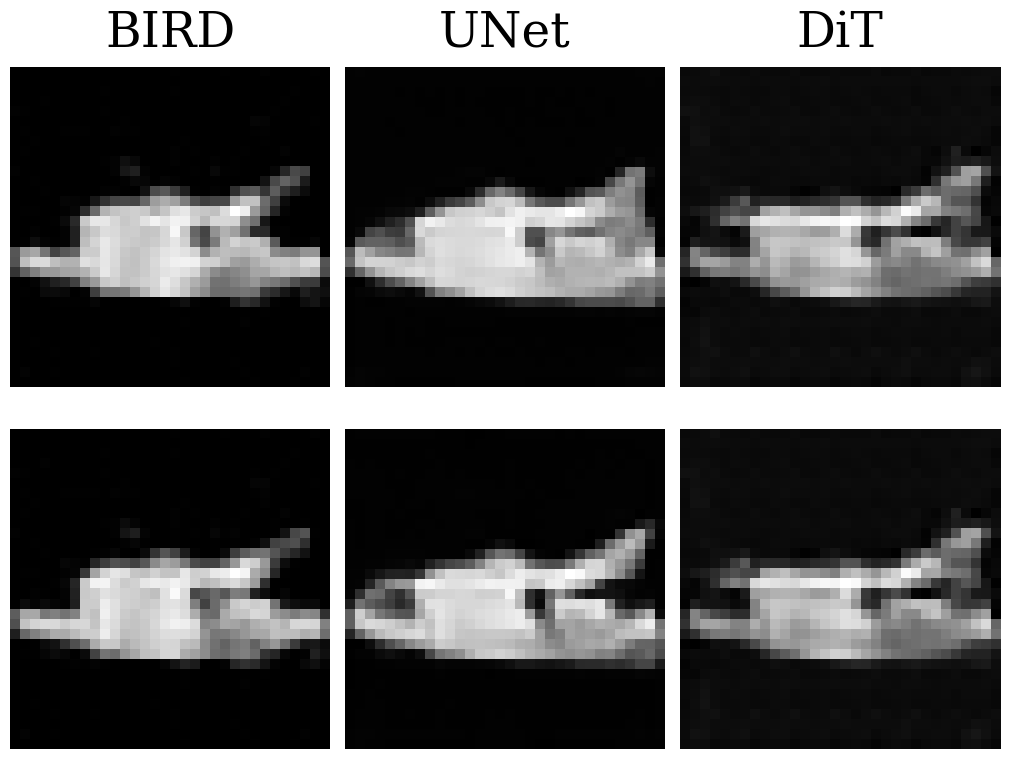}{\,\,\,$\mathcal{D}_2\,\,\,\,\,\,\,\,\,\,\,\,\,\,\,\,\,\,\,\,\mathcal{D}_1$\,\,\,\,}
        \includegraphics[width=0.3\linewidth]{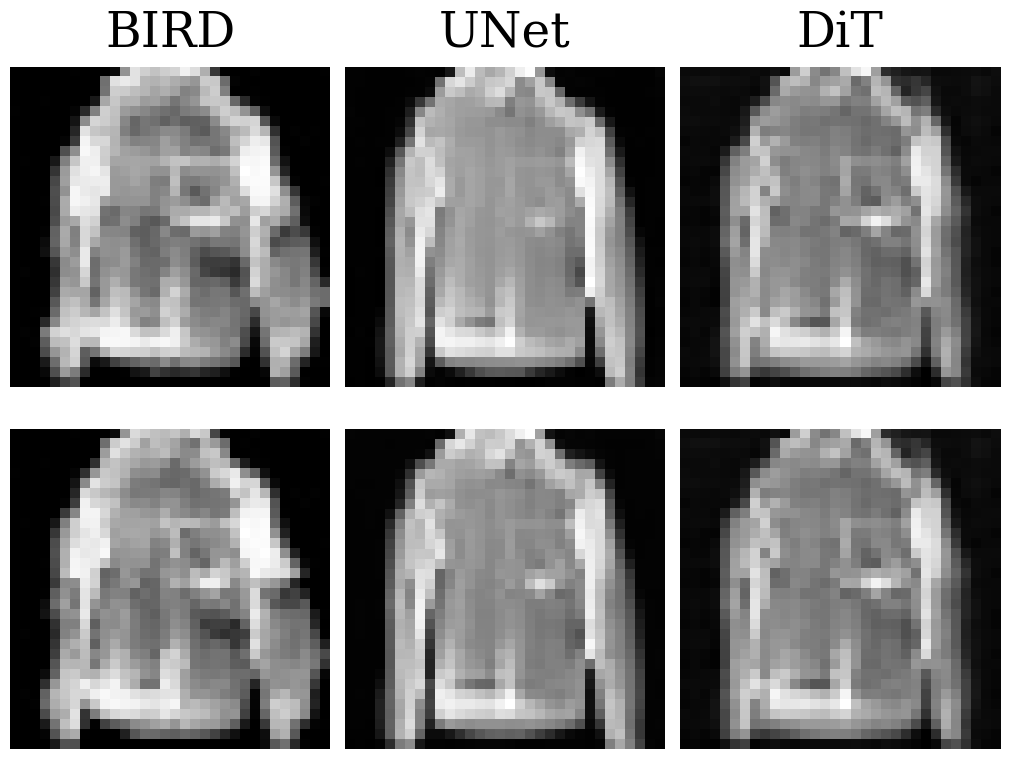}
    \includegraphics[width=0.3\linewidth]{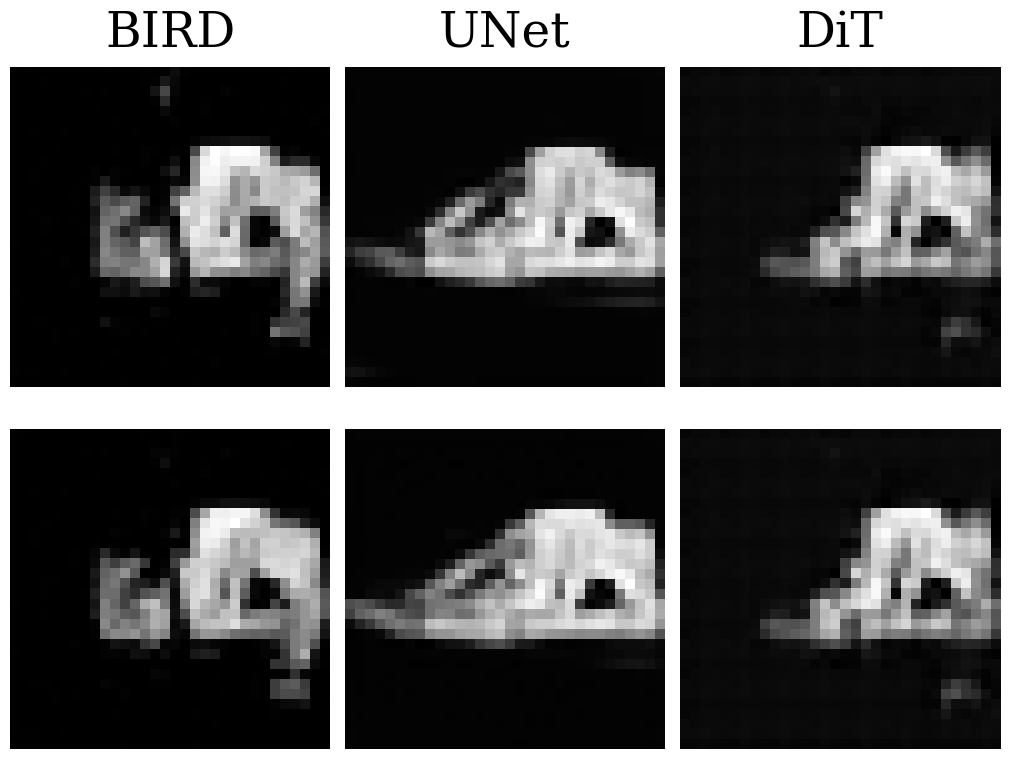}
    
    \sidecaptionimage{0.3\linewidth}{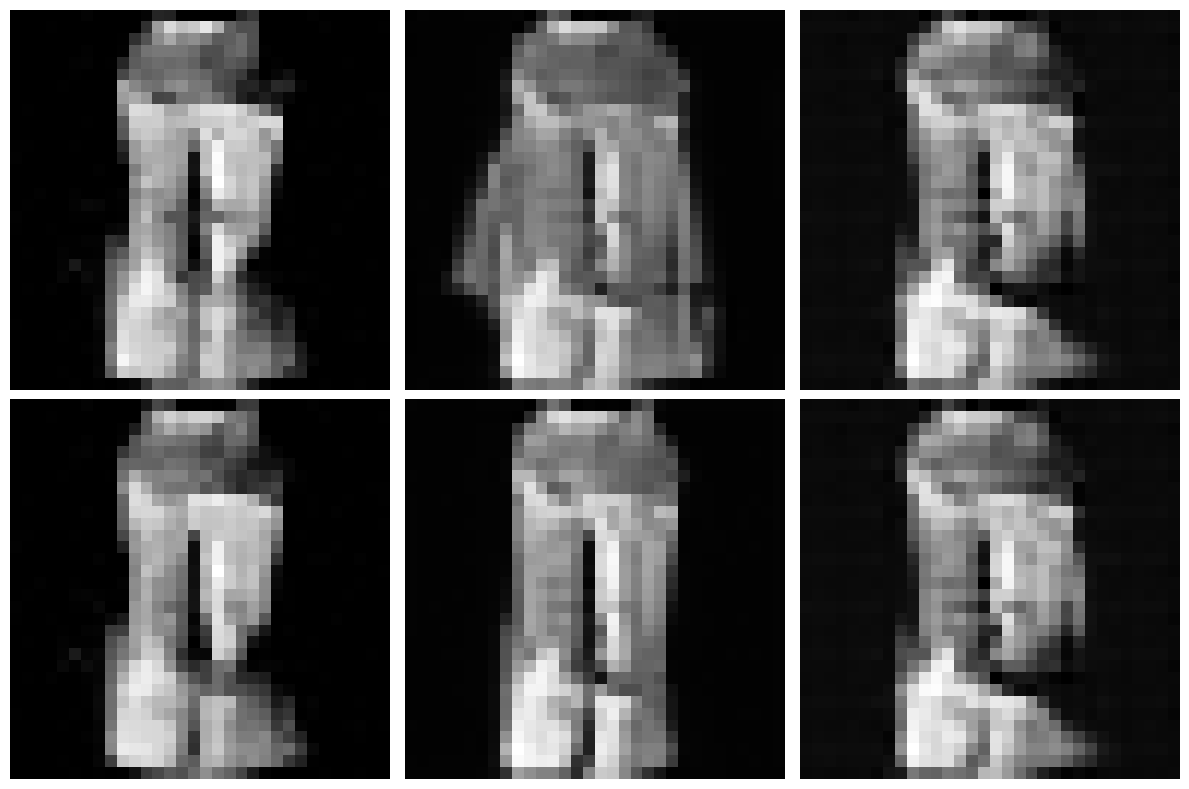}{\,\,\,$\mathcal{D}_2\,\,\,\,\,\,\,\,\,\,\,\,\,\,\,\,\mathcal{D}_1$}
    \includegraphics[width=0.3\linewidth]{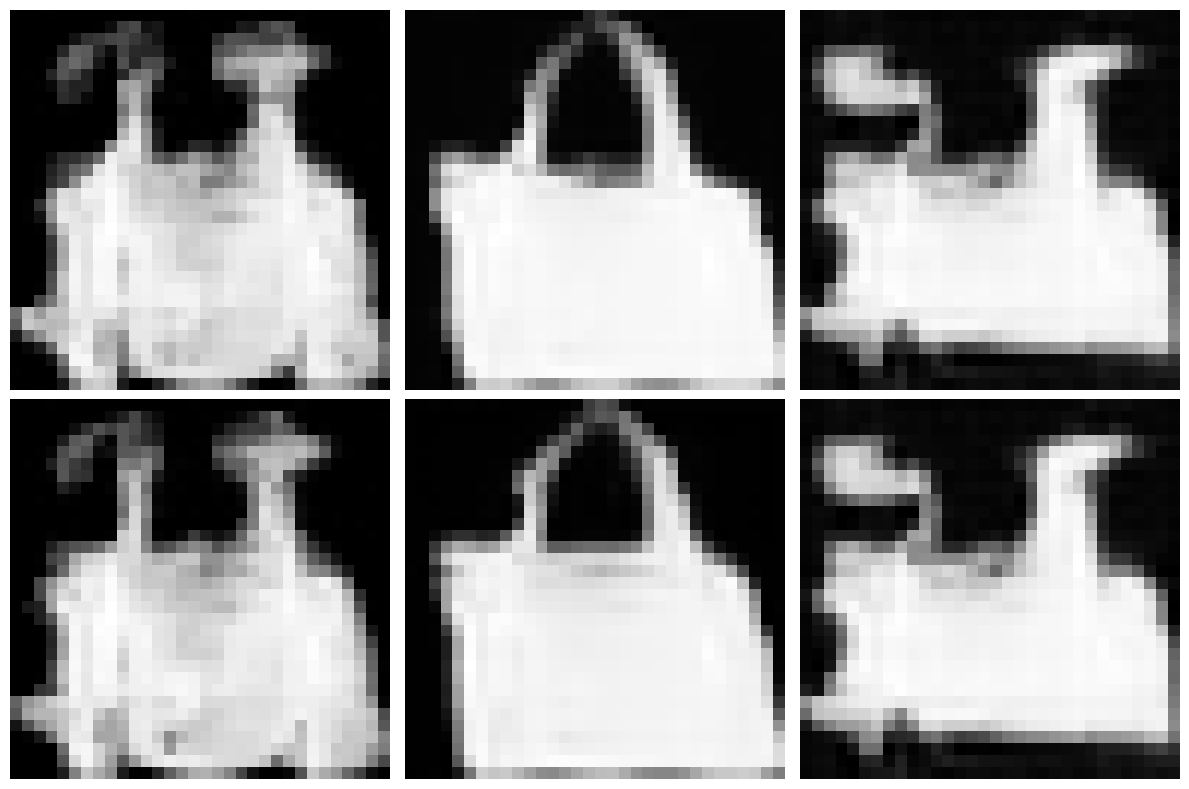}
    \includegraphics[width=0.3\linewidth]{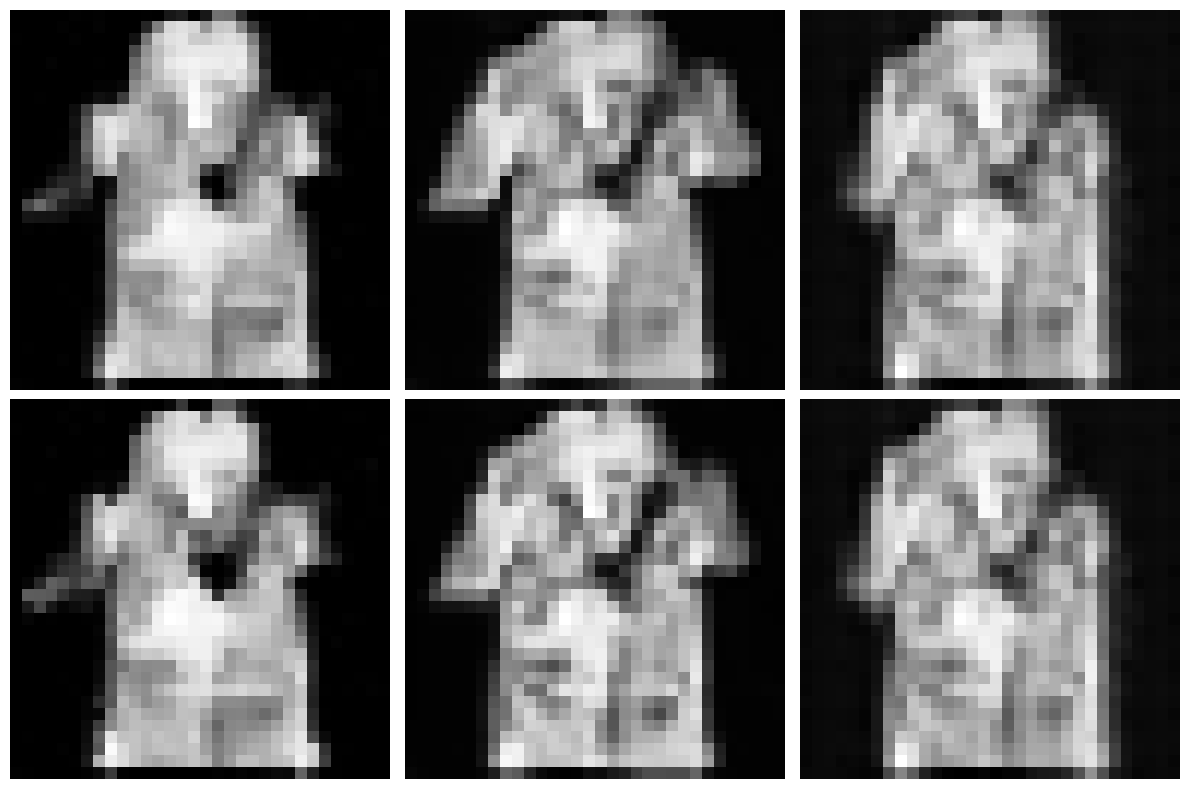}

    \sidecaptionimage{0.3\linewidth}{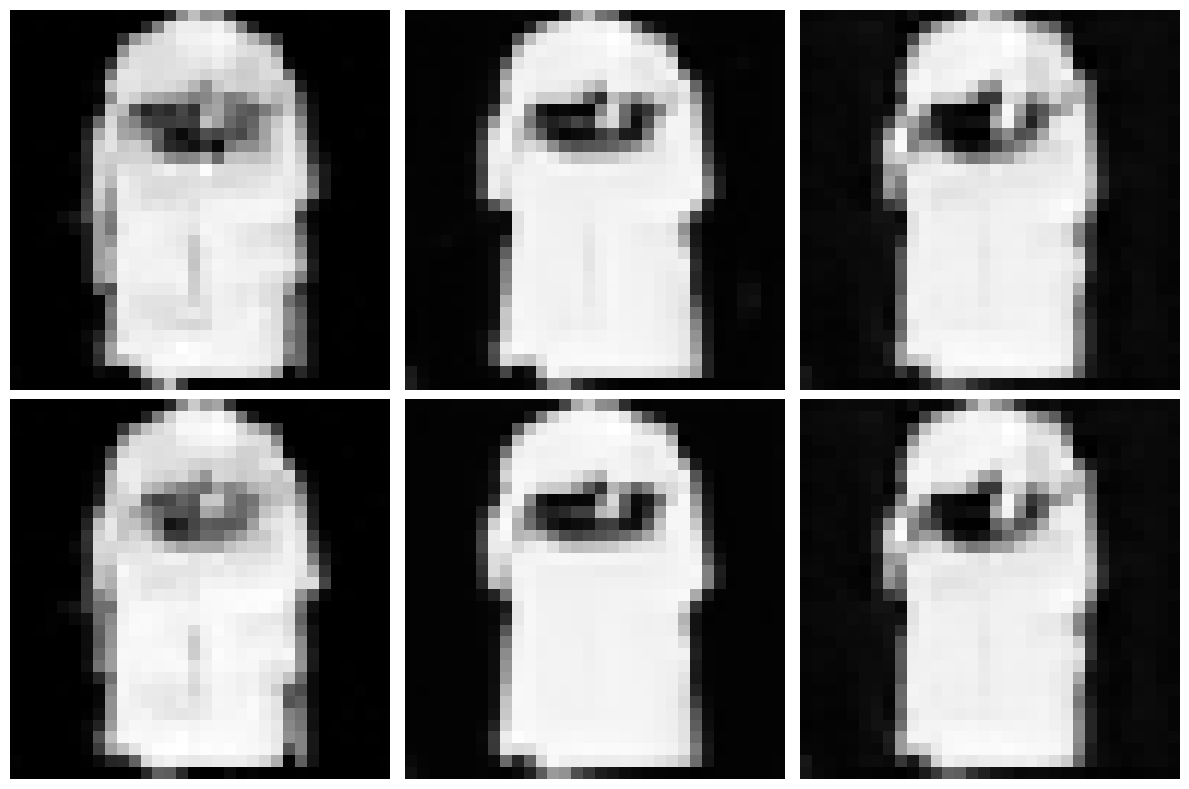}{\,\,\,$\mathcal{D}_2\,\,\,\,\,\,\,\,\,\,\,\,\,\,\,\,\mathcal{D}_1$}
    \includegraphics[width=0.3\linewidth]{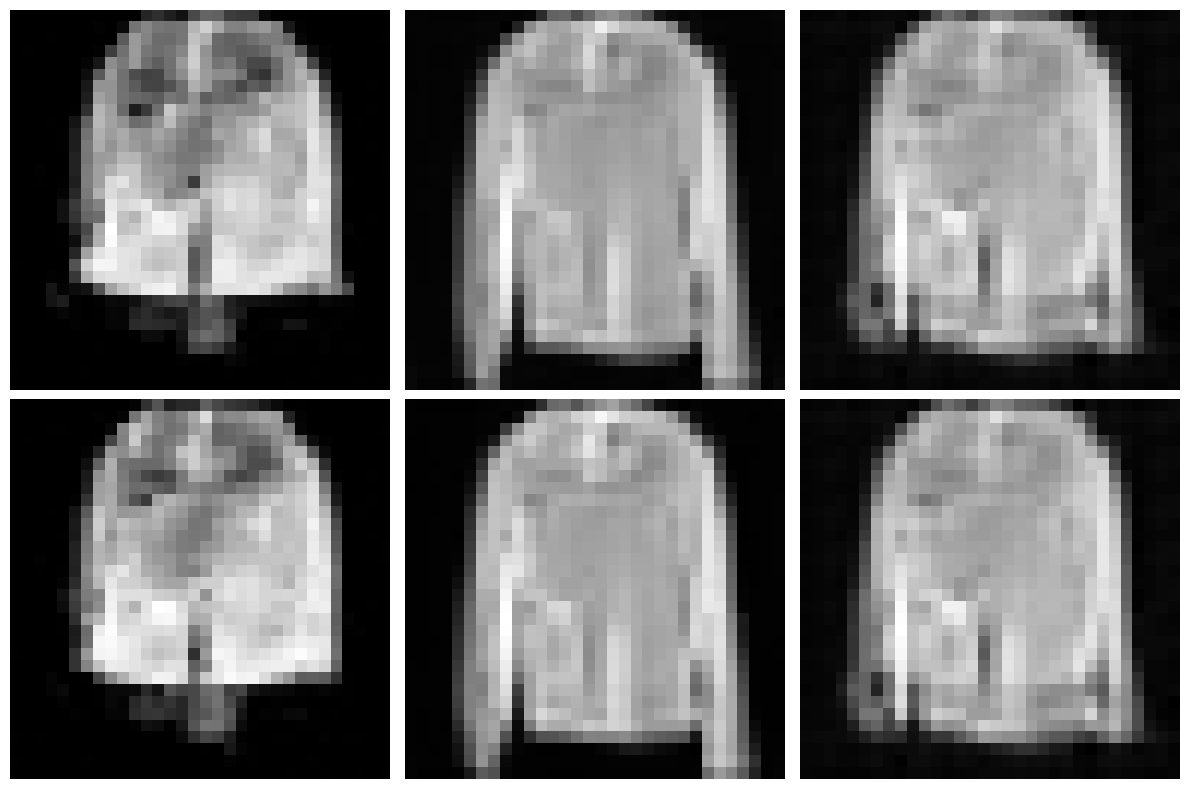}
    \includegraphics[width=0.3\linewidth]{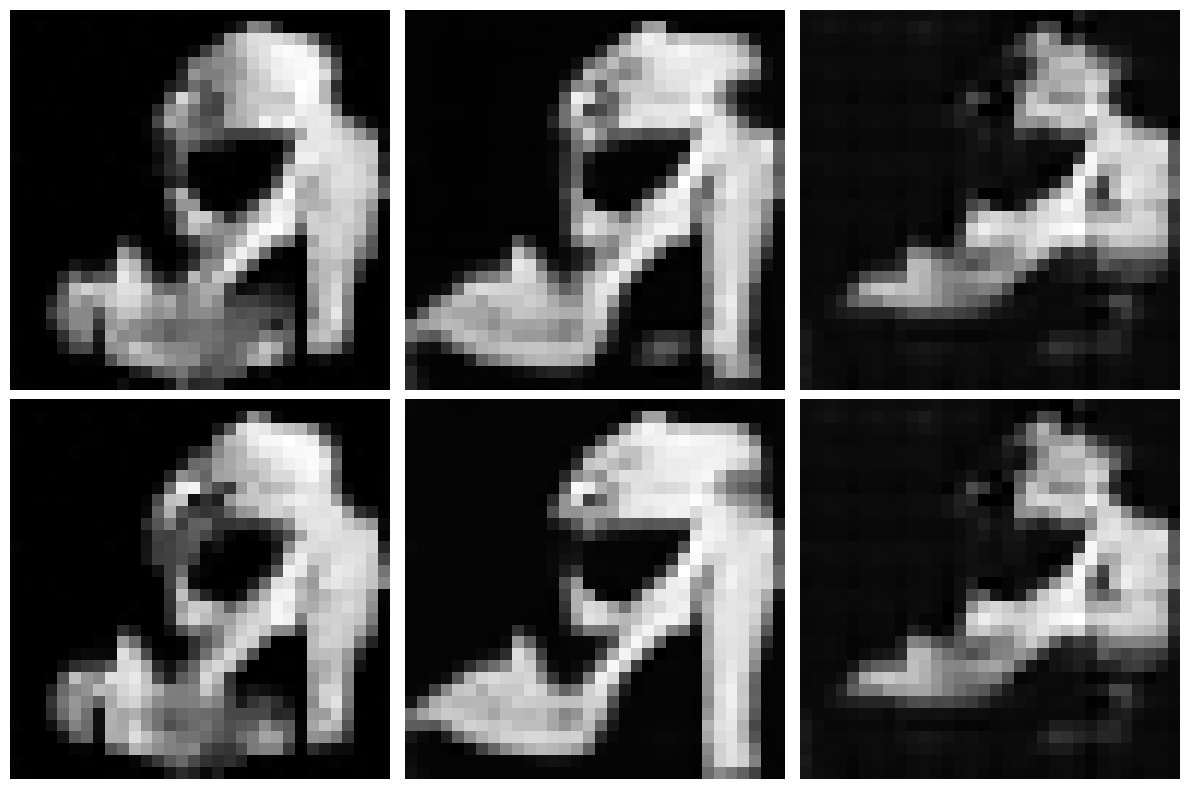}

    \caption{Further samples comparing DiTs and UNets early in training, trained on two disjoint subsets of $10^4$ samples from FashionMNIST, and the outputs of calibrated local score models. UNet samples are from 10 epochs of training, DiT samples are from 15 epochs of training. Samples not curated for quality.}
    \label{fig:app_samps_fmnist}
\end{figure}

\begin{figure}
    \centering
    \sidecaptionimage{0.3\linewidth}{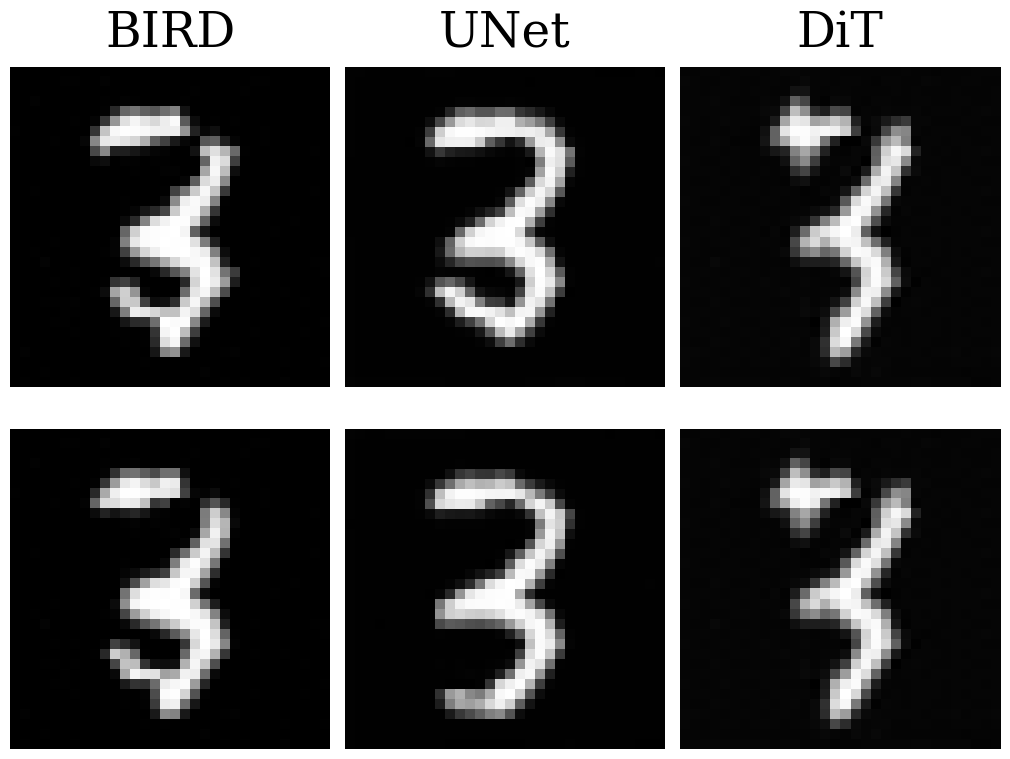}{\,\,\,$\mathcal{D}_2\,\,\,\,\,\,\,\,\,\,\,\,\,\,\,\,\,\,\,\,\mathcal{D}_1$\,\,\,\,}
        \includegraphics[width=0.3\linewidth]{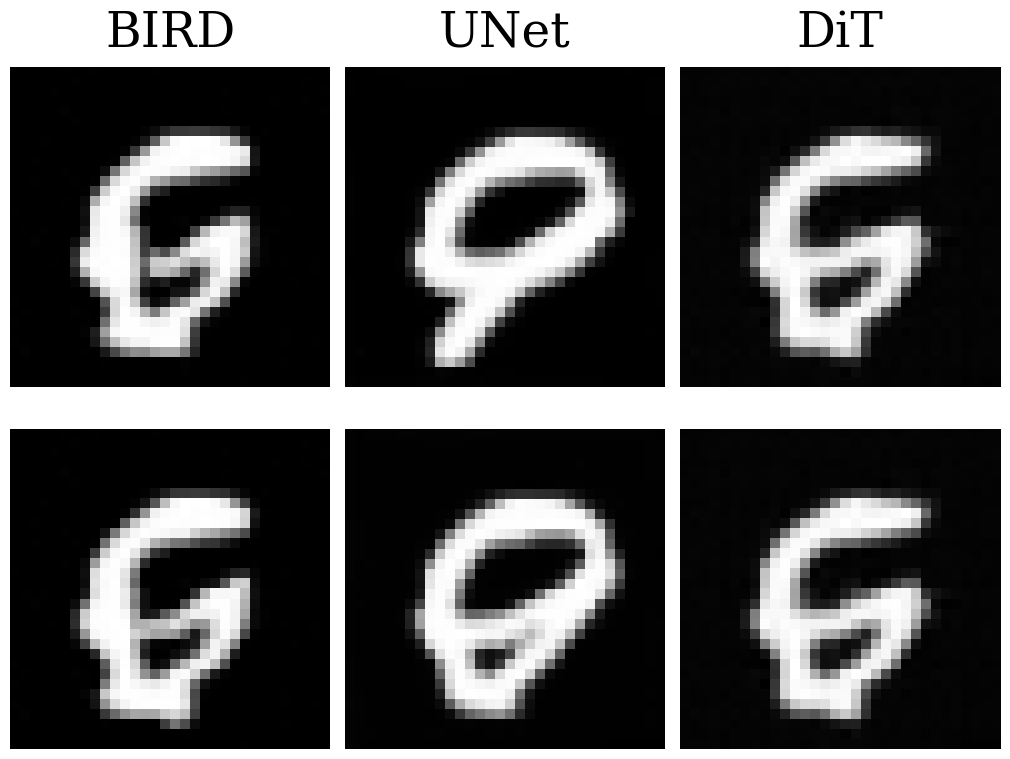}
    \includegraphics[width=0.3\linewidth]{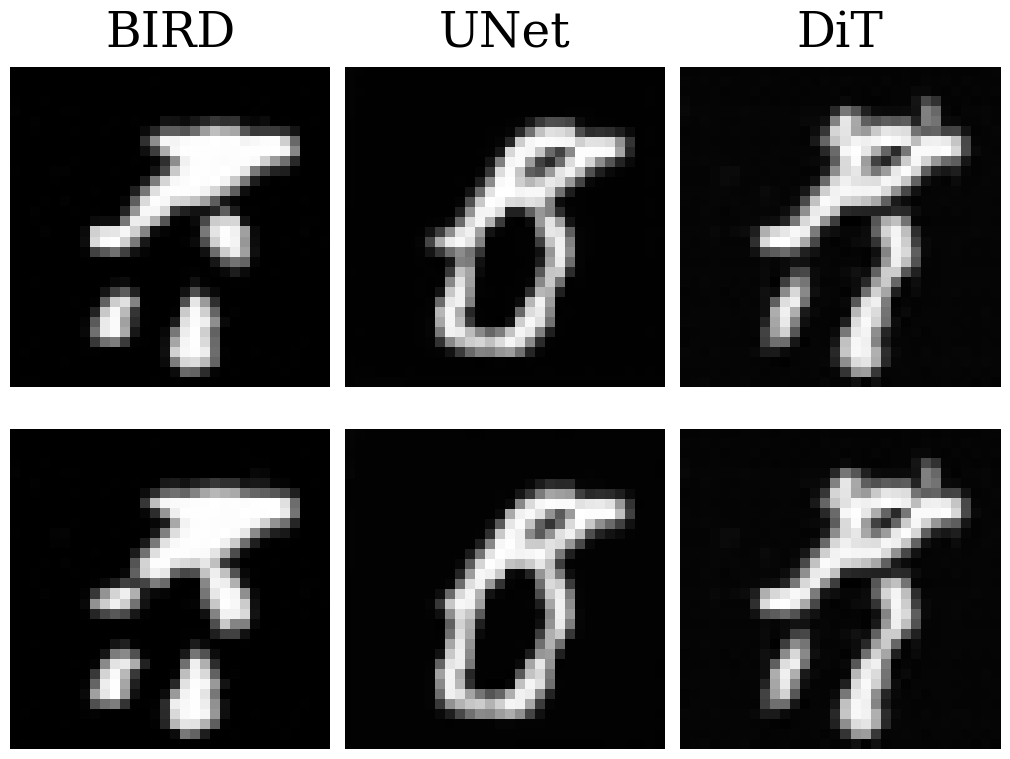}
    
    \sidecaptionimage{0.3\linewidth}{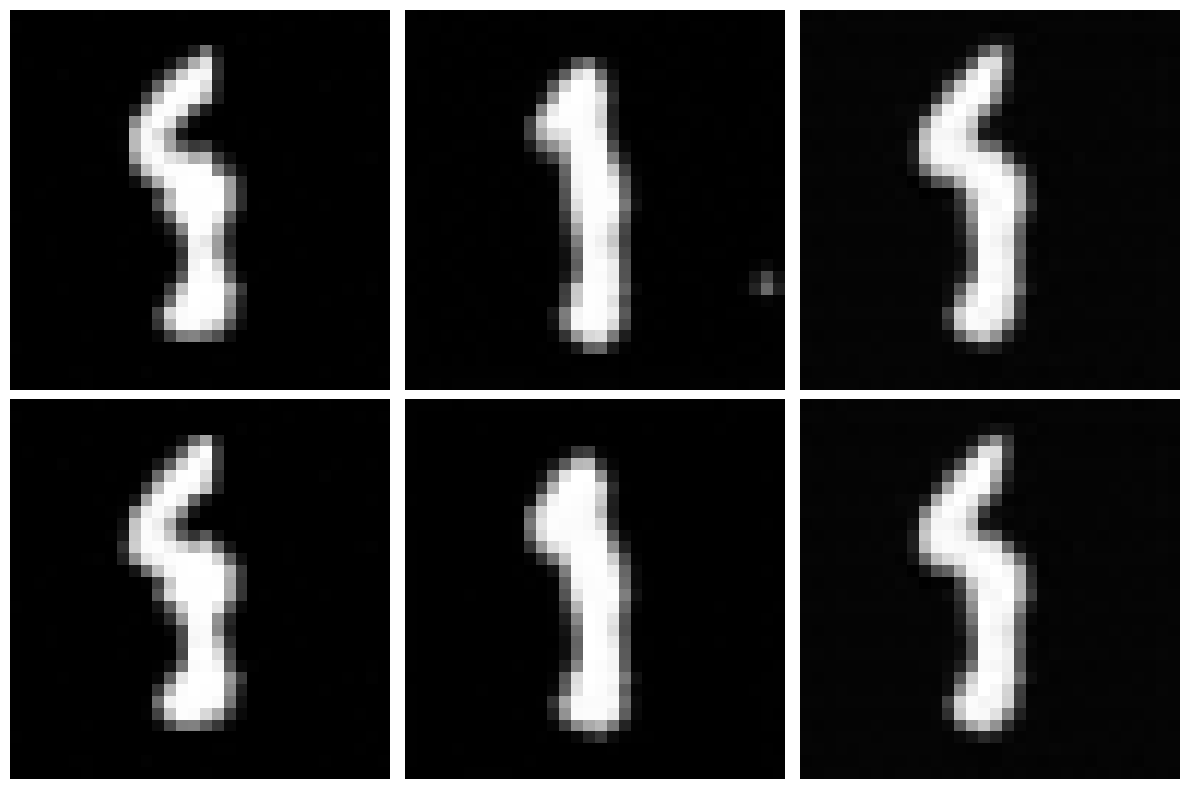}{\,\,\,$\mathcal{D}_2\,\,\,\,\,\,\,\,\,\,\,\,\,\,\,\,\mathcal{D}_1$}
    \includegraphics[width=0.3\linewidth]{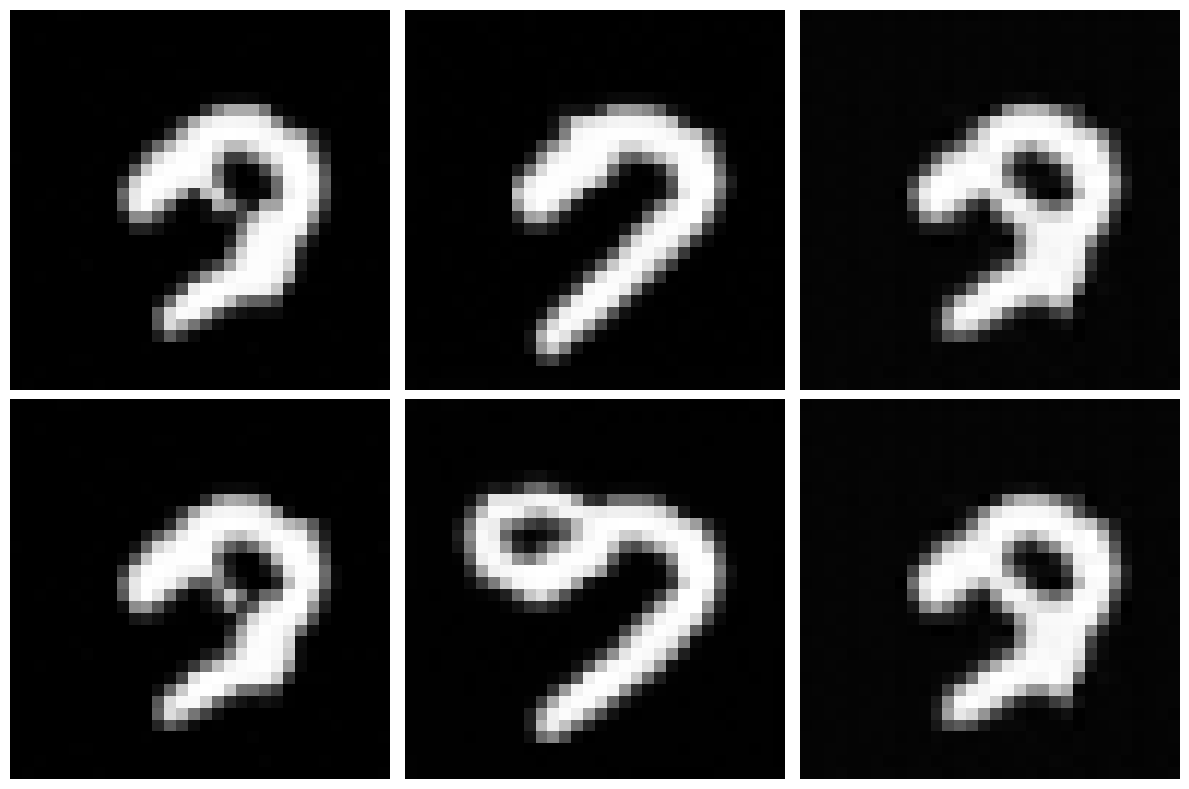}
    \includegraphics[width=0.3\linewidth]{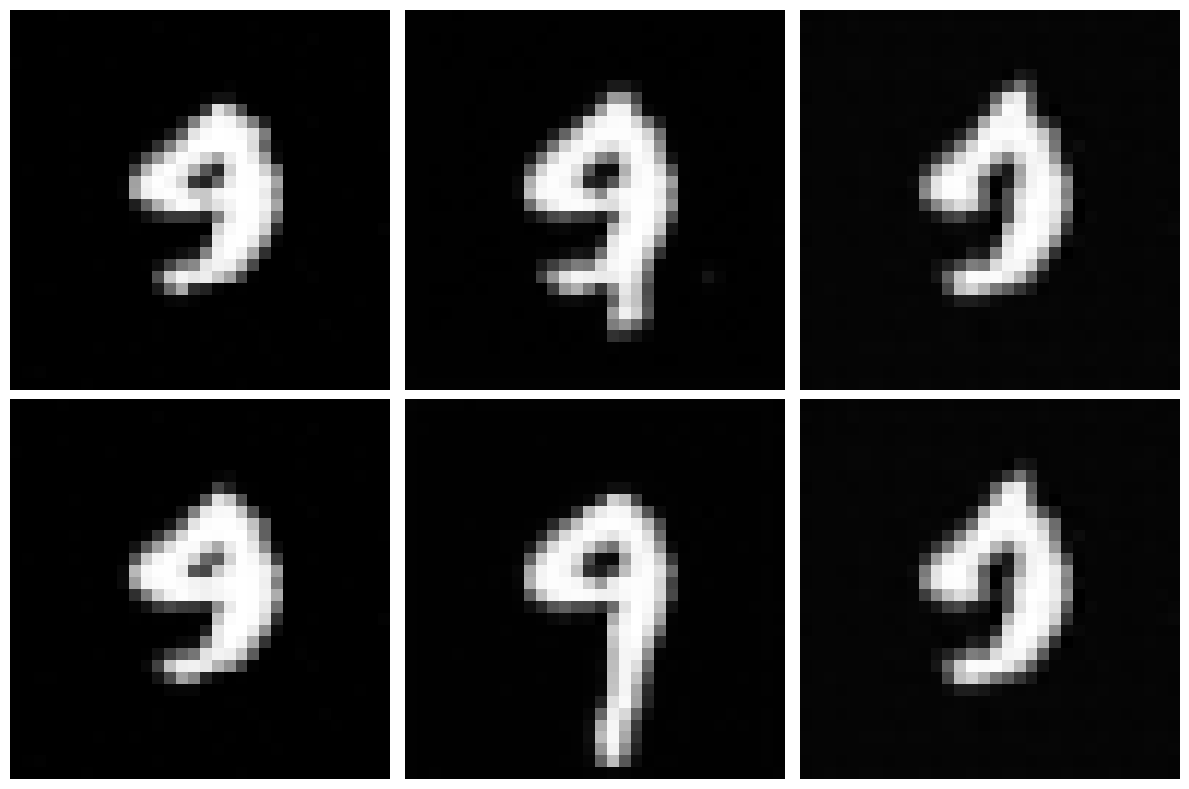}

    \sidecaptionimage{0.3\linewidth}{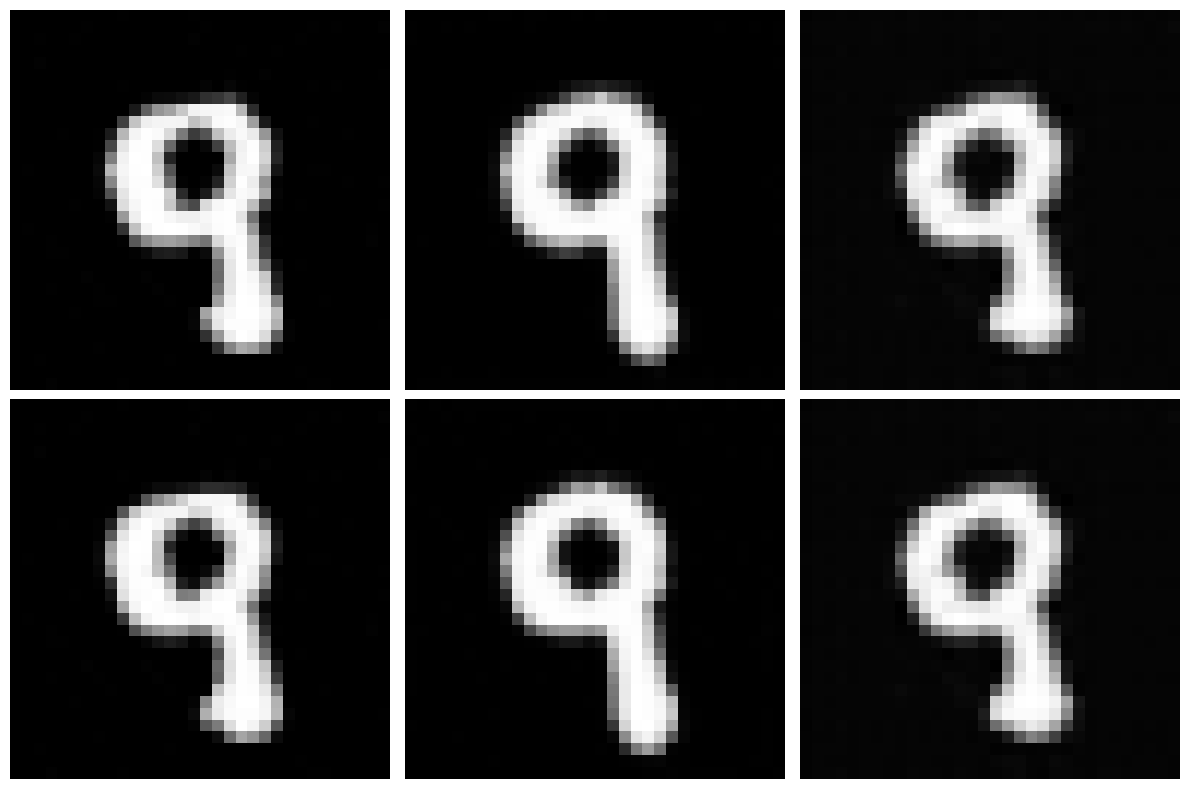}{\,\,\,$\mathcal{D}_2\,\,\,\,\,\,\,\,\,\,\,\,\,\,\,\,\mathcal{D}_1$}
    \includegraphics[width=0.3\linewidth]{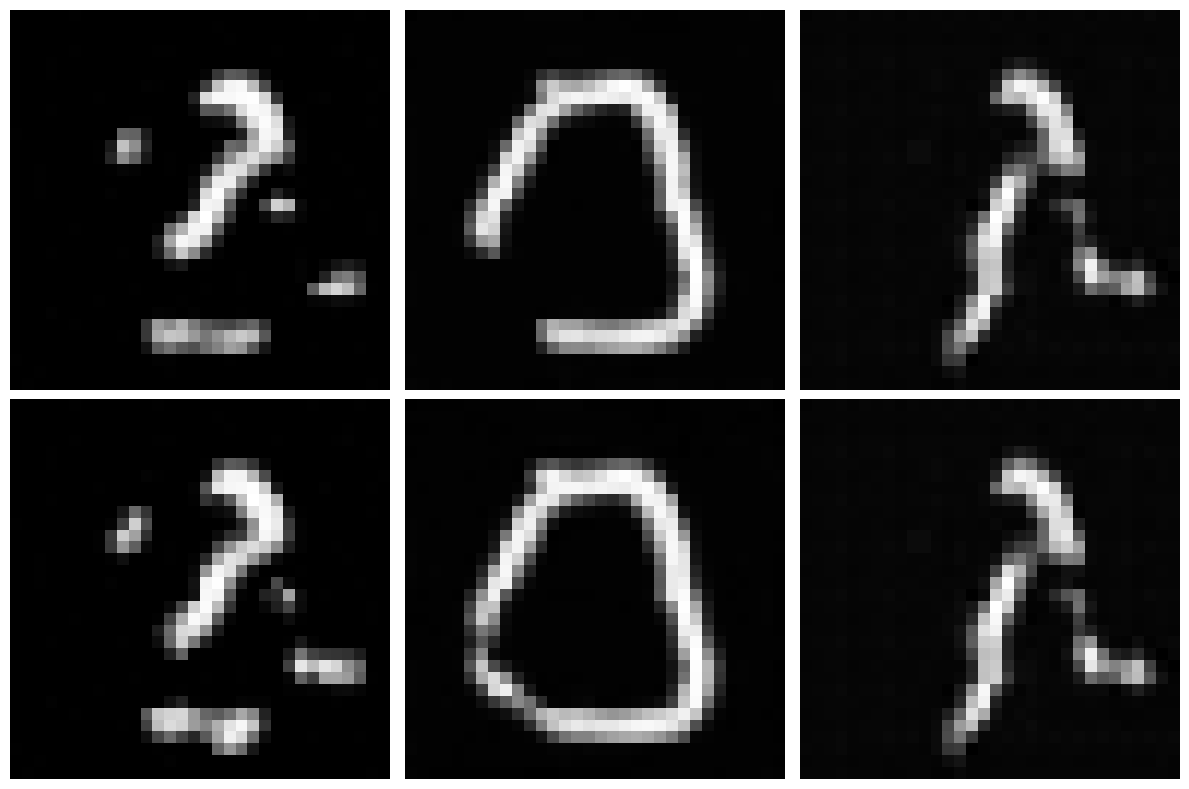}
    \includegraphics[width=0.3\linewidth]{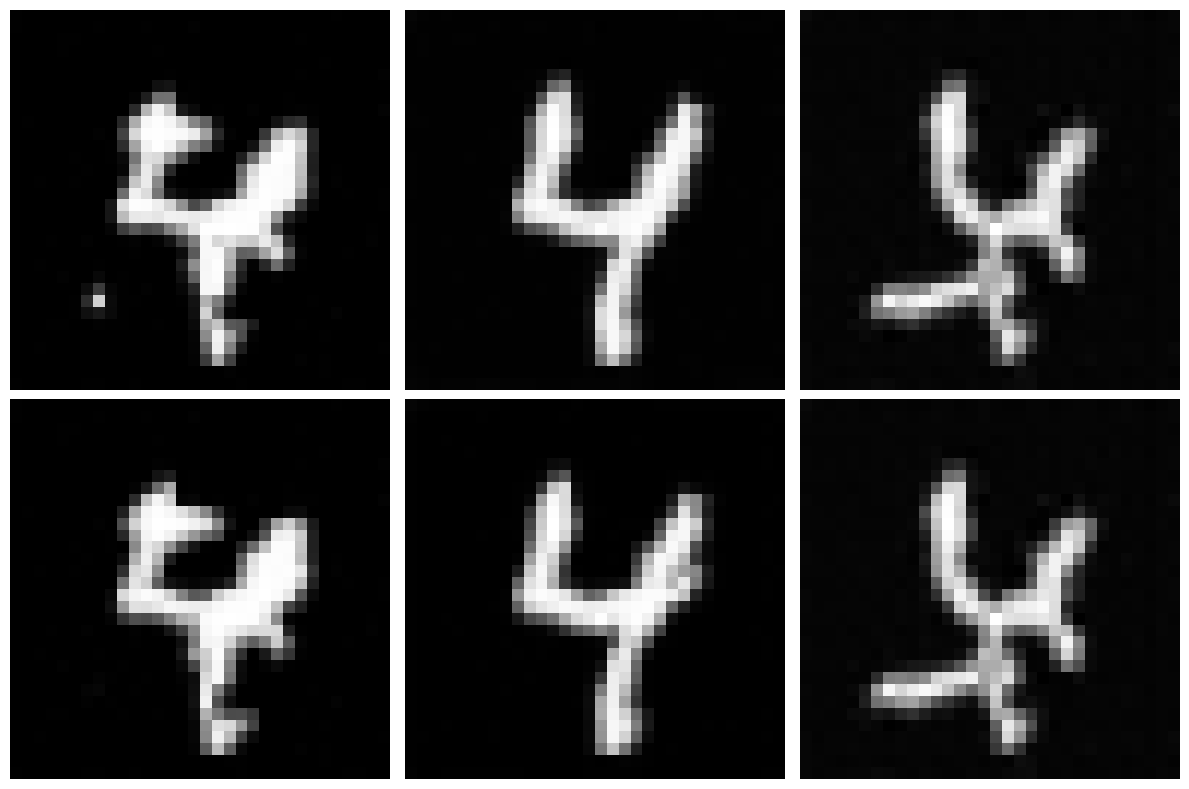}

    \caption{Further samples comparing DiTs and UNets early in training, trained on two disjoint subsets of $10^4$ samples from MNIST, and the outputs of calibrated local score models. UNet samples are from 10 epochs of training, DiT samples are from 15 epochs of training. Samples not curated for quality.}
    \label{fig:app_samps_mnist}
\end{figure}

\begin{figure}
    \centering
    \sidecaptionimage{0.49\linewidth}{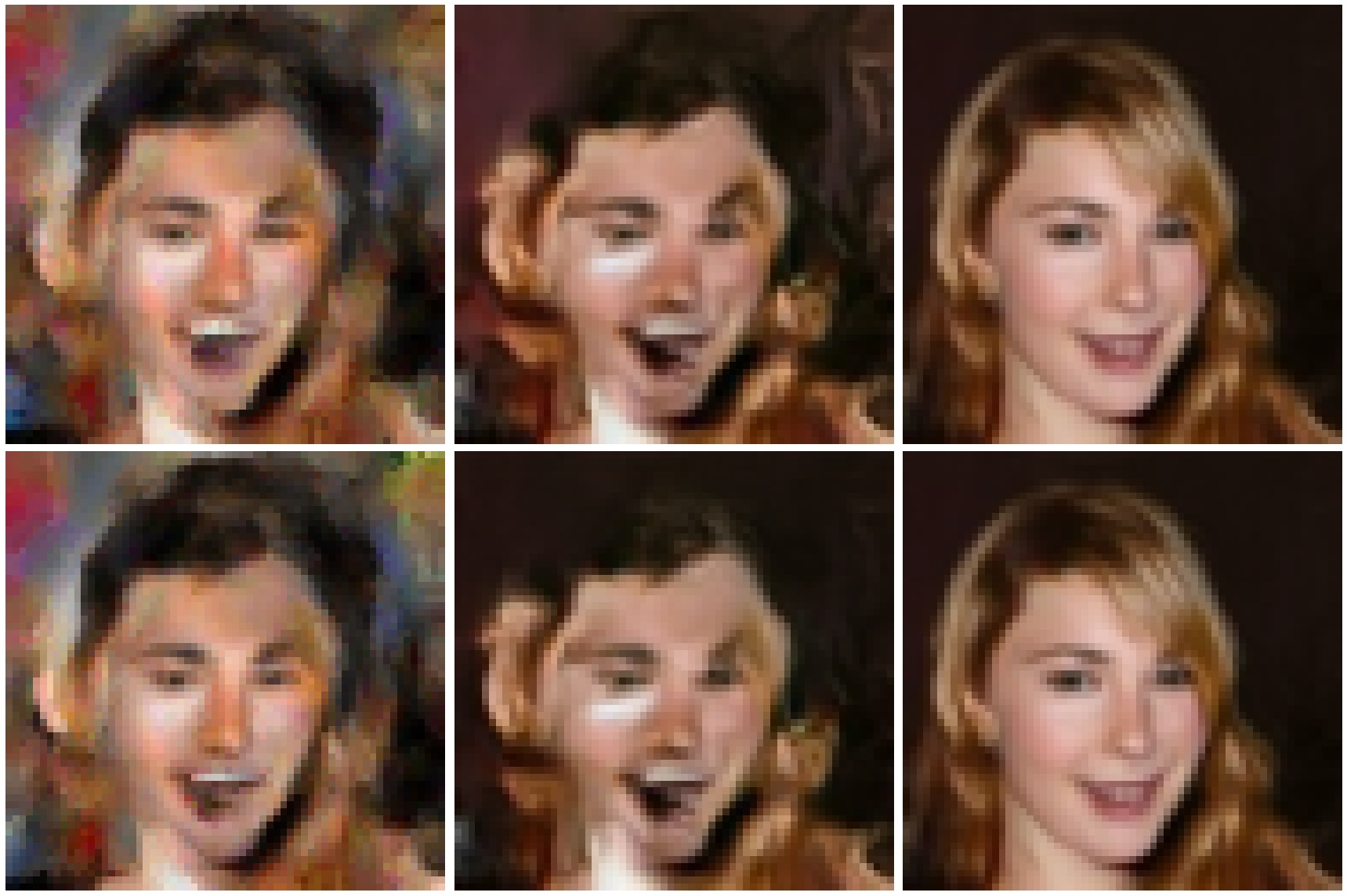}{\,\,\,\,\,\,\,$\mathcal{D}_2\,\,\,\,\,\,\,\,\,\,\,\,\,\,\,\,\,\,\,\,\,\,\,\,\,\,\,\,\,\mathcal{D}_1\,\,\,$} \includegraphics[width=0.49\linewidth]{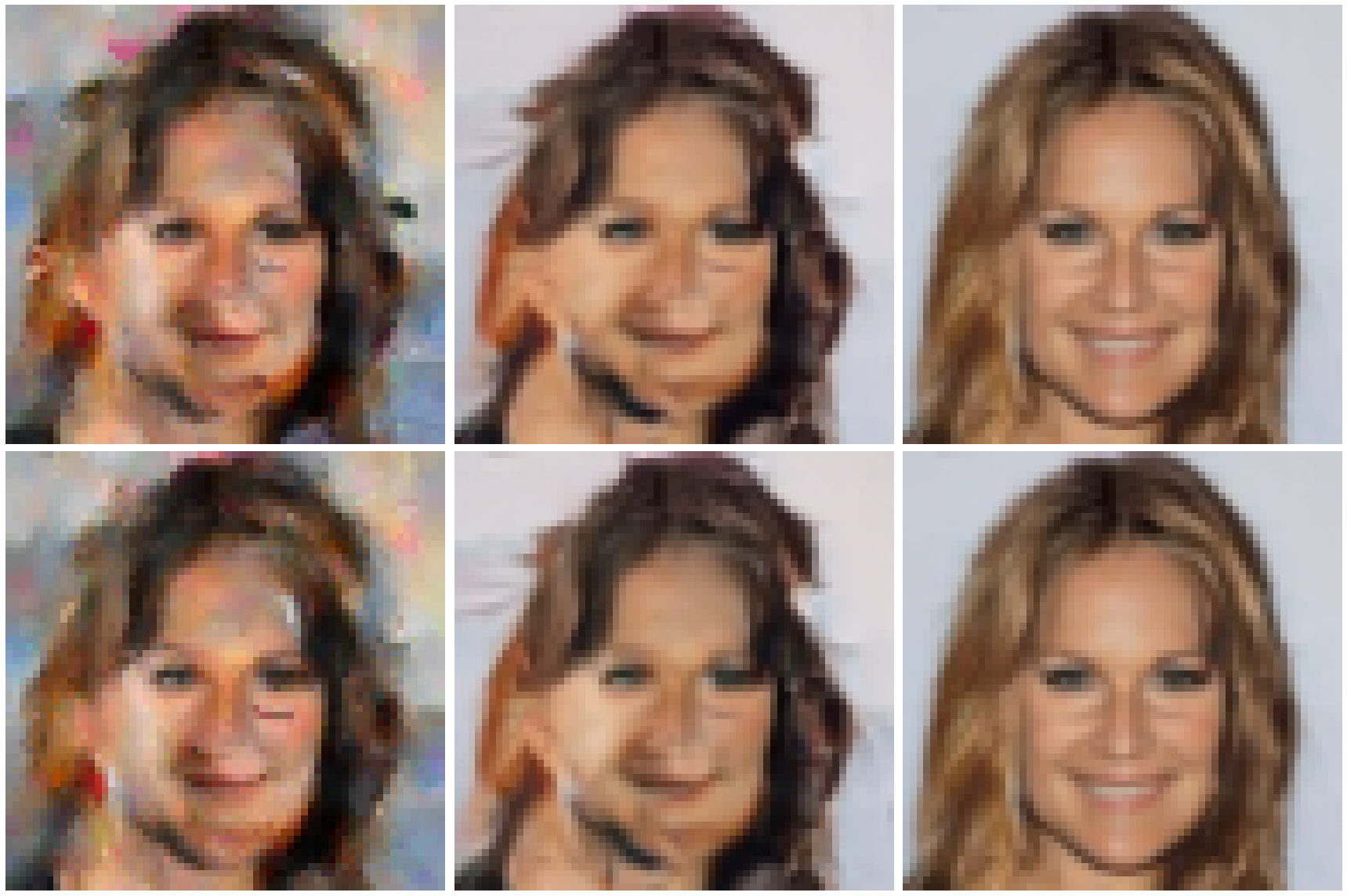}
    \sidecaptionimage{0.49\linewidth}{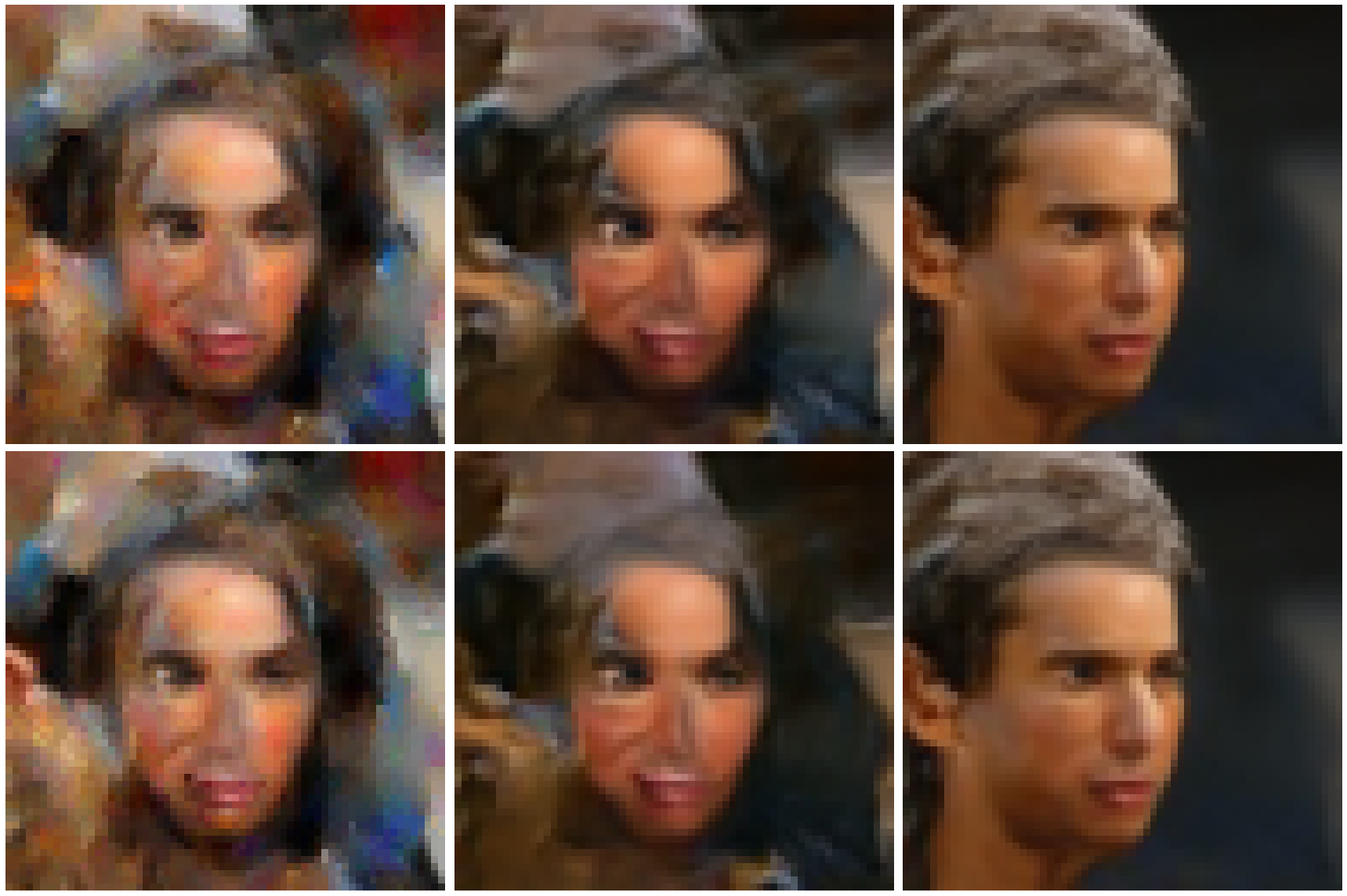}{\,\,\,\,\,\,\,$\mathcal{D}_2\,\,\,\,\,\,\,\,\,\,\,\,\,\,\,\,\,\,\,\,\,\,\,\,\,\,\,\,\,\mathcal{D}_1\,\,\,$}   
    \includegraphics[width=0.49\linewidth]{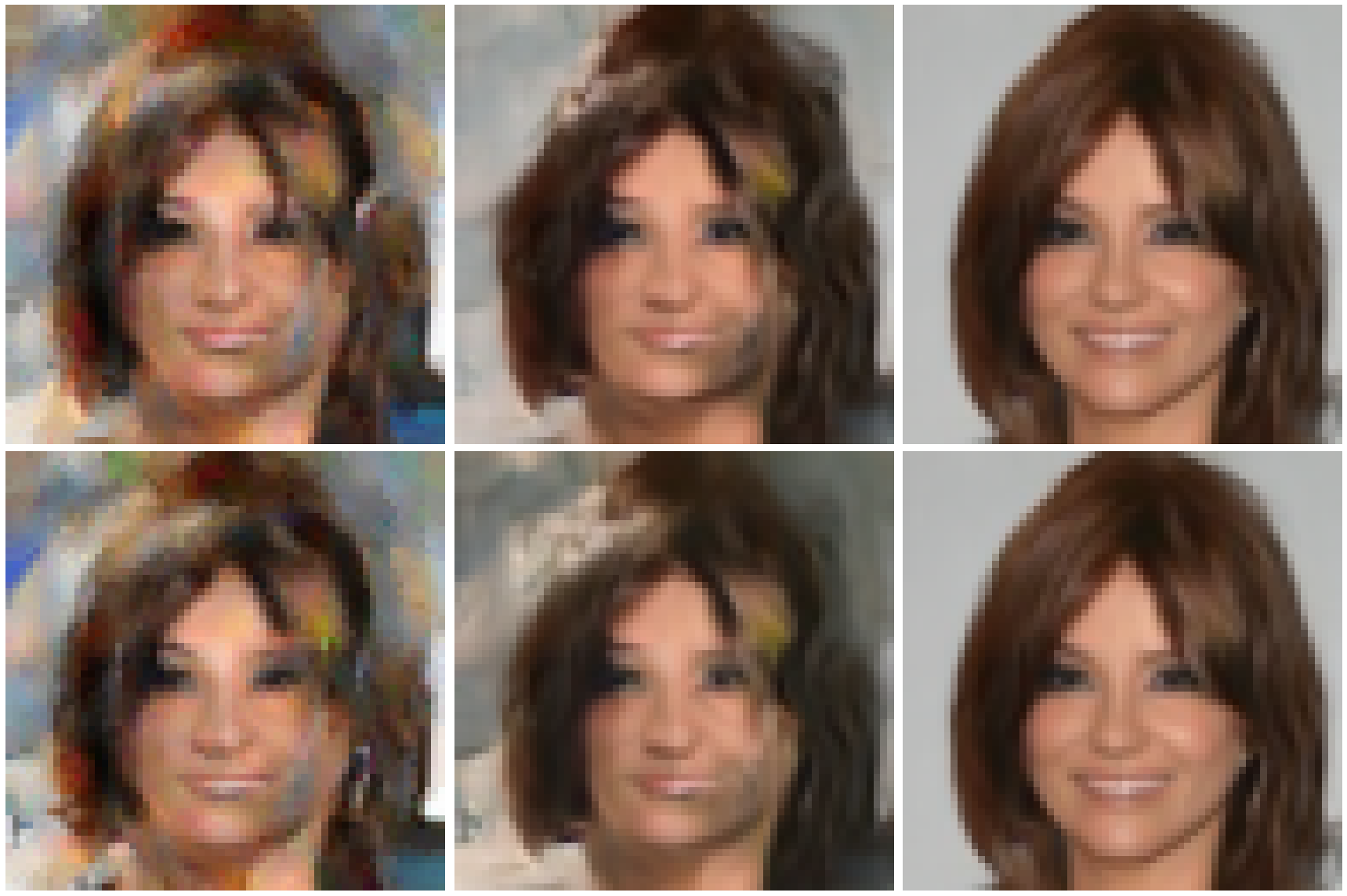}

    \caption{As training progresses, models evolve from `patch mosaic' style outputs \citep{kambanalytic} towards semantically coherent, nonlocally consistent generations, which nevertheless remain robust to dataset shifts. Leftmost columns are analytical, middle columns are DiT after 30 epochs on a $10^4$ sample subset, and rightmost is after 300 epochs on a $10^4$ sample subset.}
    \label{fig:carveout}
\end{figure}

\end{appendices}


\end{document}